\documentclass[12pt, a4paper, twoside, english, final]{ETHthesis}

\usepackage{ETHthesis}
\usepackage[english]{babel}

\usepackage[OT1]{fontenc}
\usepackage[utf8]{inputenc}
\usepackage{lmodern}

\usepackage{amsmath, amssymb, amsthm}
\usepackage[usenames, dvipsnames]{xcolor}

\usepackage{csquotes}
\usepackage[backend=bibtex, style=alphabetic, natbib=true,
            isbn=false, backref=true, maxcitenames=3,
            maxbibnames=100]{biblatex}

\usepackage{subcaption}
\usepackage{tikz}
\usepackage{bbm}
\usepackage{dsfont}
\usepackage{mathtools}
\usepackage{graphicx}

\usepackage{tabularx}
\usepackage{booktabs}

\usepackage{algorithm}
\usepackage[noend]{algorithmic}
\renewcommand\algorithmicthen{}

\usepackage{eqparbox}

\usepackage{xspace}

\usepackage{enumitem}
\setlist[enumerate]{leftmargin=*}

\usepackage{titlesec}
\usepackage{hyperref}
\hypersetup{pdfauthor={Francesco Locatello}, pdftitle={}, pdfsubject={A thesis submitted to attain the degree of Doctor of Sciences to ETH Zurich},
            pdflang={en}, bookmarksnumbered, pageanchor=true, plainpages=false, colorlinks=true, linktocpage=true,
            citecolor=MidnightBlue, filecolor=BrickRed, linkcolor=NavyBlue, urlcolor=ForestGreen}

\definecolor{light-gray}{gray}{0.65}
\titleformat{\chapter}
[display]
{\bfseries\Large}
{\filleft{\fontsize{80}{100}\selectfont\color{light-gray}\thechapter}}
{-25pt}
{\filleft\fontsize{26}{35}\selectfont}
[\vspace{5pt}\titlerule]

\iftwoside
\else
\fi
\newcommand{\cleardoublepageempty}{
  \clearpage
  \thispagestyle{empty}
  \cleardoublepage
}

\usepackage{array}
\newcolumntype{H}{>{\setbox0=\hbox\bgroup}c<{\egroup}@{}} 

\newtheorem*{rep@theorem}{\rep@title}
\newcommand{\newreptheorem}[2]{%
\newenvironment{rep#1}[1]{%
 \def\rep@title{#2 \ref{##1}}%
 \begin{rep@theorem}}%
 {\end{rep@theorem}}}
\newreptheorem{lemma}{Lemma'}
\newreptheorem{theorem}{Theorem'}
\newreptheorem{corollary}{Corollary'}
\newtheorem{theorem}{Theorem}
\newtheorem{lemma}[theorem]{Lemma}

\newtheorem{definition}[theorem]{Definition}
\newtheorem{proposition}[theorem]{Proposition}

\usepackage{url}            
\usepackage{amsfonts}       
\usepackage{nicefrac}       
\usepackage{microtype}      
\usepackage{amsbsy}%
\usepackage{calc}
\usepackage{moresize}
\usepackage{siunitx}
\usepackage{printlen}

\usepackage{multicol}
\usepackage{footmisc}
\usetikzlibrary{positioning, matrix, shapes,trees}
\usepackage{mathtools}
\usepackage[export]{adjustbox}
\usepackage{wrapfig}

\makeatletter
\newcommand{\LINEIF}[3][default]{%
  \ALC@it\algorithmicif\ #2\ \algorithmicthen%
  \ALC@com{#1}\ #3\ %
}
\makeatother

\usepackage{tikz}
\usetikzlibrary{calc}
\tikzset{mynode/.style={draw,circle, minimum size = 0.7cm}}
\usetikzlibrary{positioning,shapes,arrows}
\usepackage{xcolor}

\definecolor{gblue}{RGB}{207,226,243}
\definecolor{gred}{RGB}{244,204,204}
\definecolor{gyellow}{RGB}{255,229,153}
\definecolor{gyellow2}{RGB}{252,229,205}
\definecolor{ggreen}{RGB}{217,234,211}
\definecolor{ggray}{RGB}{238,238,238} 
\definecolor{ggray2}{RGB}{81,84,87} 
\definecolor{gpurple}{RGB}{217,210,233} 

\newcommand{\ie}{\textit{i.e.}\xspace}
\newcommand{\eg}{\textit{e.g.}\xspace}
\newcommand{\iid}{\textit{i.i.d.}\xspace}

\makeatletter
\newcommand{\upmax}{\def\blx@maxcitenames{99}}
\newcommand{\dnmax}{\def\blx@maxcitenames{1}}
\makeatother
\newcommand{\myfullcite}[1]{\upmax\fullcite{#1}\dnmax}

\newcommand{\Or}[1]{\mathcal{O}\mathopen{}\left(#1\right)\mathclose{}}

\newcommand{\Th}[1]{\Theta\mathopen{}\left(#1\right)\mathclose{}}

\newcommand{\encase}[1]{{\left[#1\right]}}
\newcommand{\brac}[1]{{\left(#1\right)}}

\newcommand{\R}{\mathbb{R}}

\newcommand{\trace}{\operatorname{tr}}

\newcommand{\norm}[1]{\Vert #1 \Vert}
\newcommand{\ip}[2]{\big\langle #1, \, #2 \big\rangle}
\newcommand{\prox}[1]{\mathrm{prox}_{#1}}

\newcommand{\lmo}{\textit{lmo}\xspace}
\def\bbE{{\mathbb{E}}}
\def\E{{\mathbb{E}}}
\def\sR{{\mathbb{R}}}
\def\sZ{{\mathbb{Z}}}

\DeclareMathOperator*{\argmin}{arg\,min}
\DeclareMathOperator*{\argmax}{arg\,max}

\DeclareMathOperator{\conv}{conv}
\DeclareMathOperator{\lin}{lin}
\DeclareMathOperator{\cone}{cone}

\newcommand{\dkl}{D^{KL}}
\newcommand{\KL}{D^{KL}}
\newcommand{\relbo}{\textsc{RELBO}\xspace}

\newcommand{\FW}{{\textsf{\tiny FW}}}

\DeclareMathOperator{\l1}{L1-ball}
\DeclareMathOperator{\Cf}{C_f}
\DeclareMathOperator{\radius}{\mathrm{radius}}
\DeclareMathOperator{\diam}{\mathrm{diam}}
\DeclareMathOperator{\mdw}{mDW(\cA)}

\DeclareMathOperator{\amx}{\left\lbrace\cA\cup-\frac{\theta_k}{\|\theta_k\|_\cA}\right\rbrace}
\DeclareMathOperator{\amxnok}{\left\lbrace\cA\cup-\frac{\theta}{\|\theta\|_\cA}\right\rbrace}
\DeclareMathOperator{\faces}{faces}
\DeclareMathOperator{\gfaces}{g-faces}

\DeclareMathOperator{\cw}{CWidth(\cA)}
\DeclareMathOperator{\clip}{clip}

\def\cA{\mathcal{A}}
\def\cB{\mathcal{B}}
\def\cD{\mathcal{D}}
\def\cF{\mathcal{F}}
\def\cH{\mathcal{H}}
\def\cK{\mathcal{K}}

\def\cQ{\mathcal{Q}}
\def\cS{\mathcal{S}}
\def\cT{\mathcal{T}}
\def\cZ{\mathcal{Z}}

\def\rvu{{\mathbf{i}}}

\def\rvs{{\mathbf{s}}}

\def\rvu{{\mathbf{u}}}

\def\rvw{{\mathbf{w}}}
\def\rvx{{\mathbf{x}}}
\def\rvy{{\mathbf{y}}}
\def\rvz{{\mathbf{z}}}

\def\sfA{\mathsf{A}}

\def\sfZ{\mathsf{Z}}

\def\rz{{\textnormal{z}}}

\def\vu{{\textbf{u}}}
\def\vv{{\textbf{v}}}

\def\vx{{\textbf{x}}}

\def\mA{{\textbf{A}}}

\def\mI{{\textbf{I}}}

\def\gM{{\mathcal{M}}}

\def\gS{{\mathcal{S}}}

\def\gX{{\mathcal{X}}}

\def\gZ{{\mathcal{Z}}}

\renewcommand{\P}{P}
\newcommand{\Q}{Q}

\DeclareRobustCommand{\myhammer}{%
  \begingroup\normalfont
  \includegraphics[height=1.1\fontcharht\font`\B]{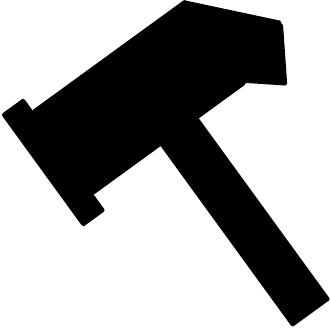}%
  \endgroup
}
\newcommand{\PA}{\textbf{PA}}
\newcommand{\supp}{\operatorname{supp}}

\newcommand\blfootnote[1]{%
  \begingroup
  \renewcommand\thefootnote{}\footnote{#1}%
  \addtocounter{footnote}{-1}%
  \endgroup
}

\newcommand{\softmax}[1]{\texttt{Softmax}\,(#1)}

\newcommand{\gru}[2]{\texttt{GRU}\,(#1, \ #2)}
\newcommand{\slots}[0]{\texttt{slots}}
\newcommand{\attn}[0]{\texttt{attn}}
\newcommand{\inp}[0]{\texttt{inputs}}
\newcommand{\mlp}[0]{\texttt{MLP}}
\newcommand{\sam}[0]{Slot Attention\xspace}

\newcommand{\updates}[0]{\texttt{updates}}
\newcommand{\layernorm}[0]{\texttt{LayerNorm}}

\bibliography{library}

\begin{document}
\initETHthesis

\dissnum{27248}
\title{Enforcing and Discovering Structure in Machine Learning}
\author{Francesco Locatello}
\acatitle{MSc ETH in Computer Science, ETH Zurich}
\dateofbirth{01.08.1992}
\citizen{Italy}
\Year{2020}
\examiners{
  Prof.\ Dr.\ Gunnar R\"atsch (ETH Zurich), examiner\linebreak
  Prof.\ Dr.\ Andreas Krause (ETH Zurich), co-examiner \linebreak
  Prof.\ Dr.\ Bernhard Sch\"olkopf (MPI T\"ubingen), co-examiner \linebreak
  Prof.\ Dr.\ Volkan Cevher (EPFL), co-examiner
}
\support{}
\disclaimer{}
\pagestyle{empty}
\hypersetup{pageanchor=false}
\maketitle
\hypersetup{pageanchor=true}

\pagestyle{plain}
\pagenumbering{roman}
\chapter*{Abstract}
The world is structured in countless ways. It may be prudent to enforce corresponding structural properties to a learning algorithm’s solution, such as incorporating prior beliefs, natural constraints, or causal structures. Doing so may translate to faster, more accurate, and more flexible models, which may directly relate to real-world impact. In this dissertation, we consider two different research areas that concern structuring a learning algorithm’s solution: when the structure is known and when it has to be discovered.

First, we consider the case in which the desired structural properties are known, and we wish to express the solution of our learning algorithm as a sparse combination of elements from a set. We assume that this set is given in the form of a constraint for an optimization problem.  Specifically, we consider convex combinations with additional affine constraints, linear combinations, and non-negative linear combinations. In the first case, we develop a stochastic optimization algorithm suitable to minimize non-smooth objectives with applications to Semidefinite Programs. In the case of linear combinations, we establish a connection in the analysis of Matching Pursuit and Coordinate Descent, which allows us to present a unified analysis of both algorithms. We also show the first accelerated convergence for both matching pursuit and steepest coordinate descent on convex objectives. On convex cones, we present the first principled definitions of non-negative MP algorithms which provably converge on convex and strongly convex objectives. Further, we consider the applications of greedy optimization to the problem of approximate probabilistic inference. We present an analysis of existing boosting variational inference approaches that yields novel theoretical insights and algorithmic simplifications. 

Second, we consider the case of learning the structural properties underlying a dataset by learning its \textit{factors of variation}. This is an emerging field in representation learning that starts from the premise that real-world data is generated by a few explanatory factors of variation, which can be recovered by (unsupervised) learning algorithms. Recovering such factors should be useful for arbitrary downstream tasks. We challenge these ideas and provide a sober look at the common practices in the training and evaluation of such models. From the modeling perspective, we discuss under which conditions factors of variation can be disentangled and perform extensive empirical evaluations in the unsupervised, semi-supervised, and weakly-supervised settings. Regarding the evaluation, we discuss the biases and usefulness of the disentanglement metrics and the downstream benefits of disentanglement, particularly for fairness applications.
Overall, we find that the unsupervised learning of disentangled representations is theoretically impossible, and unsupervised model selection appears challenging in practice. On the other hand, we also find that little and imprecise explicit supervision (in the order of 0.01--0.5\% of the dataset) is sufficient to train and identify disentangled representations in the seven datasets we consider.
Motivated by these results, we propose a new weakly-supervised disentanglement setting that is theoretically identifiable and does not require explicit observations of the factors of variation, providing useful representations for diverse tasks such as abstract visual reasoning, fairness, and strong generalization. Finally, we discuss the conceptual limits of disentangled representation and propose a novel paradigm based on attentive grouping. We propose a differentiable interface mapping perceptual features in a distributed representational format to a set of high-level, task-dependent variables that we evaluate on set prediction tasks. 

\chapter*{Sommario}
Il mondo \`e strutturato in innumerevoli modi. Potrebbe essere utile imporre simili propriet\`a strutturali alla soluzione di un algoritmo di apprendimento automatico, ad esempio incorporando conoscenza pregressa, vincoli naturali o strutture causali. Ciò potrebbe tradursi in modelli più veloci, più accurati e più flessibili, che potrebbero avere un impatto nel mondo reale. In questa tesi consideriamo due diverse aree di ricerca, che riguardano la struttura della soluzione di un algoritmo di apprendimento automatico: quando la struttura è nota e quando deve essere individuata.

\looseness=-1Innanzitutto consideriamo il caso in cui le proprietà strutturali desiderate sono note e si vuole esprimere la soluzione del nostro algoritmo di apprendimento automatico come una combinazione sparsa di elementi di un insieme. Assumiamo che questo insieme sia dato sotto forma di vincolo per un problema di ottimizzazione. In particolare consideriamo combinazioni convesse con aggiuntivi vincoli affini, combinazioni lineari e combinazioni lineari non negative. Nel primo caso sviluppiamo un algoritmo di ottimizzazione stocastica adatto a minimizzare funzioni non lisce con applicazioni a Programmi Semidefiniti. Nel caso di combinazioni lineari stabiliamo una connessione nell'analisi di Matching Pursuit e Coordinate Descent, che ci permette di presentare un'analisi unificata di entrambi gli algoritmi. Dimostriamo anche la prima convergenza accelerata sia per Matching Pursuit che per Coordinate Descent su funzioni convesse. Per vincoli conici convessi presentiamo i primi algoritmi di MP non negativi, per i quali dimostriamo la convergenza su funzioni convesse e fortemente convesse. Inoltre consideriamo le applicazioni dell'ottimizzazione greedy al problema dell'inferenza probabilistica approssimata. Presentiamo quindi un'analisi degli approcci esistenti di inferenza variazionale che risultano in nuove intuizioni teoriche e semplificazioni algoritmiche.

\looseness=-1In secondo luogo consideriamo il caso di dover apprendere le proprietà strutturali sottostanti a un insieme di dati, scoprendo i suoi \textit{fattori di variazione}. Si tratta di un settore emergente nell'apprendimento automatico delle rappresentazioni, partendo dal presupposto che i dati nel mondo reale sono generati da pochi fattori di variazione esplicativi, che possono essere scoperti da algoritmi di apprendimento automatico (non supervisionati). La scoperta di tali fattori dovrebbe essere utile per arbitrari problemi successivi. Mettiamo in discussione queste idee e forniamo un'analisi sobria delle pratiche comuni nell'apprendimento e nella valutazione di tali modelli. Dal punto di vista della modellizzazione discutiamo in quali condizioni i fattori di variazione possono essere scoperti ed eseguiamo estese valutazioni empiriche  in contesti non supervisionati, semi-supervisionati e debolmente supervisionati. Per quanto riguarda la valutazione discutiamo la parzialit\`a, l'utilit\`a delle metriche e i vantaggi per successivi compiti, in particolare per le applicazioni di predizioni eque.
Nel complesso scopriamo che l'apprendimento automatico senza supervisione dei fattori di variazione è teoricamente impossibile e la selezione dei modelli senza supervisione appare impegnativa nella pratica. D'altra parte, troviamo anche che una limitata supervisione esplicita ed imprecisa (nell'ordine dello 0,01-0,5 \% del set di dati) è sufficiente per imparare ed identificare i fattori di variazione nei sette set di dati che consideriamo.
Motivati da questi risultati, proponiamo un nuovo scenario debolmente supervisionato, che è teoricamente identificabile e non richiede osservazioni esplicite dei fattori di variazione, fornendo rappresentazioni utili per svariati compiti come ragionamento visivo astratto, predizioni eque e generalizzazione forte. Infine discutiamo i limiti concettuali dell'imparare fattori di variazione e proponiamo un nuovo paradigma basato sul raggruppamento attento. Proponiamo quindi un'interfaccia differenziabile, che mappa da una descrizione percettiva dei dati in un formato rappresentazionale distribuito ad un insieme di variabili di alto livello, che dipendono dall'obiettivo dell'apprendimento e che valuitamo in attivit\`a di predizione di insiemi. 

\chapter*{Acknowledgments}
I wish to thank everyone that supported me in the past few years in the effort leading to this dissertation.

\looseness=-1First, I'd like to thank my advisors, Gunnar R\"atsch and  Bernhard Sch\"olkopf, for supporting and supervising me over these years. I'm grateful for the freedom you gave me while helping me focus on the important questions and giving me your precious feedback and advice! Thank you for being patient with me and for being such great advisors and mentors.
I am also grateful for the stimulating research environment in their labs at ETH and MPI. I'd like to thank present and past members of the BMI group and Empirical Inference department (especially the Empirical Inferencers pub quiz team). I would also like to thank Prof. Dr. Andeas Krause and Prof. Dr. Volkan Cevher for serving on my Ph.D. committee.

\looseness=-1I would like to give a special thanks to Olivier Bachem, who has been a great friend, manager, and mentor. He made my time at Google Zurich the most productive and most fun of my PhD. He also helped me and supported me over the past two years, giving me invaluable feedback for both my research and professional growth. I really cannot thank you enough!

\looseness=-1Another special thanks goes to everyone else that mentored me and helped me during the PhD, in particular Stefan Bauer, Volkan Cevher, Martin Jaggi, Rajiv Khanna, Mario Lucic, Quentin de Laroussilhe, Thomas Kipf, Nal Kalchbrenner, Olivier Bousquet, Sylvain Gelly, Nicolas Le Roux, and Michael Tschannen. Thank all of you for listening to me, for your precious advice, and for helping me navigate the difficult times.

\looseness=-1I am grateful to all my co-authors: Alp Yurtsever, Olivier Fercoq, Volkan Cevher, Anant Raj, Sai Praneeth Reddy Karimireddy,  Sebastian Urban Stich, Martin Jaggi, Michael Tschannen, Rajiv Khanna, Joydeep Ghosh, Gideon Dresdner, Isabel Valera,  Stefan Bauer, Mario Lucic, Sylvain Gelly, Olivier Bachem,  Gabriele  Abbati,  Tom  Rainforth, Ben Poole, Dirk Weissenborn, Thomas Unterthiner, Aravindh Mahendran, Georg Heigold, Damien Vincent, Ilya Tolstikhin, Jakob Uszkoreit, Alexey Dosovitskiy, Thomas Kipf, Nan Rosemary Ke, Nal Kalchbrenner, Anirudh Goyal, Yoshua Bengio, Bernhard Sch\"olkopf and Gunnar R\"atsch. I'm honored I could work with all of you, you were all fundamental for both my growth as a researcher and for this dissertation! I also thank all the other people with whom I co-authored a paper, in particular Sjoerd Van Steenkiste, Alp Yurtsever, Vincent Fortuin, Waleed Gondal, Manuel Wuthrich, Luigi Gresele, Paul Rubenstein, Frederik Tr\"auble, Stefan Stark, Joanna Ficek, Geoffrey Negiar, and Andrea Dittadi. Thank you for involving me in your projects, collaborating with you has been a pleasure!

I'd like to thank Natalia for her friendship and endless help. Most of the things I did in my PhD would just not have been possible without you. 

I'd like to thank Annika, Stefan, Anirudh, Rosemary, Thomas K., Klaus, Alexey, Jakob, Gunnar, Bernhard, Yoshua and Nal for the philosophical discussions and the influence they had on my research and this dissertation. 

I'd like to thank all the members of the Brain Team in Zurich for the great time I had, in particular Olivier, Mario, Marvin, Michael, and Joan for the many foosball matches. Thank you Mario for teaching me how to stay reasonable.

I am also very grateful to Nal Kalchbrenner for having me in his team at Google Brain Amsterdam for six months and to Thomas Kipf for being a fantastic collaborator and my go-to person for just about everything during my internship. Thank you both for all the mentorship, inspiring discussions, and support. I'd like to thank also all the other members of the Amsterdam and Berlin teams for the great athmosphere and the fun times (despite the lockdown), in particular Klaus, Alexey, Jakob, Tim, Rianne, Mostafa, Avital, Thomas U., Aravindh, Sindy, Manoj, Nal, and Thomas K.

I gratefully acknowledge the institutions that supported my research, in particular the
Department of Computer Science at ETH Zurich, the Max Planck Institute for Intelligent Systems, the ETH MPI Center for Leaerning Systems, ELLIS, and Google (through a 2019 Google PhD Fellowship, the MSRA Partnership with ETH and MPI, and an internship).

\looseness=-1Many thanks to all my friends, in particular Stefano, Tommaso, Riccardo, Giuseppe, Stefan, Stefan, Gideon, Fabio, Michael, Gabriele, Andrii,  Kseniia, and Dani (who tragically left us). 

I am grateful to Sarah, for her unconditional support and for enriching my life.

Finally, I'd like to thank my family, Alessandro, Chiara, and Silvano, but also Marisa and Piero (I wish you were still here) for believing in me, enduring my absence, and for helping me follow my dreams!

\cleardoublepageempty{}
\selectlanguage{english}

\setcounter{tocdepth}{1}
\pagestyle{headings}
\pdfbookmark[0]{\contentsname}{contents}
\tableofcontents
\cleardoublepageempty{}

\pagenumbering{arabic}
\cleardoublepageempty{}



\chapter{Introduction}
\section{Broad Motivation}
\looseness=-1Structure is pervasive in human's understanding and description of reality. In Physics, we describe phenomena with an intricate combination of elementary concepts that can be composed with each other.
For example, we model rigid body motion with the laws of classical mechanics and the behavior of charged objects with Maxwell's equations.
Structuring scientific knowledge has the advantage that elementary laws can be more easily and individually falsified~\citep{popper2005logic}. 
Interestingly, this structure may not be a human construct arising from our scientific pursuits. Philosophers argued that the existence of a natural structure should play a part in the correctness of counterfactual claims~\citep{chisholm1946contrary,chisholm1955law,goodman1947problem} which are in turn related to human's learning. While Kant believed that causal understanding is innate in Humans~\citep{kant1977prolegomena}, Hume argued that the mind is unable to perceive causal relations directly and can only extrapolate them through counterfactual reasoning~\citep{hume2000enquiry}.
Work in cognitive psychology \citep{epstude2008functional} found that counterfactuals indeed allow reasoning about the usefulness of past actions and transferring these insights to corresponding behavioral intentions in future scenarios \citep{roese1994functional,reichert1999reflective,landman1995missed}. Without an underlying natural structure, the induction from experimental evidence to law formulation would likely not be possible, although this conclusion is not yet widely accepted in modern philosophy~\citep{gabriel2015world}.

Similarly, the fundamental concept of \textit{generalization} emerged in the early days of machine learning~\citep{Solomonoff1964} under the assumption that future data will be similar to past data. Even today, this assumption is present in the form of the widespread \iid assumption. The predominant philosophy of modern machine learning is connectionism, which attempts to model the brain using an artificial neural network~\citep{rosenblatt1958perceptron} learning a distributed representation of the data~\citep{hinton1984distributed}, where knowledge is distributed in a pattern of activations across multiple computing elements (\ie the neurons). As deep learning approaches are dominating the field in computer vision~\citep{he2016deep,krizhevsky2012imagenet}, natural language processing~\citep{devlin2018bert}, and speech recognition~\citep{graves2013speech}, an emerging body of literature is questioning the robustness of the prediction of state-of-the-art models~\citep{hendrycks2019benchmarking,karahan2016image,michaelis2019benchmarking,roy2018effects,azulay2019deep,barbu2019objectnet,engstrom2017exploring,zhang2019making,gu2019using,shankarimage,barbu2019objectnet}. In fact, the assumption that future data will follow the same distribution of training data is often violated in practice as a model is deployed to solve real-world tasks in the wild. While scale is a viable short-term answer to this issue~\citep{brown2020language}, new work is arguing for the advantages of structuring representations to better align with our understanding of physics and human cognition~\citep{bengio2013representation,PetJanSch17,scholkopf2020towards}. Pragmatically, several real-world successful applications of deep learning are already trained from the internal representation of simulators~\citep{battaglia2016interaction,sanchez2020learning} and game engines~\citep{berner2019dota,vinyals2019grandmaster} rather than raw observational data. 

These arguments serve as a conceptual motivation for the present dissertation. Starting from the premise that there exists a natural structure underlying some data distribution, we argue that incorporating the same structure into a learning algorithm may be beneficial. This structure may come from prior beliefs, physical constraints, or knowledge about the causality of a system, and could either be given or it has to be discovered. 
More concretely, we restrict ourselves to the settings of (1) Enforcing constraints to the solution of a learning algorithm through its optimization formulation. In particular, we focus on constraints that can be written as a combination (convex, linear, or non-negative linear) of a set of \textit{atoms} (that may be vectors, matrices, or sometimes functions). This corresponds to the case where the structural constraints are known. 
(2) We investigate under which conditions and to which extent neural networks can discover underlying factors of variation in a dataset with various degrees of supervision. In this scenario, we wish to discover the natural structure underlying a dataset and investigate the feasibility of this problem and its usefulness for different downstream tasks.

\section{Summary of Main Results}

In this dissertation, we explore two main research questions related to enforcing and discovering structure:

\begin{enumerate}
  \item \emph{How can we efficiently constrain a learning algorithm to express its solution as a combination of elements from a set?}
  

Learning problems with real-life applications often benefit from incorporating requirements such as constraints given by natural laws. 
For example, a structured solution is desirable in many applications, due to the underlying physics or for the sake of interpretability. 
The most general convex optimization template can be formulated as follows:
\begin{align}\label{eq:main_template}
\min_{\theta\in\cD} f(\theta)
\end{align}
where $f:\cH\rightarrow \R$, $\cH$ is an Hilbert space with associated inner product $\langle \theta, y\rangle$ $\forall$ $\theta, \ y \in\cH$ and $\cD\subset \cH$. $\cD$ encodes the desired structure of the solution as a constraint for the optimization problem. We here consider the general setup of optimization over Hilbert spaces as we aim at applying our results to both convex optimization algorithms in Euclidean spaces and approximate inference.

As an example, consider $\cD = \R^n$ and $f$ convex and smooth. In this case, the most common approach is gradient descent. Gradient descent is an iterative algorithm in which the iterate is updated as $\theta_{k+1} = \theta_k - \gamma_k \nabla f(\theta_k)$, where $\gamma_k$ is some step size. If $\cD$ is a proper subset of $\R^n$, one has to project the iterate at each step to make sure it remains feasible, for example computing $\theta_{k+1} = \argmin_{\theta\in\cD} \|\theta -  \theta_k + \gamma_k \nabla f(\theta_k)\|^2$, where $\|\cdot\|$ is the norm induced by the inner product in $\cH$. Depending on $\cD$, this can be a challenging optimization problem. The recent developments in machine learning applications with vast data brought the scalability of first-order optimization methods like projected gradient descent under scrutiny. As a result, there has been a renewed interest in projection-free optimization algorithms. 
We study mainly three classes of algorithms, which differ in the type of constraints they handle. Let $\cA\subset\cH$ be a compact set. Then, we consider the case in which $\cD$ is the convex hull ($\cD = \conv(\cA)$), the conic-hull ($\cD = \cone(\cA)$) and the linear span ($\cD = \lin(\cA)$) of the set $\cA$.
To solve these problems we use Frank-Wolfe (FW)~\citep{frank1956algorithm}, Non-Negative Matching Pursuit (NNMP) and Matching Pursuit (MP, also known as boosting when $\cA$ is a set of functions)~\citep{Mallat:1993gu,locatello2017unified}. The structure in the solution is given by the choice of $\cA$. For example, $\cA$ can be the set of rank one matrices, and so $\conv(\cA)$ is the trace norm ball, $\cone(\cA)$ is the set of all positive semidefinite matrices, and $\lin(\cA)$ is the set of all matrices.
Our work covers both deterministic and stochastic optimization and focuses on proving convergence guarantees for new and existing algorithms.


In Chapter~\ref{ch:SFW}~\citep{locatello2020stochastic}, we propose a stochastic FW method for solving stochastic convex minimization problems with affine constraints over a compact convex domain:
\begin{equation}\label{eq:shcgm}
\underset{\theta\in\conv(\cA)}{\text{minimize}} \quad \bbE_\rvw  f(\theta,\omega) +g(A\theta)
\end{equation}
where $g$ can be an indicator function. In this scenario, vanilla stochastic FW and projection methods suffer from high computational complexity. There are many practical applications for this template including scalable stochastic optimization of Semidefinite Programs (SDPs) and splitting methods for stochastic and online optimization. Stochastic and online optimization of SDPs alone have countless applications: clustering, online max cut, optimal power-flow, sparse PCA, kernel learning, blind deconvolution, community detection, convex relaxation of combinatorial problems, amongst many others. 
Our algorithm has a $\Or{k^{-1/3}}$ convergence rate in expectation on the objective residual, and $\Or{k^{-5/12}}$ in expectation on the feasibility gap. Surprisingly, our rate on the objective residual is asymptotically identical to recent rates for the stochastic FW with constant batch size on problems without the affine constraint~\citep{mokhtari2020stochastic}. 
Furthermore, the rate on the feasibility gap is only $\Or{k^{-1/12}}$ worse in expectation than the one with full gradient information~\cite{yurtsever2018conditional}. 

\looseness=-1In Chapter~\ref{cha:mpcd}~\citep{locatello2018matching}, we consider the connection between Matching Pursuit and Coordinate Descent. Coordinate descent can be seen as special case of MP as it solves the optimization problem moving the iterate along coordinates~\cite{nesterov2012efficiency}, while MP considers a generalized notion of directions (in CD $\cA$ is fixed to contain coordinates and solving the linear problem of MP is equivalent to finding the steepest coordinate). In light of this connection, we unify the analysis of both algorithms, proving affine invariant sublinear $\Or{1/k}$ rates on convex and smooth objectives and linear convergence on strongly convex smooth objectives. Furthermore, we prove the first accelerated convergence rate $\Or{1/k^2}$ for matching pursuit and steepest coordinate descent (only accelerated rates for random coordinate descent were known before~\cite{nesterov2012efficiency}) on convex objectives. 

\looseness=-1In Chapter~\ref{cha:cone}~\citep{locatello2017greedy}, we consider the case of optimization over the \textit{convex cone}, parametrized as the conic hull  of a generic atom set, leading to the first principled definitions of non-negative MP algorithms. Concrete examples of this setup include unmixing problems, projections, and non-negative matrix and tensor factorizations (using heuristic oracles). 
We derive sublinear ($\Or{1/k}$) convergence on general smooth and convex objectives, and linear convergence ($\Or{e^{-k}}$) on strongly convex objectives, in both cases for general sets of atoms. 

In Chapter~\ref{cha:boostingVI}~\citep{LocKhaGhoRat18,locatello2018boosting}, we use tools from the convex optimization literature to study the problem of approximate Bayesian inference. Here, the optimization is over spaces of probability distributions. Approximating probability densities is a core problem in Bayesian statistics and representation learning, where inference translates to the computation of a posterior distribution.
Posterior distributions depend on the modeling assumptions and can rarely be computed  exactly. Many methods popular today rely on a "flexibly parametrize, optimize and hope for the best" paradigm. In our work, we consider boosting variational inference that has been proposed as a new principled approach to approximate a posterior distribution~\cite{Guo:2016tg,Miller:2016vt}:
\begin{equation*}
\min_{q\in\conv(\cA)} \dkl(q(z)\|p(z|x))
\end{equation*}
where $\conv(\cA)$ now represents the space of mixtures of the densities in some family $\cA$. Boosting algorithms construct a mixture of these densities by greedily adding components to the solution. Assuming that one can find the components, building a mixture is a convex problem for which we discuss convergence properties.
Further, we rephrase the linear optimization problem of the FW subroutine and propose to maximize the Residual ELBO (RELBO), which replaces the standard ELBO (Evidence Lower BOund) optimization in VI.
These theoretical enhancements allow for black-box implementation of the boosting subroutine.

\item \emph{If we want learning algorithms to enrich and complement our understanding of reality, how can they discover new structure that we did not already know?}

\looseness=-1Learning useful representations from data is considered crucial in Machine Learning~\cite{bengio2013representation}. It is often argued that a representation can hide or reveal the underlying mechanisms governing the data. A recent trend in the community is to learn representations that \textit{disentangle} the factors of variation in a data set. The common wisdom is that disentangled representations are useful not only for (semi-)supervised downstream tasks but also transfer and few-shot learning. Unfortunately, a vague definition of what disentangled representations actually means has brought confusion to the community. After the $\beta$-VAE paper~\citep{higgins2016beta}, several approaches where proposed to learn disentangled representations within the VAE framework~\citep{kim2018disentangling,eastwood2018framework,kumar2017variational,chen2018isolating,ridgeway2018learning,suter2018interventional}.

In \emph{Variational Autoencoders (VAEs)}~\cite{kingma2013auto},  one assumes a prior $P(\rvz)$ on the latent space and parameterizes the conditional probability $\P(\rvx|\rvz)$ using a deep neural network (i.e., a \textit{decoder network} $p_\theta(\rvx|\rvz)$). The posterior distribution is approximated by a variational distribution $\Q(\rvz|\rvx)$, again parameterized using a deep neural network (i.e., an \textit{encoder network} $q_\phi(\rvz|\rvx)$). The model is then trained by maximizing a variational lower-bound to the log-likelihood:
\begin{align*}
\max_{\phi, \theta}\quad \E_{\rvx} [\E_{q_\phi(\rvz|\rvx)}[\log p_\theta(\rvx|\rvz)] -  \KL(q_\phi(\rvz|\rvx) \| p(\rvz))\ .
\end{align*}

The common approach for disentanglement is to enforce some structural constraints to the distribution learned by the encoder of a VAE through a carefully designed regularizer:
\begin{align*}
\max_{\phi, \theta}\quad \E_{\rvx} [\E_{q_\phi(\rvz|\rvx)}[\log p_\theta(\rvx|\rvz)] -  \KL(q_\phi(\rvz|\rvx) \| p(\rvz)) + \beta R_u(q_\phi(\rvz|\rvx))]
\end{align*}
This regularizer corresponds to a constraint $R_u(q_\phi(\rvz|\rvx))\leq \tau$, where $R_u$ is a function enforcing certain statistical properties to the encoder  $q_\phi$ (\eg factorizing aggregate posterior) and $\tau >0$ is a threshold.

\looseness=-1The early promising results on synthetic data sets led to several applications anecdotally linking disentangled representations learned by autoencoders to downstream benefits~\citep{steenbrugge2018improving,laversanne2018curiosity,nair2018visual,higgins2017darla,higgins2018scan}. In our work, we broadly investigate disentangled representations learned with VAEs, exploring different supervision settings, their evaluation, usefulness and conceptual limits.

In Chapter~\ref{cha:unsup_dis}~\citep{locatello2019challenging,locatello2020sober,locatello2020commentary}, we provide a sober look at the unsupervised learning of disentangled representations, discussing the recent progress and highlighting the limits. We present a theoretical result showing that the unsupervised learning of disentangled representations is impossible for arbitrary data sets. Further, we analyse the performance of state-of-the-art approaches, focusing on model selection questions that are particularly relevant for practitioners.
We observe that while the different methods successfully enforce properties ``encouraged'' by the corresponding losses, well-disentangled models seemingly cannot be identified without supervision.

In Chapter~\ref{cha:eval_dis}~\citep{locatello2020sober}, we focus on the evaluation of disentangled representations. In fact, measuring disentanglement is non-trivial and which metric should be used is debated. Therefore, we study the different ``notions'' of disentanglement being measured by the different metrics and investigate how seemingly small implementation decisions can affect the end results. These considerations are important to better interpret the results of Chapters~\ref{cha:unsup_dis},~\ref{cha:semi_sup},~\ref{cha:fairness}, and~\ref{cha:weak}.

In Chapter~\ref{cha:semi_sup}~\citep{locatello2019disentangling}, we investigate the impact of explicit supervision on state-of-the-art disentanglement methods. We observe that a small number of labeled examples (0.01--0.5\% of the data set), with potentially imprecise and incomplete labels, is sufficient to perform model selection. Further, we investigate the benefit of incorporating supervision into the training process.
Overall, we empirically validate that it is possible to reliably learn disentangled representations with little and imprecise supervision. Although this setting is clearly less elegant than a purely unsupervised approach, we argue that imprecise explicit supervision may be cheaply obtained in some applications. 

In Chapter~\ref{cha:fairness}~\citep{locatello2019fairness}, we investigate the usefulness of the notions of disentanglement studied in Chapter~\ref{cha:eval_dis} for improving the fairness of simple downstream classification tasks. We consider the setting of predicting a target variable based on a learned representation of high-dimensional observations (such as images) that depend on both the target variable and an \emph{unobserved} sensitive variable. We make the additional assumption that target and sensitive variables are only dependent conditioned on the observations. While this may seem restrictive, we show that training fair classifiers is still non-trivial. Analyzing the representations of \num{12600} models we trained for the analysis in Chapter~\ref{cha:fairness}, we observe that several disentanglement scores are consistently correlated with increased fairness, suggesting that disentanglement may be a useful property to encourage fairness when sensitive variables are not observed (under the assumption that disentangled representations can be learned and identified without explicit suppervision).

In Chapter~\ref{cha:weak}~\citep{locatello2020weakly}, we consider the setting where an agent is trying to learn disentangled representations observing changes in their environment. We model this setting sampling pairs of non-i.i.d. images sharing most of the underlying factors of variation. These can be thought of as nearby frames in a video, under the assumption that changes in temporally close frames should be sparse. We prove that this setting is theoretically identifiable under some additional assumptions, such as knowing \textit{how many} factors have changed, but not which ones. Inspired by the analysis, we provide methods to learn disentangled representations from paired observations. In a large-scale empirical study, we show that this type of weak supervision allows learning of disentangled representations on several benchmark data sets. Further, we find that these representations are \emph{simultaneously} useful on a diverse suite of tasks, including generalization under covariate shifts, fairness, and abstract reasoning. Overall, our results demonstrate that weak supervision enables learning of useful disentangled representations in arguably realistic scenarios.

In Chapter~\ref{cha:slot_attn}~\citep{locatello2020object}, we discuss a critical limitation of disentanglement. Pragmatically, disentangled representations as described in Chapters~\ref{cha:unsup_dis}--\ref{cha:weak} are the output of a (convolutional) neural network and are represented in a vector format. This is problematic as it prohibits compositional generalization and does not permit the network to learn a notion of objects and their description through independently controllable factors of variation. In fact, the capacity of the representation is fixed and only a fixed and constant number of objects (across the data set) can be disentangled. Turning to a simpler supervised learning scenario, we introduce Slot Attention, a novel architectural component that maps perceptual features such as the output of a CNN to a set of slots with a common representational format. These slots are exchangeable and can bind to any object in the input by specializing through a competitive procedure over multiple rounds of attention. 
In a supervised set prediction task, we show that Slot Attention succeeds in learning a set representation of the input that generalizes to a different number of objects at test time. 

\end{enumerate}

\paragraph{Personal Retrospective} \looseness=-1From the beginning of my PhD, I was interested in improving learning algorithms by incorporating natural structure. In the spirit of tackling simpler problems first, I spent the first two years of my PhD focusing on how to efficiently incorporate known structure. My goal was to develop a framework to express the solution of a learning algorithm as a general combination of elements from a set. This line of research proved to be rather fruitful from the theoretical perspective, but I struggled to find immediate and exciting real-world applications (with the partial exception of Chapter~\ref{ch:SFW}, which describes my last paper in optimization as a first author). In the last two years of the Ph.D., I became more interested in learning the structure. My interest sparked during ICML 2018, where I visited several disentanglement talks and posters with Olivier Bachem. At the superficial level, it was easy to be convinced of the usefulness of factorizing information in a disentangled format, as also argued in the seminal paper from~\citet{bengio2013representation}. We decided to investigate which inductive bias was key for disentanglement, develop new state-of-the-art algorithms, and explore applications. This problem turned out to be much more challenging than expected and led to a whole research agenda. In Chapter~\ref{cha:disent_background}, I describe in more details our research agenda and on how our paper ``\textit{Challenging Common Assumptions in the Unsupervised Learning of Disentangled Representations}''~\citep{locatello2019challenging} (best paper award at ICML 2019) changed my views on this problem shaping my following research. From the theoretical impossibility and the practical limitations, we set off to address all the issues we discovered. Most came together in the ``\textit{Weakly-Supervised Disentanglement without Compromises}''~\citep{locatello2020weakly} described in Chapter~\ref{cha:weak}, where we could present identifiability results with practical algorithms that were useful on different tasks. In Chapters~\ref{cha:slot_attn} and~\ref{cha:conclusion}, I will highlight conceptual limitations of the disentanglement framework. I believe that addressing these issues will require different architectures and move machine learning closer to causality.

\section{Publications relevant to this dissertation}
\label{sec:publications}
This dissertation (including the present section) is based upon the following publications and technical reports:

\textbf{Part~\ref{part:opt}: \nameref{part:opt}}
\begin{itemize}
  \item \myfullcite{locatello2020stochastic}
  \item \myfullcite{locatello2018matching}
  \item \myfullcite{locatello2017greedy}
  \item \myfullcite{LocKhaGhoRat18} 
  \item \myfullcite{locatello2018boosting} 
\end{itemize}

\textbf{Part~\ref{part:rep_learn}: \nameref{part:rep_learn}}
\begin{itemize}
  \item \myfullcite{locatello2019challenging} 
  \item \myfullcite{locatello2020commentary}
  \item \myfullcite{locatello2020sober}
  \item \myfullcite{locatello2019fairness}
  \item \myfullcite{locatello2019disentangling}
  \item \myfullcite{locatello2020weakly}
  \item \myfullcite{locatello2020object} 
  \item \myfullcite{scholkopf2020towards}
  
\end{itemize}
The following publications and technical reports are also relevant to but not covered in this dissertation.
\begin{itemize}
  \item \myfullcite{locatello2017unified}
  \item \myfullcite{yurtsever2018conditional}
  \item \myfullcite{locatello2018clustering}
  \item \myfullcite{fortuin2018deep}
  \item \myfullcite{gondal2019transfer}
  \item \myfullcite{van2019disentangled}
  \item \myfullcite{gresele2019incomplete}
  \item \myfullcite{trauble2020independence}
  \item \myfullcite{stark2020scim}
  \item \myfullcite{negiar2020stochastic}
  \item \myfullcite{dittadi2020transfer}
\end{itemize}

\section{Collaborators}
The content of this dissertation was developed across multiple institutions and with multiple collaborators. The work described in Part~\ref{part:opt}, was done while Francesco Locatello was at ETH Zurich and at the Max-Planck Institute for Intelligent Systems and over several visits to EPFL. This work was developed in collaboration with (listed in random order) Alp Yurtsever, Olivier Fercoq, Volkan Cevher, Anant Raj, Sai Praneeth Karimireddy, Gunnar R\"atsch, Bernhard Sch\"olkopf, Sebastian U. Stich, Martin Jaggi, Michael Tschannen, Gideon Dresdner, Rajiv Khanna, Isabel Valera, Joydeep Ghosh. The work described in Part~\ref{part:rep_learn}, was done while Francesco Locatello was at ETH Zurich, at the Max-Planck Institute for Intelligent Systems, and at Google Research (Brain Zurich and Amsterdam teams). 
This work was developed in collaboration with (listed in random order) Ben Poole, Stefan Bauer, Mario Lucic, Gunnar R\"atsch, Gabriele Abbati, Tom Rainforth, Sylvain Gelly, Bernhard Sch\"olkopf, Olivier Bachem, Rosemary Nan Ke, Nal Kalchbrenner, Anirudh Goyal, Yoshua Bengio, Michael Tschannen, Dirk Weissenborn, Thomas Unterthiner, Aravindh Mahendran, Georg Heigold, Jakob Uszkoreit, Alexey Dosovitskiy, and Thomas Kipf.
The specific contributions are highlighted as a dedicated paragraph at the beginning of each chapter.

\part{Constrained Optimization}
\label{part:opt}


\chapter{Introduction and Background}\label{cha:opt_background}
Greedy algorithms led to many success stories in machine learning (\eg, boosting and iterative inference), signal processing (\eg, compressed sensing), and optimization.
The most prominent representatives are matching pursuit (MP) algorithms \cite{Mallat:1993gu} with their Orthogonal variants (\eg Orthogonal Matching Pursuit -- OMP) \cite{chen1989orthogonal, Tropp:2004gc}, Coordinate Descent \citep{nesterov2012efficiency}, and Frank-Wolfe (FW)-type algorithms \cite{frank1956algorithm}.
All operate in the setting of minimizing an objective over combinations of a given set of atoms, or dictionary elements.
These classes of methods have strong similarities. In particular, they are iterative algorithms that rely on the very same subroutine, namely selecting the atom with the largest inner product with the negative gradient.

\looseness=-1The main difference is that MP and CD methods optimize over the \textit{linear span} of the atoms, while FW methods optimize over their \textit{convex hull}. An important ``intermediate case'' between the two domain parameterizations is the \emph{conic hull} of a possibly infinite atom set. In this case, the solution can be represented as a \emph{non-negative} linear combination of the atoms. All these cases may be desirable in many applications, \eg, due to the physics underlying the problem at hand, or for the sake of interpretability. This seemingly small difference has significant implications on the analysis of these algorithms. 
Concrete classical application examples include unmixing problems \cite{esser2013method, gillis2016fast,behr2013mitie}, model selection \cite{makalic2011logistic}, and (non-negative\footnote{In this case, we remark that a tractable approximation of the subroutine is still generally missing to the best of our knowledge.}) structured matrix and tensor factorizations \cite{berry2007algorithms, kim2012fast, wang2014matrixcompletion,Yang:2015wy, yaogreedy,guo2017efficient}. Other example applications include multilinear multitask learning \citep{romera2013multilinear}, matrix completion and image denoising \citep{tibshirani2015general}, boosting \citep{meir2003introduction,buhlmann2005boosting}, structured SVM training \citep{lacoste2013block}, and particular instances of semidefinite programs \citep{vandenberghe1996semidefinite}.

\looseness=-1Despite the vast literature on MP-type methods, which typically gives recovery guarantees for sparse signals, surprisingly little is known about MP algorithms in terms of optimization, \ie, how many iterations are needed to reach a defined target accuracy. Furthermore, all existing MP variants for the conic hull case \cite{bruckstein2008uniqueness,ID52513,Yaghoobi:2015ff} are not even guaranteed to converge.
In the context of sparse recovery, convergence rates typically come as a byproduct of the recovery guarantees and depend on strong assumptions (from an optimization perspective), such as incoherence or restricted isometry properties of the atom set \cite{Tropp:2004gc,davenport2010analysis}. Motivated by this line of work, \cite{Gribonval:2006ch,Temlyakov:2013wf, Temlyakov:2014eb, nguyen2014greedy} specifically target convergence rates but still rely on incoherence properties. 
On the other hand, FW methods are well understood from an optimization perspective, with strong explicit convergence results available for a large class of input problems, see, \eg, \cite{jaggi2013revisiting,LacosteJulien:2015wj} for a recent account.
A notable example application of greedy optimization algorithms that spawned an entire sub-field in the Machine Learning community is boosting \cite{friedman2001greedy,freund1999short,meir2003introduction}. The classical analysis of boosting algorithms is related to the early convergence results of Steepest Coordinate Descent \cite{luo1992convergence,ratsch2001convergence} generalized to Hilbert spaces.

\paragraph{Our Goal} 
\looseness=-1In this part of the dissertation, we aim at unifying the convergence analysis of several first-order greedy optimization methods under a single framework. This unification effort seeks to provide a generic framework that can be instantiated to a wide class of optimization problems. Previously, the discussed approaches have been separately studied by different communities. Therefore, their theoretical understanding is often geared towards the needs of the particular community. The literature on MP is focused on recovery guarantees (which are not covered in this dissertation), the literature on CD and FW is closer to optimization while boosting variational inference is rather empirical. The advantage of this unified framework is that it allows connecting different algorithms that generally apply to diverse settings. By virtue of this connection, we can extend some properties and rates from one algorithm/approach to the next.
Our contributions span novel stochastic optimization algorithms for semidefinite programs, accelerated rates for MP and Greedy Coordinate Descent, rates for the non-negative variants of Matching Pursuit, and applications of Frank-Wolfe to boosting variational inference.

\paragraph{General Setting} To give a general perspective on greedy projection-free optimization, consider the following optimization template:
\begin{equation}\label{eq:general_problem}
\min_{\theta\in\cD} f(\theta)
\end{equation}
where $f:\cH\rightarrow \R$ is a smooth function, and $D\subset\cH$ is the optimization domain in some Hilbert space $\cH$. In this dissertation, we focus on domains that can be parameterized as a weighted combination of elements from a compact set $\cA$. Our goal is to develop and analyze algorithms that do not require projections onto the optimization domain and instead follow the general greedy template of Algorithm~\ref{algo:generalgreedy}.

\begin{algorithm}[ht]
  \caption{General Greedy Optimization over Combinations of Atoms}
  \label{algo:generalgreedy}
\begin{algorithmic}[1]
  \STATE \textbf{init} $\theta_{0} \in \conv(\cA)$, $\cS =\left\lbrace \theta_0 \right\rbrace$
  \STATE \textbf{for} {$k=0\dots K$}
  \STATE \qquad Find $z_k := (\text{Approx-}) \lmo_{\cA}(\nabla f(\theta_{k}))$
  \STATE \qquad $\cS = \cS\cup z_k$
  \STATE \qquad $\theta_{k+1} = \textsf{update}(\theta_k,\cS,f) $
    \STATE \qquad \emph{Optional:} Correction of some/all atoms $z_{0\ldots k}$
  \STATE \textbf{end for}
\end{algorithmic}
\end{algorithm}
\paragraph{Linear Oracles}
Instead of projections, the primitive operation we assume we can efficiently solve are \textit{linear projections} over the set $\cA$. 
At each step of the optimization procedure, we query a linear minimization oracle (\lmo) to find the closest direction among the set~$\cA$:
\begin{equation}\tag{\lmo}\label{eqn:lmo}
\lmo_\cA(y) := \argmin_{z\in\cA} \,\ip{y}{z} \,,
\end{equation}
for a given query vector $y\in\cH$. Whether projected gradient steps or linear problems are more efficient depends on the shape and parameterization of $\cD$ through the set $\cA$. In practice, however, the solution of the \lmo subroutine is rarely computed exactly. In particular, for matrix problems (\eg $\cA$ contains rank one matrices), the \lmo can be \textit{approximated} efficiently using shifted power methods or the randomized subspace iterations \cite{HMT11:Finding-Structure}. Following \citep{jaggi2013revisiting}, we consider additive and multiplicative errors.
For given quality parameter $\delta\geq 0$, an objective with curvature $\Cf$ (see Equation 
\ref{def:Cf} or \citep{jaggi2013revisiting,locatello2017unified} for a definition) and any direction $d\in\cH$, the approximate \lmo with additive error returns at iteration $k$ of Algorithm \ref{algo:generalgreedy} a vector $\tilde{z}\in\cA$ such that:
\begin{equation}\tag{\lmo \ -- additive}\label{eqn:lmo_add}
    \ip{d}{\tilde{z}} \leq \min_{z\in\cA} \ip{d}{z} + \frac{\delta\Cf }{k + 2} 
\end{equation}
The definition of additive error for the approximate oracle is however rather specific as it depends on both the curvature of the objective and the iteration of the algorithm. Instead, we can define the multiplicative error for given quality parameter $\delta\in \left( 0,1\right]$ as:
\begin{equation}\tag{\lmo \ -- multiplicative}\label{eqn:lmo_mult}
    \ip{d}{\tilde{z}} \leq \min_{z\in\cA} \delta\ip{d}{z} 
\end{equation}

\paragraph{Update Step}\looseness=-1The update function is chosen to maintain the feasibility of the iterate. In other words, the optimization domain $\cD$ is closed under the update function.
We allow flexibility on the choice of the update rule and analyze different strategies. For example, the update might depend on the gradient, the function itself, the most recent atom $z_k$ or all the previously selected atoms.
Common to most of our analysis is that the function decrease is measured with an upper bound of $f$ at $\theta_k$, given as:
 \begin{equation}\label{eq:QuadraticUpperBound}
f(\theta) \leq g_{\theta_{k}}(\theta) \, \quad \forall \, \quad  \theta, \quad \text{where} \quad g_{\theta_{k}}(\theta)  := f(\theta_{k}) + \langle\nabla f(\theta_{k}), \theta-\theta_{k}\rangle+\frac{L}{2}\|\theta-\theta_{k}\|^2
\end{equation}
which is also considered as an alternative to line search on the true objectiv. $L$ is an upper bound on the smoothness constant of~$f$ with respect to the Hilbert norm $\|\cdot\|$.

\section{Frank-Wolfe}
The FW algorithm, also referred to as \textit{Conditional Gradient Method} (CGM), dates back to the 1956 paper of Frank and Wolfe \cite{frank1956algorithm}. It did not acquire much interest in machine learning until the last decade because of its slower convergence rate compared to the (projected) accelerated gradient methods. However, there has been a resurgence of interest in FW and its variants, following the seminal papers of Hazan \cite{hazan2008sparse} and Jaggi \cite{jaggi2013revisiting}. They demonstrated that FW might offer superior computational complexity than state-of-the-art methods in many large-scale optimization problems (that arise in machine learning) despite its slower convergence rate, thanks to its lower per-iteration cost.

The original method by Frank and Wolfe \cite{frank1956algorithm} was proposed for smooth convex minimization on polytopes. 
The analysis is extended for smooth convex minimization on simplex by Clarkson \cite{Clarkson2010}, spectrahedron by Hazan \cite{hazan2008sparse}, and finally for arbitrary compact convex sets by Jaggi \cite{jaggi2013revisiting}. All these methods are restricted to smooth objectives.

The FW algorithm, presented in Algorithm \ref{algo:FW}, targets the optimization problem
\begin{align}\label{eq:FWproblem}
\min_{\theta\in\conv(\cA)} f(\theta),
\end{align}
where $\conv(\cA) \subset \cH$ is convex and bounded and $f$ is a smooth function. In many applications (\eg lasso and low rank matrix problems), the domain is parameterized as the convex hull of a dictionary $\cA$, \ie, $\cD=\conv(\cA)$.

\begin{algorithm}[ht]
  \caption{Frank-Wolfe constant step size and line search variants \citep{frank1956algorithm,hazan2008sparse,jaggi2013revisiting,locatello2017unified}}
  \label{algo:FW}
\begin{algorithmic}[1]
  \STATE \textbf{init} $\theta_{0} \in \conv(\cA)$
  \STATE \textbf{for} {$k=0\dots K$}
  \STATE \qquad Find $z_k := (\text{Approx-}) \lmo_{\cA}(\nabla f(\theta_{k}))$
  \STATE \qquad \emph{Variant 0:} $\gamma := \frac{2}{t+2}$
  \STATE \qquad \emph{Variant 1:} $\gamma := \displaystyle\argmin_{\gamma\in[0,1]}  f({\theta_{k} \!+\! \gamma(z_k-\theta_{k})})$\vspace{-1mm}
  \STATE \qquad \emph{Variant 2:} $\gamma := \clip_{[0,1]}\!\big[ \frac{\langle -\nabla f(\theta_k), z_k - \theta_k\rangle}{\diam_{\|.\|\!}(\cA)^2 L} \big]$
  \STATE \qquad \emph{Variant 3:} $\gamma := \clip_{[0,1]}\!\big[ \frac{\langle -\nabla f(\theta_k), z_k - \theta_k\rangle}{\|z_k - \theta_k\|^2 L} \big]$
  \STATE \qquad Update $\theta_{k+1}:= \theta_{k} + \gamma(z_k-\theta_{k})$
    \STATE \qquad \emph{Optional:} Correction of some/all atoms $z_{0\ldots t}$
  \STATE \textbf{end for}
\end{algorithmic}
\end{algorithm}
\paragraph{Intuitive Explanation}
At each iteration, the $\lmo$ returns an element of $\cA$, which is also a descent direction. Then, a convex combination between this descent direction and the previous iterate ensures that the next iterate remains in the convex hull. Importantly, we do not need to know the smoothness parameter $L$ exactly; an upper bound is always sufficient to ensure convergence.
The convergence of Algorithm~\ref{algo:FW} can be intuitively motivated by the fact that the $\lmo$ is minimizing the supporting hyperplane to the graph of $f$ computed in $\theta_k$ on the constraint set $\conv(\cA)$. By convexity of $f$, the linearization lies beneath the graph of $f$ inducing the notion of \textit{poor man duality} \cite{jaggi2011convex}. Let the dual variable $w(\theta)$ be the minimum value obtained by the linear approximation computed at $\theta$ on $\conv(\cA)$. Intuitively, the distance between $w(\theta)$ and $f(\theta)$ is zero at the optimum and by weak duality is an upper bound of the primal error $f(\theta)-f(\theta^\star)$ where $\theta^\star$ is the minimizer of Equation~\eqref{eq:FWproblem}. Note that the minimum of the linear approximation is obtained at the solution of the $\lmo$. Therefore, the $\lmo$ is selecting at each iteration the point that minimizes the duality gap, thus yielding convergence.

\paragraph{Variants for Non-Smooth Optimization}
Existing variants of non-smooth Frank-Wolfe \cite{Lan2014,Lan2016} are based on Nesterov smoothing for Lipschitz continuous objectives.

Nesterov Smoothing \cite{Nesterov2005} approximates a Lipschitz continuous function $g$ as:
\begin{equation*}
g_\beta (z) = \max_{y \in \R^d} \ip{z}{y} - g^*(y) - \frac{\beta}{2} \norm{y}^2,
\end{equation*}
where $\beta>0$ controls the tightness of smoothing and $g^*$ denotes the Fenchel conjugate of $g$.
It is easy to see that $g_\beta$ is convex and $\frac{1}{\beta}$ smooth.
Optimizing $g_\beta (z)$ guarantees progress on $g(z)$ when $g(z)$ is $L_g$-Lipschitz continuous as
$
g_\beta(z)\leq g(z)\leq g_\beta(z) + \frac{\beta}{2} L_g^2.
$
The challenge of smoothing an affine constraint consists in the fact that the indicator function is not Lipschitz. Therefore, $g^*$ does not have bounded support, so adding a strongly convex term to it does not guarantee that $g$ and its smoothed version are uniformly close. 

To smooth constraints which are not Nesterov smoothable in the Frank-Wolfe setting, we proposed in \cite{yurtsever2018conditional} a Homotopy transformation on $\beta$ which can be intuitively understood as follows. If $\beta$ decreases during the optimization, optimizing $g_\beta(z)$ will progressively become similar to optimizing $g(z)$. When $g$ is the indicator function of an affine constraint, the iterate will converge to the feasibility set as $\beta$ goes to zero. This technique comes with a reduction in the rate from $\mathcal{O}(1/k)$ for the smooth setting to $\mathcal{O}(1/\sqrt{k})$.
In a follow-up work \cite{yurtsever2019cgal}, extended this method from quadratic penalty to an augmented Lagrangian formulation for empirical benefits. Gidel et al., \cite{Gidel2018} also proposed an augmented Lagrangian FW with a different analysis. We refer to the references in \cite{yurtsever2019cgal, yurtsever2018conditional} for other variants in this direction. 

\paragraph{Stochastic Variants and Further Non-Smooth Extensions}
So far, we have focused on deterministic variants of Frank-Wolfe. The literature on stochastic variants can be traced back to Hazan and Kale's projection-free methods for online learning \cite{Hazan2012}. 
When $g$ is a non-smooth but Lipschitz continuous function, their method returns an $\varepsilon$-solution in $\mathcal{O}(1/\varepsilon^4)$ iterations. 

The standard extension of FW to the stochastic setting gets $\mathcal{O}(1/\varepsilon)$ iteration complexity for smooth minimization, but with an increasing minibatch size. 
Overall, this method requires $\mathcal{O}(1/\varepsilon^3)$ sample complexity, see \cite{Hazan2016} for the details. More recently, Mokhtari et al., \cite{mokhtari2020stochastic} proposed a new variant with $\mathcal{O}(1/\varepsilon^3)$ iteration complexity, but the proposed method can work with a single sample at each iteration. Hazan and Luo \cite{Hazan2016} and Yurtsever et al., \cite{yurtsever2019spiderfw} incorporated various variance reduction techniques for further improvements. Goldfarb et al., \cite{pmlr-v54-goldfarb17a} introduced two stochastic FW variants, with away-steps and pairwise-steps (see next section). These methods enjoy linear convergence rate (however, the batchsize increases exponentially) but for strongly convex objectives and only in polytope domains. None of these stochastic FW variants work for non-smooth (or composite) problems. 

The non-smooth conditional gradient sliding by Lan and Zhou \cite{Lan2016} also has extensions to the stochastic setting, \eg the lazy variant of Lan et al., \cite{lan2017conditional}.
Note, however, that, as in their deterministic variants, these methods are based on the Nesterov's smoothing and are not suitable for problems with affine constraints.   

Garber and Kaplan \cite{garber2018fast} consider composite problems and propose a variance reduced algorithm that solves a smooth relaxation of the template (see Definition~1 Section~4.1). 

Lu and Freund \cite{lu2020generalized} also studied a composite template but their method incorporates the non-smooth term into the linear minimization oracle. 
This is restrictive in practice because the indicator function can significantly increase the oracle's cost, \eg in the case of semidefinite programs. 

To the best of our knowledge, there is no stochastic projection-free method for convex composite problems.

\paragraph{Non-Convex Setting}
In recent years, FW has also been extended for non-convex problems. These extensions are beyond the scope of this dissertation. We refer to Yu et al., \cite{yu2017generalized} and Julien-Lacoste \cite{julienlacoste2016nonconvexfw} for the non-convex extensions in the deterministic setting, and to Reddi et al., \cite{reddi2016stochastic}, Yurtsever et al., \cite{yurtsever2019spiderfw}, and Shen et al. \cite{Shen2019} in the stochastic setting.

\paragraph{Limitations of the Classical Algorithm}
Unfortunately, when the optimum lies on a face of $\conv(\cA)$ convergence is known to be slow. Indeed, the next iterate can only be formed as a convex mixture between a vertex of the set $\conv(\cA)$ and the current iterate. Therefore, to reach a face, the \lmo will alternate between its vertexes. This problem is known as the \textit{zig-zagging phenomenon}, see \cite{LacosteJulien:2015wj} for a more detailed overview. 

\subsection{Corrective Variants}
To alleviate the zig-zagging problem, \textit{corrective variants} were introduced to allow for a richer set of possible updates. In particular, Wolfe proposed to include the possibility to move away from an active atom in $\cS$ whenever it would yield larger descent than the regular FW step \cite{Wolfe:1970wy}. This method is called \textit{away-step Frank Wolfe} (AFW), and is presented in Algorithm~\ref{algo:AFW}. Alternatively, one can selectively swap weight between one active atom and the descent direction returned by the \lmo \cite{LacosteJulien:2015wj}. By doing so, the update selectively shrinks the weight of a single element of $\cS$ rather than uniformly shrinking the weight of every active atom. This algorithm is known as \textit{Pairwise Frank-Wolfe} (PFW) and is presented in Algorithm~\ref{algo:PFW}. 

At the extreme, one can selectively refine the weight of every  atom in $\cS$ as long as they remain normalized. This approach is known as \textit{Fully Corrective Frank-Wolfe} depicted in Algorithm~\ref{algo:FCFW} Variant 1 (see, \eg, \cite{Holloway:1974ii,jaggi2013revisiting}). We instead proposed in \citep{locatello2017unified} to minimize the simpler quadratic upper bound~\eqref{eq:QuadraticUpperBound} over the atom selected at the current iteration (using line-search) or over $\conv(\cS)$.  The name ``norm-corrective'' illustrates that the algorithm employs a simple squared norm surrogate function (or upper bound on $f$, similarly to Variant 2 and 3 in Algorithm \ref{algo:FW}), which only depends on the smoothness constant $L$. 
Finding the closest point in norm on the simplex may be more efficient than solving the more general optimization problem as in the ``fully-corrective'' variant (Variant 0 of Algorithm \ref{algo:FCFW}).
Approximately solving the subproblem in Variant 1 can be done efficiently using projected gradient steps on the weights (as the projection onto the simplex and L1 ball is efficient). Assuming a fixed quadratic subproblem as in Variant 1, the CoGEnT algorithm of \cite{Rao:2015df} uses the same ``enhancement'' steps. The difference in the presentation here is that we address general~$f$ so the quadratic correction subproblem changes in every iteration in our case.

\begin{algorithm}[ht]
	\caption{Away-steps Frank-Wolfe algorithm \citep{Wolfe:1970wy}}
	\label{algo:AFW}
	\begin{algorithmic}[1]
	\STATE Let $\theta_0 \in \cA$, and $\cS := \{\theta_0\}$
	\FOR{$k=0\dots K$}
		\STATE Let $z_k := \lmo \!\left(\nabla f(\theta_k)\right)$ %
		   and $d_k^\FW := s_k - \theta_k$ \qquad~~ \emph{\small(the FW direction)}
		\STATE Let $v_k \in \displaystyle\argmax_{v \in \cS} \textstyle\left\langle \nabla f(\theta_k), v \right\rangle$ and $d_k^A := \theta_k - v_k$ \qquad \emph{\small(the away direction)}
		  \IF{$\left\langle -\nabla f(\theta_k), d_k^\FW\right\rangle  \geq \left\langle -\nabla f(\theta_k), d_k^A\right\rangle$ }
		  \STATE $d_k :=  d_k^\FW$, and $\gamma_{\max} := 1$  
			     \hspace{22mm}\emph{\small(choose the FW direction)}
		  \ELSE
		  \STATE $d_k :=  d_k^A$, and $\gamma_{max} := \alpha_{v_k} / (1- \alpha_{v_k})$
		  	\hspace{5mm}\emph{\small(choose away direction; maximum feasible step-size)}
		  \ENDIF	
		  \STATE Line-search: $\gamma_k \in \displaystyle\argmin_{\gamma \in [0,\gamma_{max}]} \textstyle f\left(\theta_k + \gamma d_k\right)$ 
		  \STATE Update $\theta_{k+1} := \theta_{k} + \gamma_k d_k$  
		  	\hspace{2cm}\emph{\small(and accordingly for the weights $\alpha^{(t+1)}$, see text)}
		  \STATE Update $\cS := \{v \in \cA \,\: \mathrm{ s.t. } \,\: \alpha^{(t+1)}_{v} > 0\}$
		    \STATE \qquad \emph{Optional:} Correction of some/all atoms $z_{0\ldots t}$
	\ENDFOR
	\end{algorithmic}
\end{algorithm}
\begin{algorithm}
	\caption{Pairwise Frank-Wolfe algorithm \citep{LacosteJulien:2015wj}}
	\label{algo:PFW}
	\begin{algorithmic}[1]
	\STATE $\ldots$ as in Algorithm~\ref{algo:AFW}, except replacing lines~6 to~10 with:  $\,\, d_k = d^{PFW}_k \!\!\!:= z_k - v_k$, and $\gamma_{\max} := \alpha_{v_k}$.
	\end{algorithmic}
\end{algorithm}

\begin{algorithm}[ht]
\caption{Fully-Corrective Frank-Wolfe \citep{locatello2017unified,LacosteJulien:2015wj}}
\label{algo:FCFW}
\begin{algorithmic}[1]
  \STATE \textbf{init} $\theta_{0} = 0 \in \cA, \cS = \{\theta_0 \}$
  \STATE \textbf{for} {$k=0\dots K$}
  \STATE \quad Find $z_k := (\text{Approx-}) \lmo_{\cA}(\nabla f(\theta_{k}))$
  \STATE \quad $\cS := \cS \cup \{ z_k\}$
\STATE \quad \textit{Variant 0 -- norm-corrective:} \\
 \qquad $\theta_{k+1} \! =\! \underset{{\theta\in\conv(\cS)}}{\argmin} \!\|\theta - (\theta_k -\frac{1}{L}\nabla f(\theta_k))\|^2\!$
\STATE \quad \textit{Variant 1 -- fully-corrective:} \\ \qquad $\theta_{k+1} = \argmin_{\theta\in\conv(\cS)} f(\theta)$
\STATE \quad Remove atoms with zero weights from $\cS$
  \STATE \qquad \emph{Optional:} Correction of some/all atoms $z_{0\ldots t}$
  \STATE \textbf{end for}
\end{algorithmic}
\end{algorithm}

\section{Matching Pursuit and Coordinate Descent}
In Algorithm \ref{algo:MP}, we present the generalized Matching Pursuit of \citep{locatello2017unified} to solve convex optimization problems on the linear span of a symmetric atom set $\cA$. This symmetry assumption can be relaxed to non-symmetric sets as long as the origin is in the relative interior of $\conv(\cA)$.
   \begin{algorithm}
\caption{Generalized Matching Pursuit \citep{locatello2017unified}}
  \label{algo:MP}
\begin{algorithmic}[1]
  \STATE \textbf{init} $\theta_{0} \in \lin(\cA)$
  \STATE \textbf{for} {$k=0\dots K$}
  \STATE \quad Find $k_k := (\text{Approx-}) \lmo_\cA(\nabla f(\theta_{k}))$
  \STATE \quad $\theta_{k+1} := \theta_k - \frac{\ip{\nabla f(\theta_k)}{z_k}}{L\|z_k\|^2}z_k$
  \STATE \textbf{end for}
\end{algorithmic}
\end{algorithm}
\vspace{-0.2mm}
The algorithm’s structure follows that of Algorithm \ref{algo:generalgreedy}, with the \lmo finding the steepest descent direction among the set~$\cA$. This subroutine is shared with the Frank-Wolfe method~\cite{frank1956algorithm,jaggi2013revisiting} and steepest coordinate descent \citep{nesterov2012efficiency}.
The update step of MP is computed by minimizing the quadratic upper bound given by the objective's smoothness in Equation \eqref{eq:QuadraticUpperBound}.
For $f(\theta) = \frac{1}{2}\|y -\theta\|^2$, $y \in \cH$, Algorithm~\ref{algo:MP} recovers the classical MP algorithm \cite{Mallat:1993gu}.
Besides the classical work on signal recovery, convergence rates of MP on smooth objectives were studied in \cite{ShalevShwartz:2010wq, Temlyakov:2013wf, Temlyakov:2014eb, Temlyakov:2012vg, nguyen2014greedy, locatello2017unified} (see \citep{locatello2017unified} for a more detailed review and further references). 

\paragraph{Steepest Coordinate Descent.}
\looseness=-1The MP algorithm can be seen as a generalized version of steepest coordinate descent \cite{nesterov2012efficiency}. When $\cA$ is the L1-ball the $\lmo$ problem becomes 
$i_k = \argmax_{i} | \nabla_i f(\theta)| \,,$ where $\nabla_i$ is the $i$-th component of the gradient, \ie $\ip{\nabla f(\theta)}{e_i}$ with $e_i$ being one of the natural vectors. The update step is:
\begin{align*}
\theta_{k+1} := \theta_{k+1} - \frac{1}{L}\nabla_{i_k} f(\theta_k)e_{i_k} \,.
\end{align*}
Note that assuming a symmetric atom set, the $\lmo$ problem is equivalent to finding the steepest descent direction in the set $\cA$, \ie the most aligned with the negative gradient. Therefore, the positive stepsize $- \frac{\langle \nabla f(\theta_k), z_k\rangle}{L} = - \frac{1}{L}\nabla_{i_k} f(\theta_k)$ decreases the objective. In this dissertation, we explicitly connect the two algorithms providing a unified analysis. The advantage is that our rates for CD are tighter than previous rates with global smoothness constant, and we can prove new accelerated rates for MP, generic random pursuit algorithms, and steepest CD.

\paragraph{Challenges for Accelerated MP Rates}
The connection in our analysis between CD and MP allows us to extend the accelerated analysis of \cite{Lee13,Nesterov:2017} to the latter. The main challenge is that the only accelerated rates known are for random coordinate descent. \citet{song2017accelerated} proposed an accelerated greedy coordinate descent method by using the linear coupling framework of \cite{allen2017linear}. However, the updates they perform at each iteration are not guaranteed to be sparse, which is critical as we wish to represent the iterate as a sparse combination of atoms.

\section{Non-Negative Matching Pursuit}\label{sec:back_NNMP}
The optimization convergence analysis of Algorithm~\ref{algo:MP} relies on the origin being in the relative interior of $\conv(\cA)$, which is trivially implied by the symmetry of $\cA$ \cite{locatello2017unified}. This assumption is critical as it ensures that for any possible gradient, the $\lmo$ can find a descent direction within $\cA$, see the \textit{third premise} formulated by \citet{pena2015polytope}. 
In other words, unless we are at the optimum:
$\min_{z\in\cA}\langle \nabla f(\theta_k), z \rangle < 0$
which implies that the resulting stepsize in Algorithm~\ref{algo:MP} is positive. While this sounds promising for minimization problems over convex cones, there can be non-stationary points $\theta$ in the conic hull of a set $\cA$ for which $\min_{z\in\cA}\langle \nabla f(\theta),z\rangle = 0$. 

Within the context of non-negative pursuit algorithms, existing heuristics are 
limited to the least-squares objective \cite{bruckstein2008uniqueness,Yaghoobi:2015ff} (\citep{Yaghoobi:2015ff} also has coherence-based recovery guarantees for finite atom sets but no optimization rates).
Apart from MP-type algorithms, there is a large variety of non-negative least-squares algorithms, \eg, \cite{lawson1995solving}, particularly for matrix and tensor spaces. 
The gold standard in factorization problems is projected gradient descent with alternating minimization, see \cite{Sha:2002um,berry2007algorithms,shashua2005non,kim2014algorithms}. \cite{pena2016solving} is concerned with the feasibility problem on symmetric cones and \cite{harchaoui2015conditional} introduces a norm-regularized variant of the minimization problem over convex cones that solves using FW on a rescaled version of the convex hull of $\cA$. To the best of our knowledge, general rates for non-negative pursuit algorithms on general convex objectives are not known.

\section{Boosting}
Earlier, a flavor of generalized MP in Hilbert spaces became popular in the context of boosting, see \cite{meir2003introduction,ratsch2001convergence,Buhlmann:2010fj} for a general overview.
Following \cite{ratsch2001convergence}, one can view boosting as an iterative greedy algorithm minimizing a (strongly) convex objective over the linear span of a possibly infinite hypothesis class. The convergence analysis crucially relies on the assumption of the origin being in the relative interior of the hypothesis class, 
see Theorem 1 in \cite{grubb2011generalized}. Indeed, Algorithm 5.2 of \cite{meir2003introduction} might not converge if the \cite{pena2015polytope} alignment assumption is violated.
In this dissertation, we relax this assumption while preserving essentially the same asymptotic rates of \cite{meir2003introduction,grubb2011generalized}. 


    
\chapter[Convex Hulls with Affine Constraints]{Convex Hulls with Affine Constraints: Stochastic Frank-Wolfe}
\label{ch:SFW}
In this chapter, we consider the problem of minimizing a composite convex function over a convex set. This problem template cover Semidefinite Programs (SDPs) as a special case (minimization over positive-semidefinite cone subject to some affine constraints). The presented approach is based on \citep{locatello2020stochastic} and was developed in collaboration with Alp Yurtsever, Olivier Fercoq, and Volkan Cevher. The experiments were done in collaboration with Alp Yurtsever. Code available at \url{https://github.com/alpyurtsever/SHCGM}.

\section{Problem Formulation}
This chapter focuses on the following stochastic convex optimization template with composite objective, which covers both finite sum and online learning problems:
\begin{equation}
\label{eqn:SFW_main-template}
\underset{\theta\in\mathcal{X}}{\text{minimize}} \quad \bbE_\rvw  f(\theta,\omega) +g(A\theta) := F(\theta).
\end{equation}
In this optimization template, we consider the following setting: \\
$\triangleright~~\cA \subset \R^n$ is a set of atoms in $\R^n$ and is compact, \\
$\triangleright~~\mathcal{X} = \conv(\cA) \subset \R^n$ is the convex hull of $\cA$ and is a convex and compact set, \\
$\triangleright~~\omega$ is a realization of the random variable $\rvw$ drawn from a distribution $\mathcal{P}$, \\
$\triangleright~~\bbE_\rvw  f(\, \cdot \, ,\omega):\mathcal{X} \to \R$ is a smooth  convex function, \\
$\triangleright~~A \in \R^n \to R^d$ is a given linear map, \\
$\triangleright~~g:\R^d \to \R\cup\{+\infty\}$ is a convex function (possibly non-smooth). 

We consider two distinct specific cases for $g$: \\
(i) $g$ is a Lipschitz-continuous function, for which the proximal-operator is easy to compute:
\begin{equation}
\prox{g}(y) = \arg \min_{z \in \R^d} ~ g(z) + \frac{1}{2} \norm{z - y}^2
\end{equation}
(ii) $g$ is the indicator function of a convex set $\mathcal{K} \subset \R^d$:
\begin{equation}
g(z) = \begin{cases} 
0 \quad \text{if}~~z \in \mathcal{K}, \\
+\infty \quad \text{otherwise.} 
\end{cases}
\end{equation}
The former case is useful for regularized optimization problems, common in machine learning applications to promote a desired structure to the solution. The latter handles affine constraints of the form $A\theta \in \mathcal{K}$. Note that combinations of both are also possible within our framework.

\section{Practical Motivation}
An important application for our framework is stochastic semidefinite programming (SDP):
\begin{equation}\label{eqn:SFW_stochastic-sdp}
\underset{{\theta \in \mathbb{S}_+^n, ~ \trace(\theta) \leq \beta}}{\text{minimize}} \quad \bbE_\rvw  f(\theta,\omega) \quad \text{subject to} \quad A\theta \in \mathcal{K}.
\end{equation}
where $\mathbb{S}_+^n $ denotes the positive-semidefinite cone. We are interested in this problem formulation as it does not require access to the whole data at one time, which is useful for both stochastic optimization over large data sets and online optimization with streaming data. This problem template enables new applications of SDPs in machine learning, such as online variants of clustering \citep{Peng2007}, streaming PCA \cite{dAspremont2004}, kernel learning \cite{Lanckriet2004}, community detection \cite{Abbe2018}, optimal power-flow \cite{Madani2015}, etc. 

\paragraph{Example} Consider the SDP formulation of the k-means clustering problem \cite{Peng2007}:
\begin{align}\label{eqn:clustering}
\underset{{\theta \in \mathbb{S}_+^n, ~ \trace(\theta) = k}}{\text{minimize}} \quad \ip{D}{\theta} \quad \text{subject to} \quad \theta1_n = 1_n, \quad \theta \geq 0.
\end{align}
Here, $1_n$ denotes the vector of ones, $\theta \geq 0$ enforces entrywise non-negativity, and $D$ is the Euclidean distance matrix between each pair of examples. To solve this problem with a standard SDP solver, we need access to the whole data matrix $D$ at the same time. Note that $D$ is quadratic on the size of the data set. Instead, our formulation allows observing only a subset of entries of $D$ at each iteration, for example, by computing the pairwise Euclidean distance over a batch of observations.

\paragraph{Splitting}
An important use-case of affine constraints in \eqref{eqn:SFW_main-template} is splitting (see Section~5.6 in \cite{yurtsever2018conditional}). Suppose that $\mathcal{X}$ can be written as the intersection of two (or more) simpler (in terms of the computational cost of $\lmo$ or projection) sets $\mathcal{U} \cap \mathcal{V}$. By using the standard product space technique, we can reformulate this problem in the extended space $(u,v) \in \mathcal{U} \times \mathcal{V}$ with the constraint $u=v$: 
\begin{equation}
\underset{(u,v) \in \mathcal{U} \times \mathcal{V}}{\text{minimize}} \quad \bbE_\rvw  f(u,\omega) \quad \text{subject to} \quad u = v.
\end{equation}
This allows us to decompose the difficult optimization domain $\mathcal{U}\cap\mathcal{V}$ into simpler pieces for which the \lmo in equation \eqref{eqn:lmo} can be called separately. Alternatively, if projecting onto \eg $\mathcal{V}$ is cheap we can reformulate the problem minimizing over $\mathcal{U}$ with the affine constraint $v \in \mathcal{V}$:
\begin{equation}
\underset{u \in \mathcal{U}}{\text{minimize}} \quad \bbE_\rvw  f(u,\omega) \quad \text{subject to} \quad u\in \mathcal{V}.
\end{equation}

A notable application is the SDP relaxation of k-means clustering. The \lmo over the intersection of the positive-semidefinite cone and the first orthant can only be computed in $\mathcal{O}(n^3)$ with the Hungarian method. Instead, we can split the constraint and run the \lmo over the semidefinite cone and perform projections onto the first orthant cheaply.

\paragraph{Lack of Scalable Alternative Approaches}
\looseness=-1The template problem in Equation \eqref{eqn:SFW_main-template} can be solved with operator splitting methods (see \cite{Cevher2018sfdr} and the references therein). These approaches require to project the iterate onto $\mathcal{X}$. This projection requires a full eigendecomposition in SDP applications, which has a cubic cost with respect to the problem dimension. Instead, our approach has computational cost of $\Th{N_\nabla/\delta}$ where $N_{\nabla}$ is the number of non-zeros of the gradient and $\delta$ is the accuracy of the approximate $\lmo$. When computing the full gradients is feasible, the deterministic approach of \cite{yurtsever2018conditional} offers $\Or{1/\sqrt{t}}$ convergence rates. Due to its practical relevance, several relaxations to the template problem in Equation \eqref{eqn:SFW_stochastic-sdp} have been proposed in the literature \cite{garber2018fast,Hazan2016,lu2020generalized,hazan2008sparse}. None of them is guaranteed to find the optimal solution, see Table \ref{table:SFW_complexity} for an overview.

\begin{table}[ht!]
\begin{center}
\small{
\begin{tabular}{ccccc}
\toprule
Algorithm & Iter. compl. & Sample compl.  & Solves~\eqref{eqn:SFW_stochastic-sdp} & Iter. cost for~\eqref{eqn:SFW_stochastic-sdp}\\
\midrule
\cite{yurtsever2018conditional} & $\mathcal{O}({1}/{\varepsilon^{2}})$ &  $N$ & Yes & $\Theta(N_{\nabla}/\delta)$ \\
\cite{garber2018fast} & $\mathcal{O}({1}/{\varepsilon^2})$ &  $\mathcal{O}({1}/{\varepsilon^4})$ & No & $\Theta(N_{\nabla}/\delta)$ \\
\cite{Hazan2016}  & $\mathcal{O}(1/\varepsilon)$ & $\mathcal{O}(1/\varepsilon^3)$ & No & $\Theta(N_{\nabla}/\delta)$ \\
\cite{lu2020generalized} & $\mathcal{O}(1/\varepsilon)$ & $\mathcal{O}(1/\varepsilon^2)$ & No & SDP \\
\cite{hazan2008sparse} &  $\mathcal{O}(1/\varepsilon)$ & $N$ & No & $\Theta(N_{\nabla}/\delta)$\\
\phantom{$^*$}\cite{Cevher2018sfdr}$^*$ & $-$ &  $-$ & Yes & $\Theta(n^3)$ \\
\textbf{Our} & $\mathcal{O}({1}/{\varepsilon^{3}})$ &  $\mathcal{O}({1}/{\varepsilon^{3}})$ & Yes & $\Theta(N_{\nabla}/\delta)$ \\
\bottomrule\\
\end{tabular}}
 \caption{Existing algorithms for stochastic semidefinite programs as in Equation~\eqref{eqn:SFW_stochastic-sdp}. $N$ is the size of the data set. $n$ is the dimension of each data point. $N_{\nabla}$ is the number of non-zeros of the gradient. $\delta$ is the accuracy of the approximate $\lmo$. The per-iteration cost of \cite{lu2020generalized} is the cost of solving a SDP in the canonical form. \\ 
{$^*$\cite{Cevher2018sfdr} has $\mathcal{O}({1}/{\varepsilon^2})$ iteration and sample complexity when the objective function is strongly convex. This is not the case in our model problem, and \cite{Cevher2018sfdr} only has an asymptotic convergence guarantee. }
}\label{table:SFW_complexity}
\end{center}
\end{table}

\newpage
Without affine constraints, we could solve the stochastic SDP formulation of \eqref{eqn:SFW_stochastic-sdp} with stochastic Frank-Wolfe variants \eg \cite{mokhtari2020stochastic}. For positive-semidefinite cones with trace norm constraints, the $\lmo$ becomes:
\begin{equation}\label{eqn:SFW_SDP_lmo}
S = \arg\min_{Y} 
~ \left \{ \ip{\nabla f (\theta,\omega)}{Y} :  \quad Y \in \mathbb{S}_+^n,~ \trace(Y) \leq \beta \right \}
\end{equation}
which corresponds to finding the eigenvector with smallest eigenvalue of the matrix $\nabla f (\theta,\omega)$. This can be done efficiently using shifted power methods or the randomized subspace iterations \cite{HMT11:Finding-Structure}.
When we add affine constraints in our problem template, the $\lmo$ becomes an SDP instance in the canonical form, which renders the application of vanilla stochastic Frank-Wolfe algorithms computationally challenging.
\newpage
\section{Stochastic Homotopy CGM (SHCGM)}
\label{sec:algorithm}
\hfill
\begin{algorithm}[H]
   \caption{Stochastic Homotopy CGM (SHCGM) \cite{locatello2020stochastic}}
   \label{alg:stochastic_HFW}
\begin{algorithmic}
   \STATE {\bfseries Input:} $\theta_1 \in \mathcal{X}, ~ \beta_0 > 0$, $d_0 = 0$
   \FOR{$k=1,2, \ldots, $}
   \STATE $\eta_k = 9/(k+8)$
   \STATE $\beta_k =\beta_0 / (k+8)^{\frac{1}{2}}$
   \STATE $\rho_k = 4/(k+7)^{\frac{2}{3}}$
   \STATE $d_k = (1-\rho_k)d_{k-1} + \rho_k\nabla_\theta f(\theta_k,\omega_k)$
   \STATE $v_k = d_k + \beta_k^{-1} A^\top \big(A \theta_k -  \prox{\beta_k g}(A\theta_k)\big)$
	\STATE $s_k = \arg\min_{\theta \in \mathcal{X}}\ip{v_k}{\theta}$
   	\STATE $\theta_{k+1} = \theta_k + \eta_k(s_k - \theta_k)$
   \ENDFOR
\end{algorithmic}
\end{algorithm}
\vspace{-1em}

\looseness=-1Our approach is depicted in Algorithm \ref{alg:stochastic_HFW}. Let $g_\beta$ be a smooth approximation of $g$, parametrized by the penalty (or smoothing) parameter $\beta > 0$: 
\begin{equation*}
g_\beta (z) = \max_{y \in \R^d} \ip{z}{y} - g^*(y) - \frac{\beta}{2} \norm{y}^2, \quad \text{where} \quad g^\ast (x) = \max_{v \in \R^d} \ip{x}{v} - g(v).
\end{equation*}
It is easy to show that $g_\beta$ is $1/\beta$-smooth and its gradient can be computed as:
\begin{equation*}
\nabla_\theta g_{\beta}(A\theta) = A^\top \prox{\beta^{-1}g^*}(\beta^{-1} A\theta) = \beta^{-1} A^\top \left( A\theta - \prox{\beta g}(A\theta) \right),
\end{equation*}
where the second equality follows from the Moreau decomposition. In our prior work \cite{yurtsever2018conditional}, we combined Nesterov smoothing \cite{Nesterov2005} with the analysis of Frank-Wolfe, replacing the non-smooth term $g$ by the smooth approximation $g_\beta$. As we anneal $\beta \to 0$ at an appropriate rate, $g_\beta \to g$ and thus the smoothed problem (where we consider $g_\beta$ instead of $g$) converges to the original template of Equation \eqref{eqn:SFW_main-template}. As a result, while the decision variable converges to a minimizer of the non-smooth problem, we can use the smoothness of $g_\beta$ for the analysis. Unfortunately, the analysis of \cite{yurtsever2018conditional} is only applicable to deterministic gradients. Instead, we use the biased gradient estimator of \citet{mokhtari2020stochastic}:
\begin{equation}\label{eqn:grad-est}
d_k = (1-\rho_k)d_{k-1} + \rho_k\nabla_\theta f(\theta_k,\omega_k).
\end{equation}
The main advantage of this gradient estimator is that it does not require increasing the batch size to reduce the variance. However, it only yields convergence at a rate of   $\mathcal{O}({1}/{k^{\frac{1}{3}}})$ for smooth convex minimization.

In our Stochastic Homotopy Conditional Gradient Method (SHCGM), we rely on a stochastic gradient estimator $v_k$ for the smooth approximation of the composite objective,
\begin{equation}
\begin{aligned}
\nabla_\theta F_{\beta_k}(\theta) = \nabla_\theta \bbE_\rvw  f(\theta,\omega) +\nabla_\theta g_{\beta_k}(A\theta)  \quad \implies \quad
v_k & = d_k + \nabla_\theta g_{\beta_k}(A\theta).
\end{aligned}
\end{equation}
where $d_k$ is computed as in Equation \eqref{eqn:grad-est}. Crucially, we have three coupled learning rates $\eta_k,~\beta_k$, and $\rho_k$ that govern the stepsize for the iterate update, the smoothing parameter, and the gradient averaging parameter, respectively. $\rho_k$ controls the variance of the gradient estimator and $\beta_k$ the smoothness of the surrogate minimization of $F_{\beta_k}(\theta)$.

\section{Analysis}
We denote $f^\star := \bbE_\rvw f(\theta^\star, \omega)$ where $\theta^\star$ is the solution of~\eqref{eqn:SFW_main-template}. Throughout the chapter, $y^\star$ represents the solution of the dual problem of~\eqref{eqn:SFW_main-template}. For problems with affine constraints, we further assume that the strong duality holds. 
Slater's condition is a common sufficient condition for strong duality. 
By Slater's condition, we mean
\begin{equation*}
\mathrm{relint} (\mathcal{X} \times \mathcal{K})  ~\cap~ \left\{ (\theta,r) \in \R^n \times \R^d : A\theta = r \right\} \neq \emptyset.
\end{equation*}
Recall that the strong duality ensures the existence of a finite dual solution. 
We denote a solution to \eqref{eqn:SFW_main-template} and the optimal value by $\theta^\star$ and $F^\star$ respectively:
\begin{equation*}
F^\star = F(\theta^\star) \leq F(\theta), \qquad \forall \theta \in \mathcal{X}.
\end{equation*}

The main assumption we make in order to prove convergence is that the stochastic gradient has \textit{bounded variance}:
\begin{equation*}
\bbE\left[\norm{ \nabla_\theta f(\theta,\omega) - \nabla_\theta \bbE_\rvw f(\theta, \omega)}^2 \right]\leq \sigma^2 < +\infty.
\end{equation*}
\begin{theorem}[Lipschitz-continuous regularizer]\label{cor:g_lip}
Assume that $g: \R^d \to \R$ is $L_g$-Lipschitz continuous. 
Then, the sequence $\theta_k$ generated by Algorithm~\ref{alg:stochastic_HFW} satisfies the following convergence bound: 
\begin{equation}
\bbE F(\theta_{k+1}) - F^\star  
\leq 9^\frac{1}{3} \frac{C}{(k+8)^\frac{1}{3}} + \frac{\beta_0 L_g^2}{2\sqrt{k+8}},
\end{equation}
where $C := \frac{81}{2}D_{\mathcal{X}}^2(L_f + \beta_0\norm{A}^2) + 36\sigma D_{\mathcal{X}} + 27\sqrt{3}L_fD^2_{\mathcal{X}}$, $L_f$ is the smoothness of $f$, and $D_{\mathcal{X}}$ is the diameter of $\mathcal{X}$.
\end{theorem}

\textit{Proof sketch.}
The proof follows the following steps: \\
(i) Relate the stochastic gradient to the full gradient (Lemma~\ref{lemma:linear_exact_additive}). \\
(ii) Show convergence of the gradient estimator to the full gradient (Lemma~\ref{lemma:stoch_gradient_convergence}). \\
(iii) Show $\mathcal{O}(1/k^{\frac{1}{3}})$ convergence rate on the smooth gap $\bbE F_{\beta_k}(\theta_{k+1}) - F^\star$ (Theorem~\ref{thm:composite}). \\
(iv) Translate this bound to the actual sub-optimality $\bbE F(\theta_{k+1}) - F^\star$ by using the envelope property for Nesterov smoothing, see Equation (2.7) in \cite{Nesterov2005}. 
\hfill$\square$

\paragraph{Discussion} \looseness=-1\cite{Hazan2012,Lan2016,lan2017conditional} also has guarantees for stochastic Frank-Wolfe type algorithms on non-smooth but Lipschitz continuous $g$. Our rate is slower than the $\mathcal{O}(\frac{1}{\varepsilon^2})$ in \cite{Lan2016,lan2017conditional}, but we obtain $\mathcal{O}(\frac{1}{\varepsilon^3})$ sample complexity in the statistical setting as opposed to $\mathcal{O}(\frac{1}{\varepsilon^4})$. 
The main advantage of our approach is that it allows affine constraints that are not Lipschitz continuous, such as the indicator function for which Nesterov smoothing cannot be used.

\begin{theorem}[Affine constraints]\label{cor:indicator-exact}
\looseness=-1Suppose that $g: \R^d \to \R$ is the indicator function of a simple convex set $\mathcal{K}$. 
Assuming that the strong duality holds, the sequence $\theta_k$ generated by Algorithm~\ref{alg:stochastic_HFW} satisfies:
\begin{equation}
\begin{aligned}
\bbE \bbE_\rvw f(\theta_{k+1}, \omega) -  f^\star  & \geq -\norm{y^\star} ~\bbE  \mathrm{dist}(A\theta_{k+1},\mathcal{K}) \\
\bbE \bbE_\rvw f(\theta_{k+1}, \omega) -  f^\star & \leq  9^\frac{1}{3} \frac{C}{(k+8)^\frac{1}{3}} \\
\bbE \mathrm{dist}(A\theta_{k+1}, \mathcal{K}) & \leq \frac{2 \beta_0 \norm{y^\star} }{\sqrt{k+8}} +\frac{2\sqrt{2\cdot9^\frac13 C\beta_0}}{(k+8)^{\frac{5}{12}}}
\end{aligned}
\end{equation}
\end{theorem}

\textit{Proof sketch.}
We re-use the ingredients of the proof of Theorem~\ref{cor:g_lip}, except that at step (iv) we translate the bound on the smooth gap (penalized objective) to the actual convergence measures (objective residual and feasibility gap) using the Lagrange saddle point formulations and strong duality. 
\hfill$\square$

\paragraph{Discussion} Surprisingly, \\
$\triangleright$~$\mathcal{O}(1/k^{\frac{1}{3}})$ rate in objective residual matches the rate in~\cite{mokhtari2020stochastic} for smooth minimization. \\
$\triangleright$~$\mathcal{O}(1/k^{\frac{5}{12}})$ rate in feasibility gap is only an order of $k^{\frac{1}{12}}$ worse than the deterministic variant in~\cite{yurtsever2018conditional}. 

\paragraph{Inexact Oracles} \looseness=-1In Theorems~\ref{cor:g_lip} and \ref{cor:indicator-exact}, We assumed to use the exact solutions of \ref{eqn:lmo} for simplicity of exposition. 
In many applications including SDP problems, however, it is much easier to find an approximate solution to the \eqref{eqn:lmo} problem. Our convergence result can be extended to both additive and multiplicative errors. For the sake of brevity, we will only prove the case with additive error in Section \ref{sec:shcgm_proof} and defer to \citep{locatello2020stochastic} for the other case. 

\section{Empirical Evaluation}
This section presents the empirical performance of the proposed method for the stochastic k-means clustering, and matrix completion problems.
As a baseline for the stochastic k-means, we compare our algorithm against the Homotopy CGM as it is the only projection free method that handles affine constraints even though it is deterministic. In the stochastic matrix completion, we compare against the Stochastic Three-Composite Convex Minimization algorithm (S3CCM)~ \cite{Yurtsever2016s3cm} to show that our approach is favorable at scale since it avoids expensive projections. We also compare against the stochastic Frank-Wolfe (SFW) of \cite{mokhtari2020stochastic} \textit{without the additional constraints}. With this comparison, we empirically validate that adding the indicator constraints on SFW does not impact the convergence rate of the objective as predicted by the theory. 
\subsection{Stochastic k-means Clustering}
\label{sec:kmeans}
We consider the SDP formulation \eqref{eqn:clustering} of the k-means clustering problem. 
The same problem is used in numerical experiments by Mixon et al. \cite{Mixon2017}, and we design our experiment based on their problem setup\footnote{D.G. Mixon, S. Villar, R.Ward. | Available at \url{https://github.com/solevillar/kmeans_sdp}} with a sample of $1000$ data points from the MNIST data\footnote{Y. LeCun and C. Cortes. | Available at \url{http://yann.lecun.com/exdb/mnist/}}. 
See \cite{Mixon2017} for details on the preprocessing.

We solve this problem with SHCGM and compare it against HCGM \cite{yurtsever2018conditional} as the baseline. 
HCGM is a deterministic algorithm; hence it uses the full gradient. 
For SHCGM, we compute a gradient estimator by randomly sampling $100$ data points at each iteration. 
Remark that this corresponds to observing approximately $1$ percent of the entries of $D$. 

We use $\beta_0 = 1$ for HCGM and $\beta_0 = 10$ for SHCGM. 
We set these values by tuning both methods by trying $\beta_0 = 0.01, 0.1, ..., 1000$. 
We display the results in Figure~\ref{fig:clustering}, where we denote a full pass over the entries of $D$ as an epoch. 
Figure~\ref{fig:clustering} demonstrates that SHCGM performs similarly to HCGM although it uses fewer data. 

\begin{figure}[th!]
\centering
\includegraphics[width=\linewidth]{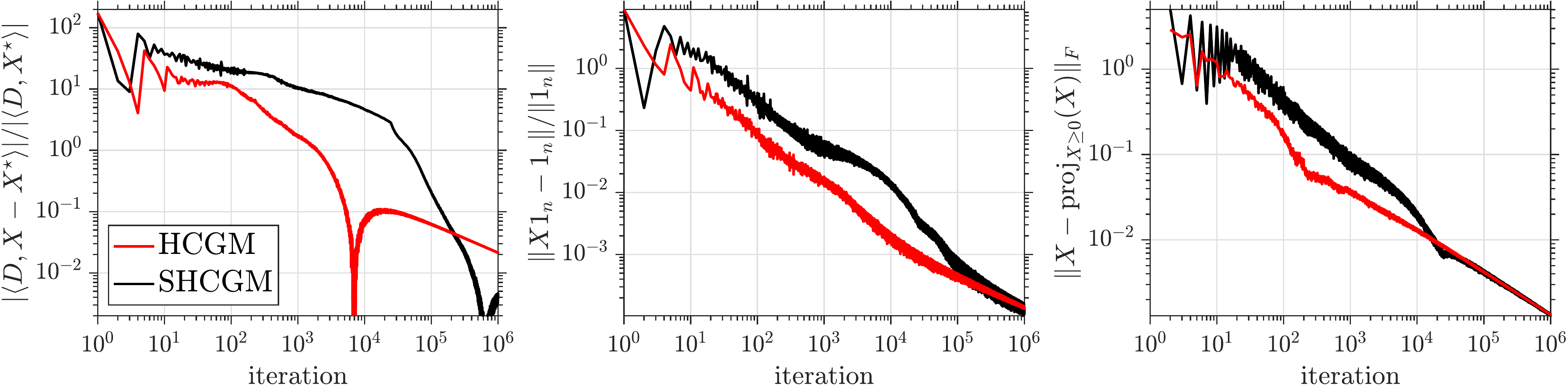}\\[0.5em]
\includegraphics[width=\linewidth]{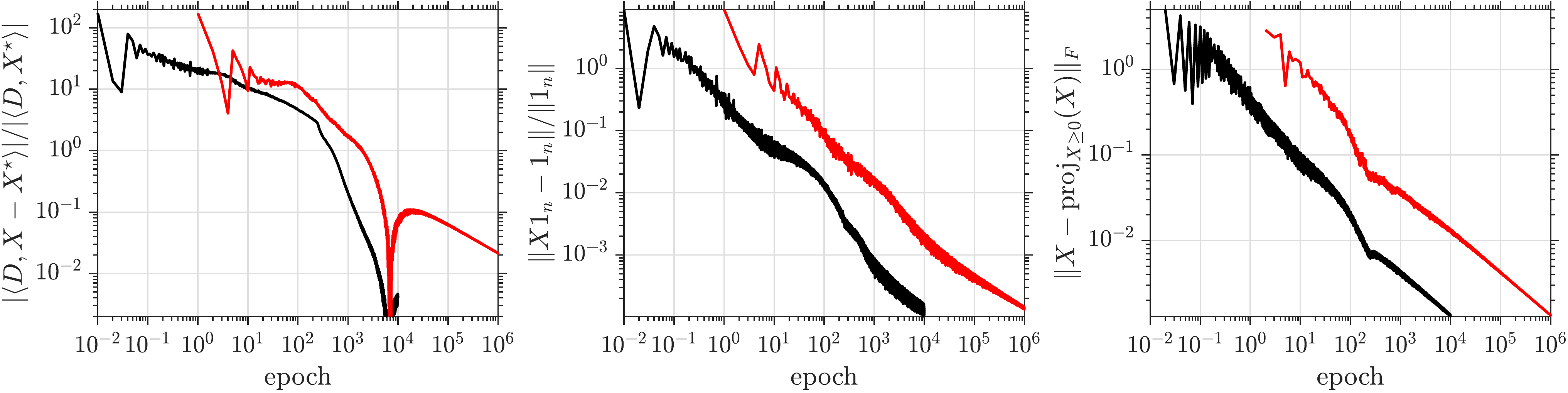}
        \caption{\small Comparison of SHCGM with HCGM for the SDP relaxation of k-means clustering described in Section~\ref{sec:kmeans}. The first column reports an estimate of the suboptimality for the objective and the latter two columns the iterate feasibility (assignment sums to one and non-negativity respectively). The first row reports these results in terms of number of iterations and the second row in epochs. We clearly see that SHCGM is competetive with the deterministic variant and is more data efficient.
        }
        \label{fig:clustering}
\end{figure}

\subsection{Stochastic Matrix Completion}
We consider the problem of matrix completion with the following mathematical formulation:
\begin{equation}\label{eqn:exp_online_mat_comp}
\underset{\|X\|_* \leq \beta_1}{\text{minimize}} \quad  \sum_{(i,j) \in \rvw} (X_{i,j} - Y_{i,j})^2  \quad \text{subject to} \quad 1\leq X\leq 5,
\end{equation}
where, $\rvw$ is the set of observed ratings (samples of entries from the true matrix $Y$ that we try to recover), and $\|X\|_*$ denotes the nuclear-norm (sum of singular values). The affine constraint $1\leq X\leq 5$ imposes a hard threshold on the estimated ratings (in other words, the entries of $X$).

\begin{figure}[ht!]
\begin{minipage}{0.65\textwidth}
\includegraphics[width=0.95\linewidth]{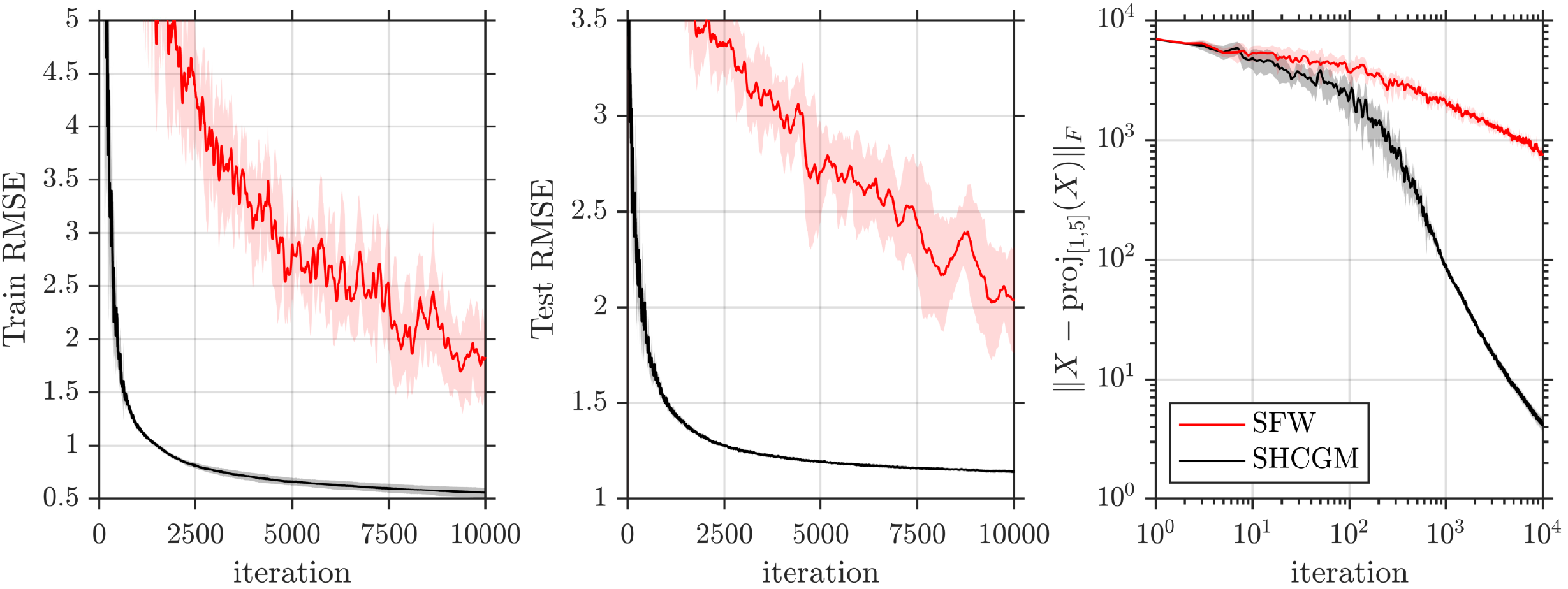}
\end{minipage}
\begin{minipage}{0.33\textwidth}
\center
\begin{tabular}{ l c }
\toprule
   & \small{train RMSE } \\
  \midrule
  \footnotesize
 {SHCGM} & {0.5574 \!$\pm$\! 0.0498} \\
 {SFW} & 1.8360 \!$\pm$\! 0.3266\\
  \bottomrule
\end{tabular}
\vspace{0.5em}
\begin{tabular}{ l c }
\toprule
   & \small{ test RMSE } \\
  \midrule
  \footnotesize
  {SHCGM} & {1.1446 \!$\pm$\! 0.0087} \\
  {SFW} & 2.0416  \!$\pm$\!  0.2739\\
  \bottomrule
\end{tabular}
\end{minipage}
        \caption{\small Training Error, Feasibility gap and Test Error for MovieLens~100k. Table shows the mean values and standard deviation of train and test RMSE over 5 different train/test splits at the end of $10^4$ iterations. We observe that adding the additional constraint does not negatively impact the convergence. On the contrary, our SHCGM train faster and generalize better compared to the vanilla stochastic FW.}
        \label{fig:movielens100k}
\end{figure}

We first compare SHCGM with the Stochastic Frank-Wolfe (SFW) from \cite{mokhtari2020stochastic}. 
We consider a test setup with the MovieLens100k data set\footnote{\label{ftnote:foot3}F.M. Harper, J.A. Konstan. | Available at \url{https://grouplens.org/data sets/movielens/}} \cite{harper2016movielens}. 
This data set contains $\sim$100'000 integer valued ratings between $1$ and $5$, assigned by $1682$ users to $943$ movies. 
This experiment aims to emphasize the flexibility of SHCGM. 
Recall that SFW does not directly apply to \eqref{eqn:exp_online_mat_comp} as it cannot handle the affine constraint $1\leq X\leq 5$. 
Therefore, we apply SFW to a relaxation of  \eqref{eqn:exp_online_mat_comp} that omits this constraint. 
Then, we solve \eqref{eqn:exp_online_mat_comp} with SHCGM and compare the results. 

We use the default \texttt{ub.train} and \texttt{ub.test} partitions provided with the original data. 
We set the model parameter for the nuclear norm constraint $\beta_1 = 7000$, and the initial smoothing parameter $\beta_0 = 10$. 
At each iteration, we compute a gradient estimator from $1000$ \textit{iid} samples. 
We perform the same test independently for $10$ times to compute the average performance and confidence intervals. 
In Figure~\ref{fig:movielens100k}, we report the training and test errors (root mean squared error) and the feasibility gap. 
The solid lines display the average performance, and the shaded areas show $\pm$ one standard deviation. 
Note that SHCGM performs uniformly better than SFW, both in terms of the training and test errors. 
The table shows the values achieved at the end of $10'000$ iterations. 

Finally, we compare SHCGM with the stochastic three-composite convex minimization method (S3CCM) from \cite{Yurtsever2016s3cm}. 
S3CCM is a projection-based method that applies to \eqref{eqn:exp_online_mat_comp}. 
In this experiment, we aim to demonstrate the advantages of the projection-free methods for problems in large-scale. 

We consider a test setup with the MovieLens1m data set$^{\ref{ftnote:foot3}}$ with $\sim$1 million ratings from $\sim$$6000$ users on $\sim$$4000$ movies. 
We partition the data into training and test samples with a $80/20$ train/test split. 
We use  $10'000$ \textit{iid} samples at each iteration to compute a gradient estimator. 
We set the model parameter $\beta_1 = 20'000$. 
We use $\beta_0 = 10$ for SHCGM, and we set the step-size parameter $\gamma = 1$ for S3CCM. 
We implement the \ref{eqn:lmo} efficiently using the power method. 

\begin{figure}
\centering
\includegraphics[width=0.7\linewidth]{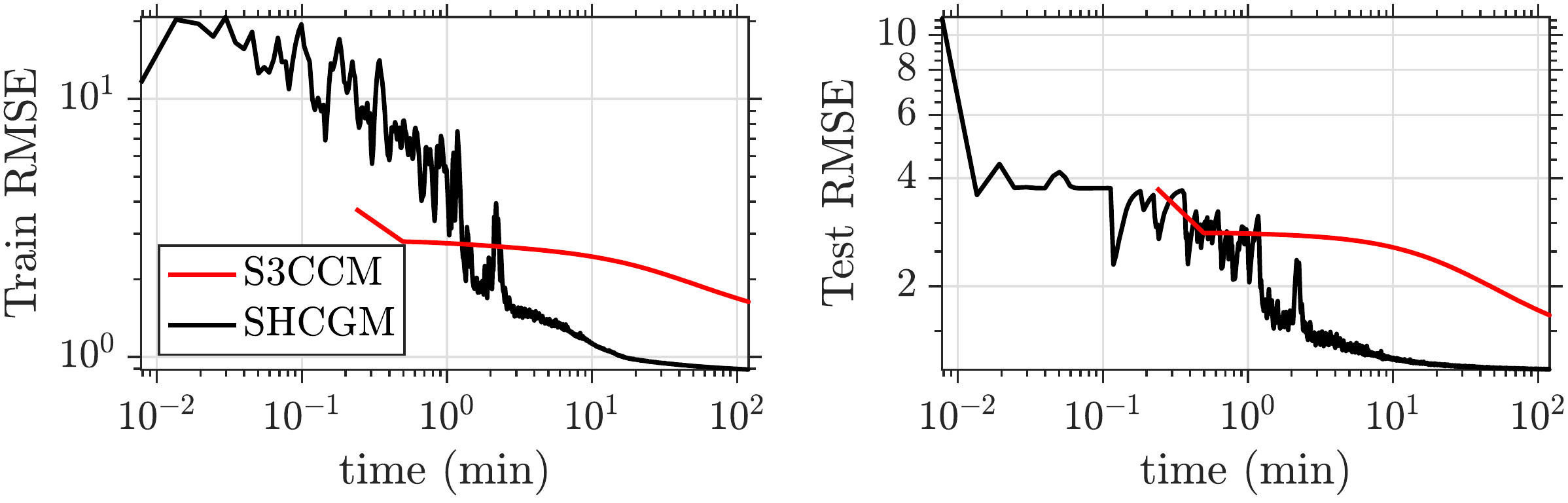}
        \caption{\small SHCGM vs S3CCM with MovieLens-1M. We observe that, on large data sets where projections are very expensive, the small iteration cost of our SHCGM makes it competitive against a projection-based approach.}
        \label{fig:movielens1m}
\end{figure}
Figure~\ref{fig:movielens1m} reports the outcomes of this experiment. 
SHCGM clearly outperforms S3CCM in this test. 
We run both methods for $2$ hours. 
Within this time limit, SHCGM can perform $27'860$ iterations while S3CCM can gets only up to $435$ because of the high computational cost of the projection. We remark that S3CCM needs much fewer iterations. However, the iteration cost is much better for our algorithm which proves to be an advantage on large datasets. On smaller datasets (such as MovieLens100k), projection based methods such as S3CCM would be faster.

\section{Proofs}\label{sec:shcgm_proof}
We first prove some key lemmas. This section builds on top of the analysis of \cite{mokhtari2020stochastic} and the homotopy CGM framework of \citet{yurtsever2018conditional}. 
All these results are for the inexact oracle with additive error, the exact oracle case can be obtained setting $\delta = 0$.

First, we adapt the definition of the additive error \ref{eqn:lmo_add} to the smoothed objective $F_{\beta_k}$.
At iteration $k$, for the given $v_k$, we assume that the approximate \ref{eqn:lmo} returns an element $\tilde{s}_k\in\mathcal{X}$ such that:
\begin{equation}\label{eqn:SFW_inexact_LMO}
\ip{v_k}{\tilde{s}_k} \leq \ip{v_k}{s_k} + \delta \frac{\eta_k}{2}D_{\mathcal{X}}^2\left(L_f + \frac{\|A\|^2}{\beta_k}\right)
\end{equation}
for some $\delta > 0$, where $s_k$ is the exact \ref{eqn:lmo} solution. 

\begin{lemma}\label{lemma:linear_exact_additive}
For any given iteration $k\geq 1$ of Algorithm~\ref{alg:stochastic_HFW} the following relation holds:
 \begin{align*}
 \ip{ \nabla F_{\beta_k}(\theta_k)}{s_k - \theta_k}&\leq \norm{\nabla_\theta \bbE_\rvw f(\theta_k, \omega)- d_k}D_\mathcal{X} +  f^\star  - \bbE_\rvw f(\theta_k, \omega) + g(A\theta^\star) - g_{\beta_k}(A\theta_k)\\ &\qquad - \frac{\beta_k}{2} \norm{  y^\ast_{\beta_k}(A\theta_k) }^2  + \delta\frac{\eta_k}{2}D_\mathcal{X}^2\left( L_f + \frac{\norm{A}^2}{\beta_k}\right)
 \end{align*}
 where $\delta \geq 0$ is the accuracy of the inexact \ref{eqn:lmo} with additive error as in Equation \eqref{eqn:SFW_inexact_LMO}.
 \begin{proof}
\begin{align}
\ip{ \nabla F_{\beta_k}&(\theta_k)}{\tilde s_k - \theta_k}
= 
\ip{\nabla_\theta \bbE_\rvw f(\theta_k, \omega)}{\tilde s_k - \theta_k} + \ip{A^\top \nabla g_{\beta_k}(A\theta_k)}{\tilde s_k - \theta_k }\nonumber\\
&= \ip{\nabla_\theta \bbE_\rvw f(\theta_k, \omega) - d_k}{\tilde s_k - \theta_k} + \ip{d_k + A^\top \nabla g_{\beta_k}(A\theta_k)}{\tilde s_k - \theta_k } \nonumber
\end{align}
\begin{align}
\quad
&\leq \ip{\nabla_\theta \bbE_\rvw f(\theta_k, \omega) - d_k}{s_k - \theta_k} \nonumber\\&\qquad + \ip{d_k + A^\top \nabla g_{\beta_k}(A\theta_k)}{s_k - \theta_k } + \delta\frac{\eta_k}{2}D_\mathcal{X}^2\left( L_f + \frac{\norm{A}^2}{\beta_k}\right)\label{eq_proof:lemma_additive_inexact_def}\\
&\leq \ip{\nabla_\theta \bbE_\rvw f(\theta_k, \omega) - d_k}{s_k - \theta_k} \nonumber\\&\qquad + \ip{d_k + A^\top \nabla g_{\beta_k}(A\theta_k)}{\theta^\star - \theta_k } + \delta\frac{\eta_k}{2}D_\mathcal{X}^2\left( L_f + \frac{\norm{A}^2}{\beta_k}\right)\label{eq_proof:lemma_additive_min_oracle}\\
&\leq \norm{\nabla_\theta \bbE_\rvw f(\theta_k, \omega) - d_k}\norm{s_k - \theta^\star} \nonumber\\&\qquad + \ip{\nabla_\theta \bbE_\rvw f(\theta_k, \omega)+ A^\top \nabla g_{\beta_k}(A\theta_k)}{\theta^\star - \theta_k }+ \delta\frac{\eta_k}{2}D_\mathcal{X}^2\left( L_f + \frac{\norm{A}^2}{\beta_k}\right)\label{eq_proof:lemma_additive_cs} \\
&\leq \norm{\nabla_\theta \bbE_\rvw f(\theta_k, \omega) - d_k}D_\mathcal{X} \nonumber\\&\qquad + \ip{\nabla_\theta \bbE_\rvw f(\theta_k, \omega) + A^\top \nabla g_{\beta_k}(A\theta_k)}{\theta^\star - \theta_k } + \delta\frac{\eta_k}{2}D_\mathcal{X}^2\left( L_f + \frac{\norm{A}^2}{\beta_k}\right)\label{eq_proof:lemma_additive_diam}
\end{align}
where Equation~\eqref{eq_proof:lemma_additive_inexact_def} is the definition of inexact oracle with additive error, Equation~\eqref{eq_proof:lemma_additive_min_oracle} follows from $s_k$ being a solution of $\min_{\theta \in \mathcal{X}} \ip{d_k + A^\top \nabla g_{\beta_k}(A\theta_k)}{\theta}$, Equation~\eqref{eq_proof:lemma_additive_cs} is Cauchy-Schwarz and the Equation~\eqref{eq_proof:lemma_additive_diam} is the definition of diameter.
Now, convexity of $\bbE_\rvw f(\theta_k, \omega)$ ensures $\langle\nabla_\theta \bbE_\rvw f(\theta_k, \omega),\theta^\star - \theta_k \rangle\leq f^\star  - \bbE_\rvw f(\theta_k, \omega)$. 

From Lemma 10 in \citep{TranDinh2017} we have that:
\begin{align}\label{eqn:smoothing-prop-2}
g(z_1) & \geq g_\beta(z_2) + \ip{\nabla g_{\beta}(z_2)}{z_1 - z_2} + \frac{\beta}{2} \norm{y^\ast_\beta(z_2)}^2.
\end{align}
Therefore: 
\begin{align*}
\ip{A^\top \nabla g_{\beta_k} (A\theta_k)}{\theta^\star - \theta_k} 
&\leq 
g(A\theta^\star) - g_{\beta_k}(A\theta_k) - \frac{\beta_k}{2} \norm{  y^\ast_{\beta_k}(A\theta_k) }^2. 
\end{align*}
Which yields the desired result:
\begin{align*}
\ip{ \nabla F_{\beta_k}(\theta_k)}{\tilde s_k - \theta_k}&\leq \norm{\nabla_\theta \bbE_\rvw f(\theta_k, \omega) - d_k}D_\mathcal{X} +  f^\star ) - \bbE_\rvw f(\theta_k, \omega) + g(A\theta^\star) - g_{\beta_k}(A\theta_k)\\&\qquad - \frac{\beta_k}{2} \norm{  y^\ast_{\beta_k}(A\theta_k) }^2 + \delta\frac{\eta_k}{2}D_\mathcal{X}^2\left( L_f + \frac{\norm{A}^2}{\beta_k}\right)
\end{align*}
 \end{proof}
\end{lemma}

\begin{lemma}\label{lemma:stoch_gradient_convergence}
For any $k\geq 1$ the estimate of the gradient computed in Algorithm~\ref{alg:stochastic_HFW} satisfies:
\begin{align*}
\mathbb{E}\left[ \norm{\nabla_\theta \bbE_\rvw f(\theta_k, \omega) - d_k }^2\right] &\leq \frac{Q}{( k +8)^\frac{2}{3}}
\end{align*}
where $Q = \max\left\lbrace \norm{\nabla_\theta \bbE_\rvw f(\theta_{1}, \omega) - d_1 }^2 7^\frac{2}{3},16\sigma^2 + 81L_f^2D_\mathcal{X}^2\right\rbrace$
\begin{proof}
This lemma applies Lemma 1 and Lemma 17 of \cite{mokhtari2020stochastic} to our different stepsize sequences. First, we invoke Lemma 1:
\begin{align}
\mathbb{E}&\left[ \norm{\nabla_\theta\bbE_\rvw f(\theta_k, \omega) - d_k }^2\right] \leq \left(1-\frac{\rho_k}{2} \right) \norm{\nabla_\theta\bbE_\rvw f(\theta_{k-1}, \omega) - d_{k-1} }^2 + \rho_k^2\sigma^2 + \frac{2L_f^2D_\mathcal{X}^2\eta_{k-1}^2}{\rho_k}\nonumber\\
&\leq \left(1-\frac{2}{(k+7)^\frac{2}{3}} \right) \norm{\nabla_\theta \bbE_\rvw f(\theta_{k-1}, \omega) - d_{k-1} }^2 + \frac{16\sigma^2 + 81L_f^2D_\mathcal{X}^2}{( k+7)^\frac{4}{3}}\label{eq_proof:lemma1}
\end{align}
where we used $\rho_k = \frac{4}{(k+7)^\frac{2}{3}}$.
Applying Lemma 17 of \cite{mokhtari2020stochastic} to Equation~\eqref{eq_proof:lemma1}  with $k_0 = 7$, $\alpha = \frac{2}{3}$, $c = 2$, $b = 16\sigma^2 + 81L_f^2D_\mathcal{X}^2$ gives:
\begin{align*}
\mathbb{E}\left[ \norm{\nabla_\theta \bbE_\rvw f(\theta_k, \omega) - d_k }^2\right] &\leq \frac{Q}{( k +8)^\frac{2}{3}}
\end{align*}
where $Q = \max\left\lbrace \norm{\nabla \bbE_\rvw f(x_{1}, \omega) - d_1 }^2 7^\frac{2}{3},16\sigma^2 + 81L_f^2D_\mathcal{X}^2\right\rbrace$
\end{proof}
\end{lemma}

\subsection{Proof of Theorem~\ref{thm:composite}}
We prove Theorem~\ref{thm:composite} with the oracle with additive error. The proof without additive error can be obtained with $\delta = 0$.
\begin{theorem}\label{thm:composite}
The sequence $\theta_k$ generated by Algorithm~\ref{alg:stochastic_HFW} satisfies the following bound for $k\geq 1$:
\begin{align*}
\bbE F_{\beta_k}(\theta_{k+1}) - F^\star 
\leq  9^\frac{1}{3} \frac{C_\delta}{(k+8)^\frac{1}{3}},
\end{align*}
where $C_\delta := \frac{81}{2}D_{\mathcal{X}}^2(L_f + \beta_0\norm{A}^2)(1+\delta) + 9D_\mathcal{X}\sqrt{Q}$, \newline
$Q = \max\left\lbrace 4\norm{\nabla \bbE_\rvw f(\theta_1, \omega) - d_1}^2, 16\sigma^2 + 2L_f^2D_\mathcal{X}^2 \right\rbrace$ and $\delta\geq 0$.
\end{theorem}
\begin{proof}
First, we use the smoothness of $F_{\beta_k}$ to upper bound the progress. 
Note that $F_{\beta_k}$ is $(L_f + \norm{A}^2/\beta_k)$-smooth. 
\begin{align}
F_{\beta_k}(\theta_{k+1}) 
&\leq F_{\beta_k}(\theta_k)  +\eta_k 
\ip{\nabla F_{\beta_k}(\theta_k) }{\tilde s_k - \theta_k}
+ \frac{\eta_k^2}{2}D_{\mathcal{X}}^2(L_f + \frac{\norm{A}^2}{\beta_k})\nonumber,
\end{align}
where $s_k\in\mathcal{X}$ denotes the atom selected by the \ref{eqn:lmo}.  
We now apply Lemma~\ref{lemma:linear_exact_additive} and obtain:
\begin{align}
F_{\beta_k}(\theta_{k+1}) 
&\leq 
F_{\beta_k}(\theta_k) +\eta_k \left( f^\star  - \bbE_\rvw f(\theta_k, \omega) + g(A\theta^\star) - g_{\beta_k}(A\theta_k) - \frac{\beta_k}{2} \norm{ \nabla y^\ast_{\beta_k}(A\theta_k) }^2 \right) \nonumber \\
& \qquad + \frac{\eta_k^2}{2}D_{\mathcal{X}}^2(L_f + \frac{\norm{A}^2}{\beta_k})(1+\delta) + \eta_k \norm{\nabla_\theta \bbE_\rvw f(\theta_k, \omega) - d_k}D_\mathcal{X}. \label{eqn:proof-recursion-1} 
\end{align}

Now, using Lemma 10 of~\cite{TranDinh2017} we get:
\begin{equation}
g_{\beta} (z_1) \leq g_{\gamma}(z_1) + \frac{\gamma - \beta}{2} \norm{y^\ast_\beta(z_1)}^2 \label{eqn:smoothing-prop-3}
\end{equation}
 and therefore:
\begin{align*}
F_{\beta_k}(\theta_k) 
& = 
\bbE_\rvw f(\theta_k, \omega) + g_{\beta_k}(A\theta_k) \\
& \leq 
\bbE_\rvw f(\theta_k, \omega) + g_{\beta_{k-1}}(A\theta_k) + \frac{\beta_{k-1}-\beta_k}{2}\norm{y^\ast_{\beta_k}(A\theta_k)}^2 \\
& = 
F_{\beta_{k-1}}(\theta_k) + \frac{\beta_{k-1}-\beta_k}{2}\norm{y^\ast_{\beta_k}(A\theta_k)}^2.
\end{align*}
We combine this with \eqref{eqn:proof-recursion-1} and subtract $F^\star $ from both sides to get
\begin{align*}
F_{\beta_k}(\theta_{k+1}) -F^\star 
&\leq 
(1 - \eta_k) \big( F_{\beta_{k-1}}(\theta_k) -  F^\star  \big) + \frac{\eta_k^2}{2}D_{\mathcal{X}}^2(L_f + \frac{\norm{A}^2}{\beta_k})(1+\delta) \\ 
&  +  \big( (1-\eta_k) (\beta_{k-1} - \beta_k) - \eta_k \beta_k \big) \frac{1}{2} \norm{ y^\ast_{\beta_k}(A\theta_k) }^2 + \eta_k \norm{\nabla_\theta \bbE_\rvw f(\theta_k, \omega) - d_k}D_\mathcal{X} .
\end{align*}

Let us choose $\eta_k$ and $\beta_k$ in a way to vanish the last term. 
By choosing $\eta_k = \frac{9}{k+8}$ and $\beta_k = \frac{\beta_0}{(k+8)^{\frac{1}{2}}}$ for $k \geq 1$ with some $\beta_0 > 0$, we get $(1-\eta_k)(\beta_{k-1}-\beta_k) - \eta_k\beta_k < 0$. Hence, we end up with 
\begin{align*}
F_{\beta_k}(\theta_{k+1}) - F^\star 
&\leq 
(1 - \eta_k) \big( F_{\beta_{k-1}}(\theta_k) -  F^\star  \big) + \frac{\eta_k^2}{2}D_{\mathcal{X}}^2(L_f + \frac{\norm{A}^2}{\beta_k})(1+\delta) \\
&\qquad + \eta_k \norm{\nabla_\theta \bbE_\rvw f(\theta_k, \omega) - d_k}D_\mathcal{X} .
\end{align*}

We now compute the expectation and use Jensen inequality and Lemma~\ref{lemma:stoch_gradient_convergence} to obtain the final recursion:
\begin{align*}
\mathbb{E} F_{\beta_k}(\theta_{k+1}) - F^\star &\leq 
(1 - \eta_k) \big(\mathbb{E} F_{\beta_{k-1}}(\theta_k) -  F^\star  \big) + \frac{\eta_k^2}{2}D_{\mathcal{X}}^2(L_f + \frac{\norm{A}^2}{\beta_k})(1+\delta)+ \frac{9D_\mathcal{X}\sqrt{Q}}{(k+8)^{\frac{4}{3}}}
\end{align*}

Now, note that: 
\begin{align*}
\frac{\eta_k^2}{2}D_{\mathcal{X}}^2(L_f + \frac{\norm{A}^2}{\beta_k}) &= \frac{\eta_k^2}{2}D_{\mathcal{X}}^2L_f + \frac{\eta_k^2}{2}D_{\mathcal{X}}^2\frac{\norm{A}^2}{\beta_k}
\leq\frac{\frac{81}{2}}{(k+8)^{\frac{4}{3}}}D_{\mathcal{X}}^2L_f + \frac{\frac{81}{2}}{(k+8)^{\frac{4}{3}}}\beta_0 D_{\mathcal{X}}^2\norm{A}^2
\end{align*}
Therefore:
\begin{align*}
\mathbb{E} F_{\beta_k}(\theta_{k+1}) - F^\star 
&\leq \left(1 - \frac{9}{k+8}\right) \big(\mathbb{E} F_{\beta_{k-1}}(\theta_k) -  F^\star  \big) \nonumber\\&+\frac{\frac{81}{2}D_{\mathcal{X}}^2(L_f + \beta_0\norm{A}^2)(1+\delta) + 9D_\mathcal{X}\sqrt{Q}}{(k+8)^{\frac{4}{3}}}
\end{align*}
For simplicity, let $C_\delta := \frac{81}{2}D_{\mathcal{X}}^2(L_f + \beta_0\norm{A}^2)(1+\delta) + 9D_\mathcal{X}\sqrt{Q}$ and $\mathcal{E}_{k+1}:= \mathbb{E} F_{\beta_k}(\theta_{k+1}) - F^\star $. Then, we need to solve the following recursive equation:
\begin{align}\label{eq_proof:induction_rec}
\mathcal{E}_{k+1}\leq \left(1 - \frac{9}{k+8}\right)\mathcal{E}_{k} + \frac{C_\delta}{(k+8)^{\frac{4}{3}}}
\end{align}
Let the induction hypothesis for $k\geq 1$ be:
\begin{align*}
\mathcal{E}_{k+1}\leq 9^\frac{1}{3} \frac{C_\delta}{(k+8)^\frac{1}{3}}
\end{align*}
For the base case $k=1$ we need to prove $\mathcal{E}_{2}\leq C_\delta$.
From Equation~\eqref{eq_proof:induction_rec} we have $\mathcal{E}_{2}\leq  \frac{C_\delta}{(9)^{\frac{4}{3}}}< C_\delta$ as $9^{\frac{4}{3}}>1$
Now:
\begin{align*}
\mathcal{E}_{k+1}&\leq \left(1 - \frac{9}{k+8}\right)\mathcal{E}_{k} + \frac{C_\delta}{(k+8)^{\frac{4}{3}}} \leq \left(1 - \frac{9}{k+8}\right)9^\frac{1}{3} \frac{C_\delta}{(k+7)^\frac{1}{3}}+ \frac{C_\delta}{(k+8)^{\frac{4}{3}}}
\leq 9^\frac{1}{3}\frac{C_\delta}{(k+8)^{\frac{1}{3}}}\\
\end{align*}
\end{proof}

\subsection{Proof of Theorem \ref{cor:g_lip}}
After proving Theorem \ref{thm:composite}, the proof of Theorem \ref{cor:g_lip} follows easily using arguments from \cite{yurtsever2018conditional} and setting $\delta=0$. 

If $g:\R^d \to \R \cup \{+\infty\}$ is $L_g$-Lipschitz continuous from equation~(2.7) in \cite{Nesterov2005} and the duality between Lipshitzness and bounded support (\textit{cf.} Lemma~5 in \cite{Dunner2016}) we have:
\begin{equation}\label{eqn:smoothing-sandwich}
g_{\beta} (z) \leq g (z) \leq g_{\beta} (z) + \frac{\beta}{2} L_g^2
\end{equation}
Using this fact, we write:
\begin{align*}
g(A\theta_{k+1}) &\leq g_{\beta_k}(A\theta_{k+1}) + \frac{\beta_kL_g^2 }{2} = g_{\beta_k}(A\theta_{k+1}) + \frac{\beta_0 L_g^2}{2\sqrt{k+8}} . 
\end{align*}
We complete the proof by adding $\bbE \bbE_\rvw f(\theta_{k+1}, \omega) - F^\star $ to both sides:
\begin{align*}
\bbE F(\theta_{k+1}) - F^\star 
&\leq 
\bbE F_{\beta_k}(\theta_{k+1}) - F^\star  + \frac{\beta_0 L_g^2}{2\sqrt{k+8}}\leq 
 9^\frac{1}{3} \frac{C_\delta}{(k+8)^\frac{1}{3}} + \frac{\beta_0 L_g^2}{2\sqrt{k+8}}.
\end{align*}

\subsection{Proof of Theorem \ref{cor:indicator-exact}}
We adapt to our setting the proof technique of Theorem 4.3 in \cite{yurtsever2018conditional}. We give the proof for the more general case of additive errors and obtain the proof for Theorem \ref{cor:indicator-exact} as a special case with $\delta=0$.

From the Lagrange saddle point theory, we know that the following bound holds $\forall \theta \in \mathcal{X}$ and $\forall r \in \mathcal{K}$:
\begin{align*}
f^\star  \leq \mathcal{L}(\theta,r,y^\star) &= \bbE_\rvw f(\theta, \omega) + \ip{y_\star}{A\theta - r} \leq \bbE_\rvw f(\theta, \omega) + \norm{y_\star}\norm{A\theta - r},
\end{align*}
Since $\theta_{k+1} \in \mathcal{X}$ and taking the expectation, we get
\begin{align}
\bbE \bbE_\rvw f(\theta_{k+1}, \omega) - f^\star  &\geq - \bbE \min_{r\in\mathcal{K}}\norm{y^\star}\norm{A\theta_{k+1} - r} =  - \norm{y^\star}\bbE \mathrm{dist}(A\theta_{k+1}, \mathcal{K}). \label{eqn:obj-lower-bound}
\end{align}
This proves the first bound in Theorem~\ref{cor:indicator-exact}.

The second bound directly follows by Theorem~\ref{thm:composite}
\begin{align*}
\bbE \bbE_\rvw f(\theta_{k+1}, \omega) - f^\star  &\bbE \leq \bbE \bbE_\rvw f(\theta_{k+1}, \omega) - f^\star  +  \frac{1}{2\beta_k}\bbE \left[\mathrm{dist}(A\theta_{k+1}, \mathcal{K})\right]^2\\
&\leq\bbE  F_{\beta_k}(\theta_{k+1}) - F^\star \leq  9^\frac{1}{3} \frac{C_\delta}{(k+8)^\frac{1}{3}}.
\end{align*}
Now, we combine this with \eqref{eqn:obj-lower-bound}, and we get
\begin{align*}
- \norm{y^\star}\bbE \mathrm{dist}(A\theta_{k+1}, \mathcal{K}) + \frac{1}{2\beta_k} \bbE \left[\mathrm{dist}(A\theta_{k+1}, \mathcal{K}) \right]^2
&\leq  9^\frac{1}{3} \frac{C_\delta}{(k+8)^\frac{1}{3}}\\
\end{align*}
This is a second order inequality in terms of $\bbE \mathrm{dist}(A\theta_k, \mathcal{K})$. Solving this inequality, we get
\begin{align*}
\bbE \mathrm{dist}(A\theta_{k+1}, \mathcal{K})  
& \leq \frac{2 \beta_0 \norm{y^\star} }{\sqrt{k+8}} +\frac{2\sqrt{2\cdot9^\frac13 C_\delta\beta_0}}{(k+8)^{\frac{5}{12}}}.
\end{align*}



\chapter{Linear Span: Matching Pursuit and Coordinate Descent}\label{cha:mpcd}
In this chapter, we consider minimizing a convex function over the linear span of a set. This problem template covers Coordinate Descent as a special case. We provide a unified tight analysis for this general case. The presented approach is based on \citep{locatello2018matching} and was developed in collaboration with Anant Raj, Sai Praneeth Karimireddy, Gunnar R\"atsch, Bernhard Sch\"olkopf, Sebastian U. Stich, and Martin Jaggi. The proof of the accelerated convergence of Matching Pursuit and Coordinate Descent was a four-hand collaboration between Francesco Locatello and Anant Raj. Francesco Locatello and Anant Raj contributed equally to this publication. 

\section{Problem Formulation}
In this chapter we focus on the minimization of a convex smooth function over the linear span of a set:

\begin{equation}
\label{eqn:MPtemplate}
\underset{\theta\in\lin(\cA)}{\text{minimize}} \quad f(\theta).
\end{equation}

In this optimization template, we consider the following setting: \\
$\triangleright~~\cH$ is a Hilbert space $\cH$ with associated inner product $\ip{x}{y} \,\forall \, x,y \in \cH$ and induced norm $\| x \|^2 := \ip{x}{x},$ $\forall \, x \in \cH$. \\
$\triangleright~~\cA \subset \cH$ is a compact and symmetric set (the ``set of atoms'' or dictionary) in $\cH$. \\
$\triangleright~~f \colon \cH \to \R$ is a convex and $L$-smooth ($L$-Lipschitz gradient in the finite dimensional case) function. If $\cH$ is an infinite-dimensional Hilbert space, then $f$ is assumed to be Fr{\'e}chet differentiable. \\
$\triangleright$ We specifically focus on projection-free algorithms that select at each iteration an element $z$ from $\cA$ and update the iterate as $\theta_{k+1} = \theta_k + \gamma_k z_k$, where $\gamma_k$ is a suitable chosen stepsize.

We consider two distinct specific oracles for the optimization algorithm:\\
(i) first-order oracle that returns a pair $(f(\theta), \nabla f(\theta))$ given $\theta$.\\
(ii) a stochastic version of the first-order oracle that returns a pair $(f(\theta), \nabla_z f(\theta))$ given $\theta$ where $z$ is a randomly sampled element of $\cA$ and $\nabla_z f(\theta) = \ip{\nabla f(\theta)}{z}$.

This formulation of the optimization template is interesting as if $\cA = \l1$, then $\lin(\cA)$ corresponds to a euclidean space and the Matching Pursuit and Coordinate Descent algorithms coincide. This allows to draw a parallelism between the two algorithms and present a tight unified analysis of both in the steepest and random pursuit cases. 

\section{Affine Invariant Algorithms}
First, let us review the definition of affine invariant algorithms (also see \citep{LacosteJulien:2013uea}):
\begin{definition} \label{def:aff_invar}
An optimization method that is invariant under affine transformations of its domain is said to be \textit{affine invariant}. Let $M:\hat\cD\rightarrow\cD$ be a \emph{surjective} linear or affine map, then the iterate sequence of an affine invariant optimization algorithm on $\min_{\theta\in\cD}f(\theta)$ and
$\min_{\hat\theta\in\hat\cD}\hat f(\hat\theta)$ for $\hat f(\hat\theta):=f(M\hat\theta)$
is the same in the sense that there is a correspondence between each step of the iterate sequences through $M$.
\end{definition}

In order to develop an affine invariant algorithm, we rely on the atomic norm $\|x\|_\cA := \inf \{ c > 0 \colon x \in c \cdot \conv (\cA) \}$ (also known as the gauge function of $\conv (\cA)$), which is affine invariant. A visualization of the atomic norm is presented in Figure \ref{fig:gface}.
We measure the smoothness of the objective using the atomic norm: 
\begin{equation*}
 L_{\cA} := \!\!\!\sup_{\substack{x,y\in\lin(\cA)\\y = x + \gamma z \\\|z\|_\cA=1, \gamma\in \R_{>0}}}\frac{2}{\gamma^2}\big[ f(y)- f(x) -  \ip{\nabla f(x)}{y - x} \big] \,.
\end{equation*}


This notion of curvature is inspired by the curvature constant of Frank-Wolfe and Matching Pursuit in~\cite{jaggi2013revisiting,locatello2017unified}. It combines the complexity of the function~$f$ as well as the set $\cA$ into a single constant that is affine invariant under transformations of our input problem~\eqref{eqn:MPtemplate}.
From this definition we can easily derive an affine invariant version of the regular smoothness upper bound \eqref{eq:QuadraticUpperBound}, this time using atomic norms:
\begin{align*}
f(y) \leq f(x) + \ip{\nabla f(x)}{y-x} + \frac{L_\cA}{2} \|y-x\|_\cA \,,
\end{align*}

The affine-invariant MP Algorithm~\ref{algo:affineMP} simply replaces the new smoothness definition into Algorithm~\ref{algo:MP}.

\begin{algorithm}[ht]
\caption{Affine Invariant Generalized Matching Pursuit \citep{locatello2018matching}}
  \label{algo:affineMP}
\begin{algorithmic}[1]
  \STATE \textbf{init} $\theta_{0} \in \lin(\cA)$
  \STATE \textbf{for} {$k=0\dots K$}
  \STATE \quad Find $z_k := (\text{Approx-}) \lmo_\cA(\nabla f(\theta_{k}))$
  \STATE \quad $\theta_{k+1} = \theta_k - \frac{\ip{\nabla f(\theta_k)}{z_k}}{L_\cA}z_k$
  \STATE \textbf{end for}
\end{algorithmic}
\end{algorithm}
\vspace{-0.2mm}

\section{Affine Invariant Rates}
First, we remark that using $\nabla \hat f = M^T \nabla f$ it is simple to show that Algorithm~\ref{algo:affineMP} is affine invariant:
\begin{align*}
M\hat\theta_{k+1} &= M\left( \hat\theta_k + \frac{\langle \nabla \hat f(\hat\theta_k), \hat b_k\rangle}{L_\cA}\hat b_k\right) = M\hat\theta_k + \frac{\langle \nabla \hat f(\hat\theta_k), \hat b_k\rangle}{L_\cA}M\hat b_k = \theta_k + \frac{\langle \nabla \hat f(\hat\theta_k), \hat b_k\rangle}{L_\cA} b_k\\
&= \theta_k + \frac{\langle M^T\nabla f(\theta_k), \hat b_k\rangle}{L_\cA} b_k = \theta_k + \frac{\langle \nabla f(\theta_k), M\hat b_k\rangle}{L_\cA} b_k = \theta_k + \frac{\langle \nabla f(\theta_k),  b_k\rangle}{L_\cA} b_k = \theta_{k+1} \,.
\end{align*}

For the analysis, we call $\theta^\star$ the minimizer of problem~\eqref{eqn:MPtemplate}. If the optimum is not unique, we pick the one that yields worst-case constants. 

\subsection{Sublinear Rates for Convex Functions}
In order to prove convergence, we start by defining the level set radius measured with the atomic norm as:
\begin{equation}
R_\cA^2 := \max_{\substack{\theta\in\lin(\cA)\\ f(\theta)\leq f(\theta_0)}}\|\theta-\theta^\star\|^2_\cA \,.
\end{equation}
When we measure this radius with the $\|\cdot\|_2$ we call it $R_2^2$, and when we measure it with $\|\cdot\|_1$ we call it $R_1^2$.

\paragraph{Remark} Measuring smoothness using the atomic norm guarantees that:
\begin{lemma}\label{lem:LwithRadius}
Assume $f$ is $L$-smooth w.r.t. a given norm $\|\cdot\|$, over $\lin(\cA)$ where $\cA$ is symmetric.
Then,
\begin{equation}
 L_\cA \leq L \, \radius_{\norm{\cdot}}(\cA)^2\,.
\end{equation}
\end{lemma}
As a motivating example, in coordinate descent, we will measure smoothness with the atomic norm being the L1-norm. Lemma~\ref{lem:LwithRadius} implies that $L_{\cA} \leq L_2$ where $L_2$ is the smoothness constant measured with the L2-norm.

We are now ready to prove the convergence rate of Algorithm~\ref{algo:affineMP} for smooth convex functions. We only consider multiplicative errors for the \lmo as in Equation \eqref{eqn:lmo_mult}.
\begin{theorem}\label{thm:sublinear_MP_rate}
Let $\cA \subset \cH$ be a closed and bounded set. We assume that $\|\cdot\|_\cA$ is a norm over $\lin(\cA)$. Let $f$ be convex and $L_\cA$-smooth w.r.t. the norm $\|\cdot\|_\cA$ over $\lin(\cA)$, and let $R_\cA$ be the radius of the level set of $\theta_0$ measured with the atomic norm.
Then, Algorithm~\ref{algo:affineMP} converges for $k \geq 0$ as
\[
f(\theta_{k+1}) - f(\theta^\star) \leq \frac{2L_{\cA}R_\cA^2}{\delta^2(k+2)} \,,
\]
where $\delta \in (0,1]$ is the relative accuracy parameter of the employed approximate \lmo as in Equation~\eqref{eqn:lmo_mult}.
\end{theorem}

\paragraph{Discussion} \looseness=-1Our proof relies on the affine invariant definition of smoothness and level set radius and the properties of atomic norms. The key ingredient for the proof is to realize that the \lmo solution can be used to compute the dual atomic norm of the gradient. Overall, the proof shares the spirit of the classical proof of steepest coordinate descent from~\citet{nesterov2012efficiency} except:\\
$\triangleright$ We do not assume orthogonal atoms and allow for potentially different L2-norms. Our atoms do not correspond to the natural basis of the ambient space.\\
$\triangleright$ We only assume that $\cA$ is closed, bounded and $\|\cdot\|_\cA$ is a norm over $\lin(\cA)$. Therefore,  $\lin(\cA)$ could be a subset of the ambient space.\\
$\triangleright$ We further do not make any incoherence assumption nor sparsity assumption. We support continuous atom sets \eg when $\cA$ is the L2-ball, our algorithm and analysis perfectly recover gradient descent.

\paragraph{Comparison with Previous MP Rates}
\looseness=-1The analysis of our sublinear convergence rate is fundamentally different from the one proved in~\cite{locatello2017unified}. As their proof is inspired by the technique used for Frank-Wolfe by~\citet{jaggi2013revisiting}, they suffer from a dependency on the atomic norm of the whole iterate sequence $\rho := \max\left\lbrace \|\theta^\star\|_{\cA}, \|\theta_{0}\|_{\cA}\ldots,\|\theta_K\|_{\cA}\right\rbrace<\infty$. Although their algorithm is also affine invariant, their notion of smoothness depends explicitly on $\rho$, which introduces a circular dependency if one wants to achieve a tight rate (otherwise, using an upper bound to the smoothness constant in the algorithm is sufficient). Our solution of redefining the smoothness constant with atomic norms is more elegant and requires a significantly simpler proof without any additional unnecessary assumption such as $\rho$ being finite.

\paragraph{Comparison with Previous CD Rates}
As a special case of our analysis we can read off the rate of steepest coordinate descent. We simply set $\cA$ to the L1-ball in an $n$ dimensional space and obtain (in the case of exact \lmo):
\begin{align*}
f(\theta_{k+1}) - f(\theta^\star) \leq \frac{2L_1R_1^2}{k+2}\leq \frac{2L_2R_1^2}{k+2} \leq \frac{2L_2nR_2^2}{k+2} \,,
\end{align*}
where the first inequality is our rate, the second is~\cite{stich2017approximate} and the last is \cite{nesterov2012efficiency}, both with global Lipschitz constant. Our rate for steepest coordinate descent is the tightest known with a global smoothness constant. Note that for coordinate-wise $L$ our definition is equivalent to the classical one. $L_\cA\leq L_2$ if the norm is defined over more than one dimension (i.e. blocks), otherwise there is equality. For the relationship of $L_1$-smoothness to coordinate-wise smoothness, see also \citep[Theorem 4 in Appendix]{karimireddy2019efficient}. 

\paragraph{Affine Invariant Coordinate Descent?} 
\looseness=-1Our approach is affine invariant as the basis becomes part of the definition of the optimization problem. If we transform the problem with an affine transformation $M$ the atoms do not correspond to the natural coordinates anymore. The transformed coordinates are $\hat{e_i} = M^{-1} e_i$ where $M^{-1}$ is the inverse of the affine map. Although MP and CD coincide for one particular choice of basis, the latter algorithm is not affine invariant (unless one transforms the basis as in MP).

\subsection{Intermezzo: Random Pursuit}
\citep{nesterov2012efficiency} argued that steepest coordinate descent is at a disadvantage compared to random coordinate descent for large scale optimization problems. In the latter case, only a random coordinate of the gradient needs to be computed with no loss in terms of convergence rate. We can likewise extend our analysis to the \emph{random pursuit} case where $z$ is randomly sampled from a distribution over $\cA$, rather than picked by a linear minimization oracle as in~\cite{stich2013optimization}.
Assuming we can cheaply compute the projection of the gradient onto a single atom $\ip{\nabla f}{z}$ (for example, approximating it with finite differences), we can exploit the definition of the inexact oracle to give a rate in the case of arbitrary atom sets:
\begin{equation}\label{eq:delta-def-random-pursuit}
\hat\delta^2 := \min_{d\in\lin(\cA)}\frac{\bbE_{z \in \cA} \ip{d}{z}^2}{\|d\|_{\cA*}^2} \,.
\end{equation}
where $\|\cdot\|_{\cA*}$ is the dual norm of $\|\cdot\|_{\cA}$.
This constant already appears in~\cite{Stich14} to measure the convergence of random pursuit ($\beta^2$ in his notation).
Uniformly sampling the corners of the L1-ball, we have $\hat\delta^2 = \frac{1}{n}$.
Our definition of $\hat\delta$ allows obtaining a rate for any sampling scheme that ensures $\hat\delta^2\neq 0$.  

We are now ready to present the sublinear convergence rate of random pursuit.
\begin{theorem}\label{thm:rand_purs_sub}
Let $\cA \subset \cH$ be a closed and bounded set. We assume that $\|\cdot\|_\cA$ is a norm. Let $f$ be convex and $L_\cA$-smooth w.r.t. the norm $\|\cdot\|_\cA$ over $\lin(\cA)$ and let $R_\cA$ be the radius of the level set of $\theta_0$ measured with the atomic norm.
Then, Algorithm~\ref{algo:affineMP} converges for $k \geq 0$ as\vspace{-1mm}
\[
\bbE_z \big[f(\theta_{k+1}) \big] - f(\theta^\star) \leq \frac{2L_{\cA}R_\cA^2}{\hat\delta^2(k+2)} \,,
\]
when the $\lmo$ is replaced with random sampling of $z$ from a distribution over $\cA$. The expectation $\bbE_z$ is computed over this distribution.
\end{theorem}

\paragraph{Discussion}
This approach is very general, as it allows to guarantee convergence for \textit{any} sampling scheme and \textit{any} set $\cA$ provided that $\hat\delta^2 \neq 0$. In particular:\\
$\triangleright$ Replacing the line search step on the quadratic upper bound given by smoothness with line search on $f$, our technique gives a rate for gradient-free approaches.\\
$\triangleright$ While the worst case convergence of steepest and random coordinate descent is the same, the best case speed-up of steepest CD is $\Th{n}$.\\
$\triangleright$  Examples of computation of $\hat\delta^2$ can be found in \citep[Section 4.2]{Stich14}. In particular, if $z$ is sampled from a spherical distribution as in~\cite{stich2013optimization}, $\hat\delta^2 = \frac{1}{n}$.\\
$\triangleright$ If the sampling distribution is preserved across affine transformations of the domain, $\hat\delta^2$ is affine invariant. 

\subsection{Linear Rates for Strongly Convex Functions}
For the linear rate we define an affine invariant notion of strong convexity:
\begin{align*}
\mu_\cA := \inf_{\substack{x,y\in\lin(\cA)\\x\neq y}} \frac{2}{\|y-x\|_{\cA}^2} D(y,x)\,.
\end{align*}
where $D(y,x):= f(y)-f(x)-\ip{\nabla f(x)}{y - x}$.
We can give the linear rate of both matching and random pursuit.
\begin{theorem}\label{thm:linear_rate}
Let $\cA \subset \cH$ be a closed and bounded set. We assume that $\|\cdot\|_\cA$ is a norm. Let $f$ be $\mu_\cA$-strongly convex and $L_\cA$-smooth w.r.t. the norm $\|\cdot\|_\cA$, both over $\lin(\cA)$.
Then, Algorithm~\ref{algo:affineMP} converges for $k \geq 0$ as
\[
\varepsilon_{k+1} \leq \big(1 - \delta^2\frac{\mu_\cA}{L_{\cA}}\big)\varepsilon_k \,.
\]
where $\varepsilon_{k} := f(\theta_{k}) - f(\theta^\star)$.
If the LMO direction is sampled randomly from $\cA$,  Algorithm~\ref{algo:affineMP} converges for $k \geq 0$ as
\[
\bbE_z \left[\varepsilon_{k+1}|\theta_k\right] \leq \big(1 - \hat\delta^2\frac{\mu_\cA}{L_{\cA}} \big)\varepsilon_k \,.
\]
\end{theorem}

\paragraph{Comparison with Previous MP Rates}
Similarly to the sublinear rate, our new proof does not rely on $\rho$ and is tighter than with any other norm choice.
Let us recall the notion of \textit{minimal directional width} from ~\cite{locatello2017unified}, which is how they measure the complexity of the atom set for a chosen norm:\vspace{-1mm}
\begin{align*}
\mdw := \min_{\substack{d\in\lin(\cA)\\d\neq 0}} \max_{z\in\cA} \ip{ \frac{d}{\|d\|}}{z} \,.
\end{align*}
Our affine invariant definition for strong convexity relates to the $\mdw$ as:
\begin{lemma}\label{thm:mu_mdw}
Assume $f$ is $\mu$-strongly convex w.r.t. a given norm $\|\cdot\|$ over $\lin(\cA)$ and $\cA$ is symmetric. Then:\vspace{-1mm}
\begin{align*}
\mu_\cA \geq \mdw^2\mu \,.
\end{align*}
\end{lemma}
This lemma allows us to recover the rate of \citep{locatello2017unified} when the norm is fixed. 

\paragraph{Comparison with Previous CD Rates}
Note that for CD we have that $\mdw = \frac{1}{\sqrt{n}}$. Our rate compares to previous work as:
\begin{align*}
\varepsilon_{k+1}\leq\left(1 - \frac{\mu_1}{L_1}\right)\varepsilon_k \leq\left(1 - \frac{\mu_1}{L}\right)\varepsilon_k \leq \left(1 - \frac{\mu}{nL}\right)\varepsilon_k \,,
\end{align*}
where the first is our rate, the second is the rate of steepest CD~\cite{Nutini:2015vd} and the last is the one for random CD~\cite{nesterov2012efficiency}.

\section{Accelerated Rates}
\paragraph{Grounding of Our Approach}
\looseness=-1We consider the acceleration technique in \cite{stich2013optimization} which in turn is based on \cite{Lee13}. Their accelerated coordinate descent algorithm maintains two sequences of iterates $\theta$ and $v$ that are updated along randomly sampled coordinates. Considering general atoms instead, this proof technique would lead us to an accelerated rate for random pursuit. To accelerate \emph{matching} pursuit we need to \emph{decouple} the updates of $\theta$ and $v$, using the steepest descent direction for $\theta$ and a randomly sampled atom for $v$.
The possibility of decoupling the updates was first noted in \citep[Corollary 6.4]{Stich14} though its implications for accelerating greedy coordinate descent or matching pursuit were not explored.

\paragraph{Additional Assumptions} For the accelerated rate, we further assume:\\
$\triangleright$ That the linear space spanned by the atoms $\cA$ is finite dimensional.\\
$\triangleright$ That random atoms are only sampled on a non-symmetric version of $\cA$ with all the atoms in the same half space. Line search ensures that sampling either $z$ or $-z$ yields the same update. \\
$\triangleright$ For simplicity, we focus on an exact \lmo.\\
Whether these assumptions are necessary or not remains an open question. \\
$\triangleright$ We make assumptions on the sampling distribution as explained in the next paragraph.\\
Whether these assumptions are necessary or not remains an open question. 

\paragraph{Key Ingredients}
The main difference between working with atoms and working with coordinates is that randomly sampling the coordinates of the gradient gives an unbiased estimate of the gradient itself in expectation. In other words, for any vector $d$ we have:
$$
  \bbE_{i \in [n]} \encase{\ip{e_i}{d}e_i \Large}= \frac{1}{n} d\,.
$$
This is not true for general atom sets. To correct this issue, we morph the geometry of the space to achieve comparable sampling properties. The disadvantage of this approach is that it makes the proof dependent on a specific norm choice and therefore is not affine invariant.

Suppose we sample the atoms from a random variable $\rvz$ defined over $\cA$, we define
$$
  \tilde{P} := \bbE_{z \sim \rvz}[z z^\top]\,.
$$
We assume that the distribution $\rvz$ is such that $\lin(\cA) \subseteq \text{range}(\tilde{P})$.
 Further let $P = \tilde{P}^{\dagger}$ be the pseudo-inverse of $\tilde{P}$. Note that both $P$ and $\tilde{P}$ are positive semi-definite matrices. We can equip our space with a new inner product
$\ip{\cdot}{P \cdot}$ and the resulting norm $\norm{\cdot}_P$. With this new dot product, for each $d\in\lin(\cA)$ we have:
$$
  \bbE_{z \sim \rvz} \encase{\ip{z}{P d}z} = \bbE_{z \sim \rvz} [ z z^\top]P d = P^{\dagger} P d = d.
$$
The last equality follows from our assumption that $\lin(\cA) \subseteq \text{range}(\tilde{P})$ and $P^{\dagger} P$ is an orthogonal projection operator onto the range of $\tilde{P}$.

The acceleration technique of \cite{stich2013optimization} works by optimizing two different quadratic subproblems at each iteration. The first is the regular smoothness upper bound. The second is a ``model'' of the function:
\begin{align}
\psi_{k+1}(x)  =
 \psi_k(x) + \alpha_{k+1} \Big( f(y_k) + \langle {z}_k^\top \nabla f(y_k) , {z}_k^\top P(x - y_k)  \rangle \Big) \,, \label{eq:model_acc_random_paper}
\end{align}
where
$\psi_0(x) = \frac{1}{2}\| x - x_0 \|_{P}^2$ and $z$ is sampled from $\rvz$.

\subsection{Analysis}

\begin{algorithm}[ht]
\begin{algorithmic}[1]
  \STATE \textbf{init} $\theta_0 = v_0 = y_0$, $\beta_0 = 0$, and $\nu$
  \STATE \textbf{for} {$k=0, 1 \dots K$}
     \STATE \qquad  Solve ${\alpha_{k+1}^2}{L\nu} = \beta_k + \alpha_{k+1}$
    \STATE  \qquad $\beta_{k+1} := \beta_k + \alpha_{k+1} $
    \STATE \qquad $\tau_k := \frac{\alpha_{k+1}}{\beta_{k+1}}$
   \STATE \qquad Compute $y_k := (1- \tau_k)\theta_k + \tau_k v_k$
   \STATE \qquad Find $z_k := \lmo_\cA(\nabla f(y_k))$
\STATE \qquad $\theta_{k+1} := y_k - \frac{\langle \nabla f(y_k),z_k\rangle}{L\|z_k\|_2^2}z_k$
\STATE \qquad  Sample $\tilde{z}_k  \sim \rvz$
\STATE \qquad $v_{k+1} := v_k - \alpha_{k+1}{\langle \nabla f(y_k),\tilde{z}_k\rangle}\tilde{z}_k$
  \STATE \textbf{end for}
\end{algorithmic}
 \caption{Accelerated Matching Pursuit \citet{locatello2018matching}}
 \label{algo:ACDM}
\end{algorithm}

\begin{algorithm}[ht]
\begin{algorithmic}[1]
    \STATE $\ldots$ as in Algorithm~\ref{algo:ACDM}, except replacing lines~7 to~10 by
   \STATE \qquad Sample $z_k \sim \rvz$
\STATE \qquad $\theta_{k+1} := y_k - \frac{\langle \nabla f(y_k),z_k\rangle}{L\|z_k\|^2_2}z_k$
\STATE \qquad $v_{k+1} := v_k - \alpha_{k+1}{\langle \nabla f(y_k),{z}_k\rangle}{z}_k$
  \STATE \textbf{end for}
\end{algorithmic}
 \caption{Accelerated Random Pursuit \citet{locatello2018matching}}
 \label{algo:RAMP}
\end{algorithm}

\begin{lemma}\label{lem:v-min-psi}
The update of $v$ in Algorithm~\ref{algo:ACDM} and~\ref{algo:RAMP} minimizes the model\vspace{-1mm}
  $$
    v_k \in \argmin_{\theta}\psi_k(\theta) \,.
  $$
\end{lemma}

For both the algorithm and the analysis we need a constant $\nu$ relating the geometry of the atom set with the sampling procedure (similar to $\hat\delta^2$ in Equation~\eqref{eq:delta-def-random-pursuit}):\vspace{-2mm}
   $$
  \nu  \leq \max_{d\in\lin(\cA)}
  \frac{\bbE\encase{ (\tilde{z}_k^\top d)^2 \norm{\tilde{z}_k}_P^2}\|{z(d)}\|_2^2}{(z(d)^\top d)^2} \, \quad \text{where} \quad z(d) = \lmo_\cA(- d)\,.
  $$

\begin{theorem} \label{thm:greedy_acc_pursuit}
Let $f$ be a convex function and $\cA$ be a symmetric compact set. Then the output of algorithm~\ref{algo:ACDM} for any $k\geq 1$ converges with the following rate:\vspace{-1mm}
$$
    \bbE[f(\theta_k)] - f(\theta^\star) \leq \frac{2L \nu}{k(k+1)}\norm{\theta^\star - \theta_0}^2_P \,.
  $$
\end{theorem}

From the rate of the greedy approach we can easily derive the rate for random pursuit:
\begin{theorem} \label{thm:random_pursuit_acc}
Let $f$ be a convex function and $\cA$ be a symmetric set. Then the output of the algorithm~\ref{algo:RAMP} for any $k\geq 1$ converges with the following rate:\vspace{-1mm}
$$\bbE [f(\theta_k)] - f(\theta^\star) \leq \frac{2L \nu'}{k(k+1)}\norm{\theta^\star - \theta_0}^2_P\,,\quad \text{where} \quad
\nu'  \leq \max_{d\in\lin(\cA)}
\frac{\bbE\encase{ ({z}_k^\top d)^2 \norm{{z}_k}_P^2}}{\bbE\encase{ ({z}_k^\top d)^2/ \norm{{z}_k}_2^2}}\,.
$$
\end{theorem}

\subsection{Accelerated Greedy Coordinate Descent.}
From our accelerated rate of random and matching pursuit, we can read the rates of random and greedy coordinate descent respectively by setting $\cA$ to the usual $\l1$ and $\rvz$ is uniformly distributed. In this case, algorithm \ref{algo:RAMP} reduces to the accelerated randomized coordinate method (ACDM) of \cite{Lee13,Nesterov:2017}. Instead, the accelerated MP algorithm yield a novel accelerated greedy coordinate descent method. The same rate was simultaneously (and independently) derived by \cite{lu2018greedy}, which was published at the same ICML conference. The rate of accelerated greedy coordinate descent is again, in the worst case, the same as random coordinate descent but can be faster up to a factor $\Th{n}$:
\begin{lemma}\label{lem:coordinate_descent_acc}
	When $\cA = \{e_i, i \in [n]\}$ and $\rvz$ is a uniform distribution over $\cA$, then $P = n I$, $\nu' = n$ and $\nu \in [1, n]$.
\end{lemma}

\section{Proofs}
\subsection{Proof of Lemma~\ref{lem:LwithRadius}}
\begin{proof}
Let $D(y,\theta):= f(y)- f(\theta) + \gamma \langle \nabla f(\theta), y - \theta \rangle$
By the definition of smoothness of $f$ w.r.t. $\|\cdot\|$,
$
D(y,\theta) \leq \frac{L}{2} \| y - \theta\|^2 \,.
$

Hence, from the definition of $L_\cA$,
\begin{align*}
L_\cA &\leq \sup_{\substack{\theta,y\in\lin(\cA)\\y = \theta + \gamma z \\\|z\|_\cA=1, \gamma\in \R_{>0}}}
 \frac{2}{\gamma^2} \frac{L}{2}  \| y - \theta\|^2 = L \sup_{z \ s.t. \|z\|_\cA = 1} \, \| z\|^2 = L \, \radius_{\norm{\cdot}}(\cA)^2 \ . \qedhere
\end{align*}
\end{proof}

\subsection{Proof of Theorem \ref{thm:sublinear_MP_rate}}

\begin{proof}

Recall that $\tilde{z}_k$ is the atom selected in iteration $k$ by the approximate \lmo  defined in \eqref{eqn:lmo_mult}. We start by upper-bounding $f$  using the definition of $L_\cA$ as follows:
\begin{align*}
f(\theta_{k+1}) &\leq  \min_{\gamma\in\R} f(\theta_k) + \gamma \langle \nabla f(\theta_k),
 \tilde{z}_k \rangle   + \frac{\gamma^2}{2} L_{\cA} \nonumber \\
&\leq   f(\theta_k) - \frac{\langle\nabla f(\theta_k),\tilde{z}_k\rangle^2}{2L_{\cA}} 
 =   f(\theta_k) - \frac{\langle\nabla_\parallel f(\theta_k),\tilde{z}_k\rangle^2}{2L_{\cA}}
\leq   f(\theta_k) - \delta^2\frac{\langle\nabla_\parallel f(\theta_k),z_k\rangle^2}{2L_{\cA}} \,. \nonumber
\end{align*}
Where $\nabla_\parallel f$ is the parallel component of the gradient wrt the linear span of $\cA$. Note that $\|d\|_{\cA*}:= \sup\left\lbrace \langle z,d\rangle, z\in\cA\right\rbrace$ is the dual of the atomic norm.
Therefore, by definition:
\begin{align*}
\langle\nabla_\parallel f(\theta_k),z_k\rangle^2 = \|-\nabla_\parallel f(\theta_k)\|^2_{\cA*} \,,
\end{align*}
which gives:
\begin{align*}
f(\theta_{k+1}) &\leq f(\theta_k) - \delta^2\frac{1}{2L_{\cA}} \|\nabla_\parallel f(\theta_k)\|^2_{\cA*}\leq f(\theta_k) - \delta^2\frac{1}{2L_{\cA}}\frac{ \left(-\langle\nabla_\parallel f(\theta_k),\theta_k-\theta^\star\rangle\right)^2}{R_\cA^2}\\
&= f(\theta_k) - \delta^2\frac{1}{2L_{\cA}}\frac{ \left(\langle\nabla_\parallel f(\theta_k),\theta_k-\theta^\star\rangle\right)^2}{R_\cA^2}\leq f(\theta_k) - \delta^2\frac{1}{2L_{\cA}}\frac{ \left(f(\theta_k)-f(\theta^\star)\right)^2}{R_\cA^2} \,,
\end{align*}
where the second inequality is Cauchy-Schwarz and the third one is convexity.
Which gives:
\begin{align*}
\varepsilon_{k+1} &\leq \frac{2L_{\cA}R_\cA^2}{\delta^2(k+2)} \,. \qedhere
\end{align*}
\end{proof}

\subsection{Proof of Theorem~\ref{thm:rand_purs_sub}}


\begin{proof}

Recall that $\tilde{z}_k$ is the atom selected in iteration $k$ by the approximate \lmo  defined in \eqref{eqn:lmo_mult}. We start by upper-bounding $f$  using the definition of $L_\cA$ as follows
\begin{eqnarray}
\bbE_z f(\theta_{k+1}) &\leq & \bbE_z\left[\min_{\gamma\in\R} f(\theta_k) + \gamma \langle \nabla f(\theta_k),
 z \rangle   + \frac{\gamma^2}{2} L_{\cA}\right] \leq  f(\theta_k) - \frac{\bbE_z\left[\langle\nabla f(\theta_k),z\rangle^2\right]}{2L_{\cA}} \nonumber\\
& = &  f(\theta_k) - \frac{\bbE_z\left[\langle\nabla_\parallel f(\theta_k),z\rangle^2\right]}{2L_{\cA}} \leq  f(\theta_k) - \hat\delta^2\frac{\langle\nabla_\parallel f(\theta_k),z_k\rangle^2}{2L_{\cA}} \,.\nonumber 
\end{eqnarray}
The rest of the proof proceeds as in Theorem~\ref{thm:sublinear_MP_rate}.
\end{proof}

\subsection{Proof of Lemma~\ref{thm:mu_mdw}}

\begin{proof}
First of all, note that for any $\theta,y\in\lin(\cA)$ with $\theta\neq y$ we have that:
\begin{align*}
\langle\nabla f(x),x-y\rangle^2\leq \|\nabla f(\theta)\|_{\cA*}^2\|\theta - y\|_{\cA}^2 \,.
\end{align*}
Therefore:
\begin{align*}
\mu_\cA &= \inf_{\substack{\theta,y\in\lin\cA\\\theta\neq y}}\frac{2}{\|y-\theta\|_{\cA}^2} D(\theta,y)
\geq \inf_{\substack{\theta,y,d\in\lin\cA\\\theta\neq y,d\neq 0}} \frac{\|d\|_{\cA*}^2}{\langle d,\theta-y\rangle^2}2D(\theta,y)\\
&\geq\inf_{\substack{\theta,y,d\in\lin\cA\\\theta\neq y,d\neq 0}}\frac{\|d\|_{\cA*}^2}{\langle d,\theta-y\rangle^2}\mu \|\theta-y\|^2
\geq\inf_{\substack{\theta,y,d\in\lin\cA\\\theta\neq y,d\neq 0}}\frac{\|d\|_{\cA*}^2}{\langle d,\frac{\theta-y}{\|\theta-y\|}\rangle^2}\mu\\
&\geq\inf_{\substack{\theta,y,d\in\lin\cA\\\theta\neq y,d\neq 0}}\frac{\|d\|_{\cA*}^2}{\|d\|^2}\mu
\geq\inf_{\substack{d\in\lin\cA\\d\neq 0}} \max_{z}\frac{\langle d,z\rangle^2}{\|d\|^2}\mu
=\mdw^2\mu \,. \qedhere
\end{align*}
\end{proof}

\subsection{Proof of Theorem~\ref{thm:linear_rate}}
\begin{proof}
\textbf{(Part 1).} 
We start by upper-bounding $f$  using the definition of $L_\cA$ as follows
\begin{eqnarray}
f(\theta_{k+1}) &\leq &  \min_{\gamma\in\R} f(\theta_k) + \gamma \langle \nabla f(\theta_k),
 \tilde{z}_k \rangle   + \frac{\gamma^2}{2} L_{\cA}  \leq  f(\theta_k) - \frac{\langle\nabla_\parallel f(\theta_k),\tilde{z}_k\rangle^2}{2L_{\cA}} \nonumber\\
 & \leq & f(\theta_k) - \delta^2\frac{\langle\nabla_\parallel f(\theta_k),z_k\rangle^2}{2L_{\cA}} =   f(\theta_k) - \delta^2\frac{\langle-\nabla_\parallel f(\theta_k),z_k\rangle^2}{2L_{\cA}} \,. \nonumber 
\end{eqnarray}
Where $\|d\|_{\cA*}:= \sup\left\lbrace \langle z,d\rangle, z\in\cA\right\rbrace$ is the dual of the atomic norm.
Therefore, by definition:
\begin{align*}
\langle-\nabla_\parallel f(\theta_k),z_k\rangle^2 = \|\nabla_\parallel f(\theta_k)\|^2_{\cA*} \,,
\end{align*}
which gives $
f(\theta_{k+1}) \leq f(\theta_k) - \delta^2\frac{1}{2L_{\cA}} \|\nabla_\parallel f(\theta_k)\|^2_{\cA*} \,.
$
From strong convexity we have that:
\begin{align*}
f(y)\geq f(\theta) + \langle\nabla f(\theta),y-\theta\rangle + \frac{\mu_\cA}{2}\|y-\theta\|_\cA^2 \,.
\end{align*}
Fixing $y = \theta_k + \gamma(\theta^\star - \theta_k)$ and $\gamma = 1$ in the LHS and minimizing the RHS we obtain:
\begin{align*}
f(\theta^\star)&\geq f(\theta_k) - \frac{1}{2\mu_\cA}\frac{\langle\nabla f(\theta_k),\theta^\star-\theta_k\rangle}{\|\theta^\star-\theta_k\|_\cA^2}\geq f(\theta_k) - \frac{1}{2\mu_\cA}\|\nabla_\parallel f(\theta_k)\|_{\cA^*}^2 \,,
\end{align*}
where the last inequality is obtained by the fact that $\langle\nabla f(\theta_k),\theta^\star-\theta_k\rangle=\langle\nabla_{\parallel} f(\theta_k),\theta^\star-\theta_k\rangle$ and  Cauchy-Schwartz.
Therefore $
\|\nabla f(\theta_k)\|_{\cA^*}\geq 2\varepsilon_k\mu_\cA \,,
$
which yields $$
\varepsilon_{k+1} \leq\varepsilon_k - \delta^2\frac{\mu_\cA}{L_{\cA}}\varepsilon_k\,. 
$$

\textbf{(Part 2).} 
We start by upper-bounding $f$  using the definition of $L_\cA$ as follows
\begin{eqnarray}
\bbE_z \left[f(\theta_{k+1})\right] &\leq &  \bbE_z\left[\min_{\gamma\in\R} f(\theta_k) + \gamma \langle \nabla f(\theta_k),
 \tilde{z}_k \rangle   + \frac{\gamma^2}{2} L_{\cA} \right]\nonumber \\
& \leq &  f(\theta_k) - \bbE_z\left[\frac{\langle\nabla f(\theta_k),\tilde{z}_k\rangle^2}{2L_{\cA}}\right] 
\leq   f(\theta_k) - \hat\delta^2\frac{\langle\nabla f(\theta_k),z_k\rangle^2}{2L_{\cA}}\nonumber \\
& = &  f(\theta_k) - \hat\delta^2\frac{\langle\nabla_\parallel f(\theta_k),\tilde{z}_k\rangle^2}{2L_{\cA}}\nonumber  =   f(\theta_k) - \hat\delta^2\frac{\langle-\nabla_\parallel f(\theta_k),\tilde{z}_k\rangle^2}{2L_{\cA}} \,. \nonumber
\end{eqnarray}
The rest of the proof proceeds as in Part 1 of the proof of Theorem~\ref{thm:linear_rate}.
\end{proof}

\subsection{Accelerated Matching Pursuit}\label{app:acc_match_pursuit}
Our proof follows the technique for acceleration from~\cite{Lee13,Nesterov:2017,Nesterov:2004gx,stich2013optimization} 

We define $\norm{\theta}^2_P = \theta^\top P\theta $. We start our proof by first defining the model function $\psi_k$. For $k=0$, we define:
$$\psi_{0}(\theta) = \frac{1}{2}\norm{\theta - v_0}_P^2 \,.$$
Then for $k >1$, $\psi_k$ is inductively defined as
\begin{align}
\psi_{k+1}(\theta) &= \psi_k(\theta) +  \alpha_{k+1} \Big( f(y_k) + \langle \tilde{z}_k^\top \nabla f(y_k) , \tilde{z}_k^\top P(\theta - y_k)  \rangle \Big) \,. \label{eq:model_acc_random}
\end{align}

\begin{proof}[\textbf{Proof of Lemma~\ref{lem:v-min-psi}}]
  We will prove the statement inductively. For $k=0$, $\psi_{0}(\theta) = \frac{1}{2}\norm{\theta - v_0}_P^2$ and so the statement holds. Suppose it holds for some $k \geq 0$. Observe that the function $\psi_k(\theta)$ is a quadratic with Hessian $P$. This means that we can reformulate $\psi_k(\theta)$ with minima at $v_k$ as
  $$
    \psi_k(\theta) = \psi_k(v_k) + \frac{1}{2}\norm{\theta - v_k}_P^2 \,.
  $$
  Using this reformulation,
  \begin{align*}
    \argmin_{\theta}\psi_{k+1}(\theta) &= \argmin_{\theta} \Big\{ \psi_k(\theta) +  \alpha_{k+1} \Big( f(y_k) + \langle \tilde{z}_k^\top \nabla f(y_k) , \tilde{z}_k^\top P(\theta - y_k)  \rangle \Big)\Big\}\\
    &= \argmin_{\theta} \Big\{ \psi_k(v_k) + \frac{1}{2}\norm{\theta - v_k}_P^2 +  \alpha_{k+1} \Big( f(y_k) + \langle \tilde{z}_k^\top \nabla f(y_k) , \tilde{z}_k^\top P(\theta - y_k)  \rangle \Big)\Big\}\\
    &= \argmin_{\theta} \Big\{ \frac{1}{2}\norm{\theta - v_k}_P^2 +  \alpha_{k+1}  \langle \tilde{z}_k^\top \nabla f(y_k) , \tilde{z}_k^\top P(\theta - v_k)  \rangle \Big\}\\
    &= v_k - \alpha_{k+1}{\langle \nabla f(y_k),\tilde{z}_k\rangle}\tilde{z}_k = v_{k+1}\,. \qedhere
  \end{align*}
\end{proof}
\begin{lemma}[Upper bound on $\psi_k(\theta)$]\label{lem:psi-upper}
  $$
    \bbE[\psi_k(\theta)] \leq \beta_k f(\theta) + \psi_0(\theta) \,.
  $$
\end{lemma}
\begin{proof}
  We will also show this through induction. The statement is trivially true for $k=0$ since $\beta_0 = 0$. Assuming the statement holds for some $k \geq 0$,
  \begin{align*}
    \bbE[\psi_{k+1}(\theta)] &= \bbE\Big[\psi_k(\theta) +  \alpha_{k+1} \Big( f(y_k) + \langle \tilde{z}_k^\top \nabla f(y_k) , \tilde{z}_k^\top P(\theta - y_k)  \rangle \Big)\Big]\\
    &= \bbE\Big[\psi_k(\theta)\Big] +  \alpha_{k+1}\bbE\Big[ \Big( f(y_k) + \langle \tilde{z}_k^\top \nabla f(y_k) , \tilde{z}_k^\top P(\theta - y_k)  \rangle \Big)\Big] \\
    &\leq \beta_kf(\theta) + \psi_0(\theta) +  \alpha_{k+1} \Big( f(y_k) + \nabla f(y_k)^\top \bbE\Big[ \tilde{z}_k \tilde{z}_k^\top\Big] P(\theta - y_k)  \rangle \Big)\\
    &= \beta_kf(\theta) + \psi_0(\theta) +  \alpha_{k+1} \Big( f(y_k) + \nabla f(y_k)^\top P^{\dagger} P (\theta - y_k)  \rangle \Big)\\
    &= \beta_kf(\theta) + \psi_0(\theta) +  \alpha_{k+1} \Big( f(y_k) + \nabla f(y_k)^\top (\theta - y_k)  \rangle \Big)\\
    &\leq  \beta_kf(\theta) + \psi_0(\theta) +  \alpha_{k+1} f(\theta)\,.
  \end{align*}
  In the above, we used the convexity of the function $f(\theta)$ and the definition of $P$.
\end{proof}
\begin{lemma}[Bound on progress]\label{lem:progress-bound}
	For any $k\geq 0$ of algorithm \ref{algo:ACDM},
	$$
		f(\theta_{k+1}) - f(y_k) \leq - \frac{1}{2L \norm{z_k}^2_2}\nabla f(y_k)^\top\encase{ z_kz_k^\top} \nabla f(y_k)\,.
	$$
\end{lemma}
\begin{proof}
	  The update $\theta_{k+1}$ along with the smoothness of $f(\theta)$ guarantees that for $\gamma_{k+1} = \frac{\langle \nabla f(y_k),z_k\rangle}{L\|z_k\|^2}$,
	\begin{align*}
	f(\theta_{k+1}) &= f(y_k + \gamma_{k+1} z_k) \leq f(y_k) + \gamma_{k+1}\ip{\nabla f(y_k)}{z_k} + \frac{L\gamma_{k+1}^2}{2}\norm{z_k}^2 \\
	&= f(y_k) - \frac{1}{2L \norm{z_k}^2_2}\nabla f(y_k)^\top\encase{ z_kz_k^\top} \nabla f(y_k)\,.
	\end{align*}
\end{proof}
\begin{lemma}[Lower bound on $\psi_k(\theta)$]\label{lem:psi-lower}
Given a filtration $\cF_k$ upto time step $k$, $$
  \bbE[\min_{\theta}\psi_k(\theta)| \cF_k] \geq \beta_k f(\theta_k) \,.
$$
\end{lemma}
\begin{proof}
  This too we will show inductively. For $k=0$, $\psi_k(\theta) = \frac{1}{2}\norm{\theta - v_0}_P^2 \geq 0$ with $\beta_0 =0$. Assume the statement holds for some $k \geq 0$. Recall that $\psi_k(\theta)$ has a minima at $v_k$ and can be alternatively formulated as $\psi_k(v_k) + \frac{1}{2}\norm{\theta - v_k}_P^2$. Using this,
  \begin{align*}
    \psi_{k+1}^\star &= \min_{\theta}\encase{\psi_k(\theta) + \alpha_{k+1}\brac{\ip{{\tilde{z}}_k^\top\nabla f(y_k)}{{\tilde{z}}_k^\top P(\theta - y_k)} + f(y_k)}}\\
    &= \min_{\theta}\encase{\psi_k(v_k) + \alpha_{k+1}\brac{\ip{{\tilde{z}}_k^\top\nabla f(y_k)}{{\tilde{z}}_k^\top P(\theta - y_k)} + \frac{1}{2 \alpha_{k+1}}\norm{\theta - v_k}_P^2 + f(y_k)}}\\
    &= \psi_k^\star + \alpha_{k+1} f(y_k) + \alpha_{k+1} \min_{\theta}\encase{\ip{P {\tilde{z}}_k{\tilde{z}}_k^\top\nabla f(y_k)}{\theta - y_k} + \frac{1}{2 \alpha_{k+1}}\norm{\theta - v_k}_P^2 }\,.
  \end{align*}
  Since we defined $y_k = (1- \tau_k)\theta_k + \tau_k v_k$, rearranging the terms gives us that
  $$
    y_k - v_k = \frac{1 - \tau_k}{\tau_k}(\theta_k - y_k)\,.
  $$
  Let us take now compute $\bbE[\psi_{k+1}^\star|\cF_k]$ by combining the above two equations:
  \begin{align*}
    \bbE[\psi_{k+1}^\star|\cF_k] &= \psi_k^\star + \alpha_{k+1} f(y_k) + \frac{\alpha_{k+1} (1 - \tau_k)}{\tau_k} \ip{P \bbE_k[\tilde{z}_k\tilde{z}_k^\top]\nabla f(y_k)}{y_k - \theta_k} \\ &\hspace{1in} + \alpha_{k+1}\bbE_k \min_{\theta}\encase{\ip{P \tilde{z}_k\tilde{z}_k^\top\nabla f(y_k)}{\theta - v_k} + \frac{1}{2 \alpha_{k+1}}\norm{\theta - v_k}_P^2 }\\
    &= \psi_k^\star + \alpha_{k+1} f(y_k) + \frac{\alpha_{k+1} (1 - \tau_k)}{\tau_k} \ip{\nabla f(y_k)}{y_k - \theta_k} \\ &\hspace{1in} + \alpha_{k+1} \bbE_k \min_{\theta}\encase{\ip{P \tilde{z}_k\tilde{z}_k^\top\nabla f(y_k)}{\theta - v_k} + \frac{1}{2 \alpha_{k+1}}\norm{\theta - v_k}_P^2 }\\
    &= \psi_k^\star + \alpha_{k+1} f(y_k) + \frac{\alpha_{k+1} (1 - \tau_k)}{\tau_k} \ip{\nabla f(y_k)}{y_k - \theta_k} \\ &\hspace{1in} - \frac{\alpha_{k+1}^2}{2} \nabla f(y_k)^\top \bbE_k\encase{ \tilde{z}_k\tilde{z}_k^\top P P^{-1} P \tilde{z}_k\tilde{z}_k^\top}\nabla f(y_k)\\
    &= \psi_k^\star + \alpha_{k+1} f(y_k) + \frac{\alpha_{k+1} (1 - \tau_k)}{\tau_k} \ip{\nabla f(y_k)}{y_k - \theta_k} \\ &\hspace{1in} - \frac{\alpha_{k+1}^2}{2} \nabla f(y_k)^\top \bbE_k\encase{ \tilde{z}_k\tilde{z}_k^\top P \tilde{z}_k\tilde{z}_k^\top}\nabla f(y_k)\,.
  \end{align*}
  
  Let us define a constant $\nu \geq 0$ such that it is the smallest number for which the below inequality holds for all $k$,
  $$
  \nu \nabla f(y_k)^\top\frac{\encase{ z_kz_k^\top}}{2L \norm{z_k}^2_2} \nabla f(y_k) \geq
  \nabla f(y_k)^\top \bbE\encase{ \tilde{z}_k\tilde{z}_k^\top P \tilde{z}_k\tilde{z}_k^\top}\nabla f(y_k)\,.
  $$

  Also recall from Lemma \ref{lem:progress-bound} that
  $$
  f(\theta_{k+1}) - f(y_k) \leq - \frac{1}{2L \norm{z_k}^2_2}\nabla f(y_k)^\top\encase{ z_kz_k^\top} \nabla f(y_k)\,.
  $$
  Using the above two statements in our computation of $\psi_{k+1}^\star$, we get
  \begin{align*}
  \bbE[\psi_{k+1}^\star | \cF_k] &= \psi_k^\star + \alpha_{k+1} f(y_k) + \frac{\alpha_{k+1} (1 - \tau_k)}{\tau_k} \ip{\nabla f(y_k)}{y_k - \theta_k} \\ &\hspace{1in} - \frac{\alpha_{k+1}^2}{2} \nabla f(y_k)^\top \bbE_k\encase{ \tilde{z}_k\tilde{z}_k^\top P \tilde{z}_k\tilde{z}_k^\top}\nabla f(y_k) \\
  &\geq \psi_k^\star + \alpha_{k+1} f(y_k) + \frac{\alpha_{k+1} (1 - \tau_k)}{\tau_k} \ip{\nabla f(y_k)}{y_k - \theta_k} \\ &\hspace{1in} - \frac{\alpha_{k+1}^2 \nu}{2} \nabla f(y_k)^\top\encase{ z_kz_k^\top} \nabla f(y_k)\\
  &\geq \psi_k^\star + \alpha_{k+1} f(y_k) + \frac{\alpha_{k+1} (1 - \tau_k)}{\tau_k} \ip{\nabla f(y_k)}{y_k - \theta_k} \\ &\hspace{1in} + {\alpha_{k+1}^2 L\nu} (f(\theta_{k+1}) - f(y_k))\\
  &\geq \psi_k^\star + \alpha_{k+1} f(y_k) + \frac{\alpha_{k+1} (1 - \tau_k)}{\tau_k} (f(y_k) - f(\theta_k)) \\ &\hspace{1in} + {\alpha_{k+1}^2 L\nu} (f(\theta_{k+1}) - f(y_k))\,.
  \end{align*}

  Let us pick $\alpha_{k+1}$ such that it satisfies ${\alpha_{k+1}^2}{\nu L} = \beta_{k+1}$. Then the above equation simplifies to
  \begin{align*}
  \bbE[\psi_{k+1}^\star| \cF_k] &\geq \psi_k^\star + \frac{\alpha_{k+1}}{\tau_k} f(y_k) - \frac{\alpha_{k+1} (1 - \tau_k)}{\tau_k} f(\theta_k) + \beta_{k+1} (f(\theta_{k+1}) - f(y_k))\\
  &= \psi_k^\star - \beta_k f(\theta_k) + \beta_{k+1}f(y_k) - \beta_{k+1}f(y_k) + \beta_{k+1} f(\theta_{k+1}) \\
  &= \psi_k^\star - \beta_k f(\theta_k) + \beta_{k+1} f(\theta_{k+1}) \,.
  \end{align*}

  We used that $\tau_k = \alpha_{k+1}/\beta_{k+1}$. Finally we use the inductive hypothesis to conclude that
  \begin{align*}
  \bbE[\psi_{k+1}^\star| \cF_k] &\geq \psi_k^\star - \beta_k f(\theta_k) + \beta_{k+1} f(\theta_{k+1}) \geq \beta_{k+1} f(\theta_{k+1})\,. \qedhere
  \end{align*}
\end{proof}

\begin{lemma}[Final convergence rate]\label{lem:acc-mp-rate}
	For any $k \geq 1$ the output of algorithm \ref{algo:ACDM} satisfies:
  $$
    \bbE[f(\theta_k)] - f(\theta^\star) \leq \frac{2L \nu}{k(k+1)}\norm{\theta^\star - \theta_0}^2_P \,.
  $$
\end{lemma}
\begin{proof}

Putting together Lemmas \ref{lem:psi-upper} and \ref{lem:psi-lower}, we have that
$$
  \beta_k \bbE[f(\theta_k)] \leq \bbE[\psi_k^\star] \leq \bbE[\psi_k(\theta^\star)] \leq \beta_k f(\theta^\star) + \psi_0(\theta^\star)\,.
$$
Rearranging the terms we get
$$
  \bbE[f(\theta_k)] - f(\theta^\star) \leq \frac{1}{2 \beta_k}\norm{\theta^\star - \theta_0}^2_P\,.
$$
To finish the proof of the theorem, we only have to compute the value of $\beta_k$.
Recall that
$$
  {\alpha_{k+1}^2}{L\nu} = \beta_k + \alpha_{k+1} \,.
$$
We will inductively show that $\alpha_k \geq \frac{k}{2L \nu}$. For $k=0$, $\beta_0 =0$ and $\alpha_1 = \frac{1}{2L \nu}$ which satisfies the condition. Suppose that for some $k\geq 0$, the inequality holds for all iterations $i \leq k$. Recall that $\beta_k = \sum_{i=1}^k \alpha_i$ i.e. $\beta_k \geq \frac{k(k+1)}{4 L \nu}$. Then
\begin{align*}
  (\alpha_{k+1}L\nu)^2 - \alpha_{k+1}L \nu = \beta_k L \nu
  \geq \frac{k(k+1)}{4}\,.
\end{align*}
The positive root of the quadratic $x^2 - x - c = 0$ for $c \geq 0$ is $x = \frac{1}{2}\brac{1 + \sqrt{4c + 1}}$. Thus
$$
  \alpha_{k+1}L \nu \geq \frac{1}{2}\brac{1 + \sqrt{k(k+1) + 1}} \geq \frac{k + 1}{2}\,.
$$
This finishes our induction and proves the final rate of convergence.
\end{proof}

\begin{lemma}[Understanding $\nu$]
  $$
    \nu \leq \max_{d\in\lin(\cA)}
  \frac{\bbE\encase{ (\tilde{z}_k^\top d)^2 \norm{\tilde{z}_k}_P^2}\|{z(d)}\|_2^2}{(z(d)^\top d)^2} \,,
  $$
\end{lemma}
\begin{proof}
  Recall the definition of $\nu$ as a constant which satisfies the following inequality for all iterations $k$
  $$
  \nu \nabla f(y_k)^\top\frac{\encase{ z_kz_k^\top}}{2L \norm{z_k}^2_2} \nabla f(y_k) \geq
  \nabla f(y_k)^\top \bbE\encase{ \tilde{z}_k\tilde{z}_k^\top P \tilde{z}_k\tilde{z}_k^\top}\nabla f(y_k)\,.
  $$
which then yields the following sufficient condition for $\nu$:
  $$
  \nu  \leq \max_{d\in\lin(\cA)}
  \frac{\bbE\encase{ (\tilde{z}_k^\top d)^2 \norm{\tilde{z}_k}_P^2}\|{z(d)}\|_2^2}{(z(d)^\top d)^2} \,,
  $$
  where $z(d)$ is defined to be\vspace{-2mm}
  $$
  z(d) =  \lmo_\cA(- d)\,.
  $$
\end{proof}
\paragraph{Proof of Theorem \ref{thm:random_pursuit_acc}}
The proof of Theorem \ref{thm:random_pursuit_acc} is exactly the same as that of the previous except that now the update to $v_k$ is also a random variable. The only change needed is the definition of $\nu'$ where we need the following to hold:
$$
\nu' \nabla f(y_k)^\top\frac{1}{2L}\bbE_k\encase{ z_kz_k^\top/{\norm{z_k}^2_2}} \nabla f(y_k) \geq
\nabla f(y_k)^\top \bbE\encase{ {z}_k{z}_k^\top P {z}_k{z}_k^\top}\nabla f(y_k)\,.
$$ \qedhere

\paragraph{Proof of Lemma \ref{lem:coordinate_descent_acc}}
When $\cA = \{e_i, i \in [n]\}$ and $\cZ$ has a uniform distribution over $\cA$, then $\tilde{P} = 1/n I$ and $P = nI$. A simple computation shows that $\nu' = n$ and $\nu \in [1, n]$. Note that here $\nu$ could be upto $n$ times smaller than $\nu'$ meaning that our accelerated greedy coordinate descent algorithm could be $\sqrt{n}$ times faster than the accelerated random coordinate descent. In the worst case $\nu = \nu'$, but in practice one can pick a smaller $\nu$ compared to $\nu'$ as the worst case gradient rarely happen. It is possible to tune $\nu$ and $\nu'$ empirically but we do not explore this direction.


\chapter{Convex Cones: Non-Negative Matching Pursuit}\label{cha:cone}
In this chapter, we consider the problem of minimizing a convex function over convex cones. The presented approach is based on \citep{locatello2017greedy} and was developed in collaboration with Michael Tschannen, Gunnar R\"atsch, and Martin Jaggi. In this dissertation, we present a new unpublished proof for the sublinear rate. The new proof is shorter and more elegant than the one in \citep{locatello2017greedy} despite the rate being essentially the same. 

\section{Problem Formulation}
In this chapter we study greedy projection-free optimization of smooth convex function over convex cones: 

\begin{equation}
\label{eqn:NNMPtemplate}
\underset{\theta\in\cone{A}}{\text{minimize}} \quad f(\theta).
\end{equation}
We consider the following setting: \\
$\triangleright~~\cH$ is a Hilbert space $\cH$ with associated inner product $\ip{x}{y} \,\forall \, x,y \in \cH$ and induced norm $\| x \|^2 := \ip{x}{x},$ $\forall \, x \in \cH$. \\
$\triangleright~~\cA \subset \cH$ is a compact set (the ``set of atoms'' or dictionary) in $\cH$. As opposed to chapter \ref{cha:mpcd}, we do not assume that $\cA$ is symmetric.
\\$\triangleright$ Without loss of generality, we point the cone to the origin and assume that $0 \in \cA$ (see next paragraph). \\
$\triangleright~~f \colon \cH \to \R$ is a convex and $L$-smooth ($L$-Lipschitz gradient in the finite dimensional case) function. If $\cH$ is an infinite-dimensional Hilbert space, then $f$ is assumed to be Fr{\'e}chet differentiable. \\
$\triangleright$ We specifically focus on projection-free algorithms that maintain the iterate as a non-negative linear combination of elements from $\cA$.

\paragraph{Convex Cone}
The cone $\cone(\cA-y)$ tangent to the convex set $\conv(\cA)$ at a point $y$  is formed by the half-lines emanating from $y$ and intersecting $\conv(\cA)$ in at least one other point. 
Without loss of generality we consider $0\in\cA$ and consider the set $\cone(\cA)$ (\ie with $y=0$) to be closed. If $\cA$ is finite, the cone constraint can be written as
$$
\cone(\cA):=\lbrace x: x = \sum_{i=1}^{|\cA|}\alpha_i z_i \;\;  \text{s.t.} \;  z_i\in\cA, \ \alpha_i\geq 0 ~\forall i\rbrace\!.
$$

This setting is theoretically interesting as it represents an intermediate case between Frank-Wolfe and Matching Pursuit. Non-negative combinations of atoms are very natural in some applications \cite{esser2013method, gillis2016fast,behr2013mitie,makalic2011logistic,berry2007algorithms, kim2012fast}. On the other hand, the convergence properties of these heuristic algorithms remain unclear.

\section{A Simple Non-Negative MP Method}
First, we present a simple method in Algorithm \ref{algo:NNMP} and its convergence on general convex functions. Unfortunately, this method does not enjoy a linear rate for strongly convex objectives. We will fix this issue using corrective variants in Section \ref{sec:opt/corrective}. Note that the algorithm definition is not affine invariant but can be made so replacing $L\|z_k\|^2$ by $L_{\cone(\cA)}$ defined as in Equation \eqref{eqn:cone_L}. 

The main difference between Algorithm \ref{algo:NNMP} and the general greedy template of Algorithm~\ref{algo:generalgreedy} is how we query the \lmo. We do not only look for the steepest descent direction in $\cA$ but also add an iteration-dependent atom $-\frac{\theta_k}{\|\theta_k\|_\cA}$ to the set of possible search directions. Note that this only makes sense when $\theta_k \neq 0$; in that case, no direction is added. Further, we stress that $\|y\|_\cA$ is only well defined if $y\in\cone(\cA)$. Since the set $\cA$ contains the origin on a corner, $\|\cdot\|_\cA$ is not a valid norm in general.
Intuitively, one can think that if the \lmo selects an atom in $\cA$, the update will always increase its weight. Otherwise, it will shrink all non-zero weights taking a step towards the origin.

\begin{algorithm}[H]
\caption{Non-Negative Matching Pursuit \citep{locatello2017greedy}}
\label{algo:NNMP}
\begin{algorithmic}[1]
  \STATE \textbf{init} $\theta_{0} = 0 \in \cA$
  \STATE \textbf{for} {$k=0\dots K$}
  \STATE \quad Find $\bar{z}_k := (\text{Approx-}) \lmo_{\cA}(\nabla f(\theta_k))$
  \STATE \quad $z_k = \argmin_{z\in \left\lbrace\bar{z}_k,\frac{-\theta_k}{\|\theta_k\|_\cA}\right\rbrace}\langle \nabla f(\theta_k),z\rangle$ 
  \STATE \quad $\gamma := \frac{\langle -\nabla f(\theta_k), z_k   \rangle}{L\|z_k\|^2}$
  \STATE \quad   Update $\theta_{k+1}:= \theta_k + \gamma z_k$
  \STATE \textbf{end for}
\end{algorithmic}
\end{algorithm}

Recall the alignment assumption from \citet{pena2015polytope} discussed in Section~\ref{sec:back_NNMP}. Our variation on the \lmo follows this assumption. In fact, if $\theta_k$ is not an optimum of the optimization problem in Equation \eqref{eqn:NNMPtemplate} and if $\min_{z\in\cA}\langle \nabla f(\theta_k), z\rangle = 0$, the vector $-\frac{\theta_k}{\|\theta_k\|_\cA}$ is aligned with $\nabla f(\theta_k)$ in the sense that $\langle \nabla f(\theta_k),-\frac{\theta_k}{\|\theta_k\|_\cA} \rangle < 0$. This effectively ensure that the algorithm does not stop unless $\theta_k$ is a solution of the optimization problem.

Formally, we define the set of feasible descent directions of Algorithm~\ref{algo:NNMP} at a point $\theta\in \cone(\cA)$ as:
\begin{equation}
T_\cA(\theta)  :=  \left\lbrace  d\in\cH\!: \exists z\in\cA\cup \Big\lbrace  -\frac{\theta}{\|\theta\|_\cA} \Big\rbrace \, \ \text{s.t.} \; \langle d,z\rangle<0  \right\rbrace .
\end{equation}
Intuitively, $T_\cA(\theta)$ is the set of all possible directions for which the Algorithm would not terminate if the iterate were $\theta$. 
If at some iteration $k$ the gradient $\nabla f(\theta_k)$ is not in $T_\cA(\theta_k)$, then Algorithm~\ref{algo:NNMP} terminates as $\min_{z\in\cA}\langle d,z\rangle = 0$ and $\langle d,-\theta_k\rangle\geq 0$ (which yields $z_k = 0$). This is expected as in that case $\theta_k$
 would be optimal.

\begin{lemma}\label{lemma:originOpt}
If $\tilde\theta \in \cone(\cA)$ and $\nabla f(\tilde\theta)\not\in T_\cA$ then $\tilde\theta$ is a solution to $\min_{\theta\in\cone(\cA)}f(\theta)$.
\begin{proof}
By contradiction, assume $\theta^\star\neq \tilde\theta$ and $\nabla f(\tilde\theta)\not\in T_\cA$. Now, by convexity of $f$ we have:
\begin{align*}
f(\theta^\star)\geq f(\tilde\theta) + \langle \nabla f(\tilde\theta),\theta^\star-\tilde\theta\rangle
\end{align*}
Since $\theta^\star\neq \tilde\theta$ we have also that $f(\theta^\star)<f(\tilde\theta)$. Therefore:
\begin{align*}
0< f(\tilde\theta) - f(\theta^\star)\leq  \langle - \nabla f(\tilde\theta),\theta^\star-\tilde\theta\rangle
\end{align*}
which we rewrite as $\langle  \nabla f(\tilde\theta),\theta^\star\rangle+\langle  \nabla f(\tilde\theta),-\tilde\theta\rangle<0$. Now we note that
by the assumption that $\nabla f(\tilde\theta)\not\in T_\cA$ we have that both these inner products are non negative which is absurd.
To draw this conclusion note that $\theta^\star\in \cone(\cA)$ we have that $\theta^\star = \sum_i \alpha_i z_i$ where $z_i\in\cA$ and $\alpha_i\geq 0 \ \forall \ i $. 
\end{proof}
\end{lemma}
Notably, $\theta$ remains feasible since $\theta_0 = 0$. The reason is that we compute $\gamma$ minimizing the quadratic upper bound on $f$ given by smoothness. Since $f$ is convex, for $k>0$ we have $f(\theta_k)\leq f(0)$. Hence, the minimum of the upper bound lies in between $\theta_k$ and the origin.

For the analysis, let $$R_{\cone(\cA)}^2 := \max_{\substack{\theta\in\cone(\cA)\\ f(\theta)\leq f(\theta_0)}}\|\theta-\theta^\star\|^2_{\amxnok}$$ and 
\begin{equation}\label{eqn:cone_L}
L_{\cone(\cA)} := \!\!\!\sup_{\substack{\theta,y\in\cone(\cA)\\y = \theta + \gamma z \\\|z\|_{\amxnok}=1, \gamma\in \R_{>0}}}\frac{2}{\gamma^2}\big[ f(y)- f(\theta) -  \ip{\nabla f(\theta)}{y - \theta} \big] \,.
\end{equation}
Note that when $\theta=0$ the norm $\|y\|_{\cA}$ is well defined only for all $y\in\cone(\cA)$ and $y\neq 0$ (see Figure~\ref{fig:gface}). This is not an issue because $\theta_0 = 0$. 
\begin{theorem}
Let $\cA \subset \cH$ be a bounded set, $f$ be convex and $L_{\cone(\cA)}$ smooth and let $R_{\cone(\cA)}$ be the radius of the level set of $\theta_0$ both measured with the atomic norm.
Then, Algorithm~\ref{algo:NNMP} converges for $k \geq 1$ as 
\[
\varepsilon_{k+1}\leq \frac{2L_{\cone(\cA)} R_{\cone(\cA)}^2}{\delta^2(k+2)} ,
\]
where $\delta \in (0,1]$ is the relative accuracy parameter of the employed approximate \lmo in Equation~\eqref{eqn:lmo_mult}.
\end{theorem}
As opposed to other theorems, we report the full proof in this main section as it is a unpublished, new, and independent proof compared to the one in \citep{locatello2017greedy}.
\begin{proof}
Recall that $\tilde{z}_k\in\cA\cup\left\lbrace -\frac{\theta_k}{\|\theta_k\|_\cA}\right\rbrace$ is the atom selected in iteration $k$ by the approximate \lmo. This ensures that the search direction is always negatively correlated with the gradient and hence is a descent direction and the algorithm does not stop unless at the constrained minimum. Furthermore, the quadratic upper bound holds also if $z_k = -\frac{\theta_k}{\|\theta_k\|_{\cA}}$ as it has atomic norm of one. 
We start by upper-bounding $f$ using the definition of $L_{\cone(\cA)}$:
\begin{eqnarray}
f(\theta_{k+1}) &\leq &  \min_{\gamma\in\R} f(\theta_k) + \gamma \langle \nabla f(\theta_k), 
 \tilde{z}_k \rangle   + \frac{\gamma^2}{2} L_{\cone(\cA)} \nonumber \\
& \leq &  f(\theta_k) - \frac{\langle\nabla f(\theta_k),\tilde{z}_k\rangle^2}{2L_{cone(\cA)}}  =   f(\theta_k) - \frac{\langle\nabla_\parallel f(\theta_k),\tilde{z}_k\rangle^2}{2L_{\cone(\cA)}}\nonumber \\
& \leq &  f(\theta_k) - \delta^2\frac{\langle\nabla_\parallel f(\theta_k),\tilde{z}_k\rangle^2}{2L_{\cone(\cA)}}\nonumber \\
& = & f(\theta_k) - \delta^2\frac{1}{2L_{\cone(\cA)}} \|\nabla_\parallel f(\theta_k)\|^2_{\amx*}
\end{eqnarray}
Where we used that $\|d\|_{\amx	*}:= \sup\left\lbrace \langle z,d\rangle, z\in\amx\right\rbrace$ is the dual of the atomic norm.
 Using Lemma \ref{lemma:originOpt}, it is easy to show that:
  \begin{align*}
\|\nabla_\parallel f(\theta_k) \|^2_{\amx*} &=\|\nabla_\parallel f(\theta_k)-\nabla_\parallel f(\theta^\star) \|^2_{\amx*}
\end{align*}
which in turn yields:
\begin{align*}
f(\theta_{k+1}) &\leq f(\theta_k) - \delta^2\frac{1}{2L_{\cone(\cA)}} \|\nabla_\parallel f(\theta_k)-\nabla_\parallel f(\theta^\star)\|^2_{\amx*}\\
&\leq f(\theta_k) - \delta^2\frac{1}{2L_{\cone(\cA)}}\frac{ \langle\nabla_\parallel f(\theta_k),\theta_k-\theta^\star\rangle^2}{R_{\cone(\cA)}^2}\\
&\leq f(\theta_k) - \delta^2\frac{1}{2L_{\cone(\cA)}}\frac{ \left(f(\theta_k)-f(\theta^\star)\right)^2}{R_{\cone(\cA)}^2}\\
\end{align*}
Which gives the rate after solving the recursion:
\begin{align*}
\varepsilon_{k+1}\leq \frac{2L_{\cone(\cA)} R_{\cone(\cA)}^2}{\delta^2(k+2)}
\end{align*}%
\end{proof}

\paragraph{Difference form the proof in \citep{locatello2017greedy}} In this dissertation, we presented a new proof for the sublinear rate of Algorithm \ref{algo:NNMP}. This proof is significantly simpler than the one in the original publication and allows for an affine invariant algorithm without the circular dependency in the smoothness definition described in Chapter \ref{cha:mpcd}. The original proof separates the cases where the weights increase or decrease, which the new proof avoids using the properties of atomic norms.

\paragraph{Limitations of Algorithm \ref{algo:NNMP}:}
Let us call \textit{active} the atoms which have nonzero weights in the representation of $\theta_k = \sum_{i=0}^{k-1} \alpha_i z_i$ computed by Algorithm~\ref{algo:NNMP}. Formally, the set of active atoms is defined as $\cS := \{ z_i \colon \alpha_i > 0, i = 0,1,\ldots,k-1 \}$.
The main drawback of Algorithm~\ref{algo:NNMP} is that when the direction $-\frac{\theta_k}{\|\theta_k\|_\cA}$ is selected, the weight of \textit{all} active atoms is reduced. 
This uniform reduction can lead the algorithm to alternately select $-\frac{\theta_k}{\|\theta_k\|_\cA}$ and an atom from $\cA$, thereby slowing down convergence in a similar manner as the \textit{zig-zagging} phenomenon well-known in the Frank-Wolfe framework \cite{LacosteJulien:2015wj}. A visualization of this issue is presented in Figure \ref{fig:zig_zag_cone} (left).

\section{Corrective Variants}\label{sec:opt/corrective}
To achieve linear convergence in the strongly convex case, we introduce corrective variants of Algorithm~\ref{algo:NNMP} that are inspired by corresponding approaches in FW \cite{frank1956algorithm,LacosteJulien:2015wj}.  Our algorithms are the Away-steps Non-Negative MP (ANNMP) and Pairwise Non-Negative MP (PWNNMP), presented in Algorithm~\ref{algo:ANNMPPWNNMP} and the fully-corrective Non-Negative MP (FCNNMP) in Algorithm \ref{algo:FCNNMP}.Now, we query the \lmo a second time on the active set $\cS$ (the set of atoms with non-zero weight in the iterate) to find the direction of steepest \textit{ascent}. We use this information to selectively ``reduce'' the weight on this atom  or ``swap'' it with the steepest descent direction. A visualization of the away step is presented in Figure \ref{fig:zig_zag_cone} (right).
\begin{figure}
    \centering
    \includegraphics[width=0.6\textwidth]{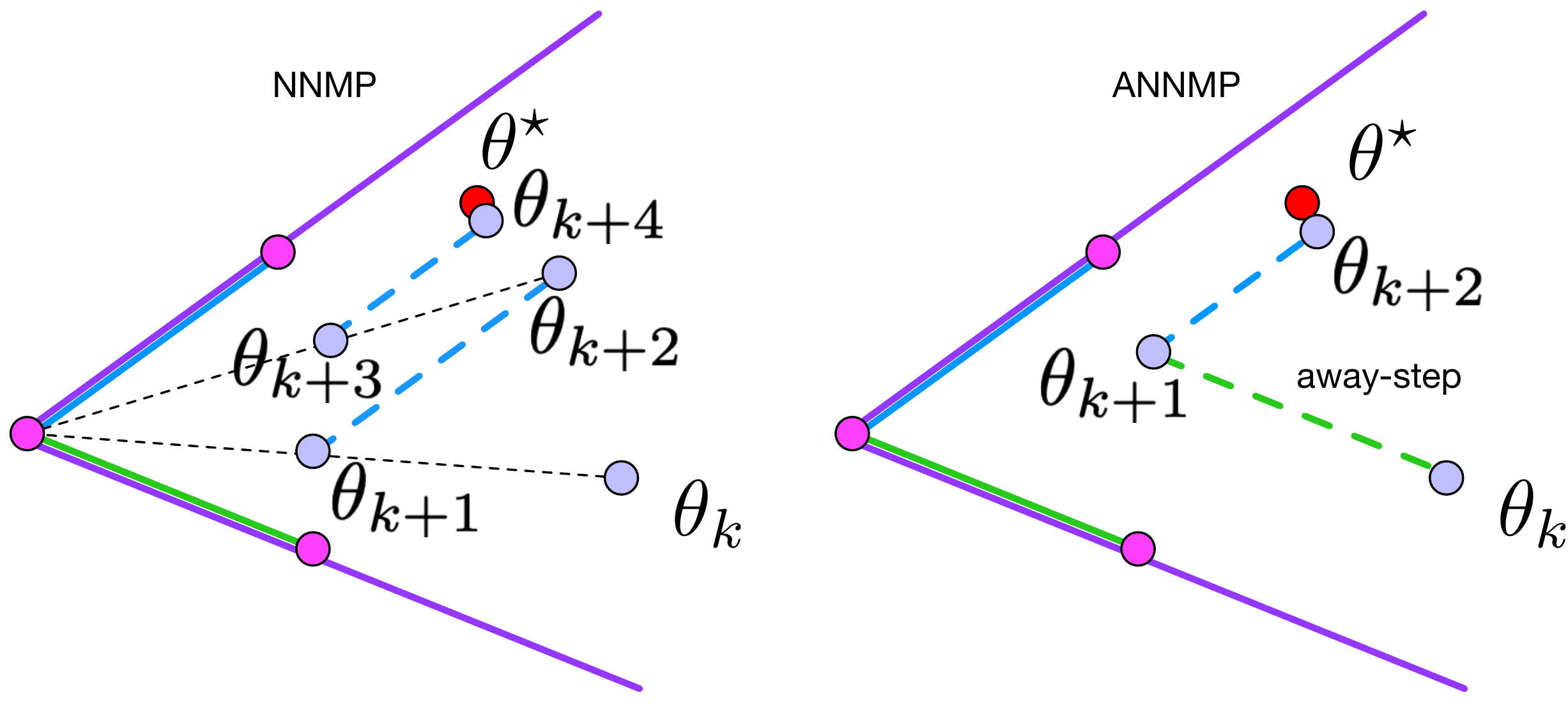}
    \caption{\small Visualization of the \textit{zig-zagging} phenomenon of the NNMP. As the vanilla algorithm cannot selectively reduce the weights of the iterate it takes more step to reach the optimum. Instead, the ANNMP allows to take an away-step that selectively reduce the weight of an atom (the green atom in the visualization) which  leads to faster convergence.}
    \label{fig:zig_zag_cone}
\end{figure}
\looseness=-1At each iteration, Algorithm~\ref{algo:ANNMPPWNNMP} updates the weights of $z_k$ and $v_k$ as $\alpha_{z_k}=\alpha_{z_k}+\gamma$ and $\alpha_{v_k}=\alpha_{v_k}-\gamma$, respectively. To ensure that $\theta_{k+1}\in\cone(\cA)$, $\gamma$ has to be clipped according to the weight which is currently on $v_k$, i.e., $\gamma_{\max}=\alpha_{v_k}$. If $\gamma = \gamma_{\max}$, we set $\alpha_{v_k}=0$ and remove $v_k$ from $\cS$ as the atom~$v_k$ is no longer active. If $d_k \in\cA$ (i.e., we take a regular MP step and not an away step), the line search is unconstrained (i.e., $\gamma_{\max}=\infty$).
Note that while the second $\lmo$ query increases the computational complexity, an exact search may. be feasible in practice as at most $k$ atoms can be active at iteration $k$. 
\looseness=-1Taking an additional computational burden, one can update the weights of all active atoms in the spirit of Orthogonal Matching Pursuit \citep{Mallat:1993gu}. This approach is implemented in the Fully Corrective MP (FCNNMP), Algorithm \ref{algo:FCNNMP}.

\begin{algorithm}[H]
\caption{Away-steps (ANNMP) and Pairwise (PWNNMP) Non-Negative MP \citep{locatello2017greedy}}
\label{algo:ANNMPPWNNMP}
\begin{algorithmic}[1]
  \STATE \textbf{init} $\theta_{0} = 0 \in \cA$, and $\cS:=\{\theta_{0}\}$
  \STATE \textbf{for} {$k=0\dots K$}
  \STATE \quad Find $z_k := (\text{Approx-}) \lmo_{\cA}(\nabla f(\theta_{k}))$
  \STATE \quad Find $v_k := (\text{Approx-}) \lmo_{\cS}(-\nabla f(\theta_{k}))$
  \STATE \quad $\cS = \cS \cup z_k$
  \STATE \quad \textit{ANNMP:} $d_k \!=\! \argmin_{d\in\lbrace z_k,-v_k\rbrace} \!\langle\nabla f(\theta_k),d \rangle\!$
  \STATE \quad \textit{PWNNMP:} $d_k = z_k-v_k$
  \STATE \quad $\gamma := \min\left\lbrace\frac{\langle -\nabla f(\theta_k),d_k\rangle}{L\|d_k\|^2},\gamma_{\max}\right\rbrace$ \\ \qquad($\gamma_{\max}$ see text)
  \STATE \quad Update $\alpha_{z_k}$, $\alpha_{v_k}$ and $\cS$ according to $\gamma$ \\ \qquad ($\gamma$ see text)
  \STATE \quad   Update $\theta_{k+1}:= \theta_k + \gamma d_k$
  \STATE \textbf{end for}
\end{algorithmic}
\end{algorithm}

\begin{algorithm}[H]
\caption{Fully Corrective Non-Negative Matching Pursuit (FCNNMP) \citep{locatello2017greedy}}
\label{algo:FCNNMP}
\begin{algorithmic}[1]
  \STATE \ldots as Algorithm~\ref{algo:ANNMPPWNNMP}, except replacing lines 6-10 with
\STATE \quad \textit{Variant 0 (norm corrective):} \\
 \qquad $\theta_{k+1} \! =\! \underset{{\theta\in\cone(\cS)}}{\argmin} \! \,\|\theta - (\theta_k -\frac{1}{L}\nabla f(\theta_k))\|^2\!$
\STATE \quad \textit{Variant 1:} \\ \qquad $\theta_{k+1} = \argmin_{\theta\in\cone(\cS)} f(\theta)$
\STATE \quad Remove atoms with zero weights from $\cS$
\end{algorithmic}
\end{algorithm}

At each iteration, Algorithm \ref{algo:FCNNMP} maintains the set of active atoms $\cS$ (\ie the atoms with non-zero weight forming the iterate) by adding $z_k$ and removing atoms with zero weights after the update. 
In Variant 0, the algorithm minimizes the quadratic upper bound $g_{\theta_{k}}(\theta)$ on $f$ at $\theta_{k}$ imitating a gradient descent step with projection onto a ``varying'' target, i.e., $\cone(\cS)$.  
In Variant~1, the original objective $f$ is minimized over $\cone(\cS)$ at each iteration, which is in general more efficient than minimizing $f$ over $\cone(\cA)$ if $k$ is small. 

For $f(\theta) = \frac{1}{2}\|y  -\theta\|^2$, $y  \in \cH$, Variant 1 recovers Algorithm~1 in \cite{Yaghoobi:2015ff} and the OMP variant in \cite{bruckstein2008uniqueness} which both only apply to this specific objective.
\subsection{Linear Rates}
The linear rate case is interesting, as it turns out that the rate is surprisingly closer in spirit to the one of Frank-Wolfe \citep{LacosteJulien:2015wj} rather than the one of MP \citep{locatello2017unified}. To illustrate these relations, we only discuss the non-affine invariant linear rate in this dissertaiton  and refer to \citep{locatello2017greedy} for the remaining affine invariant case.

We start by recalling some of the geometric complexity quantities that were introduced in the context of FW by \citet{LacosteJulien:2015wj}
and are adapted here to the optimization problem we aim to solve (minimization over $\cone(\cA)$ instead of $\conv(\cA)$). 
\vspace{-2mm}
\paragraph{Directional Width {\normalfont\cite{LacosteJulien:2015wj}}}
The directional width of a set $\cA$ w.r.t. a direction $r\in\cH$ is defined as:
\begin{align}
dirW(\cA,r):=\max_{s,v\in\cA}\big\langle\tfrac{r}{\|r\|},s-v\big\rangle
\end{align}
\vspace{-5mm}
\paragraph{Pyramidal Directional Width  {\normalfont\cite{LacosteJulien:2015wj}}}
The Pyramidal Directional Width of a set $\cA$ with respect to a direction $r$ and a reference point $\theta\in\conv(\cA)$ is defined as:
\begin{align}
PdirW(\cA,r,\theta) := \min_{\cS\in\cS_\theta}dirW(\cS\cup
\lbrace s(\cA,r)\rbrace, r),
\end{align}
where $\cS_\theta := \lbrace \cS \ |\ \cS\subset \cA$ and $\theta$ is a proper convex combination of all the elements in $\cS\rbrace$ and $s(\cA,r) := \max_{s\in\cA}\langle\frac{r}{\|r\|},s\rangle$.

Inspired by the notion of pyramidal width in \cite{LacosteJulien:2015wj}, which is the minimal pyramidal directional width computed over the set of feasible directions, we now define the cone width of a set $\cA$ where only the generating faces ($\gfaces$) of $\cone(\cA)$ (instead of the faces of $\conv(\cA)$) are considered. Before doing so we introduce the notions of \textit{face}, \textit{generating faces} ($\gfaces$), and \textit{feasible direction}.
\paragraph{Face of a Convex Set {\normalfont\cite{LacosteJulien:2015wj}}} Let us consider a set $\cK$ with a $k-$dimensional affine hull along with a point $\theta\in\cK$. Then, $\cK$ is a $k-$dimensional face of $\conv(\cA)$ if $\cK = \conv(\cA) \cap \lbrace y \colon \langle r,y-\theta\rangle = 0\rbrace$ for some normal vector $r$ and $\conv(A)$ is contained in the half-space determined by $r$, i.e., $\langle r,y-\theta\rangle\leq 0$, $\forall \ y\in\conv(\cA)$.
Intuitively, given a set $\conv(\cA)$ one can think of  $\conv(\cA)$ being a $\mathrm{dim}(\conv(\cA))-$dimensional face of itself, an edge on the border of the set a $1$-dimensional face and a vertex a $0$-dimensional face.
\vspace{-2mm}
\paragraph{Face of a Cone and \normalfont{$\gfaces$}}
\looseness=-1Similarly, a $k-$dimensional face of a cone is an open and unbounded set $\cone(\cA) \cap \lbrace y \colon \langle r,y-\theta\rangle = 0\rbrace$ for some normal vector $r$ and $\cone(A)$ is contained in the half space determined by $r$.  We can define the generating faces of a cone as:
\begin{align*}
\gfaces(\cone(\cA)) \! := \! \left\lbrace \cB \cap \conv(\cA)\colon \! \cB\in \faces(\cone(\cA))\right\rbrace.
\end{align*}
Note that $\gfaces(\cone(\cA))\subset \faces(\conv(\cA))$ and $\conv(\cA)\in\gfaces(\cone(\cA))$.  
Furthermore, for each $\cK\in\gfaces(\cone(\cA))$, $\cone(\cK)$ is a $k-$dimensional face of $\cone(\cA)$. An intuitive visualization of $\gfaces$ is depicted in Figure \ref{fig:gface} (right).

\begin{figure}
    \begin{center}
    \includegraphics[width=0.34\textwidth]{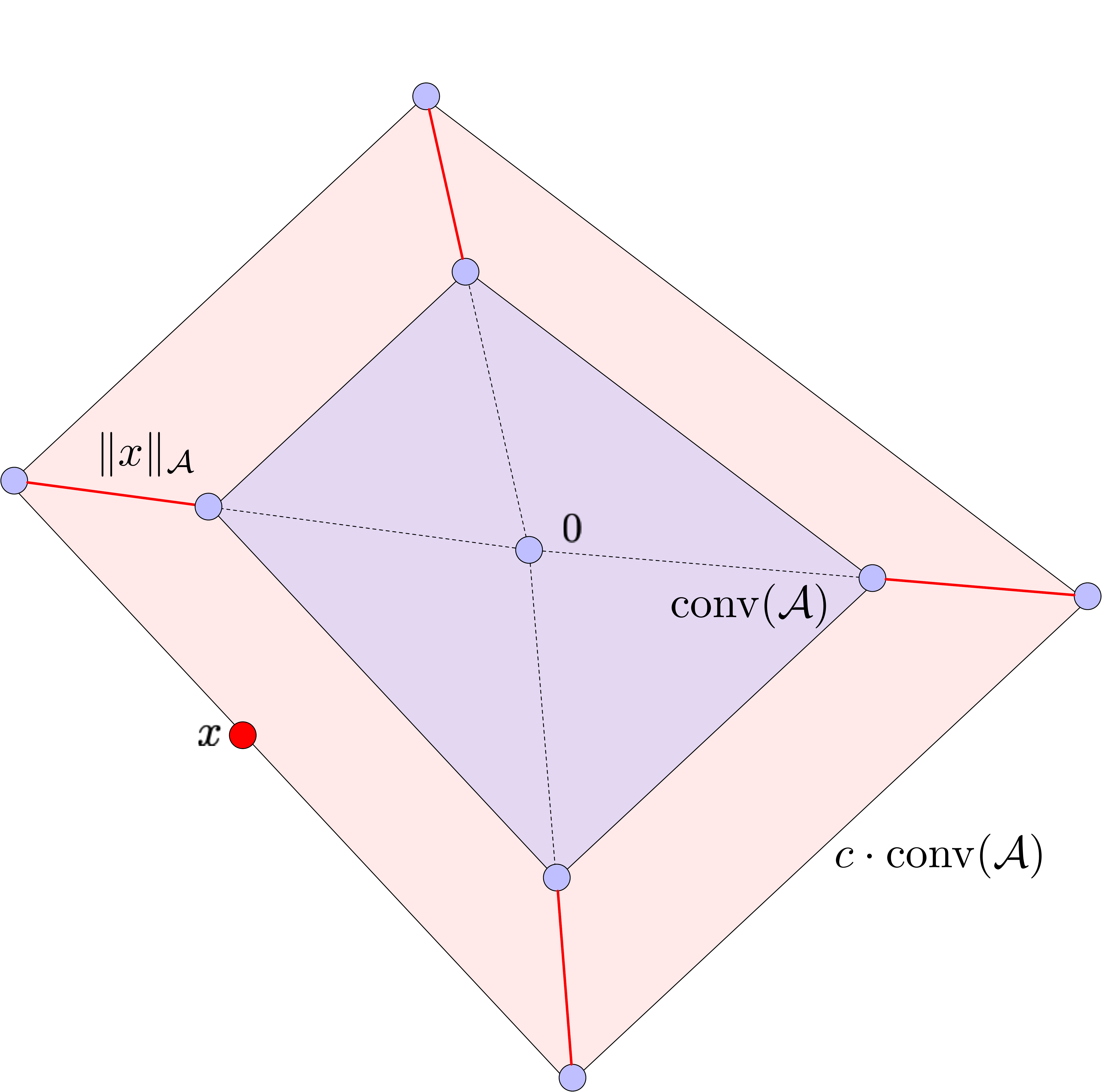}%
    \hspace{3em}
    \includegraphics[scale=0.4]{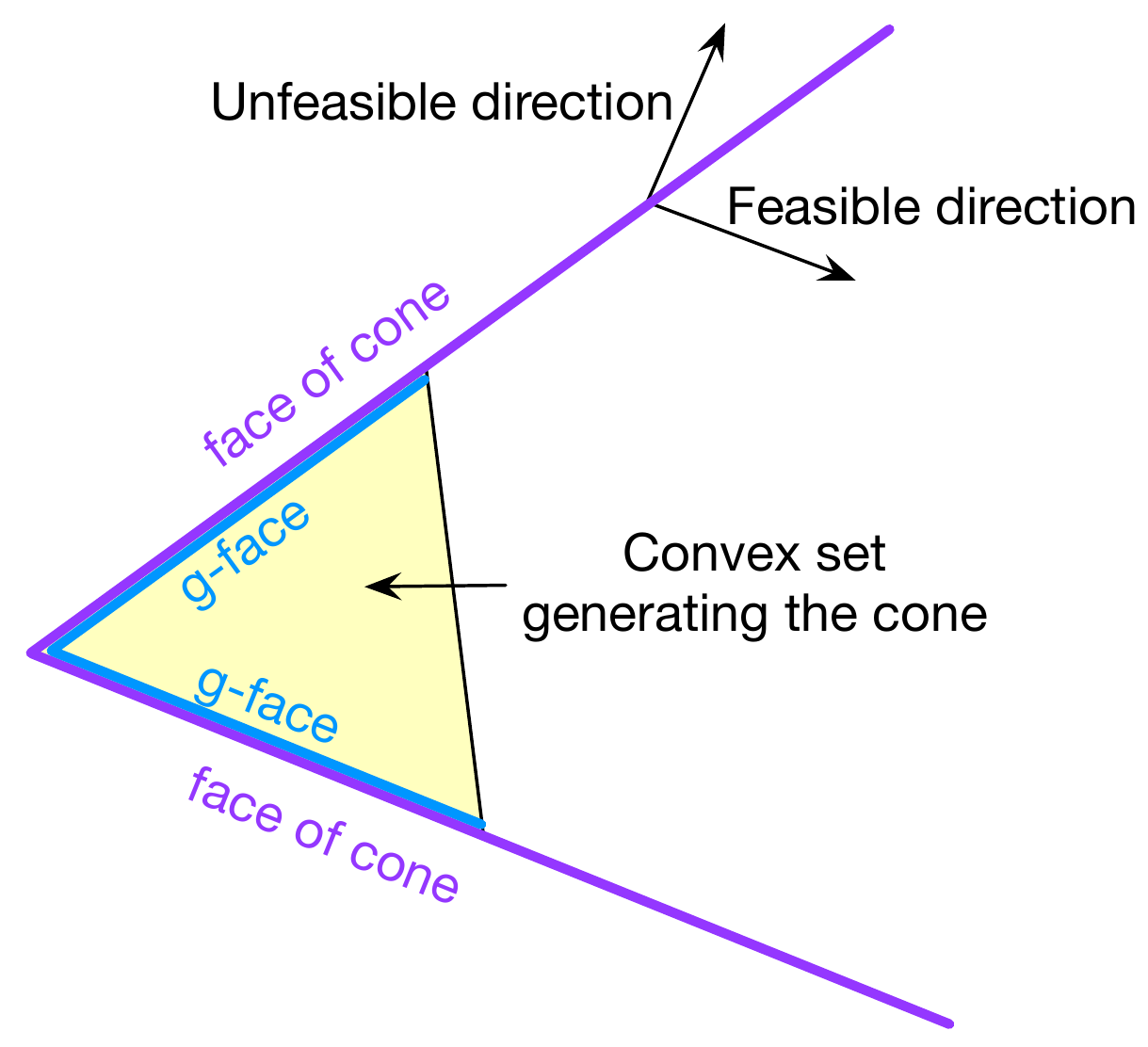}
    \end{center}
    \caption{\small (left) Visualization of the atomic norm. It corresponds to the smallest constant $c$ so that a point $x$ belong to $c\cdot\conv(\cA)$. If the origin is not in the relative interior the norm is not well defined everywhere. (right) Visualization of the $\gfaces$ of a $\cone$ generated by a convex set.}
    \label{fig:gface}
\end{figure}

We now adapt the notion of \textbf{feasible directions} from \citep{LacosteJulien:2015wj} to the cone setting. A direction $d$ is feasible from $\theta \in \cone(\cA)$ if it points inwards $\cone(\cA)$, i.e., if $\exists \varepsilon > 0$ s.t. $\theta + \varepsilon d \in \cone(\cA)$.
Since a face of the cone is itself a cone, if a direction is feasible from $\theta\in\cone(\cK)\setminus 0$, it is feasible from every positive rescaling of $\theta$. We therefore can consider only the feasible directions on the generating faces (which are closed and bounded sets). Finally, we define the cone width of $\cA$.

\paragraph{Cone Width}\vspace{-1mm}
\begin{align} \label{def:conewidth}
\cw:= \min_{\substack{\cK\in \gfaces(\cone(\cA))\\ \theta\in \cK \\ r\in\cone(\cK-\theta)\setminus \lbrace 0\rbrace}} PdirW(\cK\cap\cA,r,\theta)
\end{align}

We are now ready to show the linear convergence of Algorithms~\ref{algo:ANNMPPWNNMP} and \ref{algo:FCNNMP}.
\begin{theorem}\label{thm:PWNNMPlinear}
\looseness=-1Let $\cA \subset \cH$ be a bounded set with $0\in\cA$ and let the objective function $f \colon \cH \to \R$ be both $L$-smooth and $\mu$-strongly convex. 
Then, the suboptimality of the iterates of Algorithms~\ref{algo:ANNMPPWNNMP} and \ref{algo:FCNNMP} decreases geometrically at each step in which $\gamma < \alpha_{v_k}$ (henceforth referred to as ``good steps'') as:
\begin{equation} \label{eq:linrate}
\varepsilon_{k+1}
\leq \left(1- \beta \right)\varepsilon_{k},
\end{equation}
where $\beta := \delta^2 \frac{\mu \cw^2}{L\diam(\cA)^2}\in (0,1]$,
$\varepsilon_k := f(\theta_k) - f(\theta^\star)$ is the suboptimality at step $k$ and $\delta \in (0,1]$ is the relative accuracy parameter of the employed approximate \lmo \eqref{eqn:lmo_mult}. For ANNMP (Algorithm~\ref{algo:ANNMPPWNNMP}), $\beta^{\text{ANNMP}} = \beta/2$. If $\mu = 0$ Algorithm~\ref{algo:ANNMPPWNNMP} converges with rate $\Or{1/K(k)}$ where $K(k)$ is the number of ``good steps'' up to iteration k.
\end{theorem}
\vspace{-2mm}
\paragraph{Bad Steps} To obtain a linear convergence rate, one needs to upper-bound the number of ``bad steps'' $k-K(k)$ (i.e., steps with $\gamma \geq \alpha_{v_k}$). We have that $K(k)=k$ for Variant 1 of FCNNMP (Algorithm~\ref{algo:FCNNMP}), $K(k)\geq k/2$ for ANNMP (Algorithm~\ref{algo:ANNMPPWNNMP}) and $K(k) \geq t/(3|\cA|!+1)$ for PWNNMP (Algorithm~\ref{algo:ANNMPPWNNMP}) and Variant 0 of FCNNMP (Algorithm~\ref{algo:FCNNMP}). This yields a global linear convergence rate of $\varepsilon_k\leq \varepsilon_0 \exp\left(-\beta K(k)\right)$. The bound for PWNNMP is very loose and only meaningful for finite sets $\cA$. Further note that Variant 1 of FCNNMP (Algorithm \ref{algo:FCNNMP}) does not produce bad steps. Also note that the bounds on the number of good steps given above are the same as for the corresponding FW variants and are obtained using the same (purely combinatorial) arguments as in \cite{LacosteJulien:2015wj}.

\paragraph{Relation to previous MP rates}
The linear convergence of the generalized (not non-negative) MP variants studied in  \cite{locatello2017unified} crucially depends on the geometry of the set which is characterized by the Minimal Directional Width $\mdw$:
\begin{align}
\mdw := \min_{ \substack {d\in\lin(\cA)\\ d \neq 0}}\max_{z\in\cA}\langle \frac{d}{\|d\|},z\rangle \ . 
\end{align}
The following Lemma relates the Cone Width with the minimal directional width.
\begin{lemma} \label{lem:mdw}
If the origin is in the relative interior of $\conv(\cA)$ with respect to its linear span, then $\cone(\cA)=\lin(\cA)$ and $\cw= \mdw$.
\end{lemma}
Now, if the set $\cA$ is symmetric or, more generally, if $\cone(\cA)$ spans the linear space $\lin(\cA)$ (which implies that the origin is in the relative interior of $\conv(\cA)$), there are no bad steps. Hence, by Lemma \ref{lem:mdw}, the linear rate obtained in Theorem~\ref{thm:PWNNMPlinear} for non-negative MP variants generalizes the one presented in \cite[Theorem~7]{locatello2017unified} for generalized MP variants.

\vspace{-2mm}
\paragraph{Relation to FW rates}
Optimization over conic hulls with non-negative MP is more similar to FW than to MP itself in the following sense. For MP, every direction in $\lin(\cA)$ allows for unconstrained steps, from any iterate $\theta_k$. In contrast, for our non-negative MPs, while some directions allow for unconstrained steps from some iterate $\theta_k$, others are constrained, thereby leading to the dependence of the linear convergence rate on the cone width, a geometric constant which is very similar in spirit to the Pyramidal Width appearing in the linear convergence bound in \cite{LacosteJulien:2015wj} for FW. Furthermore, as for Algorithm~\ref{algo:ANNMPPWNNMP}, the linear rate of Away-steps and Pairwise FW holds only for good steps. We finally relate the cone width with the Pyramidal Width \cite{LacosteJulien:2015wj}. The Pyramidal Width is defined as 
\begin{align*}
\mathrm{PWidth}(\cA):= \min_{\substack{\cK\in \faces(\conv(\cA))\\ \theta\in \cK \\ r\in\cone(\cK-\theta)\setminus \lbrace 0\rbrace}} PdirW(\cK\cap\cA,r,\theta).
\end{align*}
We have $\cw \geq \mathrm{PWidth}(\cA)$ as the minimization in the definition \eqref{def:conewidth} of $\cw$ is only over the subset $\gfaces(\cone(\cA))$ of $\faces(\conv(\cA))$. As a consequence, the decrease per iteration characterized in Theorem~\ref{thm:PWNNMPlinear} is larger than what one could obtain with FW on the rescaled convex set $\tau\cA$. Furthermore, the decrease characterized in \cite{LacosteJulien:2015wj} scales as $1/\tau^2$ due to the dependence on $1/\diam(\conv(\cA))^2$. 


\section{Empirical Convergence}
We illustrate the performance of the presented algorithms on the exemplary task of finding the projection of a point in a convex cone. 

\paragraph{Synthetic data} 
We consider minimizing the least-squares objective on the conic hull of~100 unit-norm vectors sampled at random in the first orthant of $\R^{50}$. We compare the convergence of Algorithms~\ref{algo:NNMP}, \ref{algo:ANNMPPWNNMP}, and \ref{algo:FCNNMP} with the Fast Non-Negative MP (FNNOMP) of~\cite{Yaghoobi:2015ff}, and Variant 3 (line-search) of the FW algorithm in \cite{locatello2017unified} on the atom set rescaled by $\tau = 10 \|y\|$, observing linear convergence for our corrective variants. 
\begin{figure}
\centering
\includegraphics[width=0.7\linewidth]{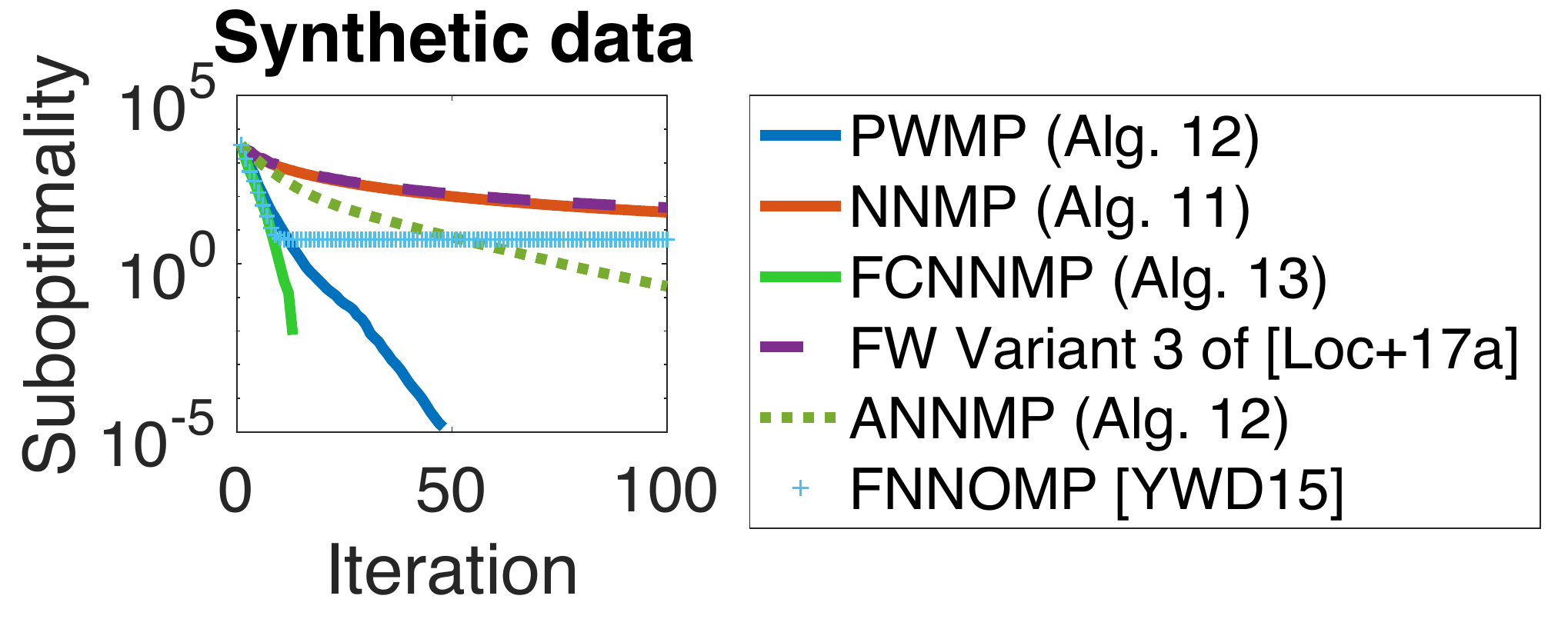}
\caption{\small Synthetic data experiment solving a least squares projection on a randomly generated cone. These curves are averages over 20 random cones and target vectors. We observe that as the objective is strongly convex, our FCNNMP, PWNNMP and ANNMP achieve a linear convergence rate.}\label{exp:synth}
\end{figure}
Figure~\ref{exp:synth} shows the suboptimality $\varepsilon_k$, averaged over 20 realizations of $\cA$ and $y$, as a function of the iteration $t$. As expected, FCNNMP achieves fastest convergence followed by PWNNMP, ANNMP and NNMP. 
The FNNOMP gets stuck instead. Indeed, \citep{Yaghoobi:2015ff} only show that their algorithm terminates and not that it converges.

\section{Proofs}
\subsection{Proof of Theorem~\ref{thm:PWNNMPlinear}}
\begin{proof}
Let us consider the case of PWNNMP.

Consider the atoms $\tilde{z}_k\in\cA$ and $\tilde{v}_k\in\cS$ selected by the \lmo at iteration $k$.
Due to the smoothness property of $f$ it holds that: 
\begin{align*}
f(\theta_{k+1})&\leq \min_{\gamma\in \mathbb{R}} f(\theta_{k}) + \gamma\langle\nabla f(\theta_{k}), \tilde{z}_k - \tilde{v}_k\rangle \\&+ \frac{L}{2}\gamma^2\|\tilde{z}_k - \tilde{v}_k\|^2.
\end{align*}
for a good step (i.e. $\gamma < \alpha_{v_k}$). Note that this also holds for variant 0 of Algorithm~\ref{algo:FCNNMP}.

We minimize the upper bound with respect to $\gamma$ setting $\gamma = -\frac{1}{L}\langle\nabla f(\theta_{k}),\frac{\tilde{z}_k - \tilde{v}_k}{\|\tilde{z}_k - \tilde{v}_k\|^2}\rangle$ . Subtracting $f(\theta^\star)$ from both sides and replacing the optimal $\gamma$ yields:
\begin{equation}\label{step:linearRateMPsmoothfast}
\varepsilon_{k+1}\leq\varepsilon_{k}-\frac{1}{2L}\left\langle\nabla f(\theta_{k}),\frac{\tilde{z}_k - \tilde{v}_k}{\|\tilde{z}_k-\tilde{v}_k\|}\right\rangle^2
\end{equation}
Now writing the definition of strong convexity, we have the following inequality holding for all  $\gamma\in\R$:
\begin{align*}
f(\theta_{k}+\gamma(\theta^\star-\theta_{k}))\geq f(\theta_{k})+\gamma\langle\nabla f(\theta_{k}),\theta^\star-\theta_{k}\rangle+\\\quad\gamma^2\frac{\mu}{2}\|\theta^\star-\theta_{k}\|^2
\end{align*}
We now fix $\gamma=1$ in the LHS and minimize with respect to $\gamma$ in the RHS:
\begin{align*}
\varepsilon_{k}\leq \frac{1}{2\mu} \left\langle \nabla f(\theta_{k}), \frac{\theta^\star-\theta_{k}}{\|\theta^\star-\theta_{k}\|} \right\rangle^2
\end{align*}
Combining this with \eqref{step:linearRateMPsmoothfast} yields: 
\begin{align}\label{step:linearRateCombinedfast}
\varepsilon_{k}-\varepsilon_{k+1}\geq \frac{\mu}{L}\frac{\big\langle \nabla f(\theta_{k}),\frac{\tilde{z}_k-\tilde{v}_k}{\|\tilde{z}_k-\tilde{v}_k\|}\big\rangle^2}{\big\langle \nabla f(\theta_{k}), \frac{\theta^\star-\theta_{k}}{\|\theta^\star-\theta_{k}\|}\big\rangle^2}\varepsilon_{k}
\end{align}%
We now use Theorem~\ref{thm:widthBound} to conclude the proof.
For Away-steps MP the proof is trivially extended since $2\min_{z\in\cA\cup-\cS}\langle \nabla f(\theta_k),z\rangle\leq \min_{z\in\cA, v\in\cS}\langle \nabla f(\theta_k),z-v\rangle$. Therefore, we obtain the same smoothness upper bound of the PWNNMP. The rest of the proof proceed as for PWNNMP with the additional $\frac12$ factor.

\paragraph{Sublinear Convergence for $\mu = 0$}
If $\mu = 0$ we have for PWNNMP:
\begin{align}\label{eq:PWsubBound}
f(\theta_{k+1})&\leq \min_{\gamma\leq \alpha_{v_k}} f(\theta_{k}) + \gamma\langle\nabla f(\theta_{k}), \tilde{z}_k - \tilde{v}_k\rangle \\&+ \frac{L}{2}\gamma^2\|\tilde{z}_k - \tilde{v}_k\|^2.
\end{align}
which can be rewritten for a good step (i.e. no clipping is necessary) as: 
\begin{eqnarray*}
 \varepsilon_{k+1} &\leq \varepsilon_{k} + \min_{\gamma\in[0,1]}\left\lbrace - \frac{\delta}{2} \gamma \varepsilon_{k} + \frac{\gamma^2}{2}L\rho^2\diam(\cA)^2\right\rbrace.\\
\end{eqnarray*}  
Unfortunately, $\alpha_{v_k}$ limits the improvement. On the other hand, we can repeat the induction only for the good steps.
Therefore:
\begin{eqnarray*}
 \varepsilon_{k+1}
 & \leq \varepsilon_{k} - \frac{2}{\delta' k+2}\delta' \varepsilon_{k} + \frac{1}{2}\left(\frac{2}{\delta' k+2}\right)^2 C,
\end{eqnarray*}
where we set $\delta' := \delta/2$, $C=L\rho^2\diam(\cA)^2$ and used $\gamma = \frac{2}{\delta' k+2} \in [0,1]$ (since it is a good step this produce a valid upper bound). Finally, we show by induction
 \begin{equation*}
 \varepsilon_k \leq \frac{4\left(\frac{2}{\delta} C + \varepsilon_0\right)}{t+4} = 2\frac{\left(\frac{1}{\delta'} C + \varepsilon_0\right)}{\delta' K(k)+2}
 \end{equation*}
where $K(k) \geq 0$ is the number of good steps at iteration $k$. 

When $K(k)=0$ we get $\varepsilon_0\leq \left(\frac{1}{\delta'}C+\varepsilon_0\right)$. Therefore, the base case holds. We now prove the induction step assuming $\varepsilon_k \leq \tfrac{2\left(\frac{1}{\delta'}C+\varepsilon_0\right)}{\delta' K(k)+2}$ as :
\begin{align*}
\varepsilon_{k+1} &\leq \left(1-\tfrac{2\delta'}{\delta' K(k) + 2}\right)\varepsilon_{k} + \tfrac12 C \left(\tfrac{2}{\delta' K(k)+2}\right)^2\\
&\leq \left(1-\tfrac{2\delta'}{\delta' K(k) + 2}\right)\tfrac{2\left(\frac{1}{\delta'}C+\varepsilon_0\right)}{\delta' K(k)+2} \\
&\quad+ \tfrac{1}{2}\left(\tfrac{2}{\delta' K(k)+2}\right)^2C + \tfrac{2}{(\delta' K(k)+2)^2}\delta'\varepsilon_0\\
&= \tfrac{2\left(\frac{1}{\delta'}C+\varepsilon_0\right)}{\delta' K(k)+2}\left(1-\tfrac{2\delta'}{\delta' K(k) +2}+\tfrac{\delta'}{\delta' K(k) +2}\right)\\
&\leq \tfrac{2\left(\frac{1}{\delta'}C+\varepsilon_0\right)}{\delta'(K(k)+1)+2}.
\end{align*}
For AFW the procedure is the same but the linear term of Equation~\ref{eq:PWsubBound} is divided by two. We proceed as before with the only difference that we call $\delta'=\delta/4$. \qedhere
\end{proof}

\subsection{Linear Convergence of FCNNMP}

\begin{proof}
The proof is trivial noticing that:
\begin{align*}
f(\theta_{k+1}) &=\min_{\theta\in\cone(\cS\cup s(\cA,r))}f(\theta)\\
&\leq \min_{\theta\in\cone(\cS\cup s(\cA,r))}g_{\theta_k}(\theta)\\
&\leq \min_{\gamma\leq\alpha_{v_k}}g_{\theta_k}(\theta_k+\gamma(z_k-v_k))
\end{align*}
which is the beginning of the proof of Theorem~\ref{thm:PWNNMPlinear}. Note that there are no bad steps for variant 1. Since we minimize $f$ at each iteration, $v_k$ is always zero and each step is unconstrained (i.e., no bad steps). \qedhere\qedhere
\end{proof}

\subsection{Cone and Pyramidal Widths}
The linear rate analysis is dominated by the fact that, similarly as in FW, many step directions are constrained (the ones pointing outside of the cone). So these arguments are in line with \cite{LacosteJulien:2015wj} and the techniques are adapted here. Lemma~\ref{lemma:faces} is a minor modification of [\cite{LacosteJulien:2015wj}, Lemma 5], see also their Figure 3. If the gradient is not feasible, the vector with maximum inner product must lie on a facet. Furthermore, it has the same inner product with the gradient and with its orthogonal projection on that facet. While first proof of Lemma~\ref{lemma:faces} follows \cite{LacosteJulien:2015wj}, we also give a different proof which does not use the KKT conditions. 

\begin{lemma}\label{lemma:faces}
Let $\theta$ be a reference point inside a polytope $\cK\in\gfaces(\cone(\cA))$ and $r\in\lin(\cK)$ is not a feasible direction from $\theta$. Then, a feasible direction in $\cK$ minimizing the angle with $r$ lies on a facet $\cK'$ of $\cK$ that includes $\theta$:
\begin{align*}
\max_{e\in\cone(\cK-\theta)} \langle r,\frac{e}{\|e\|}\rangle &= \max_{e\in\cone(\cK'-\theta)} \langle r,\frac{e}{\|e\|}\rangle \\&= \max_{e\in\cone(\cK'-\theta)} \langle r',\frac{e}{\|e\|}\rangle
\end{align*} 
where $r'$ is the orthogonal projection of $r$ onto $\lin(\cK')$
\begin{proof}
Let us center the problem in $\theta$.
We rewrite the optimization problem as:
\begin{align*}
\max_{e\in\cone(\cK),\|e\| = 1} \langle r,e\rangle
\end{align*}
and suppose by contradiction that $e$ is in the relative interior of the cone. By the KKT necessary conditions we have that $e^\star$ is collinear with $r$. Therefore $e^\star = \pm r$. Now we know that $r$ is not feasible, therefore the solution is $e^\star = -r$. By Cauchy-Schwarz we know that this solution is minimizing the inner product which is absurd. Therefore, $e^\star$ must lie on a face of the cone. The last equality is trivial considering that $r'$ is the orthogonal projection of $r$ onto $\lin(\cK')$.
\paragraph{Alternative proof}
This proof extends the traditional proof technique of \cite{LacosteJulien:2015wj} to infinitely many constraints. We also reported the FW inspired proof for the readers that are more familiar with the FW analysis.
Using proposition 2.11 of \cite{burger2003infinite} (we also use their notation) the first order optimality condition minimizing a function $J$ in a general Hilbert space given a closed set $\cK$ is that the directional derivative computed at the optimum $\bar u$ satisfy $J'(\bar u)v\geq 0$ $\forall v\in\cT(\cK-\bar u)$. Let us now assume that $\bar u $ is in the relative interior of $\cK$. Then $\cT(\cK-\bar u)=\cH$. Furthermore, $J'(\bar u)v = \langle r,v\rangle$ which is clearly not greater or equal than zero for any element of $\cH$.
\end{proof}
\end{lemma}

Theorem~\ref{thm:widthBound} is the key argument to conclude the proof of Theorem~\ref{thm:PWNNMPlinear} from Equation~\eqref{step:linearRateCombinedfast}: we have to bound the ratio of those inner products with the cone width. 

\begin{theorem}\label{thm:widthBound}
Let $r = -\nabla f(\theta_k)$, $\theta\in\cone(\cA)$, $\cS$ be the active set and $z$ and $v$ obtained as in Algorithm~\ref{algo:ANNMPPWNNMP}. Then, using the notation from Lemma \ref{lemma:faces}:
\begin{align}\label{eq:pyrThm}
\frac{\langle r, d\rangle}{\langle r,\hat e\rangle}\geq \cw
\end{align}
where $d := z-v$, $\hat e = \frac{e}{\|e\|}$ and $e= \theta^\star -\theta_k$.
\begin{proof}
As we already discussed we can consider $\theta\in\conv(\cA)$ instead of $\theta_k\in\cone(\cA)$ since both the cone and the set of feasible direction are invariant to a rescaling of $\theta$ by a strictly positive constant.
Let us center all the vectors in $\theta$, then $\hat{e}$ is just a vector with norm 1 in some face.
As $\theta$ is not optimal, by convexity we have that $\langle r,\hat e\rangle>0$. By Cauchy-Schwartz we know that $\langle r,\hat e\rangle\leq \|r\|$ since $\langle r,\hat e\rangle>0$ and $\|\hat e\| = 1$. 
By definition of $d$ we have:
\begin{align*}
\langle\frac{r}{\|r\|},d\rangle &= \max_{z\in\cA,v\in\cS} \langle\frac{r}{\|r\|},z-v\rangle\\
&\geq \min_{\cS\subset \cS_\theta}\max_{z\in\cA,v\in\cS} \langle\frac{r}{\|r\|},z-v\rangle\\
&=PdirW(\cA,r,\theta).
\end{align*}

Now, if $r$ is a feasible direction from $\theta$ Equation \eqref{eq:pyrThm} is proved (note that $PdirW(\cA,r,\theta) \geq \cw$ as $\conv(\cA) \in \gfaces(\cone(\cA))$ and $\conv(\cA) \cap \cA = \cA$). If $r$ is not a feasible direction it means that $\theta$ is on a face of $\cone(\cA)$ and $r$ points to the exterior of $\cone(\cA)$ from $\theta$. 
We then project $r$ on the faces of $\cone(\cA)$ containing $\theta$ until it is a feasible direction. We start by lower bounding the ratio of the two inner products replacing $\hat{e}$ with a vector of norm 1 in the cone that has maximum inner product with $r$ (with abuse of notation we still call it $\hat{e}$).
We then write:
\begin{align*}
\frac{\langle r, d\rangle}{\langle d,\hat e\rangle}\geq \left(\max_{z\in\cA,v\in\cS} \langle{r},z-v\rangle\right)\cdot\left(\max_{e\in\cone(\cA-\theta)}\langle r,\frac{e}{\|e\|}\rangle\right)^{-1}
\end{align*}
Let us assume that $r$ is not feasible but without loss of generality is in $\lin(\cA)$ since orthogonal components to $\lin(\cA)$ does not influence the inner product with elements in $\lin(\cA)$. 

Using Lemma~\ref{lemma:faces} we know that:
\begin{align*}
\max_{e\in\cone(\cK-\theta)} \langle r,\frac{e}{\|e\|}\rangle &= \max_{e\in\cone(\cK'-\theta)} \langle r,\frac{e}{\|e\|}\rangle \\&= \max_{e\in\cone(\cK'-\theta)} \langle r',\frac{e}{\|e\|}\rangle
\end{align*} 
Let us now consider the reduced cone $\cone(\cK')$ as $r\in\lin(\cK')$. For the numerator we obtain:
\begin{align*}
\max_{z\in\cA,v\in\cS} \langle{r},z-v\rangle
&\stackrel{\cK'\subset\cA}{\geq} \max_{z\in\cK'} \langle{r},z\rangle + \max_{v\in\cS} \langle{-r},v\rangle
\end{align*}
Putting numerator and denominator together we obtain:
\begin{align*}
\frac{\langle r, d\rangle}{\langle d,\hat e\rangle}\geq \left(\max_{\substack{z\in\cK'\\v\in\cS}}\langle{r'},z-v\rangle\right)\cdot\left(\max_{e\in\cone(\cK'-\theta)} \langle r',\frac{e}{\|e\|}\rangle\right)^{-1}
\end{align*}
Note that $\cS\subset\cK'$. Indeed, $\theta$ is a proper convex combination of the elements of $\cS$ and $\theta\in\cK'\subset\conv(\cA)$.
Now if $r'$ is a feasible direction in $\cone(\cK'-\theta)$ we obtain the cone width since $\cone(\cK')$ is a face of $\cone(\cA)$. If not we reiterate the procedure projecting onto a lower dimensional face $\cK^{''}$. Eventually, we will obtain a feasible direction. Since $\langle r,\hat e\rangle\neq 0$ we will obtain $r_{final}\neq 0$.  
\end{proof}
\end{theorem}

\paragraph{Lemma~\ref{lem:mdw}}
\textit{If the origin is in the relative interior of $\conv(\cA)$ with respect to its linear span, then $\cone(\cA)=\lin(\cA)$ and $\cw= \mdw$.}
\begin{proof}
Let us first rewrite the definition of cone width:
\begin{align*}
\cw:= \min_{\substack{\cK\in \gfaces(\cone(\cA))\\ \theta\in \cK \\ r\in\cone(\cK-\theta)\setminus \lbrace 0\rbrace}} PdirW(\cK\cap\cA,r,\theta).
\end{align*}
The minimum is over all the feasible directions of the gradient from every point in the domain.
It is not restrictive to consider $r$ parallel to $\lin(\cA)$ (because the orthogonal component has no influence). Therefore, from every point $\theta\in\lin(\cA)$ every $r\in\lin(\cA)$ is a feasible direction.
The geometric constant then becomes:
\begin{align*}
\cw=\min_{\substack{\cK\in \gfaces(\cone(\cA)) \\ \theta\in\cK\\r\in\lin(\cA)\setminus \lbrace 0\rbrace}} PdirW(\cK\cap\cA,r,\theta)
\end{align*}

 Let us now assume by contradiction that for any $\cK\in\gfaces$ we have: 
\begin{align}\label{eq:minxface}
0\not\in \argmin_{\theta\in\cK}\min_{r\in\lin(\cA)\setminus\lbrace 0\rbrace}PdirW(\cK\cap\cA,r,\theta)
\end{align} 
Therefore, $\exists v\in\cS$ such that $v\neq 0$ for any of the $\theta$ minimizing~\eqref{eq:minxface}. By definition, we have $0\in\cS$, which yields $\max_{v\in\cS}\langle r,-v\rangle\geq0$ for every $r$. Therefore, $\langle r,z-v\rangle\geq\langle r,z\rangle$ which is absurd because we assumed zero was in the set of minimizers of \eqref{eq:minxface}.
 So $0$ minimize the cone directional width which yields $\cS_\theta = \lbrace 0 \rbrace$ and $v=0$.
 In conclusion we have:
\begin{align*}
\cw = \min_{d\in\lin(\cA)} \max_{z\in\cA} \langle \frac{d}{\|d\|},z\rangle = \mdw
\end{align*}
\end{proof}



\chapter{An Optimization View on Boosting Variational Inference}\label{cha:boostingVI}
\looseness=-1In this chapter, we discuss an optimization perspective on Boosting Variational Inference Algorithm, which allows us to provide a theoretical understanding of this problem and algorithmic simplifications. The presented work is partially based on \citep{LocKhaGhoRat18} and \citep{locatello2018boosting} and was developed in collaboration with Gideon Dresdner, Rajiv Khanna, Isabel Valera, Joydeep Ghosh, and Gunnar R\"atsch. For the experimental evaluation, we only report the proof of concept results on the synthetic data that were performed by Francesco Locatello. The other experiments on real data were performed by Rajiv Khanna in \citep{LocKhaGhoRat18} and Gideon Dresdner in \citep{locatello2018boosting} (Francesco Locatello and Gideon Dresdner contributed equally to this publication). We refer the reader interested in the experiments to those papers.

\section{Variational Inference and Boosting}
Bayesian inference involves computing the posterior distribution given a model and the data.
More formally, we choose a distribution for our observations $\rvx$ given unobserved latent variables $\rvz$, called the likelihood $p(\rvx | \rvz)$, and a prior distribution over the latent variables $p(\rvz)$. Our goal is to infer the posterior, $p(\rvz | \rvx)$~\cite{Blei:2016vr}.
Bayes theorem relates these three distributions by expressing the posterior as equal to the product of prior and likelihood divided by the normalization constant, $p(\rvx)$.
The posterior is often intractable because the normalization constant $p(\rvx) = \int_\sfZ p(\rvx | \rvz) p(\rvz)d\rvz$ requires integrating over the full latent variable space.

The goal of VI is to find a \emph{tractable} approximation $q(\rvz)$ of $p(\rvz|\rvx)$.
From an optimization viewpoint, one can think of the posterior as an unknown function $p(\rvz|\rvx):\sfZ\rightarrow \R^+_{>0}$ where $\sfZ$ is a measurable set.
The task of VI is to find the best approximation, in terms of  KL divergence, to this unknown function within a family of tractable distributions $\cQ$.
Therefore, VI can be written as the following optimization problem:
\begin{align}\label{eq:VarInfProb}
\min_{q(\rvz)\in\cQ} \dkl(q(\rvz)\|p(\rvz|\rvx)).
\end{align}
Obviously, the quality of the approximation directly depends on the expressivity of the family $\cQ$. However, as we increase the complexity of  $\cQ$, the optimization problem~\eqref{eq:VarInfProb} also becomes more complex. 

The objective in Equation~\eqref{eq:VarInfProb} requires access to an unknown function $p(\rvz|\rvx)$ and is therefore not computable. Equivalently, VI maximizes instead the so-called Evidence Lower BOund (ELBO)~\cite{Blei:2016vr}: 
\begin{equation}
-\bbE_q \left[ \log q(\rvz)\right] + \bbE_q \left[ \log p(\rvx,\rvz)\right].
\end{equation}

Intuitively, variational inference aims at projecting the true posterior on the set of tractable densities $\cQ$ (for example, factorial in the mean-field case). 
There have been several efforts to improve the approximation while retaining a tractable variational family. 
Relevant to this dissertation, one could consider approximating by a mixture of \eg Gaussian distributions and allowing more than just isotropic structures.  The underlying intuition is that the family of mixtures is more expressive than any single distribution composing the mixture.
Continuing our example, a mixture of isotropic Gaussian distributions is already a much more powerful and flexible model than a single isotropic Gaussian. In fact, it is flexible enough to model any distribution arbitrarily well~\cite{parzen1962estimation}. While there has been significant algorithmic  and empirical development for studying variational inference using mixture models~\cite{Miller:2016vt,Guo:2016tg,Li:2000vt}, the theoretical understanding is limited. The boosting approach described in~\cite{Guo:2016tg,Miller:2016vt} explicitly aims at replacing $\cQ$ with $\conv(\cQ)$ thereby expanding the capacity of the variational approximation to the class of mixtures of the base family $\cQ$.

\section{Problem Formulation}
In this chapter we first analyze the problem template proposed by \citep{Guo:2016tg} for Boosting Variational Inference. We establish a connection between their approach and the Frank-Wolfe algorithm, which allows us to consider closed form stepsize estimates and corrective variants enabling faster convergence in practice as well as better rates in some specific case.

Returning to the problem formulation, \citep{Guo:2016tg} proposes to optimize the following template iteratively:
\begin{equation}\label{eq:minklQconv}
\min_{q(\rvz)\in\conv(\cQ)}\dkl(q(\rvz)||p(\rvx,\rvz)).
\end{equation}
For simplicity in the following we write $\dkl(q)$ instead of $\dkl(q(\rvz)||p(\rvx,\rvz))$ and $\dkl(q||p_\rvx)$ instead of $\dkl(q(\rvz)||p(\rvz|\rvx))$.
The boosting approach to this problem consists of specifying an iterative procedure, in which the problem is solved via the greedy combination of solutions from simpler surrogate problems. This approach was first proposed in~\cite{Guo:2016tg}, where they iteratively enrich the approximation of the ELBO by minimizing a Taylor approximation to the KL divergence:
$$
\dkl(q_{k+1}) = \dkl(q_{k}) + \alpha_k \ip{log(q_k/p)}{s_k}\,.
$$
To find the next component, this suggest to solve:
$$
\min_{s\in\cQ} \ip{-log(q_k/p)}{s}
$$
The issue with this approach is that the solution to this problem is degenerate. As a fix,~\cite{Guo:2016tg} propose to regularize this linear problem by the euclidean norm of $s$, which in turn relates to its covariance, preventing degenerate solutions. The iterate is then updated performing a convex combination between the old iterate $q_k$ and the new distribution. For the boosting analysis of \citep{zhang2003sequential} to apply, they argue that densities should be bounded from below, for example truncating their support. Performing this truncation, one can apply the classical Frank-Wolfe analysis~\citep{jaggi2013revisiting} to obtain an explicit convergence rate up to the constant approximation error induced by truncation (the true posterior is not truncated).
This error can be quantified as follows. Let:
\begin{align*}
p_{\sfA}(\rvz | \rvx) =  \begin{cases} \frac{p(\rvz| \rvx)}{\int_\sfZ p(\rvz | \rvx)\delta_\sfA(\rvz)d\rvz}, & \mbox{if } \rvz\in\sfA \\ 0, & \mbox{otherwise }\end{cases}
\end{align*}
Where $\delta_\sfA(\rvz)$ is the delta set function. Using the definition of $p_{\sfA}(\rvz)$ we have that:
\begin{align}\label{eq:klstar}
\dkl(p_\sfA(\rvz| \rvx)||p(\rvz| \rvx)) &= \int_\sfA p_\sfA(\rvz| \rvx) \log\frac{p_\sfA(\rvz| \rvx)}{p(\rvz| \rvx)} d\rvz \nonumber\\
 &= \int_\sfA p_\sfA(\rvz| \rvx) \log\frac{p(\rvz| \rvx)}{p(\rvz| \rvx)\cdot \int_\sfZ p(\rvz | \rvx)\delta_\sfA(\rvz)d\rvz } d\rvz \nonumber \\
 &=   -\log \int_\sfZ p(\rvz | \rvx)\delta_\sfA(\rvz)d\rvz
\end{align}
This error represents a trade-off between the smoothness of the objective (and therefore the rate of the boosting algorithm) and the approximation quality. Better rates might be achieved following the analysis of \citep{7472875} even without requiring truncation. We did not explore this direction, which is very promising to obtain rates in more realistic settings.

\section{The Practitioner's perspective}
Imagine a practitioner that, after designing a Bayesian model and using a VI algorithm to approximate the posterior, finds that the approximation is too poor to be useful. Standard VI does not give the practitioner the option to trade additional computational cost for a better approximation. 
As a result, the practitioner may change the variational family for a more expressive one and restart the optimization from scratch. 

Boosting is an interesting alternative as it allows finding the optimal approximating mixture adding components iteratively~\cite{Guo:2016tg, Miller:2016vt}. Further, we studied trade-off bounds for the number of iterations vs.~approximation quality. Unfortunately, these greedy algorithms require a specialized, restricted variational family to ensure convergence and, therefore, a \emph{white box} implementation of the boosting subroutine.
These restrictions include that~(a)~each potential component of the mixture has a bounded support \ie,~truncated densities, and~(b)~the subroutine should not return degenerate distributions. These assumptions require specialized care during implementation, and therefore, one cannot simply take existing VI solvers and boost them. This makes boosting VI unattractive for practitioners. Instead, we would argue that the ideal algorithm for boosting VI uses a generic VI solver as a subroutine which is iteratively queried to fit what the current approximation of the posterior is not yet capturing.

\section{Revisiting the Curvature}
To boost VI using FW in practice, we need to ensure that the assumptions are not violated. 
Assume that $\cB\subset \cQ$ is the set of probability density functions with compact parameter space as well as bounded infinity norm and $L2$ norm. 
These assumptions on the search space are easily justified since it is reasonable to assume that the posterior is not degenerate (bounded infinity norm) and has modes that are not arbitrarily far away from each other (compactness).
Under these assumptions, the optimization domain is closed and bounded. It is simple to show that the solution of the \lmo problem over $\conv(\cB)$ is an element of $\cB$. Therefore, $\cB$ is closed.
The troublesome condition that needs to be satisfied for the convergence of FW is smoothness. 
A bounded curvature is however sufficient to guarantee convergence \citep{jaggi2013revisiting}.
\begin{equation}
\label{def:Cf}
\Cf := \sup_{\substack{s\in\cA,\, q \in \conv(\cA)
 \\ \gamma \in [0,1]\\ y = q + \gamma(s- q)}} \frac{2}{\gamma^2} D(y,q),
\end{equation}
where 
\begin{equation*}
D(y,q) :=f(y) - f(q)- \langle y -q, \nabla f(q)\rangle.
\end{equation*}
It is known that $\Cf\leq L\diam(\cA)^2$ if $f$ is $L$-smooth over $\conv(\cA)$.
This condition is weaker than smoothness, which was assumed by~\cite{Guo:2016tg} and our previous analysis in \citep{LocKhaGhoRat18}. For the KL divergence, the following holds.

\begin{theorem}\label{boundedCf}
$\Cf$ is bounded for the KL divergence if the parameter space of the densities in $\cB$ is bounded.
\end{theorem}
\paragraph{Discussion}
A bounded curvature for the $\dkl$ can be obtained as long as:
\begin{align*}
\sup_{\substack{s\in\cB,\, q \in \conv(\cB)
 \\ \gamma \in [0,1]\\ y = q + \gamma(s- q)}} \frac{2}{\gamma^2}\dkl(y\| q)
\end{align*}
is bounded. The proof sketch proceeds as follows. For any pair $s$ and $q$, we need to check that $\frac{2}{\gamma^2}\dkl(y\| q)$ is bounded as a function of $\gamma\in[0,1]$. The two limit points, $\dkl(s\| q)$ for $\gamma = 1$ and $\|s-q\|_2^2$ for $\gamma = 0$, are both bounded for any choice of $s$ and $q$. Hence, the $\Cf$ is bounded as it is a continuous function of $\gamma$ in $[0,1]$ with bounded function values at the extreme points. $\dkl(s\| q)$ is bounded because the parameter space is bounded. $\|s-q\|_2^2$ is bounded by the triangle inequality and bounded $L2$ norm of the elements of $\cB$. This result is interesting since it makes the complicated truncation described before unnecessary assuming a bounded parameter space which is arguably more practical.
\vspace{-1mm}
\section{The residual ELBO}
\vspace{-1mm}
Note that the \lmo is a constrained linear problem in a function space. A complicated heuristic is developed in \cite{Guo:2016tg} to deal with the fact that the \textit{unconstrained} linear problem they consider has a degenerate solution. 
In \cite{LocKhaGhoRat18}, we explored using projected gradient descent on the parameters of $s$ with a constraint on the infinity norm of $s$. 
Such a constraint is hardly practical. Indeed, one must compute the maximum value of $s$ as a function of its parameters, which depends on the particular choice of $\cB$. In contrast, the entropy is a general term that can be approximated via sampling and allows for black box computation. We relate infinity norm and the entropy in the following lemma. 
\begin{lemma}\label{thm:bounded_entropy}
A density with bounded infinity norm has entropy bounded from below. The converse is true for many of the distributions which are commonly used in VI (for example Gaussian, Cauchy and Laplace).
\end{lemma}
In general, a bounded entropy does not always imply a bounded infinity norm. While this is precisely the statement we would need, a simple verification is sufficient to show that it holds in several cases of interest. We assume that $\cB$ is a family for which bounded entropy implies bounded infinity norm. Therefore, we can constrain the optimization problem with the entropy instead of the infinity norm. We call $\bar\cB$ the family $\cB$ without the infinity norm constraint. At every iteration, we need to solve: 
\begin{align*}
\argmin_{\substack{s\in\bar\cB\\H(s)\geq -M}} \left\langle s, \log\left(\frac{q_k}{p}\right)\right\rangle
\end{align*}
Note that the constraint on the entropy is crucial here. Otherwise, the solution of the \lmo would be a degenerate distribution as also argued in~\cite{Guo:2016tg}. 

We now replace this problem with its regularized form using Lagrange multipliers and solve for $s$ given a fixed value of $\lambda$: 
\begin{align}
\argmin_{\substack{s\in\bar\cB}} \left\langle s, \log\left(\frac{q_k}{p}\right)\right\rangle + \lambda\left( -H(s)- M\right) &= \argmin_{\substack{s\in\bar\cB}} \left\langle s, \log\left(\frac{q_k}{p}\right)\right\rangle + \langle s, \log s^\lambda\rangle\label{eq:minus_relbo}\\
&= \argmin_{\substack{s\in\bar\cB}} \left\langle s, \log\left(\frac{s^\lambda}{\frac{p}{q_k}}\right)\right\rangle\nonumber\\
&= \argmin_{\substack{s\in\bar\cB}} \left\langle s, \log\left(\frac{s}{\sqrt[\lambda]{\frac{p}{q_k}}}\right)\right\rangle. \nonumber
\end{align}
Therefore, the regularized LMO problem is equivalent to the following minimization problem: 
\begin{align*}
\argmin_{\substack{s\in\bar\cB}} \dkl(s\| \sqrt[\lambda]{\frac{p}{q_k}}Z ),
\end{align*}
where $Z$ is the normalization constant of $\sqrt[\lambda]{\frac{p}{q_k}}$. 
From this optimization problem, we can write what we call the Residual Evidence Lower Bound (RELBO) as:
\begin{align}\label{eq:relbo}
\relbo(s, \lambda) :=  \bbE_{s}[\log p] -\lambda\bbE_{s}[\log s] -\bbE_{s}[\log q_k].
\end{align}

\paragraph{Discussion}
Let us now analyze the \relbo and compare it with the ELBO in standard VI~\cite{Blei:2016vr}. First, note that we introduce the hyperparameter $\lambda$, which controls the weight of the entropy. To obtain the true LMO solution, one would need to maximize the LHS of Equation~\eqref{eq:minus_relbo} for $\lambda$ and solve the saddle point problem. Since an approximate solution is sufficient for convergence, we consider the regularized problem as a simple heuristic. One can then fix an arbitrary value for $\lambda$ or decrease it when $k$ increases. The latter amounts to allowing increasingly sharp densities as optimization proceeds. 
The other important difference between ELBO and \relbo is the residual term which is expressed through $\bbE_{s}[\log p]-\bbE_{s}[\log q_k]$. Maximizing this term amounts to looking for a density with low cross-entropy with the joint $p$ and high cross-entropy with the current iterate $q_k$. In other words, the next component $s_k$ needs to be as close as possible to the target $p$ but also sufficiently different from the current approximation $q_k$. Indeed, $s_k$ should capture the aspects of the posterior that the current mixture could not approximate yet. 

\paragraph{Failure Modes} Using a black box VI as an implementation for the \lmo represents an attractive, practical solution. Indeed, one could just run VI once and, if the result is not good enough, rerun it on the residual without changing the structure of the implementation. 
Unfortunately, two failure modes should be discussed. First, if the target posterior is a perfectly symmetric multimodal distribution, the residual is also symmetric, and the algorithm may get stuck. A simple solution to this problem is to run the black box VI for fewer iterations, breaking the residual symmetry. This failure mode arise from the non-convexity of the variational inference problem we solve as a subroutine. 
The second problem arises in scenarios where the posterior distribution can be approximated well by a single element of $\cQ$. In such cases, most of the residual will be on the tails. The algorithm will then fit the tails and, in the following iterations, re-learn a distribution close to $q^0$. Consequently, it is essential to identify reasonable solutions before investing additional computational effort by adding more components to the mixture. Note that the ELBO cannot be used for this purpose, as its value at the maximum is unknown.

\paragraph{Stopping criterion} 

We propose a stopping criterion for boosting VI, which allows us to identify when a reasonably good approximation is reached and save computational effort. To this end, we rephrase the notion of \textit{duality gap}~\cite{jaggi2013revisiting,jaggi2011convex} in the context of boosting VI, which gives a surprisingly simple stopping criterion for the algorithm.
\begin{lemma}\label{lemma:duality_gap}
The duality gap $g(q):= \max_{s\in\conv(\cB)} \langle q-s,\log\frac{q}{p}\rangle$ computed at some iterate $q\in\conv(\cB)$ is an upper bound on the primal error $\dkl(q\|p) - \dkl(q^\star\|p)$.
\end{lemma}

Note that the $\argmax_{s\in\conv(\cB)} \langle q-s,\log\frac{q}{p}\rangle$ is precisely the \lmo solution to the problem~\eqref{eqn:lmo}. Therefore, with an exact \lmo, one obtains a certificate on the primal error for free, without knowing the value of $\dkl(q^\star\|p)$. 
It is possible to show that a convergence rate also holds for the duality gap~\cite{jaggi2013revisiting}. If the oracle is inexact, the estimate of the duality gap $\tilde{g}(q)$ satisfies that $\frac{1}{\delta}\tilde{g}(q)\geq g(q)$, as a consequence of~\eqref{eqn:lmo_mult}.

\section{Experimental Proof of Concept}
\label{sec:Experiments}
\looseness=-1This experimental proof of concept aims to illustrate that the algorithm indeed learns a multimodal approximation to the posterior distribution. The experimental evaluation beside this simple case was done by Rajiv Khanna and Gideon Dresdner and, therefore, is not included in this dissertation. We refer to \citep{LocKhaGhoRat18,locatello2018boosting} for a detailed experimental evaluation. We implemented our algorithm as an extension to the {\sl Edward} probabilistic programming framework~\cite{tran2016edward} thereby enabling users to apply boosting VI to any probabilistic model and variational family which are definable in {\sl Edward}.
For comparisons to baseline VI, we use {\sl Edward}'s built-in black box VI (BBVI) algorithm without modification.
We used $\lambda = \frac{1}{\sqrt{k+1}}$. 

\begin{figure}
  \center\includegraphics[width=0.5\linewidth]{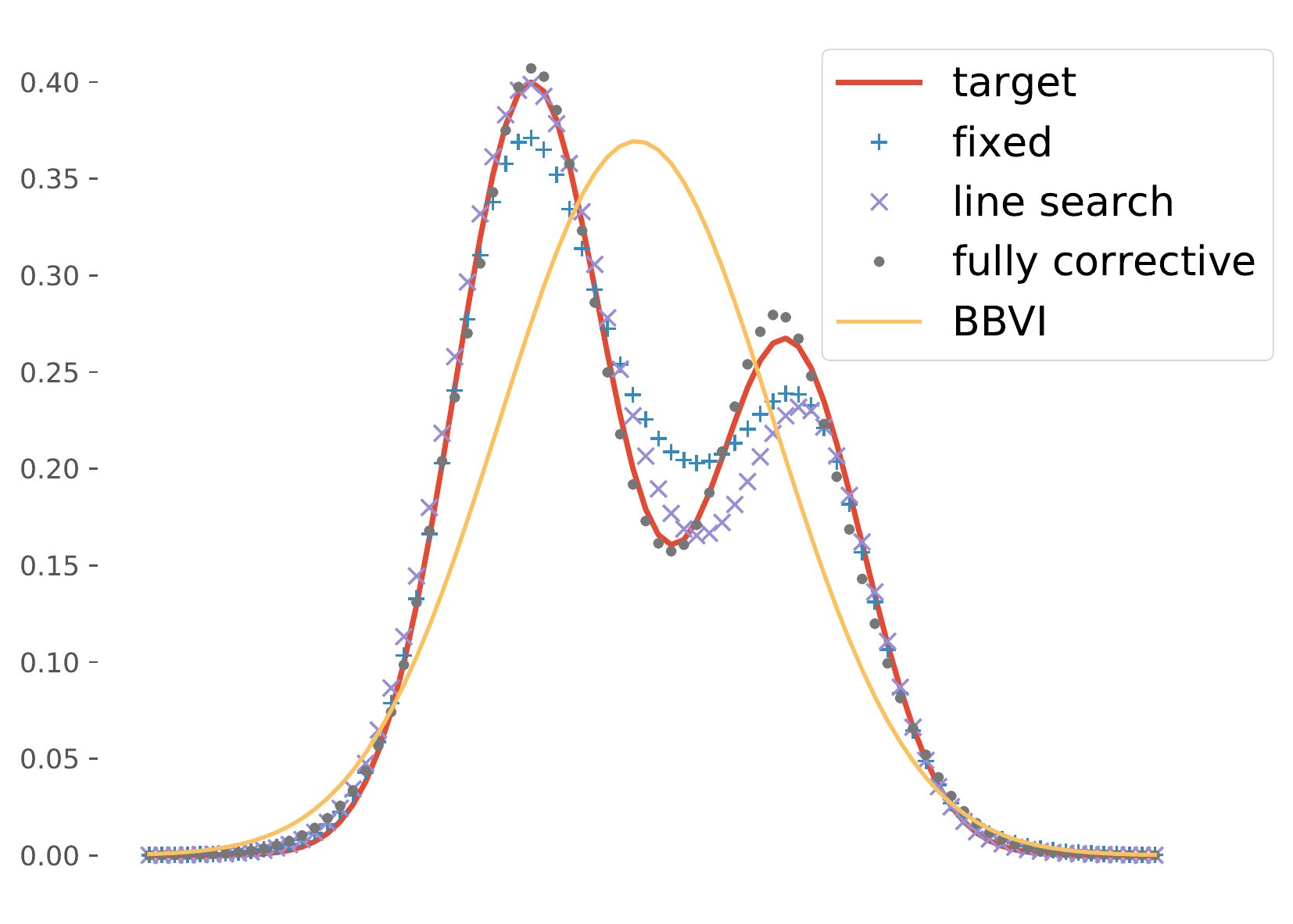}
    \caption{\small Comparison between BBVI and three variants of boosting BBVI method on a mixture of Gaussians example. We observe that the boosting methods are able to fit a simple bi-modal distribution.}
  \label{fig:bimodal-post}
  \vspace{-5mm}
\end{figure}

We use synthetic data to visualize the approximation of our algorithm of a bimodal posterior using up to $40$ iterations. 
In particular, we consider a mixture of two Gaussians with parameters $\mu = (-1,+1)$, $\sigma = (0.5,0.5)$, and mixing weights $\pi = (0.4,0.6)$. 

We performed experiments using vanilla FW with fixed stepsize, line search, and the fully corrective variant. For the fully corrective variant, we used FW to solve the subproblem of finding the optimal weights for the current atom set. Our results are summarized in Figure~\ref{fig:bimodal-post}.
We observe that unlike BBVI, all three variants can fit both modes of the bimodal target distribution. The fully corrective version gives the best fit.
This improved solution comes at a computational cost --- solving the line search and fully corrective subproblems is slower than the fixed step size variant. 

\section{Proofs}

\subsection{Proof of Theorem~\ref{boundedCf}}
\begin{proof}
First, we rewrite the divergence in the curvature definition as:
\begin{align*}
D(y,q) &= \dkl(y) - \dkl(q)- \langle y -q, \nabla \dkl(q)\rangle\\
&= \langle y, \log \frac{y}{p} \rangle - \langle q, \log \frac{q}{p} \rangle - \langle y -q, \log\frac{q}{p}\rangle \\
&= \langle y, \log \frac{y}{p} \rangle  - \langle y , \log\frac{q}{p}\rangle \\
&= \langle y, \log \frac{y}{q} \rangle\\
&= \dkl(y\| q)
\end{align*}
In order to show that $\Cf$ is bounded we then need to show that:
\begin{align*}
\sup_{\substack{s\in\cB,\, q \in \conv(\cB)
 \\ \gamma \in [0,1]\\ y = q + \gamma(s- q)}} \frac{2}{\gamma^2}\dkl(y\| q)
\end{align*}
is bounded.
For a fixed $s$ and $q$ we how that $ \frac{2}{\gamma^2}\dkl(y\| q)$ is continuous.
Since the parameter space is bounded $\dkl(y\| q)$ is always bounded for any $\gamma\geq \varepsilon > 0$ and so is the $\Cf$, therefore the $\Cf$ is continuous for $\gamma \in (0,1]$. We only need to show that it also holds for $\gamma = 0$ in order to use the result that a continuous function on a bounded domain is bounded.
When $\gamma \rightarrow 0 $ we have that both $\gamma^2$ and $\dkl(y\| q)\rightarrow 0$. Therefore we use L'Hospital Rule (H) and obtain:
\begin{align*}
\lim_{\gamma\rightarrow 0 }\frac{2}{\gamma^2}\dkl(y\| q) &\stackrel{H}{=} \lim_{\gamma\rightarrow 0 }\frac{1}{\gamma}\int_\sfZ(s-q)\log\left(\frac{y}{q}\right)
\end{align*}
\begin{align*}
\lim_{\gamma\rightarrow 0 }\frac{2}{\gamma^2}\dkl(y\| q) &\stackrel{H}{=} \lim_{\gamma\rightarrow 0 }\frac{1}{\gamma}\int_\sfZ(s-q)\log\left(\frac{y}{q}\right)
\end{align*}
where for the derivative of the $\dkl$ we used the functional chain rule. 
Again both numerator and denominators in the limit go to zero when $\gamma\rightarrow 0 $, so we use L'Hospital Rule again and obtain:
\begin{align*}
\lim_{\gamma\rightarrow 0 }\frac{1}{\gamma}\int_\sfZ(s-q)\log\left(\frac{y}{q}\right) &\stackrel{H}{=} \lim_{\gamma\rightarrow 0 }\int_\sfZ(s-q)^2\frac{q}{y}\\
&= \int_\sfZ \lim_{\gamma\rightarrow 0 }(s-q)^2\frac{q}{y}\\
&= \int_\sfZ (s-q)^2\\
\end{align*}
which is bounded under the assumption of bounded parameters space and bounded infinity norm. Indeed:
\begin{align*}
\int_\sfZ (s-q)^2\leq 4\max_{s\in\conv(\cB)} \int_\sfZ s^2
\end{align*}
Which is bounded under the assumption of bounded $L2$ norm of the densities in $\cB$ by triangle inequality.
\end{proof}

\subsection{Proof of Lemma~\ref{lemma:duality_gap}}
\begin{proof}
Let $\log\frac{q}{p}$ be the gradient of the $\dkl(q\|p )$ computed at some $q\in\conv(\cQ)$. The dual function of the $\dkl$ is:
\begin{align*}
w(q) := \min_{s\in\conv(\cB)} \dkl(q\|p ) + \langle s-q,\log\frac{q}{p}\rangle.
\end{align*}
By definition, the gradient is a linear approximation to a function lying below its graph at any point. Therefore, we have that for any $q,y\in\conv(\cB)$:
\begin{align*}
w(q) = \min_{s\in\conv(\cB)} \dkl(q\|p ) + \langle s-q,\log\frac{q}{p}\rangle \leq\dkl(q\|p ) + \langle y-q,\log\frac{q}{p}\rangle \leq \dkl(y\|p ).
\end{align*}
The duality gap at some point $q$ is the defined as the difference between the values of the primal and dual problems:
\begin{equation}
g(q) := \dkl(q\|p) - w(q) = \max_{s\in\conv(\cB)} \langle q-s,\log\frac{q}{p}\rangle.
\end{equation}
Note that the duality gap is a bound on the primal error as:
\begin{equation}
g(q) = \max_{s\in\conv(\cB)} \langle q-s,\log\frac{q}{p}\rangle \geq \langle q-q^\star,\log\frac{q}{p} \geq \dkl(q\|p) - \dkl(q^\star\|p),
\end{equation}
where the first inequality comes from the fact that the optimum $q^\star\in\conv(\cB)$ and the second from the convexity of the KL divergence w.r.t. $q$.
\end{proof}

\part{Disentangled Representations}
\label{part:rep_learn}


\chapter{Introduction and Background}\label{cha:disent_background}
\blfootnote{This chapter is partly based on discussions in \citep{locatello2019challenging,locatello2020sober,locatello2020commentary,scholkopf2020towards} that were developed in collaboration with Stefan Bauer, Mario Lucic, Gunnar R\"atsch, Sylvain Gelly, Bernhard Sch\"olkopf, Olivier Bachem, Rosemary Nan Ke, Nal Kalchbrenner, Anirudh Goyal, and Yoshua Bengio. These works were partially done when Francesco Locatello was at Google Research, Brain Teams in Zurich and Amsterdam.}

\looseness=-1
In representation learning, it is often assumed that real-world observations $\rvx$ (such as images or videos) are generated by a two-step generative process.
First, a multivariate latent random variable $\rvz$ is sampled from a distribution $\P(\rvz)$.
Intuitively, $\rvz$ corresponds to semantically meaningful \textit{factors of variation} of the observations (such as content and position of objects in an image).
Then, in a second step, the observation $\rvx$ is sampled from the conditional distribution $\P(\rvx|\rvz)$.
The key idea behind this model is that the high-dimensional data $\rvx$ can be explained by the substantially lower dimensional and semantically meaningful latent variable $\rvz$.
Informally, the goal of representation learning is to find useful transformations $r(\rvx)$ of $\rvx$ that ``\textit{make it easier to extract useful information when building classifiers or other predictors}'' \citep{bengio2013representation}. 

\looseness=-1A recent line of work has argued that \emph{disentanglement} is a desirable property of good representations~\citep{bengio2013representation,peters2017elements,lecun2015deep,bengio2007scaling,schmidhuber1992learning,lake2017building,tschannen2018recent}. Disentangled representations should contain all the information present in $\rvx$ in a compact and interpretable structure~\citep{bengio2013representation,kulkarni2015deep,chen2016infogan} while being independent from the task at hand~\citep{goodfellow2009measuring,lenc2015understanding}. 
They should be useful for (semi-)supervised learning of downstream tasks, transfer and few shot learning~\citep{bengio2013representation,scholkopf2012causal,peters2017elements}. They should enable to integrate out nuisance factors~\citep{kumar2017variational}, to perform interventions, and to answer counterfactual questions~\citep{pearl2009causality,SpiGlySch93,peters2017elements}. 

While there is no single formalized notion of disentanglement (yet) which is widely accepted, the key intuition is that a disentangled representation should separate the distinct, informative \emph{factors of variations} in the data \citep{bengio2013representation}.
A change in a single underlying factor of variation $\rz_i$ should lead to a change in a single factor in the learned representation $r(\rvx)$. This assumption can be extended to groups of dimensions as, for instance, in the work of \citet{bouchacourt2017multi} or \citet{suter2018interventional}. 
Based on this idea, a variety of disentanglement evaluation protocols have been proposed leveraging the statistical relations between the learned representation and the ground-truth factor of variations~\citep{higgins2016beta,kim2018disentangling,eastwood2018framework,kumar2017variational,chen2018isolating,ridgeway2018learning,suter2018interventional}. 

State-of-the-art approaches for unsupervised disentanglement learning are largely based on \emph{Variational Autoencoders (VAEs)} \citep{kingma2013auto}:
One assumes a specific prior $\P(\rvz)$ on the latent space and then uses a deep neural network to parameterize the conditional probability $\P(\rvx|\rvz)$. 
Similarly, the distribution $\P(\rvz|\rvx)$ is approximated using a variational distribution  $\Q(\rvz|\rvx)$, again parametrized using a deep neural network.
The model is then trained by minimizing a suitable approximation to the negative log-likelihood. The representation for $r(\rvx)$ is usually taken to be the mean of the approximate posterior distribution $\Q(\rvz|\rvx)$. 
Several variations of VAEs were proposed with the motivation that they lead to better disentanglement~\citep{higgins2016beta,burgess2018understanding,kim2018disentangling,chen2018isolating,kumar2017variational}.
The common theme behind all these approaches is that they try to enforce a factorized aggregated posterior $\int_{\rvx}\Q(\rvz|\rvx)\P(\rvx)d\rvx$, which should encourage disentanglement, see Figure~\ref{fig:vae}.

\begin{figure}
    \centering
    \includegraphics[width=0.65\textwidth]{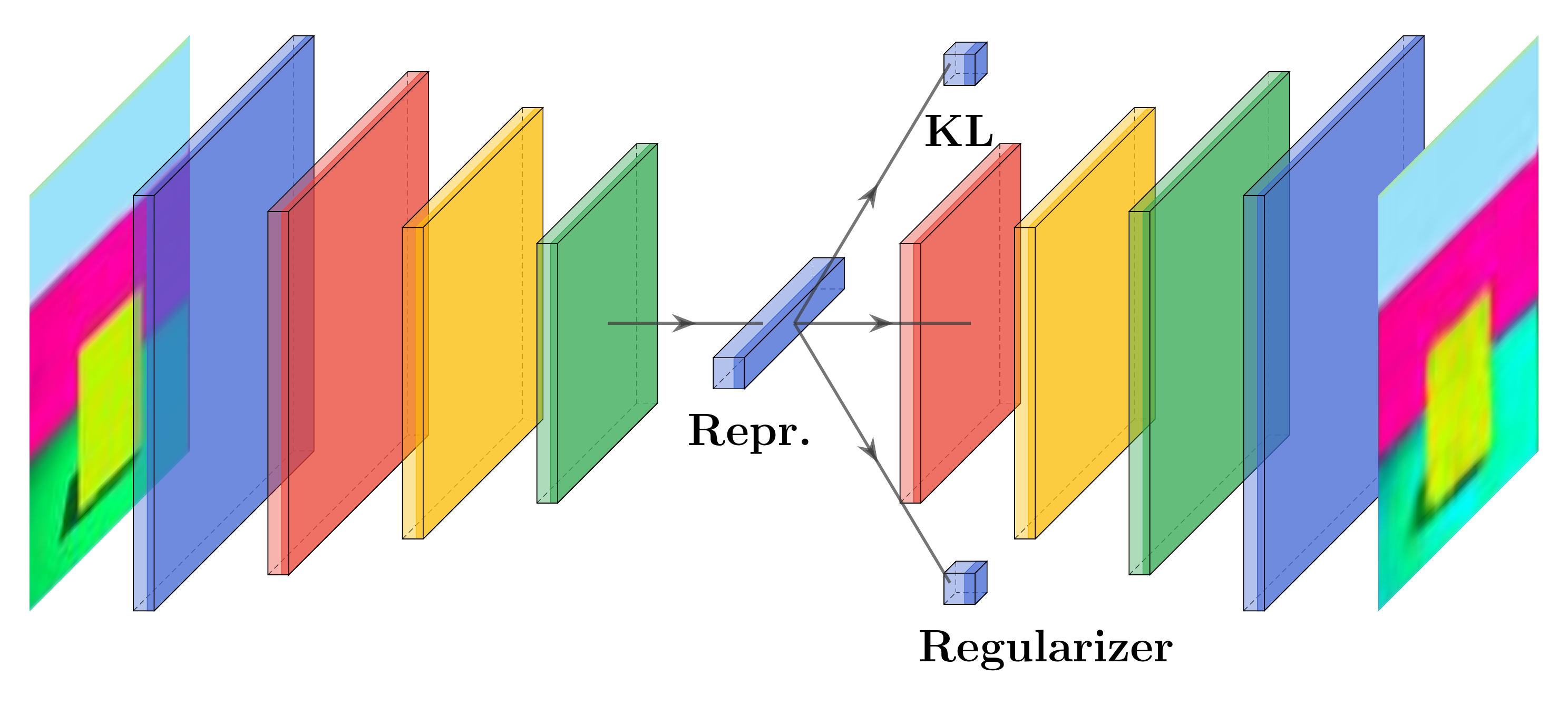}
    \caption{\small Prototypical architecture for disentanglement based on VAEs. In order to learn disentangled representations, one further regularize the latent space to enforce suitable properties as proxies for disentanglement.}
    \label{fig:vae}
\end{figure}
\paragraph{Our Goal and Retrospectives} Disentangled representations appear very attractive at a superficial level as factorizing knowledge in a way that matches independent causal modules may have several important applications. The theoretical impossibility result of Theorem~\ref{thm:impossibility} was  the starting point of our work and our initial goal was to systematically study the state-of-the-art approaches in order to precisely pinpoint which inductive bias was making disentanglement possible. Eventually, we hoped to use this knowledge to propose a new state-of-the-art approach. Therefore, our first main research question was: \textit{What's the secret ingredient that allows learning disentangled representations?} In~\citep{locatello2019challenging}, we observed that learning disentangled representations without ground-truth knowledge was challenging and we could not validate the usefulness in terms of downstream sample complexity. Due to the relevance of this problem, we received the best paper award at ICML 2019. There were two main questions at this point: \textit{``How much supervision do we actually need?''} and \textit{``Is it really worth it to learn disentangled representations?''}. To address the first question, we investigated the role of explicit supervision on state-of-the-art approaches in~\citep{locatello2019disentangling}. We performed a very large scale study and found that we actually did not need much supervision and methods were generally robust to imprecision.
To address the usefulness of disentangled representations, we investigated a Fairness~\citep{locatello2019fairness} and an abstract reasoning setting~\citep{van2019disentangled} (the latter work is not part of this dissertation). These works provided convincing evidence about the usefulness of disentangled representations but we still had no method to reliably learn them without access to the ground-truth factors of variation. In~\citep{locatello2020weakly}, we proposed a new realistic setting motivated by our recent results in non-linear ICA~\citep{gresele2019incomplete} where disentangled representations are identifiable. Inspired by our analysis, we also proposed a method that reliably learned them in practice. Perhaps most importantly, we could show that these representations were indeed useful on multiple diverse downstream tasks.

\section{Relation with Prior Work in ML}
\looseness=-1In a similar spirit to disentanglement, (non-)linear independent component analysis~\citep{comon1994independent,bach2002kernel,jutten2003advances, hyvarinen2016unsupervised} studies the problem of recovering independent components of a signal. The underlying assumption is that there is a generative model for the signal composed of the combination of statistically independent non-Gaussian components. While the
identifiability result for linear ICA \citep{comon1994independent} proved to be a milestone for the classical theory of factor analysis, similar results are in general not obtainable for the nonlinear case and the underlying sources generating the data cannot be identified
 \citep{hyvarinen1999nonlinear}. The lack of almost any identifiability result in non-linear ICA has been a main bottleneck for the utility of the approach \citep{hyvarinen2018nonlinear} and partially motivated alternative machine learning approaches~\citep{desjardins2012disentangling,schmidhuber1992learning,cohen2014transformation}.
Given that unsupervised algorithms did not initially perform well on realistic settings most of the other works have considered some more or less explicit form of supervision~\citep{reed2014learning,zhu2014multi,yang2015weakly,kulkarni2015deep,cheung2014discovering,mathieu2016disentangling,narayanaswamy2017learning,suter2018interventional}.~\citep{hinton2011transforming,cohen2014learning} assume some knowledge of the effect of the factors of variations even though they are not observed. One can also exploit known relations between factors in different samples~\citep{karaletsos2015bayesian,goroshin2015learning,whitney2016understanding,fraccaro2017disentangled,denton2017unsupervised,hsu2017unsupervised,yingzhen2018disentangled,locatello2018clustering}. This is not a limiting assumption especially in sequential data like for videos. There is for example a rich literature in disentangling pose from content in 3D objects and content from motion in videos or time series in general~\citep{yang2015weakly,li2018disentangled,hsieh2018learning,fortuin2018deep,deng2017factorized,goroshin2015learning}.
Similarly, the non-linear ICA community recently shifted to non-iid data types exploiting time dependent or grouped observations~\citep{hyvarinen2016unsupervised,hyvarinen2018nonlinear,gresele2019incomplete}

\section{Relation with Causality}
In this section, we give a brief introduction to causal inference and structural causal models and we explain the link between disentanglement and causality. 

\begin{figure*}[t]
\begin{center}
\includegraphics[width=0.7\textwidth]{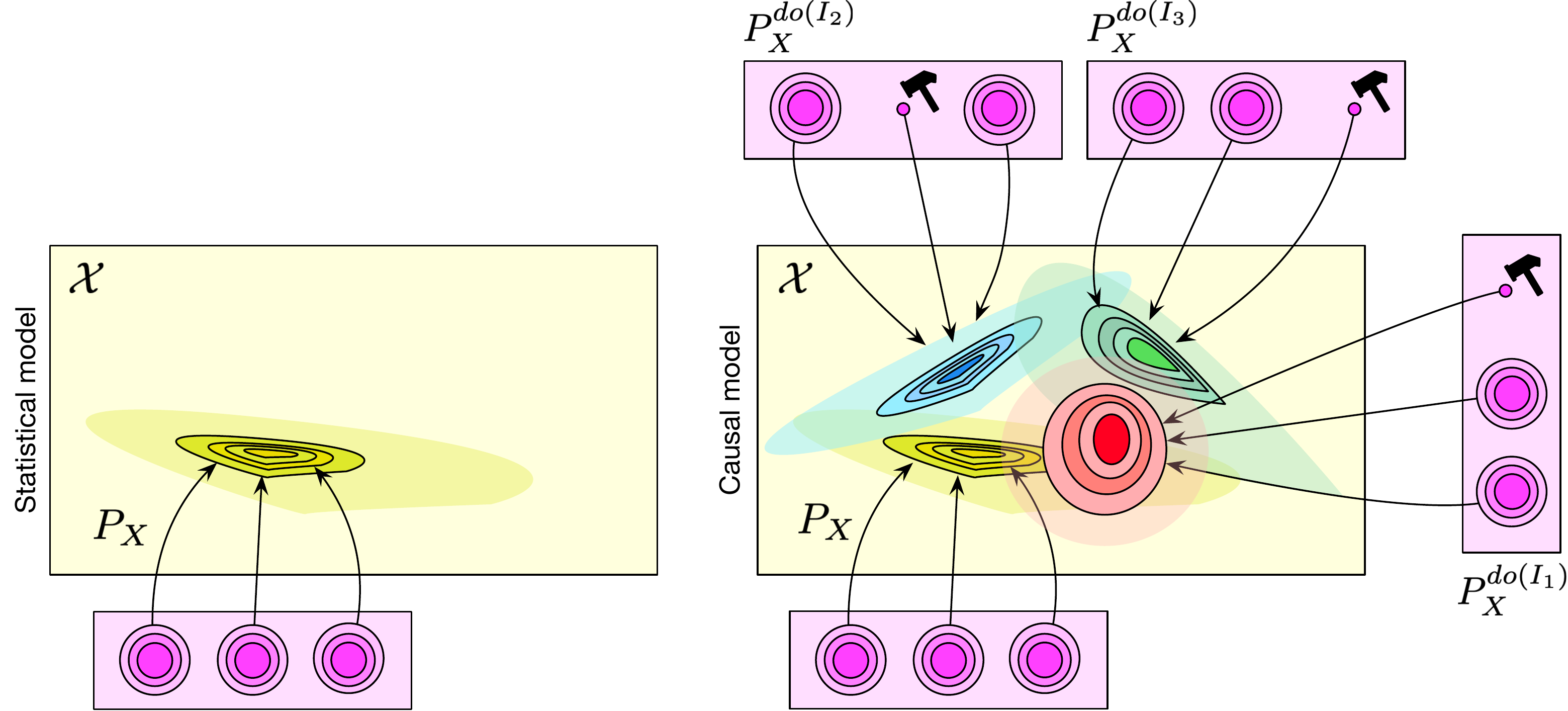}
\end{center}
  \caption{\small Difference between statistical (left) and causal models (right) on a given set of three $X = \left[\rvx_1, \rvx_2, \rvx_3 \right]$ variables. While a statistical model specifies a single probability distribution, a causal model represents a set of distributions, one for each possible intervention (indicated with a \myhammer{} in the figure). In causal graphical models, the interventions are induced by the causal graph. In both cases, one can factorize the joint distribution into a product of independent components. In causal models, we can intervene on these components, for example by fixing their value. 
 }\label{fig_stats_vs_causal}
\end{figure*}

Structural causal models consider a set of random variables $\rvx_1,\dots,\rvx_n$ and operate under the assumption that the value of each variable corresponds to a functional assignment. The causal relations between the variables are modelled in a directed acyclic graph (DAG) and are quantified via structural equations:
\begin{equation}\label{eq:SA}
\rvx_i := f_i (\PA_i, \rvu_i),   ~~~~ (i=1,\dots,n),
\end{equation}
where $f_i$ are deterministic functions taking as input $\rvx_i$'s parents in the DAG (which we call $\PA_i$) and a random variable $\rvu_i$, representing the noise source of the corresponding variable. This noise explain all the stochasticity in the variable $\rvx_i$ and allow to express a a general conditional distribution $P(\rvx_i|\PA_i)$. The critical assumption is that the noise variables are jointly independent, which can be seen as a consequence of \textit{causal sufficiency} \citep{peters2017elements}.
The parent-child relationships entailed by the structural equations and the graph structure implies a specific factorization of the joint distribution of the observables, which is called the {\em causal (or disentangled) factorization}:
\begin{equation}\label{eq:cf}
P(\rvx_1,\dots,\rvx_n) = \prod_{i=1}^n  P(\rvx_i \mid \PA_i).
\end{equation}
While many other factorizations are possible, Equation \eqref{eq:cf} decomposes the joint distribution into conditionals corresponding to the structural assignments in Equation \eqref{eq:SA}.

\paragraph{Causal vs Statistical Models}
An example of the difference between a statistical and a causal model is depicted in Figure~\ref{fig_stats_vs_causal}. A statistical model may be defined through a graphical model where the joint distribution factorizes. In general, the connections in a (generic) graphical model do not need to be causal \citep{peters2017elements}. 
A graphical model can be made causal by augmenting it with a graph  (sometimes referred to as the ``causal graph''). The addition of the graph allows to compute interventional distributions as in Figure~\ref{fig_stats_vs_causal}. When a variable is intervened upon, we disconnect it from its parents, fix its value, and perform ancestral sampling on its children. A structural causal model is composed of (i) a set of causal variables and (ii) a set of structural equations with a distribution over the noise variables $\rvu_i$ (or a set of causal conditionals). While both causal graphical models and SCMs allow to compute interventional distributions, only the SCMs allow to compute counterfactuals. To compute counterfactuals, we need to fix the value of noise variables and there are many ways to represent a conditional as a structural assignment (by picking different combinations of functions and noise variables).

\paragraph{Disentanglement}
\looseness=-1There are two ways of understanding the relation between disentanglement and causality. First, one could think of the factors of variation as latent causal parents of the observed variables \citep{suter2018interventional}. The second view, is that causal variables are not given and images are a high-dimensional measurement of these lower dimensional variables~\citep{scholkopf2019causality,scholkopf2020towards}. The goal is then to learn a mapping from the high-dimensional observations to the causal variables. To do so, we can encode the images with a neural network to obtain the unexplained noise variables $\rvu_1,\ldots,\rvu_n$ and decode them back to images. If the structural relations between the variables are known we can embed this knowledge into the architecture of the decoder. Otherwise, the decoder will implicitly model the structural relations using distributed representations. In this context, disentangled representations are appealing for causality. They allow to reason about the noise realization of an observation, which is useful to answer specific interventional and counterfactual questions. On the other hand, we shall see that disentangled representations cannot be identified from observational data and non-iid settings that are common in causality can fix this issue.

\section{Learning Disentangled Representations}\label{sec:disent_methdos}
Variants of variational autoencoders~\cite{kingma2013auto} are considered the state-of-the-art for unsupervised disentanglement learning. 
They optimize the following approximation to the maximum likelihood objective, 
\begin{align}
\max_{\phi, \theta}\quad \E_{p(\rvx)} [\E_{q_\phi(\rvz|\rvx)}[\log p_\theta(\rvx|\rvz)] -  \KL(q_\phi(\rvz|\rvx) \| p(\rvz))]\label{eq:elbo}
\end{align}
which is also know as the evidence lower bound (ELBO). By carefully considering the KL term, one can encourage various properties of the resulting presentation. We will briefly review the main approaches. We now briefly categorize the different approaches.

\paragraph{Bottleneck Capacity}
 \citet{higgins2016beta} propose the $\beta$-VAE, introducing a hyperparameter in front of the KL regularizer of vanilla VAEs. They maximize the following expression:
\begin{align*}
\E_{p(\rvx)} [\E_{q_\phi(\rvz|\rvx)}[\log p_\theta(\rvx|\rvz)] - \beta \KL(q_\phi(\rvz|\rvx) \| p(\rvz))]
\end{align*} 
By setting $\beta > 1$, the encoder distribution will be forced to better match the factorized unit Gaussian prior. This procedure introduces additional constraints on the capacity of the latent bottleneck, encouraging the encoder to learn a disentangled representation for the data. 
~\citet{burgess2018understanding} argue that when the bottleneck has limited capacity, the network will be forced to specialize on the factor of variation that most contributes to a small reconstruction error. Therefore, they propose to progressively increase the bottleneck capacity, so that the encoder can focus on learning one factor of variation at the time:
\begin{align*}
\E_{p(\rvx)} [\E_{q_\phi(\rvz|\rvx)}[\log p_\theta(\rvx|\rvz)] - \gamma | \KL(q_\phi(\rvz|\rvx) \| p(\rvz)) - C |]
\end{align*}
where C is annealed from zero to some value which is large enough to produce good reconstruction. 
In the following, we refer to this model as AnnealedVAE.

\paragraph{Penalizing the Total Correlation}
Let $I(\rvx; \rvz)$ denote the mutual information between $\rvx$ and $\rvz$ and note that the second term in~\eqref{eq:elbo} can be rewritten as 
\begin{align*}
\E_{p(\rvx)} [\KL(q_\phi(\rvz|\rvx) \| p(\rvz))] = I(\rvx;\rvz) + \KL(q(\rvz) \| p(\rvz)).
\end{align*}
Therefore, when $\beta>1$, $\beta$-VAE penalizes the mutual information between the latent representation and the data, thus constraining the capacity of the latent space. Furthermore, it pushes $q(\rvz)$, the so called \textit{aggregated posterior}, to match the prior and therefore to factorize, given a factorized prior. 
\citet{kim2018disentangling} argues that penalizing $I(\rvx;\rvz)$ is neither necessary nor desirable for disentanglement. The FactorVAE~\citep{kim2018disentangling} and the $\beta$-TCVAE~\citep{chen2018isolating} augment the VAE objective with an additional regularizer that specifically penalizes dependencies between the dimensions of the representation:
\begin{align*}
\E_{p(\rvx)} [\E_{q_\phi(\rvz|\rvx)}[\log p_\theta(\rvx|\rvz)] -  \KL(q_\phi(\rvz|\rvx) \| p(\rvz))] - \gamma \KL(q(\rvz)\| \prod_{j=1}^d q(\rz_j)).
\end{align*}
This last term is also known as \textit{total correlation}~\citep{watanabe1960information}. The total correlation is intractable and vanilla Monte Carlo approximations require marginalization over the training set. \citep{kim2018disentangling} propose an estimate using the density ratio trick~\citep{nguyen2010estimating,sugiyama2012density} (FactorVAE). Samples from $\prod_{j=1}^d q(\rz_j)$ can be obtained shuffling samples from $q(\rvz)$~\citep{arcones1992bootstrap}. Concurrently, \citet{chen2018isolating} propose a tractable biased Monte-Carlo estimate for the total correlation ($\beta$-TCVAE).

\paragraph{Disentangled Priors}
\citet{kumar2017variational} argue that a disentangled generative model requires a disentangled prior. This approach is related to the total correlation penalty, but now the aggregated posterior is pushed to match a factorized prior. Therefore
\begin{align*}
&\E_{p(\rvx)} [ \E_{q_\phi(\rvz|\rvx)}[\log p_\theta(\rvx|\rvz)] - \KL(q_\phi(\rvz|\rvx) \| p(\rvz))] - \lambda D(q(\rvz)\|p(\rvz)),
\end{align*}
where $D$ is some (arbitrary) divergence. Since this term is intractable when $D$ is the KL divergence, they propose to match the moments of these distribution. In particular, they regularize the deviation of either $\mathrm{Cov}_{p(\rvx)}[\mu_\phi(\rvx)]$ or $\mathrm{Cov}_{q_\phi}[\rvz]$ from the identity matrix in the two variants of the DIP-VAE. This results in maximizing either the DIP-VAE-I objective
\begin{align*}
\E_{p(\rvx)} [ \E_{q_\phi(\rvz|\rvx)}[\log p_\theta(\rvx|\rvz)] &- \KL(q_\phi(\rvz|\rvx) \| p(\rvz))] - \lambda_{od}\sum_{i\neq j} \left[\mathrm{Cov}_{p(\rvx)}[\mu_\phi(\rvx)]\right]_{ij}^2\\&- \lambda_{d}\sum_{i} \left(\left[\mathrm{Cov}_{p(\rvx)}[\mu_\phi(\rvx)]\right]_{ii} - 1\right)^2
\end{align*}
or the DIP-VAE-II objective
\begin{align*}
\E_{p(\rvx)} [ \E_{q_\phi(\rvz|\rvx)}[\log p_\theta(\rvx|\rvz)] &- \KL(q_\phi(\rvz|\rvx) \| p(\rvz))] - \lambda_{od}\sum_{i\neq j} \left[\mathrm{Cov}_{q_\phi}[\rvz]\right]_{ij}^2\\ &-\lambda_{d}\sum_{i} \left(\left[\mathrm{Cov}_{q_\phi}[\rvz]\right]_{ii} - 1\right)^2.
\end{align*}

\section{Measuring Disentanglement}\label{sec:disent_metrics}
After training a model, we wish to inspect the representation and quantitatively measure its disentanglement. This task is somewhat challenging, as disentanglement does not have a widely accepted formal definition. Therefore, different papers proposed different evaluation metrics (often paired with new methods), leading to inconsistencies on what the best performing methods are. Common to all these metrics is the assumption that at test time we have access to either the full generative model (with the ability to perform interventions) or to a sufficiently large set of observations with annotated ground-truth factors. We now present an overview of the metrics we consider in our empirical studies as well as simple downstream tasks and an approximation of the total correlation we will use in our evaluation. Following our critique on the difficulties of model selection without ground-truth annotations,~\citet{duan2019heuristic} proposed a stability based heuristic, which is posthumous to our studies and is therefore not discussed in this dissertation.

The \emph{BetaVAE} metric~\citep{higgins2016beta} measures disentanglement as the accuracy of a linear classifier that predicts the index of a fixed factor of variation.
\citet{kim2018disentangling} address several issues with this metric in their \emph{FactorVAE} metric by using a majority vote classifier on a different feature vector which accounts for a corner case in the BetaVAE metric.
The \emph{Mutual Information Gap (MIG)}~\citep{chen2018isolating} measures for each factor of variation the normalized gap in mutual information between the highest and second highest coordinate in $r(\rvx)$. Instead, the \emph{Modularity}~\citep{ridgeway2018learning} measures if each dimension of $r(\rvx)$ depends on at most a factor of variation using their mutual information.
The metrics of~\citet{eastwood2018framework}  compute the entropy of the distribution obtained by normalizing the importance of each dimension of the learned representation for predicting the value of a factor of variation. Their disentanglement score (which we call \emph{DCI Disentanglement} for clarity) penalizes multiple factors of variation being captured by the same code and their completeness score (which we call \emph{DCI Completeness}) penalizes a factor of variation being captured by multiple codes.
The \emph{SAP score}~\citep{kumar2017variational} is the average difference of the prediction error of the two most predictive latent dimensions for each factor. 
The \emph{Interventional Robustness Score (IRS)}~\citep{suter2018interventional} measures whether the representation is robustly disentangled by performing interventions on the factors of variations and measuring deviations in the latent space.

\section{\texttt{Disentanglement\_lib}}
The main bulk of our experimental evaluation builds on the methods in Section~\ref{sec:disent_methdos} using the metrics described in Section~\ref{sec:disent_metrics}. In this section, we discuss design choices taken in the \texttt{disentanglement\_lib}\footnote{\url{https://github.com/google-research/disentanglement_lib}} that are common to all experiments in Chapters~\ref{cha:unsup_dis}-\ref{cha:weak}.

\paragraph{Guiding Principles} In our library, we seek controlled, fair and reproducible experimental conditions. 
We consider the case in which we can sample from a well defined and known ground-truth generative model by first sampling the factors of variations from a distribution $P(\rvz)$ and then sampling an observation from $P(\rvx | \rvz)$. 
Our experimental protocol works as follows:
During training, we only observe the samples of $\rvx$ obtained by marginalizing $P(\rvx | \rvz)$ over $P(\rvz)$. After training, we obtain a representation $r(\rvx)$ by either taking a sample from the probabilistic encoder $Q(\rvz|\rvx)$ or by taking its mean. 
Typically, disentanglement metrics consider the latter as the representation $r(\rvx)$. 
During the evaluation, we assume to have access to the whole generative model: we can draw samples from both $P(\rvz)$ and $P(\rvx | \rvz)$. 
In this way, we can perform interventions on the latent factors as required by certain evaluation metrics. 
We explicitly note that we effectively consider the statistical learning problem where we optimize the loss and the metrics on the known data generating distribution. 
As a result, we do not use separate train and test sets but always take \iid samples from the known ground-truth distribution.
This is justified as the statistical problem is well defined and it allows us to remove the additional complexity of dealing with overfitting and empirical risk minimization.

\paragraph{Data Sets} 
We consider five data sets in which $\rvx$ is obtained as a deterministic function of $\rvz$: \textit{dSprites}~\citep{higgins2016beta}, \textit{Cars3D}~\citep{reed2015deep}, \textit{SmallNORB}~\citep{lecun2004learning}, \textit{Shapes3D}~\citep{kim2018disentangling} and we introduced \textit{MPI3D}~\citep{gondal2019transfer}, the first disentanglement dataset with real images (not synthetically generated).
We also introduce three data sets where the observations $\rvx$ are stochastic given the factor of variations $\rvz$: \textit{Color-dSprites}, \textit{Noisy-dSprites} and \textit{Scream-dSprites}. 
In \textit{Color-dSprites}, the shapes are colored with a random color. 
In \textit{Noisy-dSprites}, we consider white-colored shapes on a noisy background. 
Finally, in \textit{Scream-dSprites} the background is replaced with a random patch in a random color shade extracted from the famous \textit{The Scream} painting~\citep{munch_1893}. 
The dSprites shape is embedded into the image by inverting the color of its pixels.

\paragraph{Preprocessing Details} All the data sets contains images with pixels between \num{0} and \num{1}. 
\textit{Color-dSprites:} Every time we sample a point, we also sample a random scaling for each channel uniformly between \num{0.5} and \num{1}. 
\textit{Noisy-dSprites:} Every time we sample a point, we fill the background with uniform noise. 
\textit{Scream-dSprites:} Every time we sample a point, we sample a random $64\times 64$ patch of \textit{The Scream} painting. We then change the color distribution by adding a random uniform number to each channel and divide the result by two. Then, we embed the dSprites shape by inverting the colors of each of its pixels.

\paragraph{Inductive Biases} 
\looseness=-1To fairly evaluate the different approaches, we separate the effect of regularization (in the form of model choice and regularization strength) from the other inductive biases (for example, the choice of the neural architecture). 
Each method uses the same convolutional architecture, optimizer, hyperparameters of the optimizer and batch size. 
All methods use a Gaussian encoder where the mean and the log variance of each latent factor is parametrized by the deep neural network, a Bernoulli decoder and latent dimension fixed to 10. 
We note that these are all standard choices in prior work \citep{higgins2016beta,kim2018disentangling}.
We choose six different regularization strengths, that is, hyperparameter values, for each of the considered methods.
The key idea was to take a wide enough set to ensure that there are useful hyperparameters for different settings for each method and not to focus on specific values known to work for specific data sets.
However, the values are partially based on the ranges that are prescribed in the literature (including the hyperparameters suggested by the authors). 
We fix our experimental setup in advance and we run all the considered methods on each data set for 50 different random seeds (unless otherwise specified) and evaluate them on the considered metrics. 

\paragraph{Hyperparameters and Differences with Previous Implementations}
In our study, we fix all hyperparameters except one per each model. 
Model specific hyperparameters can be found in Table~\ref{table:sweep_main}. 
All the other hyperparameters were not varied and are selected based on the literature, see~\citep{locatello2019challenging} for the detailed values. 
We use a single choice of architecture, batch size and optimizer for all the methods which might deviate from the settings considered in the original papers.
However, we argue that unification of these choices is the only way to guarantee a fair comparison among the different methods such that valid conclusions may be drawn in between methods.
The largest change is that for DIP-VAE and for $\beta$-TCVAE we used a batch size of 64 instead of 400 and 2048 respectively.
However, \citet{chen2018isolating} shows in Section H.2 of the Appendix that the bias in the mini-batch estimation of the total correlation does not considerably affect the performances of their model even with small batch sizes.
For DIP-VAE-II, we did not implement the additional regularizer on the third order central moments since no implementation details are provided and since this regularizer is only used on specific data sets.

Our implementations of the disentanglement metrics deviate from the implementations in the original papers as follows:
First, we strictly enforce that all factors of variations are treated as discrete variables as this corresponds to the assumed ground-truth model in all our data sets.
Hence, we used classification instead of regression for the SAP score and the disentanglement score of~\citep{eastwood2018framework}.
This is important as it does not make sense to use regression on true factors of variations that are discrete (for example on shape on dSprites).
Second, wherever possible, we resorted to using the default, well-tested Scikit-learn~\citep{scikitlearn} implementations instead of using custom implementations with potentially hard to set hyperparameters.
Third, for the Mutual Information Gap~\citep{chen2018isolating}, we estimate the \textit{discrete} mutual information (as opposed to continuous) on the \textit{mean} representation (as opposed to sampled) on a \textit{subset} of the samples (as opposed to the whole data set).
We argue that this is the correct choice as the mean is usually taken to be the representation.
Hence, it would be wrong to consider the full Gaussian encoder or samples thereof as that would correspond to a different representation.
Finally, we fix the number of sampled train and test points across all metrics to a large value to ensure robustness.

\begin{table*}
\centering
\caption{\small Model's hyperparameters. We allow a sweep over a single hyperparameter for each model.}
\vspace{2mm}
\begin{tabular}{l l  l}
\toprule
\textbf{Model} & \textbf{Parameter} & \textbf{Values}\\
\midrule 
$\beta$-VAE & $\beta$ & $[1,\ 2,\ 4,\ 6,\ 8,\ 16]$\\
AnnealedVAE & $c_{max}$ & $[5,\ 10,\ 25,\ 50,\ 75,\ 100]$\\
& iteration threshold & $100000$\\
& $\gamma$ & $1000$\\
FactorVAE & $\gamma$ & $[10,\ 20,\ 30,\ 40,\ 50,\ 100]$\\
DIP-VAE-I & $\lambda_{od}$ & $[1,\ 2,\ 5,\ 10,\ 20,\ 50]$\\
&$\lambda_{d}$ & $10\lambda_{od}$\\
DIP-VAE-II & $\lambda_{od}$ & $[1,\ 2,\ 5,\ 10,\ 20,\ 50]$\\
&$\lambda_{d}$ & $\lambda_{od}$\\
$\beta$-TCVAE & $\beta$ & $[1,\ 2,\ 4,\ 6,\ 8,\ 10]$\\
\bottomrule
\end{tabular}
\label{table:sweep_main}
\end{table*}

\paragraph{Limitations}
While we aim to provide useful and fair experimental studies, there are clear limitations to the conclusions that can be drawn from them due to design choices that we have taken.
In all these choices, we have aimed to capture what is considered the state-of-the-art inductive bias in the community.

On the data set side, we only consider images with a heavy focus on synthetic images. Our MPI3D~\citep{gondal2019transfer} is the only real world dataset, but the images are still taken in a very controlled setting where ground-truth factors are known and perfectly independent (excluding camera aberrations).
We do not explore other modalities and we only consider the toy scenario in which we have access to a data generative process with uniformly distributed factors of variations. 
Furthermore, all our data sets have a small number of independent discrete factors of variations without any confounding variables.

For the methods, we only consider the inductive bias of convolutional architectures. 
We do not test fully connected architectures or additional techniques such as skip connections. 
Furthermore, we do not explore different activation functions, reconstruction losses or different number of layers. 
We also do not vary any other hyperparameters other than the regularization weight. 
In particular, we do not evaluate the role of different latent space sizes, optimizers and batch sizes. 

Implementing the different disentanglement methods and metrics has proven to be a difficult endeavour.
Few ``official'' open source implementations are available and there are many small details to consider.
We take a best-effort approach to these implementations and implemented all the methods and metrics from scratch as any sound machine learning practitioner might do based on the original papers.
When taking different implementation choices than the original papers, we explicitly state and motivate them. 



\chapter{Unsupervised Learning of Disentangled Representations}\label{cha:unsup_dis}
In this chapter, we discuss the unsupervised learning of disentangled representations. The presented work is based on \citep{locatello2019challenging,locatello2020sober,locatello2020commentary} and was developed in collaboration with Stefan Bauer, Mario Lucic, Gunnar R\"atsch, Sylvain Gelly, Bernhard Sch\"olkopf, and Olivier Bachem. Olivier Bachem did the first sketch of the impossibility result and Francesco Locatello contributed to the final version and the theorem statement.  The \texttt{disentanglement\_lib} was done by Francesco Locatello and Olivier Bachem. All this work was partially done when Francesco Locatello was at Google Research, Brain Team in Zurich.

\section{Theoretical Impossibility}\label{sec:impossibility}
The first question that we investigate is whether unsupervised disentanglement learning is even possible for arbitrary ground-truth generative models.
Theorem~\ref{thm:impossibility} shows that without inductive biases both on models and data set, the task is fundamentally impossible.
\begin{theorem}
\label{thm:impossibility}
For $d>1$, let $\rvz\sim\P$ denote any distribution which admits a density $p(\rvz)=\prod_{i=1}^dp(\rz_i)$.
Then, there exists an infinite family of bijective functions $f:\supp(\rvz)\to\supp(\rvz)$ such that $\frac{\partial f_i(\vu)}{\partial u_j} \neq 0$ almost everywhere for all $i$ and $j$ (implying that $\rvz$ and $f(\rvz)$ are completely entangled) and $\P(\rvz \leq \vu) = \P(f(\rvz) \leq \vu)$ for all $\vu\in\supp(\rvz)$ (they have the same marginal distribution). 
\end{theorem}
The proof sketch is presented in Figure~\ref{fig:proof_sketch}. 
The key idea is that we can construct two generative models whose latent variables $\rvz$ and $f(\rvz)$ are entangled with each other. If a representation is disentangled with one of these generative models it must be entangled with the other by construction: all the entries in the Jacobian of $f$ are non-zero, so a change in a single dimension of $\rvz$ implies that all dimensions of $\hat{\rvz}$ change.
Since $f$ is deterministic and $p(\rvz)=p(\hat{\rvz})$ almost everywhere, both generative models have the same marginal distribution of the observations $\rvx$ by construction, that is, $P(\rvx) = \int p(\rvx| \rvz)p(\rvz) d\rvz = \int p(\rvx|\hat{\rvz})p(\hat{\rvz}) d\hat{\rvz}$. It is impossible to distinguish which model $r(\rvx)$ should disentangle only observing only samples from $P(\rvx)$: both $\rvz$ and $f(\rvz)$ are equally plausible and ``look the same'' as they produce the same $\rvx$ with the same probability. 

\begin{figure}[ht]
    \centering
    \includegraphics[width=0.7\textwidth]{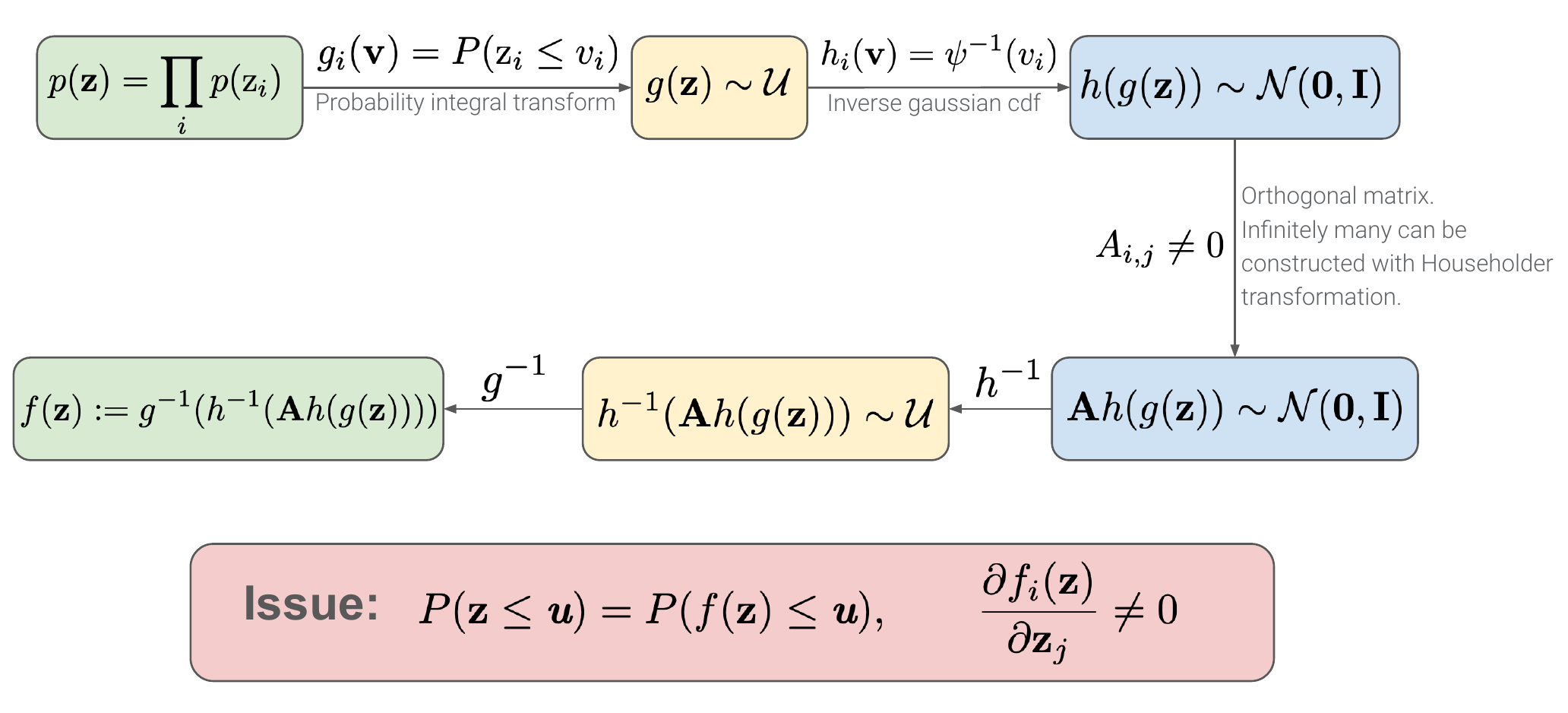}
    \caption{\small Proof sketch for the impossibility result. We construct a family of transformations $f$ that preserve the marginal distribution of the latent variable but has a non-zero jacobian almost everywhere, which ensures that $\rvz$ and $f(\rvz)$ are entangled with each other.}
    \label{fig:proof_sketch}
\end{figure}

This may not be surprising to readers familiar with the causality and ICA literature as it is consistent with the following argument: 
After observing $\rvx$, we can construct infinitely many generative models that have the same marginal distribution of $\rvx$.
Any of these models could be the true causal generative model for the data, and the right model cannot be identified given only the distribution of $\rvx$ \citep{peters2017elements}.
Similar results have been obtained in the context of non-linear ICA~\citep{hyvarinen1999nonlinear}. 
The main novelty of Theorem~\ref{thm:impossibility} is that it allows the explicit construction of latent spaces $\rvz$ and $\hat{\rvz}$ that are completely \textit{entangled} with each other in the sense of~\citep{bengio2013representation}.
We note that while this result is very intuitive for multivariate Gaussians it also holds for distributions that are not invariant to rotation, such as multivariate uniform distributions. The classical result of~\citep{hyvarinen1999nonlinear} differs from ours as we specifically show the entanglement of these equivalent models. On the technical side, (i) we do not assume that $\P(\rvx|\rvz)$ is deterministic, (ii) do not restrict $\rvx$ to be in the same space as $\rvz$, (iii) we allow for any prior that admits a factorizing density, (iv) the function $g$ in~\citep{hyvarinen1999nonlinear} maps from $\rvx$ to $\rvz$ whereas our $f$ maps from $\rvz$ to $\hat{\rvz}$. Conceptually, the function $g$ constructs alternative solutions to the non-linear ICA problem. Instead, we show that there are infinitely many completely entangled generative models for the same data. Regardless of which method is used for disentanglement, a model cannot be disentangled to all of them.

Theorem~\ref{thm:impossibility} implies that the unsupervised learning of disentangled representation is impossible for arbitrary data sets. Even in the infinite data regime, where supervised learning algorithms like k-nearest neighbors classifiers are consistent, no model can find a disentangled representation observing samples from $P(\rvx)$ only. This theoretical result motivates the need for either implicit supervision, explicit supervision, or suitable inductive biases so that the correct solution is naturally preferred. We remark that  Theorem~\ref{thm:impossibility} holds for arbitrary data sets and does not account for the structure that real-world generative models may exhibit. On the other hand, we clearly show that inductive biases are required both for the models (so that we find a specific set of solutions) and the data sets (such that these solutions match the true generative model).

\section{Can We Learn Disentangled Representations Without Supervision?}\label{sec:learning}
In this section, we provide a sober look at the performances of state-of-the-art approaches and investigate how effectively we can learn disentangled representations without looking at the labels. We focus our analysis on key questions for practitioners interested in learning disentangled representations reliably and without supervision.

\begin{figure}[p]
\centering\includegraphics[width=\textwidth]{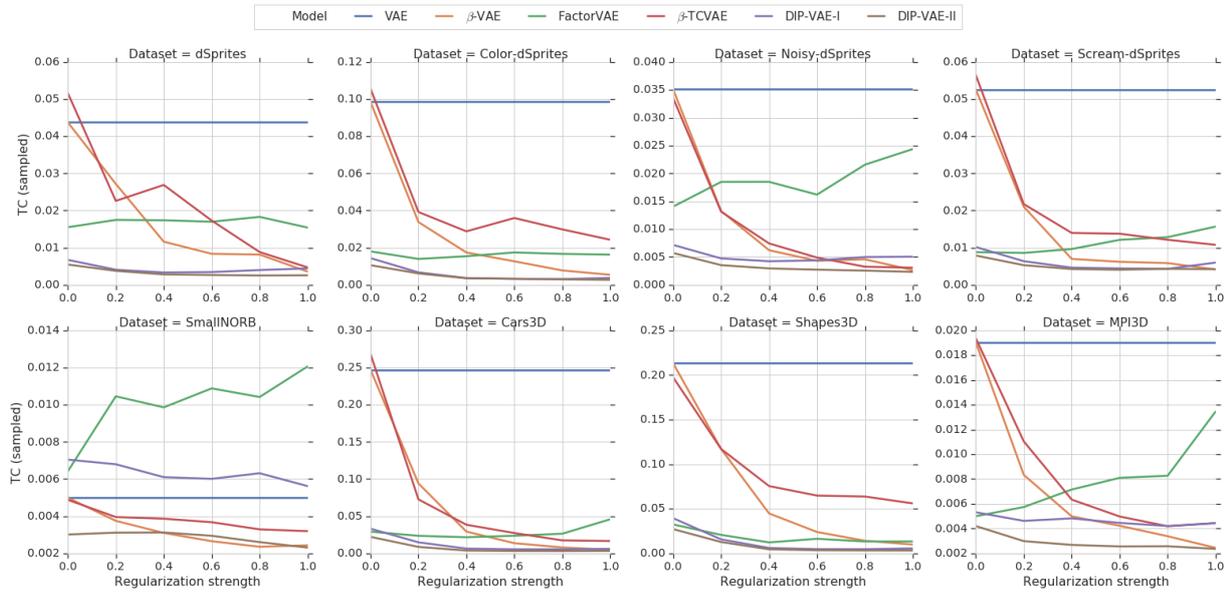}
\caption{\small Total correlation of sampled representation plotted against regularization strength for different data sets and approaches (except AnnealedVAE). The total correlation of the sampled representation decreases as the regularization strength is increased.}\label{figure:TCsampled}
\end{figure}
\begin{figure}[p]
\centering\includegraphics[width=\textwidth]{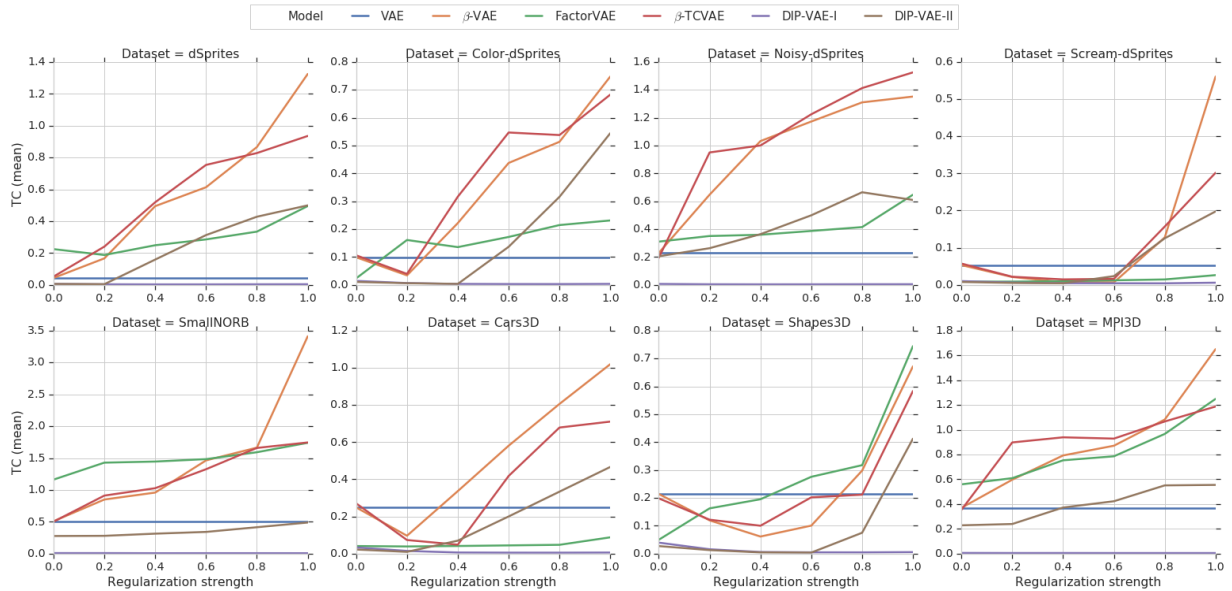}
\caption{\small Total correlation of mean representation plotted against regularization strength for different data sets and approaches (except AnnealedVAE). The total correlation of the mean representation does not necessarily decrease as the regularization strength is increased.}\label{figure:TCmean}
\end{figure}
\begin{figure}[pt]
\centering\includegraphics[width=\textwidth]{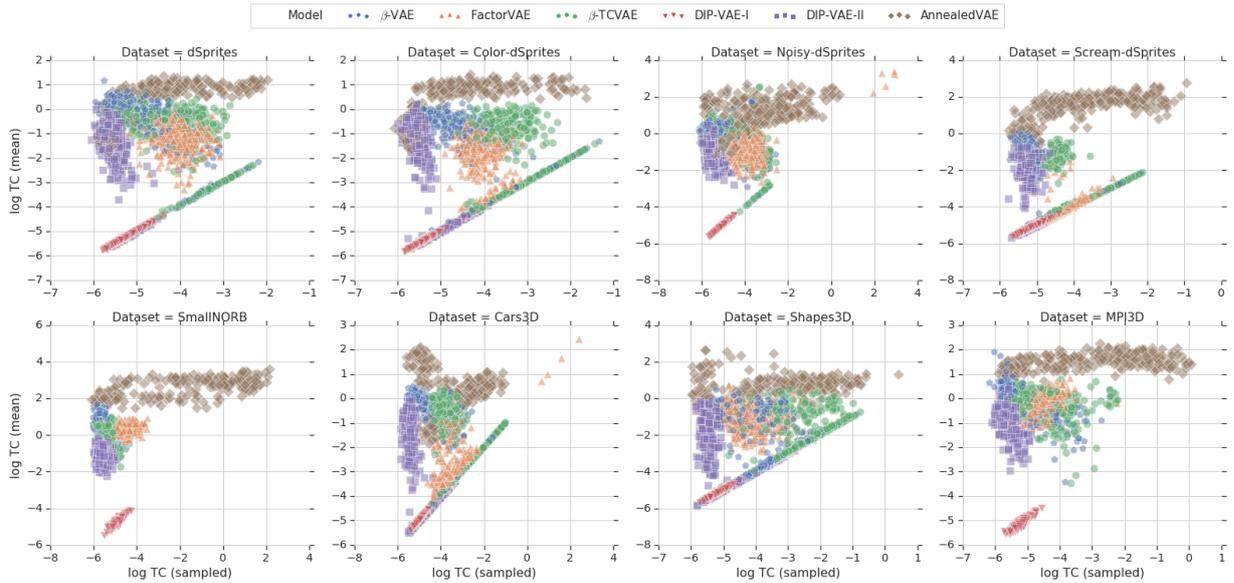}
\caption{\small Log total correlation of mean vs sampled representations. For a large number of models, the total correlation of the mean representation is higher than that of the sampled representation.}\label{figure:TCmeansampled}
\end{figure}

\subsection{Can Current Methods Enforce a Uncorrelated Aggregated Posterior and Representation?}\label{app:factorizing_q_tc}
We investigate whether the considered unsupervised disentanglement approaches are effective at enforcing a factorizing and, thus, uncorrelated aggregated posterior.
For each trained model, we sample $\num{10000}$ images and sample from the corresponding approximate posterior. 
We then fit a multivariate Gaussian distribution over these \num{10000} samples by computing the empirical mean and covariance matrix.
Finally, we compute the total correlation of the fitted Gaussian and report the median value for each data set, method and hyperparameter value.

Figure~\ref{figure:TCsampled} shows the total correlation of the sampled representation plotted against the regularization strength for each data set and method except AnnealedVAE.
\looseness=-1Overall, we observe that plain vanilla variational autoencoders (the $\beta$-VAE model with $\beta=1$) typically exhibit the highest total correlation (with the exception of DIP-VAE-I and FactorVAE on SmallNORB). For the other models, the total correlation of the sampled representation generally decreases on all data sets as the regularization strength is increased (with the exception of FactorVAE).
We did not report results for AnnealedVAE, as it is much more sensitive to the regularization strength.

While many of the considered methods aim to enforce a factorizing aggregated posterior, they use the mean vector of the Gaussian encoder as the representation and not a sample from the Gaussian encoder.
This may seem like a minor, irrelevant modification; however, it is not clear whether a factorizing aggregated posterior also ensures that the dimensions of the mean representation are uncorrelated.
To test whether this is true, we compute the same total correlation based on the mean representation (as opposed to sampled). 
Figure~\ref{figure:TCmean} shows the total correlation of the mean representation plotted against the regularization strength for each data set and method except AnnealedVAE.
We observe that, generally, increased regularization leads to an increased total correlation of the mean representations. DIP-VAE-I optimizes the covariance matrix of the mean representation to be diagonal which implies that the corresponding total correlation (as we compute it) is low. The DIP-VAE-II objective enforces the covariance matrix of the sampled representation to be diagonal, which seems to lead to a factorized mean representation on some data sets (for example, Shapes3D), but also seems to fail on others (dSprites, MPI3D).
In Figure~\ref{figure:TCmeansampled}, we further plot the log total correlations of the sampled representations versus the mean representations for each of the trained models.
It can be clearly seen that for a large number of models, the total correlation of the mean representations is much higher than that of the sampled representations. 

\paragraph{Implications}
Overall, these results lead us to conclude with minor exceptions that the considered methods are effective at enforcing an aggregated posterior whose individual dimensions are not correlated but that this does not seem to imply that the dimensions of the mean representation (usually used for representation) are uncorrelated.

\begin{figure}[p]
\centering\includegraphics[width=\textwidth]{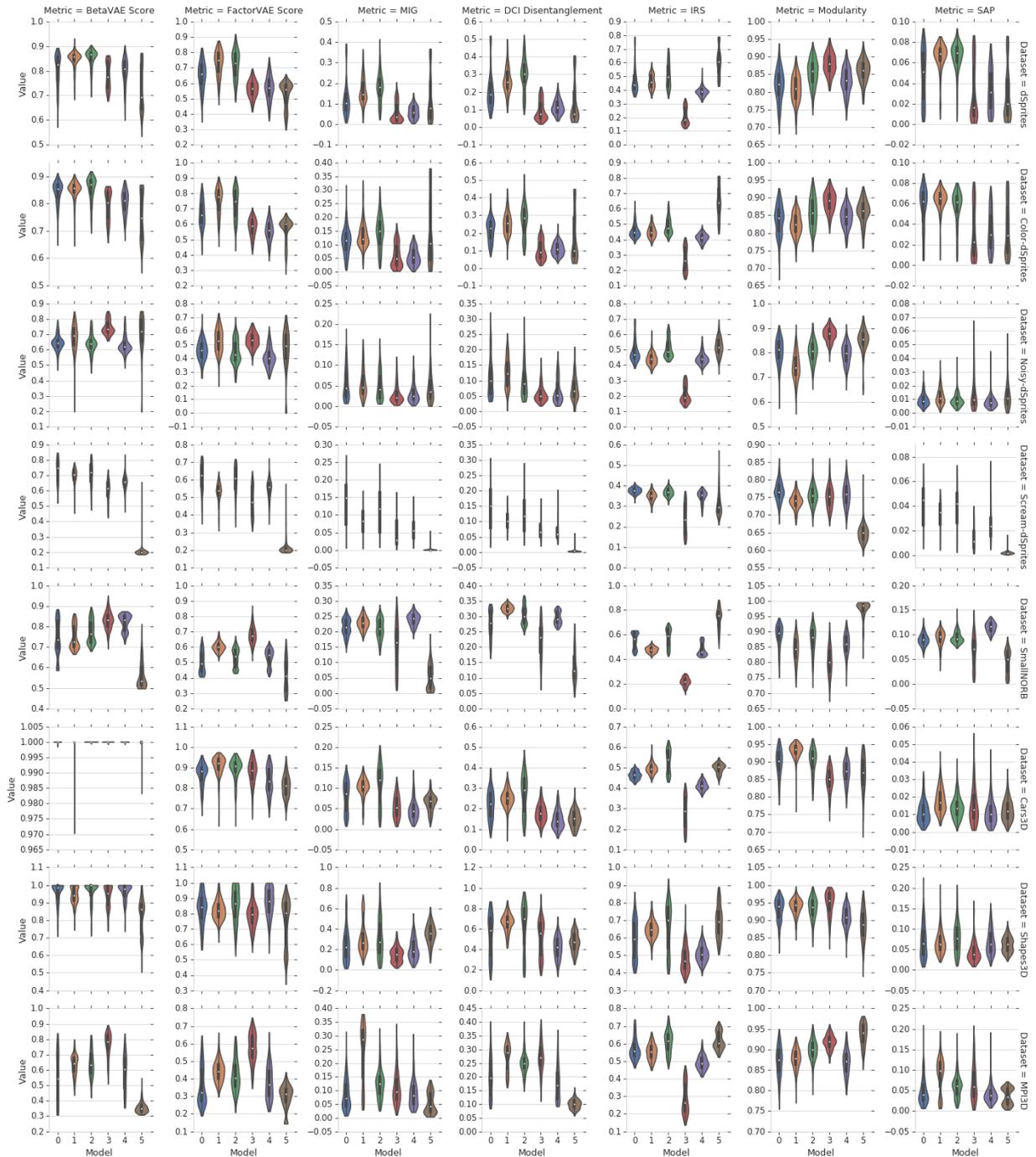}
\caption{\small Score for each method for each score (column) and data set (row) with different hyperparameters and random seed. Models are abbreviated (0=$\beta$-VAE, 1=FactorVAE, 2=$\beta$-TCVAE, 3=DIP-VAE-I, 4=DIP-VAE-II, 5=AnnealedVAE). The scores are heavily overlapping and we do not observe a consistent pattern. We conclude that hyperparameters and random seed matter more than the model choice.}\label{figure:score_vs_method}
\end{figure}
\begin{figure}[ht]
\centering\includegraphics[width=\textwidth]{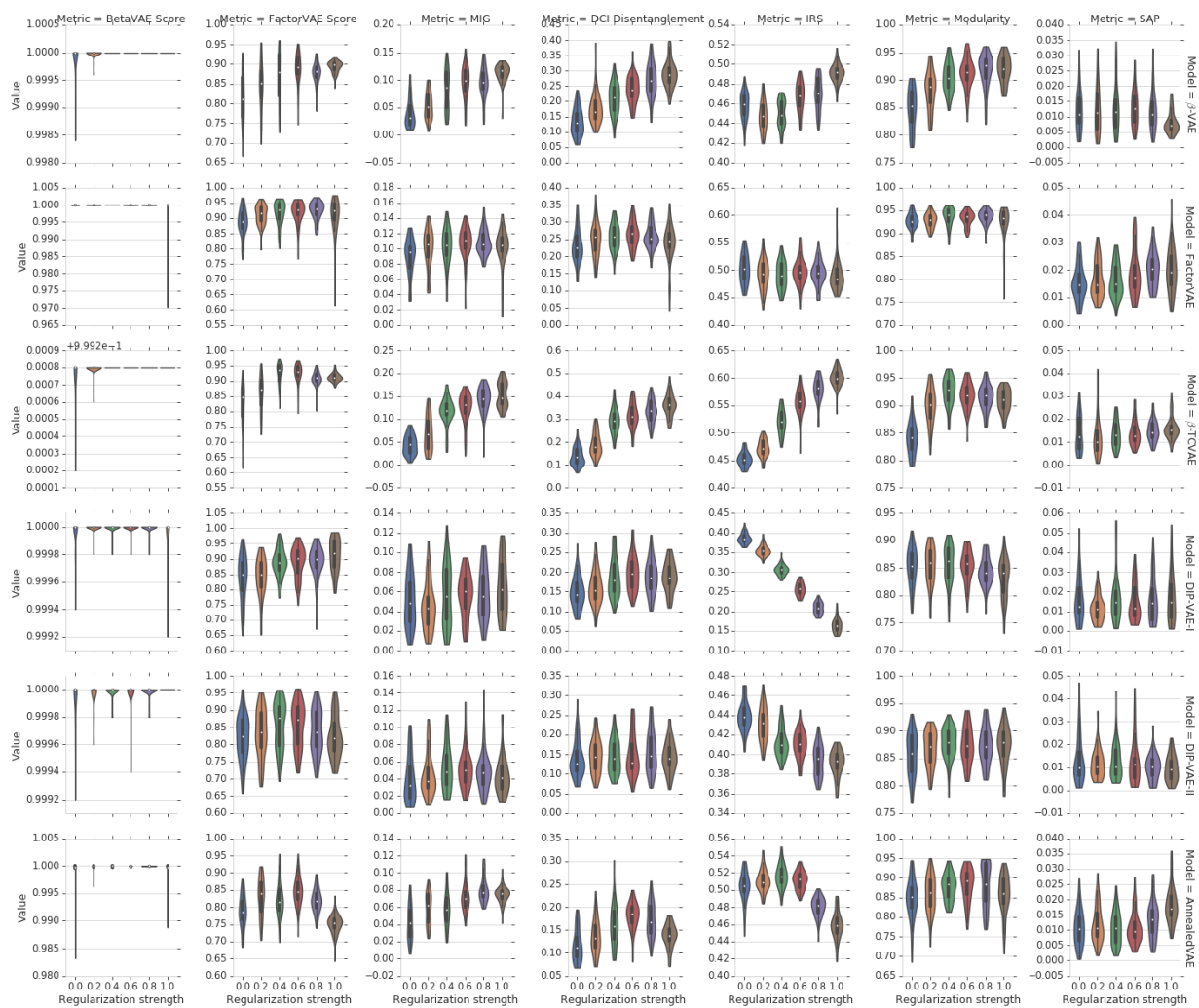}
\caption{\small Distribution of scores for different models, hyperparameters and regularization strengths on Cars3D. We clearly see that randomness (in the form of different random seeds) has a substantial impact on the attained result and that a good run with a bad hyperparameter can beat a bad run with a good hyperparameter in many cases. IRS seem to be an exception on some data sets.}\label{figure:random_seed_effect_Cars3D}
\end{figure}

\subsection{Which Method and Hyperparameter Should be Used?}\label{app:hyper_importance}
The first question a practitioner may face is how disentanglement is affected by the model choice, the hyperparameter selection, and randomness (in the form of different random seeds).
To investigate this, we compute all the considered disentanglement metrics for each of our trained models.
In Figure~\ref{figure:score_vs_method}, we show the range of attainable disentanglement scores for each method on each data set varying the regularization strength and the random seed.
We observe that these ranges are heavily overlapping for different models leading us to (qualitatively) conclude that the choice of hyperparameters and the random seed seems to be substantially more important than the objective function. 
We remark that in our study, we have fixed the range of hyperparameters a priori to six different values for each model and did not explore additional hyperparameters based on the results (as that would bias our study).
However, this also means that specific models may have performed better than in Figure~\ref{figure:score_vs_method} if we had chosen a different set of hyperparameters.

In Figure~\ref{figure:random_seed_effect_Cars3D}, we further show the impact of randomness in the form of random seeds.
Each violin plot shows the distribution of each disentanglement metric across all 50 trained models for each model and hyperparameter setting on Cars3D.
We clearly see that randomness (in the form of different random seeds) has a substantial impact on the attained result. A good run with a bad hyperparameter can beat a bad run with a good hyperparameter in many cases. We note that IRS seem to exhibit a clearer trend on some data sets. However, we remark that IRS primarily measure robustness and is often at odds with the other disentanglement metrics as discussed in Chapter~\ref{cha:eval_dis}.

To quantify these claims, we perform a variance analysis by predicting the different disentanglement scores with ordinary least squares for each data set:
If we allow the score to depend only on the objective function (categorical variable), we can only explain $37\%$ of the variance of the scores on average. 
Similarly, if the score depends on the Cartesian product of objective function and regularization strength (again categorical), we can explain $59\%$ of the variance while the rest is due to the random seed. 

\paragraph{Implications} 
The disentanglement scores of unsupervised models are heavily influenced by randomness (in the form of the random seed) and hyperparameter's choice (in the form of the regularization strength). The objective function appears to have less impact. Selecting good hyperparameters and good runs seem to be the most important.

\subsection{Are There Reliable Recipes for Model Selection?}\label{app:hyper_selection}
\looseness=-1In light of the results of Section~\ref{app:hyper_importance}, we investigate how to choose good hyperparameters how we can distinguish between good and bad training runs. 
We advocate that model selection \emph{should not} depend on the considered disentanglement score for the following reasons:
The point of unsupervised learning of disentangled representation is that there is no access to the labels as otherwise we could incorporate them and would have to compare to semi-supervised and fully supervised methods as we do in Chapter~\ref{cha:semi_sup}.
All the disentanglement metrics considered in this chapter require a considerable amount of ground-truth labels or even the full generative model (for example, for the BetaVAE and the FactorVAE metric).
Hence, one may substantially bias the results of a study by tuning hyperparameters based on (supervised) disentanglement metrics. 
Furthermore, we argue that it is not sufficient to fix a set of hyperparameters \emph{a priori} and then show that one of those hyperparameters and a specific random seed achieves a good disentanglement score as it amounts to showing the existence of a good model, but does not guide the practitioner in finding it.
Finally, in many practical settings, we might not even have access to adequate labels. It may be hard to identify the true underlying factor of variations, particularly if we consider data modalities that are less suitable to human interpretation than images.
In this study, we focus on choosing the learning model and the regularization strength corresponding to that loss function.
However, we note that in practice this problem is likely even harder as a practitioner might also want to tune other modeling choices such architecture or optimizer.
 
\subsubsection{General Recipes for Hyperparameter Selection}
We first investigate whether we may find generally applicable ``rules of thumb'' for choosing the hyperparameters.
For this, we plot in Figure~\ref{figure:score_vs_hyp} different disentanglement metrics against different regularization strengths for each model and each data set.
The values correspond to the median obtained values across 50 random seeds for each model, hyperparameter, and data set.
There seems to be no model dominating all the others and, for each model, there does not seem to be a consistent strategy in choosing the regularization strength to maximize disentanglement scores.
Furthermore, even if we could identify a good objective function and corresponding hyperparameter value, we still could not distinguish between a good and a bad training run.
\begin{figure}[ht]
\centering\includegraphics[width=\textwidth]{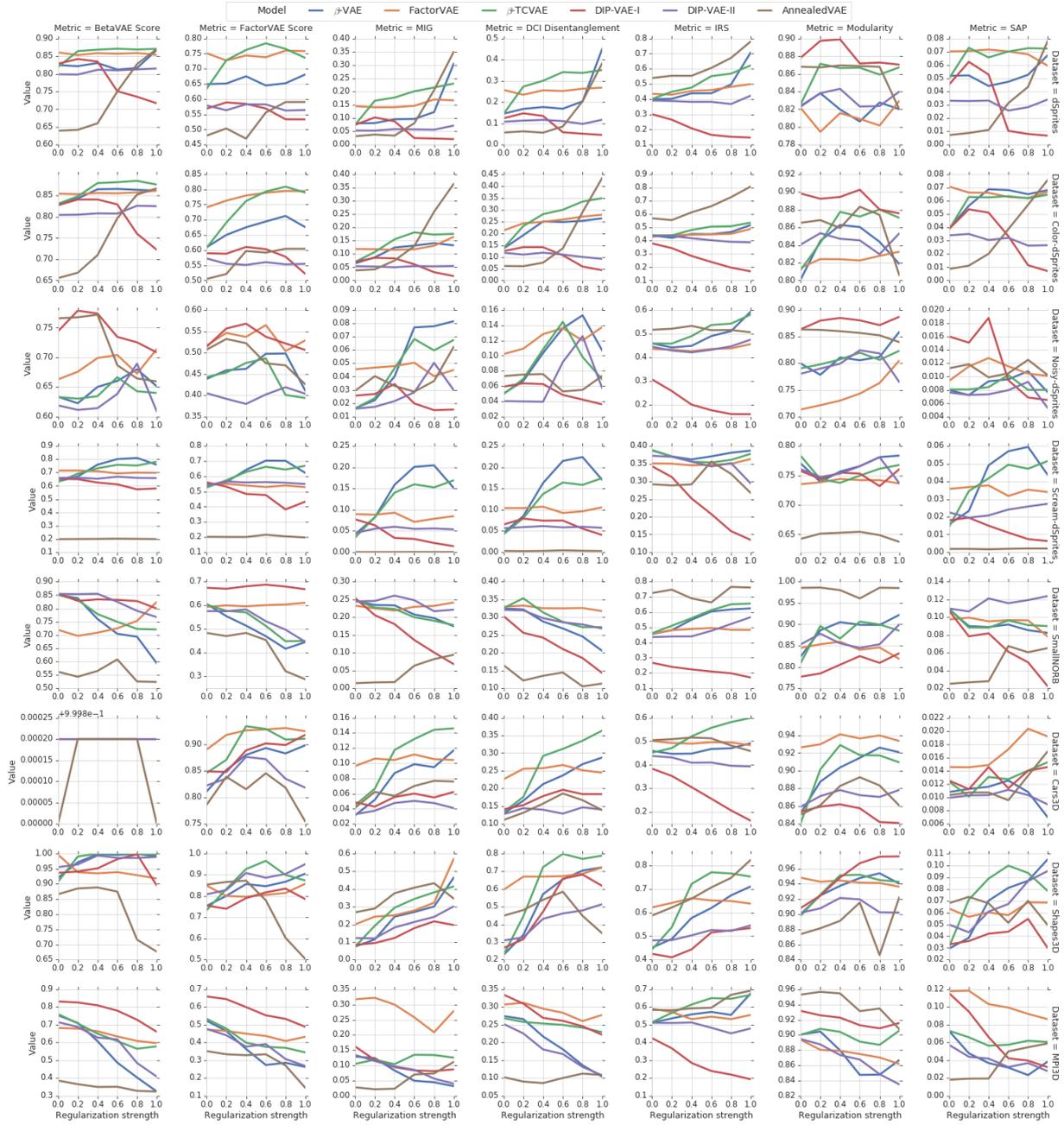}
\caption{\small Score vs hyperparameters for each score (column) and data set (row). There seems to be no model dominating all the others and for each model there does not seem to be a consistent strategy in choosing the regularization strength.}\label{figure:score_vs_hyp}
\end{figure}

\subsubsection{Model Selection Based on Unsupervised Scores}
Another approach could be to select hyperparameters based on unsupervised scores such as the reconstruction error, the KL divergence between the prior and the approximate posterior, the Evidence Lower Bound or the estimated total correlation of the sampled representation.
This would have the advantage that we could select specific trained models and not just good hyperparameter settings whose median trained model would perform well.
To test whether such an approach is fruitful, we compute the rank correlation between these unsupervised metrics and the disentanglement metrics and present it in Figure~\ref{figure:unsupervised_metrics}.
While we observe some correlations, no clear pattern emerges, which leads us to conclude that this approach is unlikely to be successful in practice.
\begin{table}[t]
    \centering
    \vspace{2mm}
\begin{tabular}{lrr}
\toprule
{} &  Random different data set &  Same data set \\
\midrule
Random different metric &                      52.7\% &          62.1\% \\
Same metric             &                      59.6\% &          81.9\% \\
\bottomrule
\end{tabular}
\caption{\small Probability of outperforming random model selection on a different random seed. A random disentanglement metric and data set is sampled and used for model selection. That model is then compared to a randomly selected model: (i) on the same metric and data set, (ii) on the same metric and a random different data set, (iii) on a random different metric and the same data set, and (iv) on a random different metric and a random different data set. The results are averaged across $\num{10000}$ random draws.}\label{table:app_random_model}
\end{table}

\begin{figure}[p]
\centering\includegraphics[width=\textwidth]{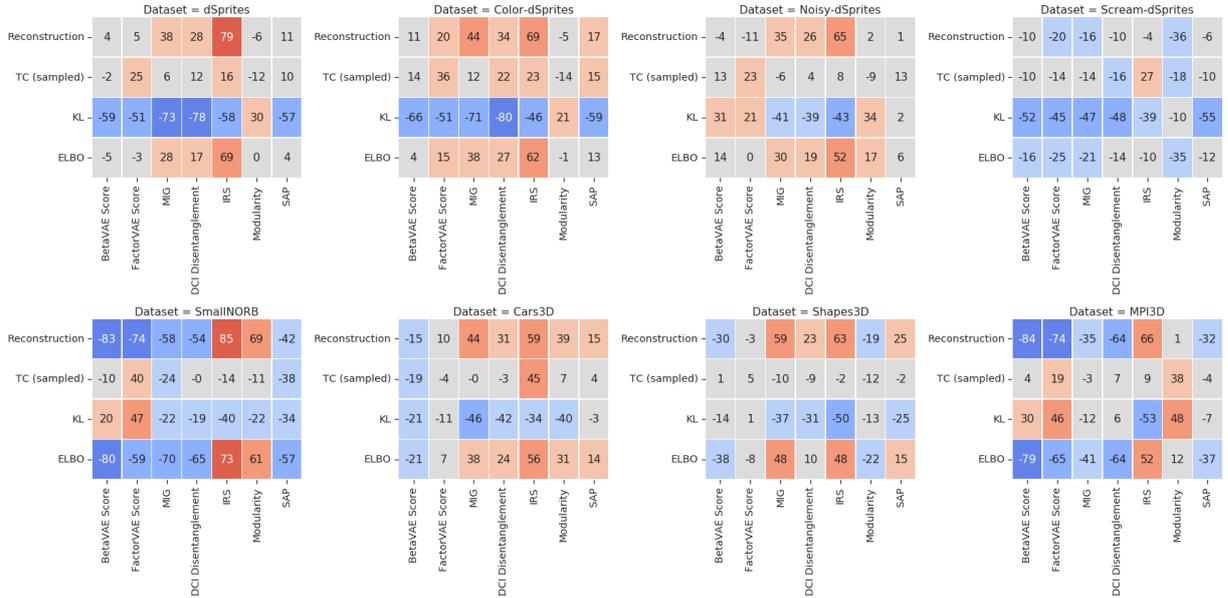}
\caption{\small Rank correlation between unsupervised scores and supervised disentanglement metrics. 
The unsupervised scores we consider do not seem to be useful for model selection.}\label{figure:unsupervised_metrics}
\end{figure}
\begin{figure}[p]
\centering\includegraphics[width=\textwidth]{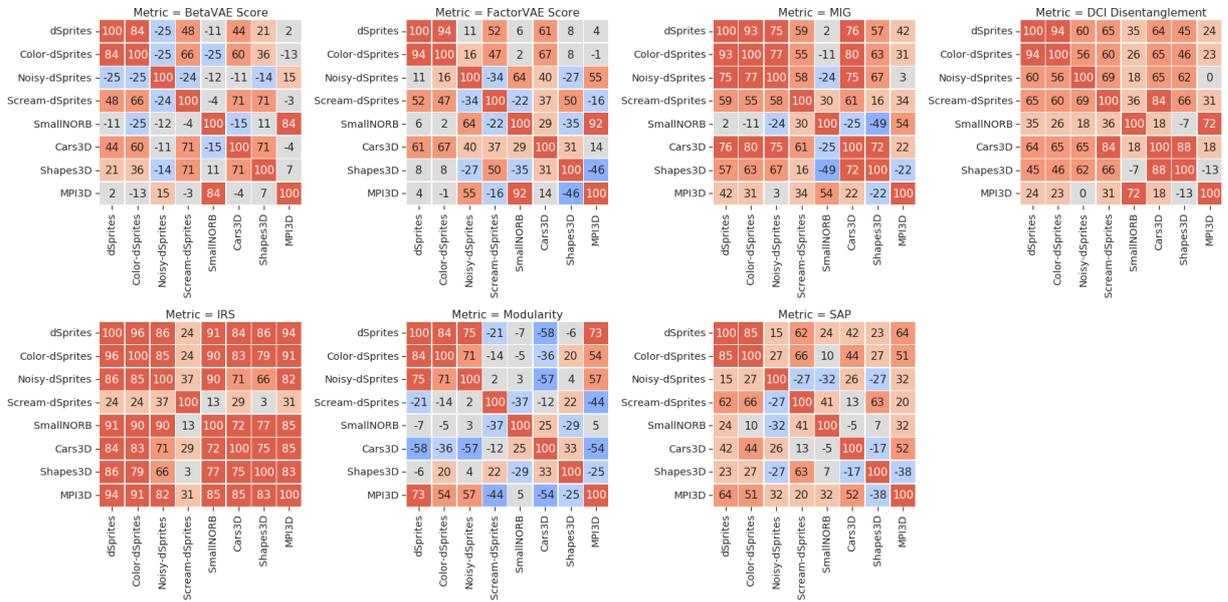}
\caption{\small Rank-correlation of different disentanglement metrics across different data sets. Good hyperparameters only seem to transfer between dSprites and Color-dSprites but not in between the other data sets.}\label{figure:dataset_vs_dataset}
\end{figure}


\subsubsection{Hyperparameter Selection Based on Transfer}
The final strategy for hyperparameter selection that we consider is based on transferring good settings across data sets.
The key idea is that good hyperparameter settings may be inferred on data sets where we have labels available (such as dSprites) and then applied to novel data sets.
In Figure~\ref{figure:dataset_vs_dataset}, we shows the rank correlations obtained between different data sets for each disentanglement scores.
While these result suggest that some transfer of hyperparameters is possible, it does not allow us to distinguish between good and bad random seeds on the target data set.

To illustrate this, we compare such a transfer based approach to hyperparameter selection to random model selection as follows: 
We first randomly sample one of our 50 random seeds and consider the set of trained models with that random seed. 
First, we sample one of our 50 random seeds, a random  disentanglement metric, and a data set and use them to select the hyperparameter setting with the highest attained score.
Then, we compare that selected hyperparameter setting to a randomly selected model on either the same or a random different data set, based on either the same or a random different metric and for a randomly sampled seed.
Finally, we report the percentage of trials in which this transfer strategy outperforms or performs equally well as random model selection across $\num{10000}$ trials in Table~\ref{table:app_random_model}.
If we choose the same metric and the same data set (but a different random seed), we obtain a score of $81.9\%$.
If we aim to transfer for the same metric across data sets, we achieve around $59.6\%$.
Finally, if we transfer both across metrics and data sets, our performance drops to $52.7\%$. The drop in performance transferring hyperparameters across different metrics may be interpreted in light of the results of Section~\ref{sec:eval_same_concept}.

\paragraph{Implications} 
\looseness=-1Unsupervised model selection remains an unsolved problem. 
Transfer of good hyperparameters between metrics and data sets does not seem to work as there appears to be no unsupervised way to distinguish between good and bad random seeds on the target task. Recent work~\citep{duan2019heuristic} may be used to select stable hyperparameter configurations. The IRS score seems to be more correlated with the unsupervised training metrics on most data sets and generally transfer the hyperparameters better. However, as we shall see in Section~\ref{sec:metrics_agreement}, IRS is not very correlated with the other disentanglement metrics.

\begin{figure}[pt]
\centering\includegraphics[width=\textwidth]{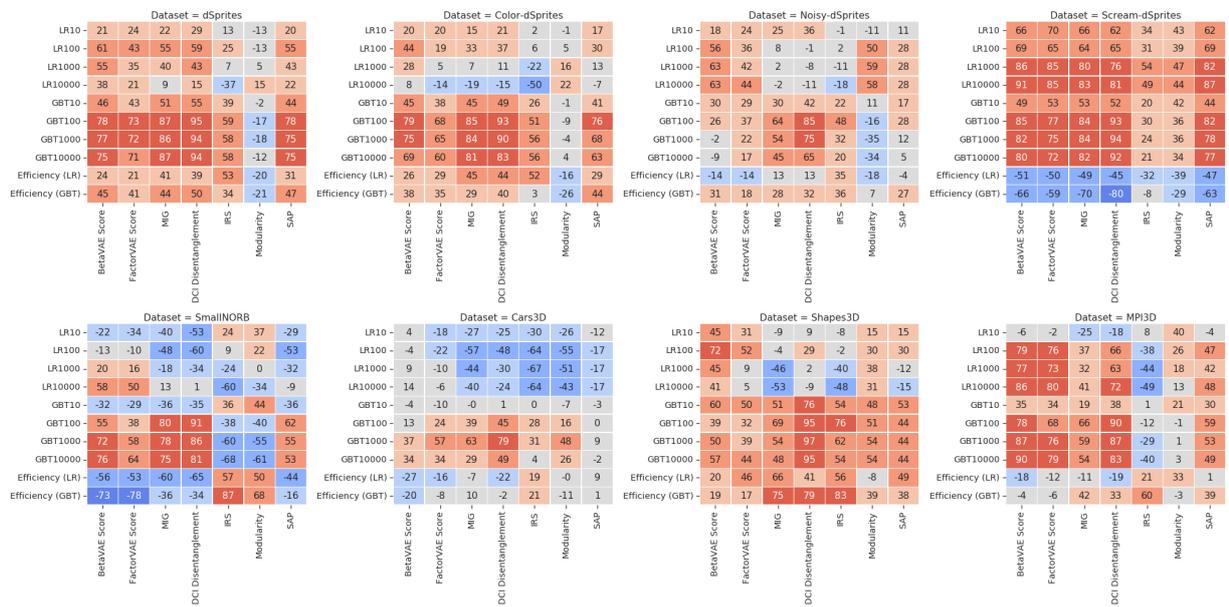}
\caption{\small Rank-correlation between the metrics and the performance on downstream task on different data sets. We observe some correlation between most disentanglement metrics and downstream performance. However, the correlation varies across data sets.
}\label{figure:downstream_tasks}
\end{figure}

\section[Downstream Accuracy and Sample Complexity]{Downstream Accuracy and Sample Complexity}\label{sec:downstream}
One of the main folklore arguments in favor of disentangled representations is that they should be useful for downstream tasks.
In particular, it is often argued that disentanglement should lead to a better sample complexity of learning~\citep{bengio2013representation,scholkopf2012causal,peters2017elements}.
In this section, we consider the simplest downstream classification task where the goal is to recover the true factors of variations from the learned representation using either multi-class logistic regression (LR) or gradient boosted trees (GBT).
We investigate the relationship between disentanglement and the average classification accuracy on these downstream tasks as well as whether better disentanglement leads to a decreased sample complexity of learning.
To compute the classification accuracy for each trained model, we sample true factors of variations and observations from our ground-truth generative models.
We then feed the observations into our trained model and take the mean of the Gaussian encoder as the representations.
Finally, we predict each of the ground-truth factors based on the representations with a separate learning algorithm.
We consider both a 5-fold cross-validated multi-class logistic regression as well as gradient boosted trees of the Scikit-learn package.
For each of these methods, we train on $\num{10}$, $\num{100}$, $\num{1000}$ and $\num{10000}$ samples.
We compute the average accuracy across all factors of variation using an additional set $\num{10000}$ randomly drawn samples.

\looseness=-1Figure~\ref{figure:downstream_tasks} shows the rank correlations between the disentanglement metrics and the downstream performance for all considered data sets.
We observe that all metrics except Modularity seem to be correlated with increased downstream performance on the different variations of dSprites and to some degree on Shapes3D. 
However, it is not clear whether this is due to the fact that disentangled representations perform better or whether some of these scores actually also (partially) capture the \textit{informativeness} of the evaluated representation. Furthermore, the correlation is weaker or inexistent on other data sets (for example, Cars3D, although this dataset may be harder as one factor of variation has significantly more classes). Generally, GBT seem to correlate better with disentanglement, in particular with DCI Disentanglement, due to its stronger axis bias. 

To assess the sample complexity argument we compute for each trained model a statistical efficiency score which we define as the average accuracy based on $\num{100}$ samples divided by the average accuracy based on $\num{10000}$ samples for either the logistic regression or the gradient boosted trees.
The key idea is that if disentangled representations lead to sample efficiency, then they should also exhibit a higher statistical efficiency score\footnote{We remark that this score differs from the definition of sample complexity commonly used in statistical learning theory.}.
Overall, we do not observe conclusive evidence in Figure~\ref{figure:downstream_tasks} that models with higher disentanglement scores also lead to higher statistical efficiency.
We observe that indeed models with higher disentanglement scores seem to often exhibit better performance for gradient boosted trees with 100 samples.
However, considering all data sets, it appears that overall increased disentanglement is rather correlated with better downstream performance (on some data sets) and not statistical efficiency.
We do not observe that higher disentanglement scores reliably lead to a higher sample efficiency.

\paragraph{Implications}
While the empirical results in this section are negative, they should also be interpreted with care.
\looseness=-1After all, we have seen in previous sections that the models considered in this study fail to reliably produce disentangled representations. 
Hence, the results in this section might change if one were to consider a different set of models, for example semi-supervised or fully supervised one.
Furthermore, there are many more potential notions of usefulness such as interpretability and fairness that we have not considered in this experimental evaluation. While prior work~\citep{steenbrugge2018improving,laversanne2018curiosity,nair2018visual,higgins2017darla,higgins2018scan} successfully applied disentanglement methods such as $\beta$-VAE on a variety of downstream tasks, it is not clear to us that these approaches and trained models performed well \emph{because of disentanglement}. Finally, we remark that disentanglement is mostly about \textit{how} the information is stored in the representation. Tasks that explicitly rely on this structure are likely to benefit more from disentanglement rather than the ones considered in this chapter. Notable examples are our applications in fairness \citep{locatello2019fairness} of Chapter~\ref{cha:fairness} and abstract visual reasoning \citep{van2019disentangled}. In the former, the we show that disentanglement can be used to isolate the effect of unobserved sensitive variables to limit their negative impact to the downstream prediction (see Chapter~\ref{cha:fairness}). In the latter, we showed compelling evidence that disentanglement is useful for abstract visual reasoning tasks in terms of sample complexity. We remark that these benefits are specific to some of the notions of disentanglement considered in this work, such as DCI Disentanglement and FactorVAE. 

\section{Proof of Theorem~\ref{thm:impossibility}}
\begin{proof}
To show the claim, we explicitly construct a family of functions $f$ using a sequence of bijective functions.
Let $d > 1$ be the dimensionality of the latent variable $\rvz$ and consider the function $g:\supp(\rvz)\to[0,1]^d$ defined by
\[
g_i(\vv) = \P(\rz_i\leq v_i) \quad \forall i =1, 2, \dots, d.
\]
Since $\P$ admits a density $p(\rvz)=\prod_ip(\rz_i)$, the function $g$ is bijective and, for almost every $\vv\in\supp(\rvz)$, it holds that $\frac{\partial g_i(\vv)}{\partial v_i}\neq 0$ for all $i$ and $\frac{\partial g_i(\vv)}{\partial v_j} = 0$ for all $i\neq j$.
Furthermore, it is easy to see that, by construction, $g(\rvz)$ is a independent $d$-dimensional uniform distribution.
Similarly, consider the function $h:(0,1]^d \to \R^d$ defined by
\[
h_i(\vv) = \psi^{-1}(v_i)\quad \forall i =1, 2, \dots, d,
\]
where $\psi(\cdot)$ denotes the cumulative density function of a standard normal distribution.
Again, by definition, $h$ is bijective with $\frac{\partial h_i(\vv)}{\partial v_i}\neq 0$ for all $i$ and $\frac{\partial h_i(\vv)}{\partial v_j} = 0$ for all $i \neq j$.
Furthermore, the random variable $h(g(\rvz))$ is a $d$-dimensional standard normal distribution.

Let $\mA\in\R^{d\times d}$ be an arbitrary orthogonal matrix with $A_{ij}\neq0$ for all $i=1,2, \dots, d$ and $j=1, 2, \dots, d$. 
An infinite family of such matrices can be constructed using a Householder transformation:
Choose an arbitrary $\alpha\in(0, 0.5)$ and consider the vector $\vv$ with $v_1=\sqrt{\alpha}$ and $v_i=\sqrt{\frac{1-\alpha}{d-1}}$ for $i=2, 3, \dots, d$. 
By construction, we have $\vv^T\vv = 1$ and both $v_i\neq0$ and $v_i\neq\sqrt{\frac12}$ for all $i=1, 2, \dots, d$.
Define the matrix $\mA = \mI_d - 2 \vv\vv^T$ and note that $A_{ii} = 1 - 2v_i^2 \neq 0$ for all $1,2, \dots, d$ as well as $A_{ij} = -v_iv_j\neq0$ for all $i\neq j$.
Furthermore, $\mA$ is orthogonal since 
\begin{align*}
\mA^T\mA = \left(\mI_d - 2 \vv\vv^T\right)^T\left(\mI_d - 2 \vv\vv^T\right) = \mI_d - 4 \vv\vv^T + 4 \vv(\vv^T\vv)\vv^T = \mI_d.
\end{align*}

Since $\mA$ is orthogonal, it is invertible and thus defines a bijective linear operator.
The random variable $\mA h(g(\rvz))\in\R^d$ is hence an independent, multivariate standard normal distribution since the covariance matrix $\mA^T\mA$ is equal to $\mI_d$.

Since $h$ is bijective, it follows that $h^{-1}(\mA h(g(\rvz)))$ is an independent $d$-dimensional uniform distribution.
Define the function $f: \supp(\rvz)\to\supp(\rvz)$
\[
f(\vu) = g^{-1}(h^{-1}(\mA h(g(\vu))))
\]
and note that by definition $f(\rvz)$ has the same marginal distribution as $\rvz$ under $\P$, \ie, $\P(\rvz \leq \vu) = \P(f(\rvz) \leq \vu)$ for all $\vu$.
Finally, for almost every $\vu\in\supp(\rvz)$, it holds that
\[
\frac{\partial f_i(\vu)}{\partial u_j} 
= 
\frac{A_{ij} \cdot\frac{\partial h_j(g(\vu))}{\partial v_j}\cdot \frac{\partial g_j(\vu)}{\partial u_j}}
{\frac{\partial h_i(h_i^{-1}(\mA h(g(\vu)))}{\partial v_i} \cdot \frac{\partial g_i(g^{-1}(h^{-1}(\mA h(g(\vu)))))}{\partial v_i}}
\neq 0,
\]
as claimed.
Since the choice of $\mA$ was arbitrary, there exists an infinite family of such functions $f$.
\end{proof}

\chapter{Evaluating Disentangled Representations}\label{cha:eval_dis}
\looseness=-1In this chapter, we discuss the evaluation of disentangled representations. The presented work is based on \citep{locatello2019challenging,locatello2020sober} and was developed in collaboration with Stefan Bauer, Mario Lucic, Gunnar R\"atsch, Sylvain Gelly, Bernhard Sch\"olkopf, and Olivier Bachem. This work was partially done when Francesco Locatello was at Google Research, Brain Team in Zurich.



\section{What Do Disentanglement Metrics Measure?}
The disentanglement of a learned representation can be seen as a certain structural property of the statistical relations between the latent space of the VAE with that of the ground-truth factors. Therefore, when evaluating disentangled representations, several metrics typically estimate these statistical dependencies first and then compute how well this structure encodes the desired properties. 
As quantifying statistical dependencies through independence testing is a challenging task~\citep{shah2018hardness}, several approaches have been proposed.
We identify two prevalent settings: using interventional~\citep{higgins2016beta,kim2018disentangling,suter2018interventional} and observational data~\citep{chen2018isolating,ridgeway2018learning,eastwood2018framework}. 

\begin{figure}
    \begin{center}
    \begin{subfigure}{0.47\textwidth}
    \includegraphics[width=\textwidth]{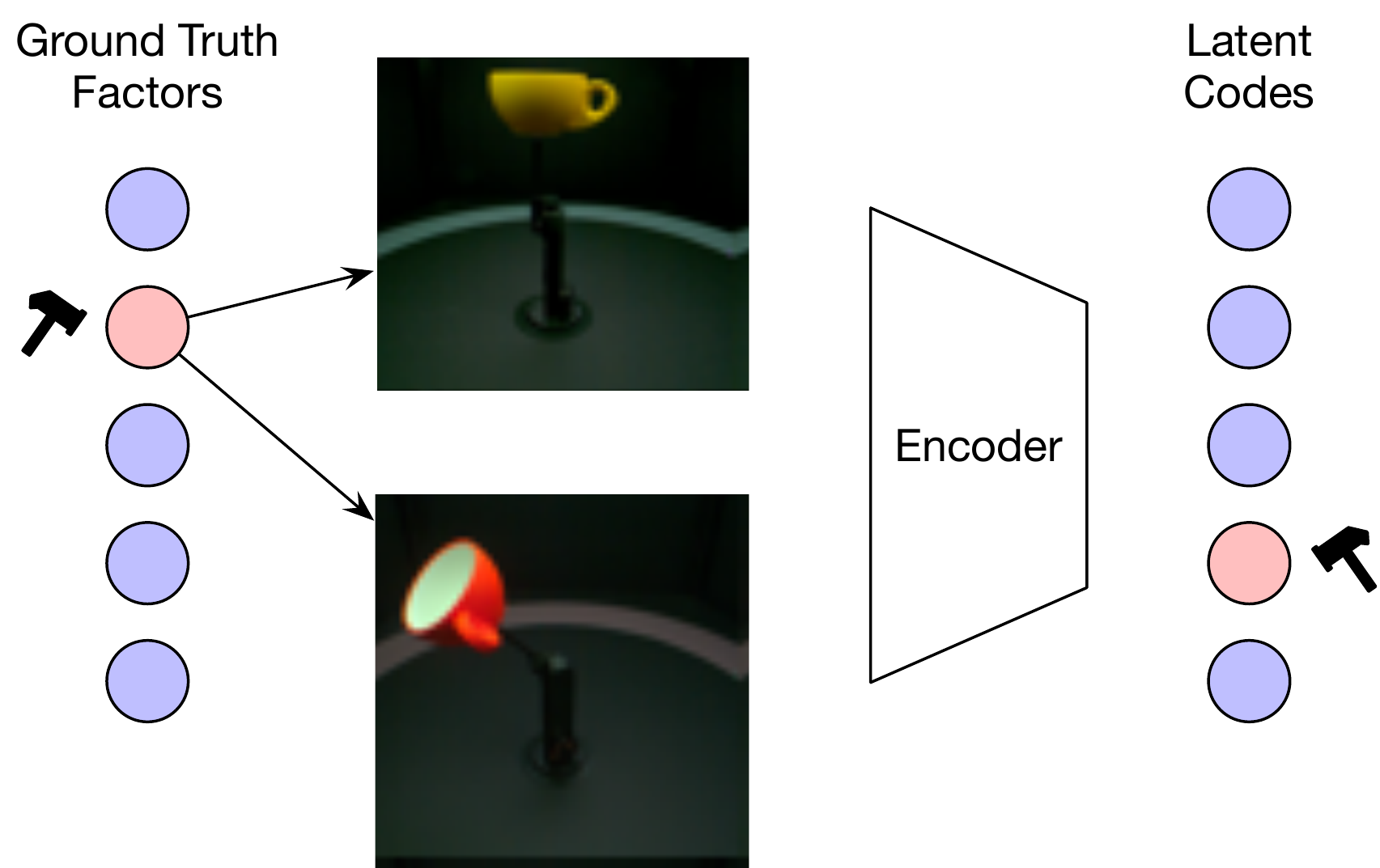}
    \caption{\small \looseness=-1Encoder \emph{consistency} \citep{shu2020weakly} for one factor of variation: intervening (\myhammer{}) on a ground-truth factor (or subset of factors) by \emph{fixing} its value corresponds to \emph{fixing} a dimension (or subset of dimensions) in the representation. In this example the object shape is constant and everything else is changing.}\label{fig:disentanaglement_notions_int_fix}
    \end{subfigure}%
    \hspace{1.5em}
    \begin{subfigure}{0.47\textwidth}
    \includegraphics[width=\textwidth]{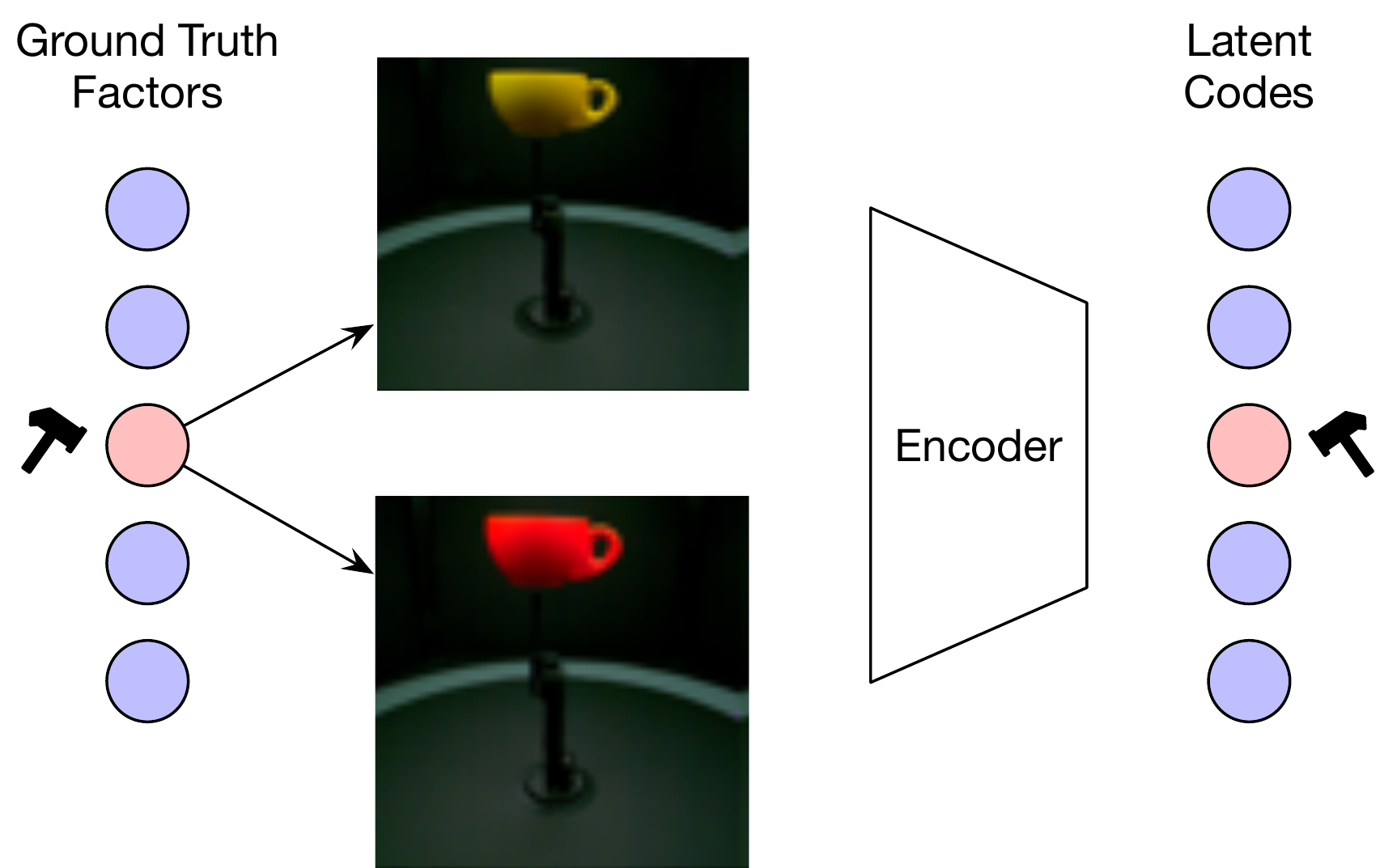}
    \caption{\small \looseness=-1Encoder \emph{restrictiveness} \citep{shu2020weakly}  for one factor of variation: intervening (\myhammer{}) on a ground-truth factor (or subset of factors) by \emph{changing} its value corresponds to \emph{changing} a dimension (or subset of dimensions) in the representation. In this example only the color is changing.}\label{fig:disentanaglement_notions_int_change}
    \end{subfigure}
    \vspace{2em}
    \begin{subfigure}{0.47\textwidth}
    \includegraphics[width=\textwidth]{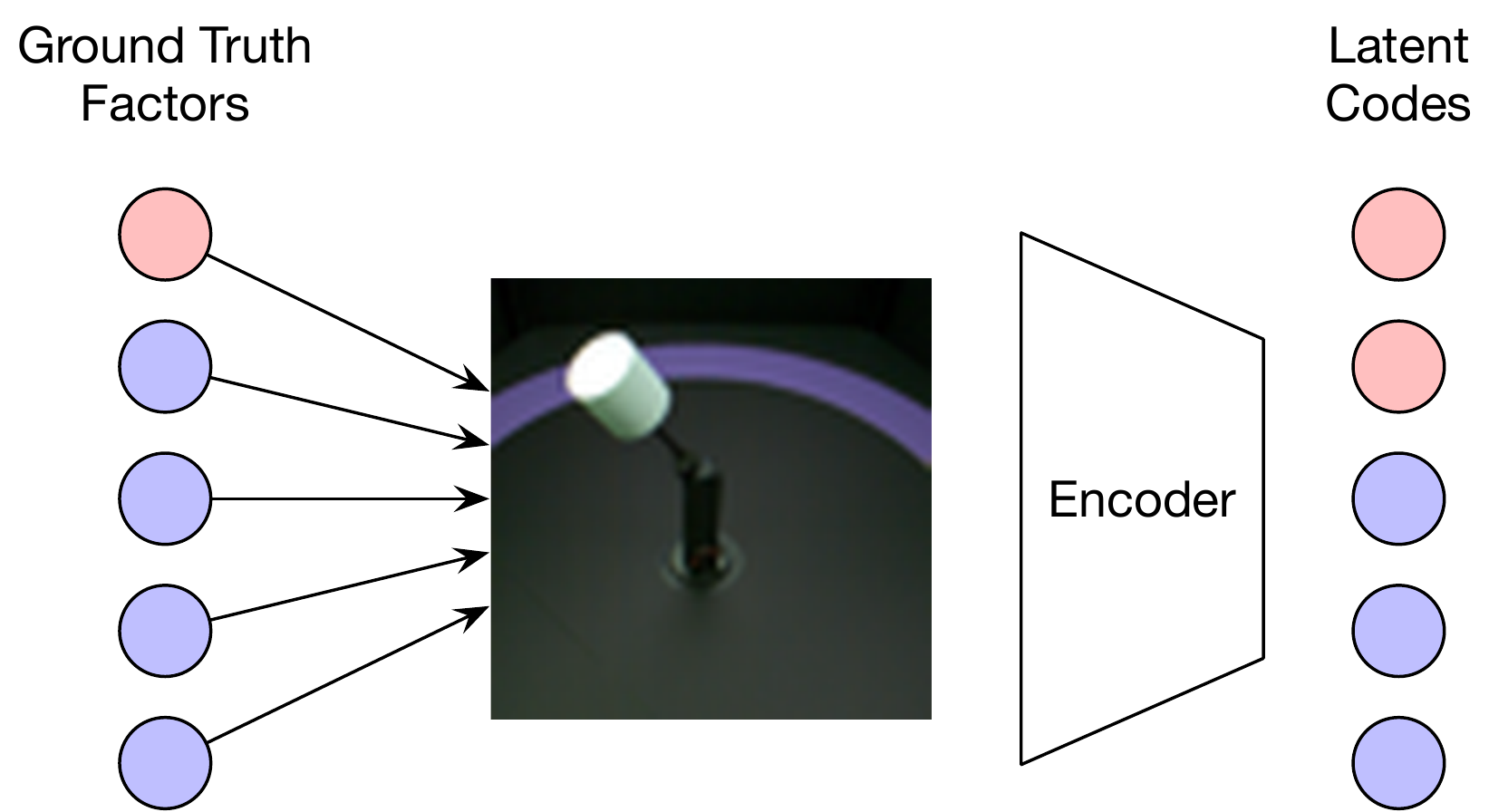}
    \caption{\small \emph{Disentangled} encoder for one factor of variation in the sense of~\citep{eastwood2018framework}: a few dimensions are capturing a single factor. }\label{fig:disentanaglement_notions_disent}
    \end{subfigure}%
    \hspace{1.5em}
    \begin{subfigure}{0.47\textwidth}
    \vspace{2em}
    \includegraphics[width=\textwidth]{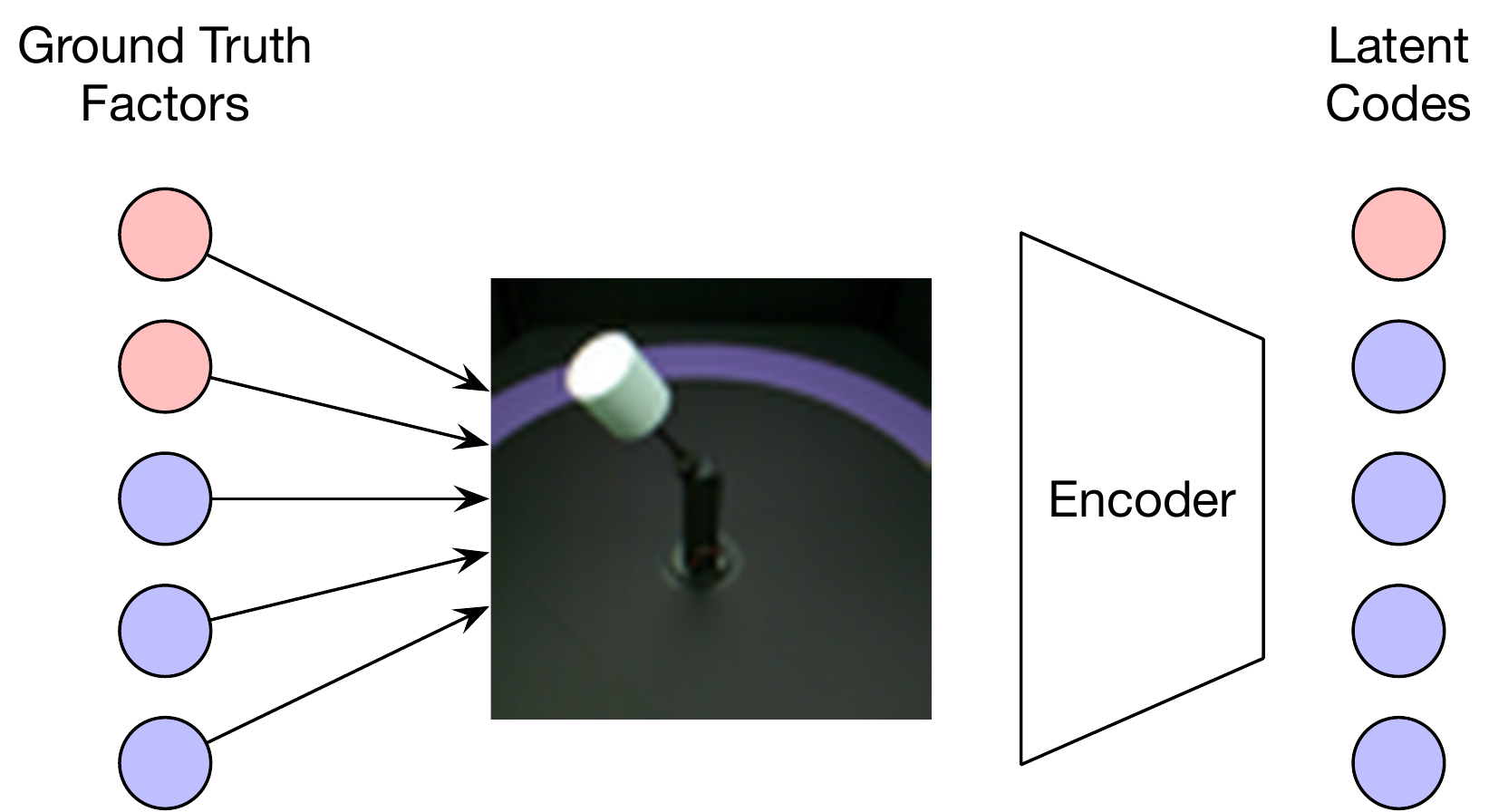}
    \caption{\small Encoder \emph{compactness} for one factor of variation in the sense of~\citep{eastwood2018framework}: a factor of variation should be captured in a single dimension. However multiple factors can still be encoded in the same dimension.}\label{fig:disentanaglement_notions_compact}
    \end{subfigure}
    \end{center}
    \caption{\small Examples of different notions of disentanglement being captured by the scores. Further, different scores \emph{measure} the same notion in different ways, which can introduce systematic differences in the evaluation. For the encoder to be consistent (a), restrictive (b), disentangled (c), or compact (d) the property highlighted in the each example should hold for each factor.}
    \label{fig:disentanaglement_notions}
\end{figure}

For interventional data, the two main properties a disentangled representation should have are \emph{consistency} and \emph{restrictiveness}~\citep{shu2020weakly}. Examples can be seen in Figures~\ref{fig:disentanaglement_notions_int_fix}~and~\ref{fig:disentanaglement_notions_int_change}. Both can be interpreted in the context of independent mechanisms~\citep{peters2017elements}: interventions on a ground-truth factor should manifest in a localized way in the representation. For example, fixing a certain factor of variation and sampling twice all others should result in a subset of dimensions being constant in the representation of the two points (consistency). This notion is used in the metrics of~\citet{higgins2016beta,kim2018disentangling}. On the other hand, changing the value of a factor of variation while keeping the others constant should result in a single change in the representation. This fact was used in the evaluation metric proposed by~\citet{suter2018interventional}. While~\citep{shu2020weakly} argue that both aspects are necessary for disentangled representations, when the ground-truth factors are independent and unconfounded the two definitions are equivalent.

On observational data, which is arguably the most practical case, there are several ways of estimating the relationship between factors and codes. For example, \citet{chen2018isolating,ridgeway2018learning} use the mutual information while \citet{eastwood2018framework,kumar2017variational} rely on predictability with a random forest classifier and a SVM respectively. The practical impact of these low-level and seemingly minor differences is not yet understood.

Once the relation between the factors and the codes is known for a given model, we need to evaluate the properties of the structure in order to measure its ``disentanglement''. Since a generally accepted formal definition for disentanglement is missing~\citep{eastwood2018framework, higgins2018towards,ridgeway2018learning}, the desired structure of the latent space compared to the ground-truth factors is a topic of debate. 
~\citet{eastwood2018framework} (and in part~\citet{ridgeway2018learning}) proposed three properties of representations: \emph{disentanglement}, \emph{compactness}, and \emph{informativeness}. A representation is disentangled if each dimension only captures a single factor of variation and compact if each factor is encoded in a single dimension, see Figures~\ref{fig:disentanaglement_notions_disent}~and~\ref{fig:disentanaglement_notions_compact}. Note that disentangled representations do not need to be compact nor compact representations need to be disentangled. Combining the two implies that a representation implements a one-to-one mapping between factors of variation and latent codes. Informativeness measures how well the information about the factors of variation is accessible in the latent representations with linear models. The degree of informativeness captured by any of the disentanglement metrics is unclear. In particular, as discussed in Chapter~\ref{sec:downstream}, it is not clear whether the correlation between disentanglement metrics and downstream performance is an artifact of the linear model used to estimate the relations between factors and code~\citep{eastwood2018framework,kumar2017variational}. 
Maintaining the terminology, the disentanglement scores in~\citep{higgins2016beta,kim2018disentangling,ridgeway2018learning,eastwood2018framework,suter2018interventional} focus on disentanglement in the sense of~\citep{eastwood2018framework} and~\citep{chen2018isolating,kumar2017variational} on compactness. Note that all these scores implement their own ``notion of disentanglement''.
Theoretically, we can characterize existing metrics in these two groups. On the other hand, observing the latent traversal of top performing models, it is not clear what the differences between the scores are and whether compactness and disentanglement are essentially equivalent on representations learned by VAEs (a compact representation is also disentangled and vice-versa).

As a motivating example consider the two models in Figure~\ref{figure:counterexample_traversals}. While visually, we may say that they are similarly disentangled, they achieve very different MIG scores, making the first model twice as good as the second one. Artifacts like this clearly impact the conclusions one may draw from a quantitative evaluation. Further, the structure of the representation may influence its usefulness downstream, and different properties may be useful for different tasks. For example, the applications in fairness~\citep{locatello2019fairness} (see Chapter~\ref{cha:fairness}), abstract reasoning~\citep{van2019disentangled} and strong generalization~\citep{locatello2020weakly} (see Chapter~\ref{cha:weak}) all conceptually rely on the disentanglement notion of~\citep{eastwood2018framework}.

\looseness=-1In this section, we first question how much the metrics agree with each other in terms of how the models are ranked. Second, we focus on the metrics that can be estimated from observational data, as we anticipate they will be more generally applicable in practice. There, we question the impact of different choices in the estimation of the factor-code matrices and in the aggregation. This last step encodes which notion of disentanglement is measured.
Finally, we investigate the sample efficiency of the different metrics to provide practical insights on which scores may be used in practical settings where labeled data is scarce.


\section{How Much Do Existing Disentanglement Metrics Agree?}
\label{sec:metrics_agreement}
As there exists no single, commonly accepted definition of disentanglement, an interesting question is to see how much the different metrics agree.
Figure~\ref{figure:metrics_rank_correlation} shows the Spearman rank correlation between different disentanglement metrics on different data sets.
Overall, we observe that all metrics except Modularity and, in part, IRS seem to be correlated strongly on the data sets dSprites, Color-dSprites, and Scream-dSprites and mildly on the other data sets.
There appear to be two pairs among these metrics that correlate well: the BetaVAE and the FactorVAE scores and the Mutual Information Gap and DCI Disentanglement. Note that this positive correlation does not necessarily imply that these metrics measure the same notion of disentanglement.

\begin{figure}[ht]
\centering\includegraphics[width=\textwidth]{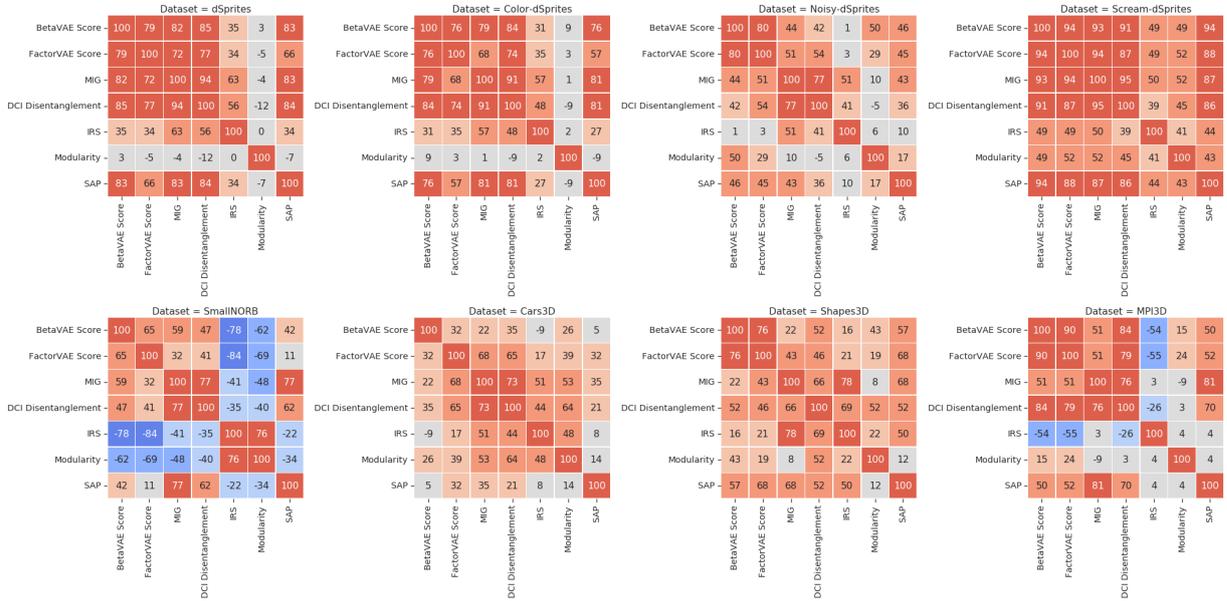}
\caption{\small Rank correlation of different metrics on different data sets. 
Overall, we observe that all metrics except Modularity seem to be strongly correlated on the data sets dSprites, Color-dSprites and Scream-dSprites and mildly on the other data sets.
 There appear to be two pairs among these metrics that capture particularly similar notions: the BetaVAE and the FactorVAE score as well as the Mutual Information Gap and DCI Disentanglement.
}\label{figure:metrics_rank_correlation}
\end{figure}

\begin{figure}[pt]
\begin{center}
{\adjincludegraphics[scale=0.4, trim={0 {0.91\height} 0 0}, clip]{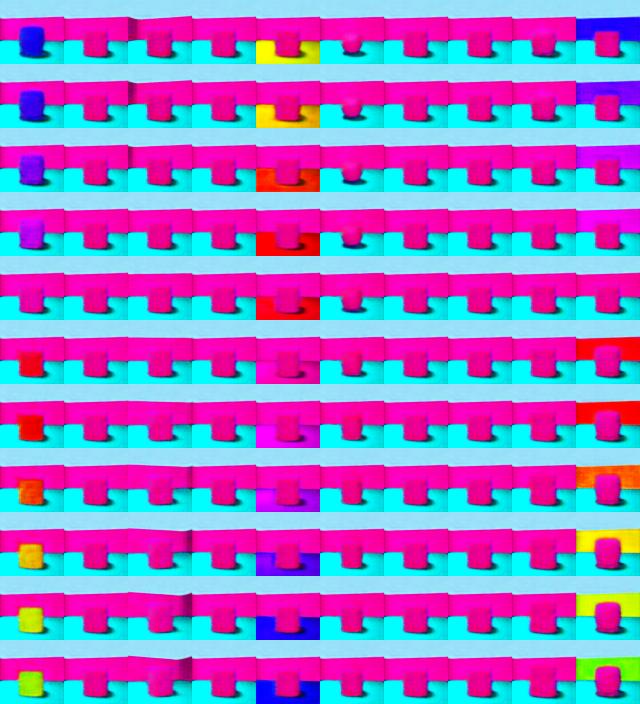}}\vspace{-0.3mm}
{\adjincludegraphics[scale=0.4, trim={0 {0.545\height} 0 {0.37\height}}, clip]{sections/disent/unsupervised/figures/counterexample.jpg}}\vspace{-0.3mm}
{\adjincludegraphics[scale=0.4, trim={0 0 0 {.91\height}}, clip]{sections/disent/unsupervised/figures/counterexample.jpg}}\vspace{2mm}
\end{center}
\begin{center}
{\adjincludegraphics[scale=0.4, trim={0 {0.91\height} 0 0}, clip]{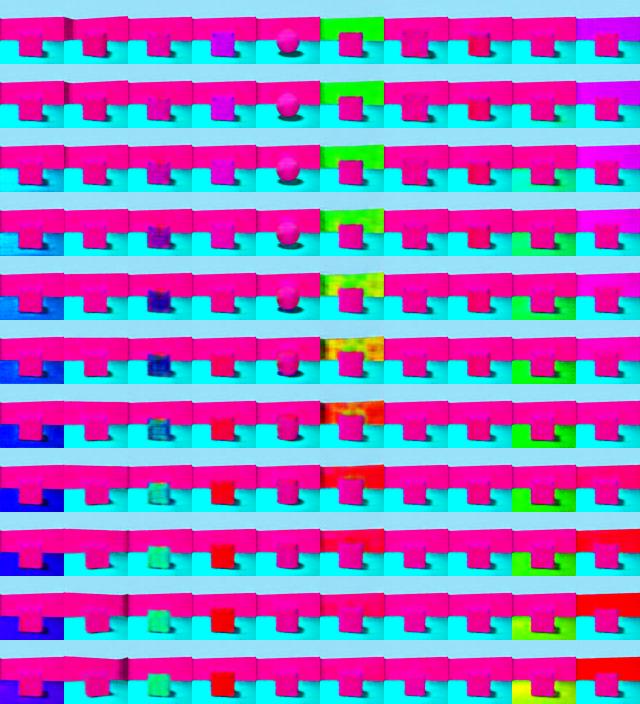}}\vspace{-0.3mm}
{\adjincludegraphics[scale=0.4, trim={0 {0.545\height} 0 {0.37\height}}, clip]{sections/disent/unsupervised/figures/traversals0.jpg}}\vspace{-0.3mm}
{\adjincludegraphics[scale=0.4, trim={0 0 0 {.91\height}}, clip]{sections/disent/unsupervised/figures/traversals0.jpg}}\vspace{2mm}
\end{center}
\caption{\small Latent traversal of a FactorVAE model (top) and a DIP-VAE-I (bottom) trained on Shapes3D. Despite dimensions 0, 5, and 8 not being perfectly disentangled (see Figure~\ref{figure:counterexample}), the model at the top achieves a MIG of 0.66 while the model at the bottom 0.33. Each column corresponds to a latent dimension. }\label{figure:counterexample_traversals}
\end{figure}

Indeed, we visualize in Figure~\ref{figure:counterexample_traversals} the latent traversals of two models that visually achieve similar disentanglement. Arguably, the bottom model may even be more disentangled that the one on the top (the shape in dimension 0 of the top model is not perfectly constant). However, the top model received a MIG of 0.66, while the model at the bottom just 0.33.  We remark that similar examples can be found for other disentanglement metrics as well by looking for models with a large disagreement between the scores. The two models in Figure~\ref{figure:counterexample_traversals} have DCI Disentanglement of 0.77 and 0.94, respectively.

The scores that require interventions and measure disentanglement computing consistency versus restrictiveness are not strongly correlated although they should be theoretically equivalent. On the other hand, we notice that the IRS is not very correlated with the other scores either, indicating that the difference may arise from how the IRS is computed.

We now investigate the differences in the scores that are computed from purely observational data: DCI Disentanglement, MIG, Modularity, and SAP Score. These scores are composed of two stages. First, they estimate a matrix relating factors of variation and latent codes. DCI Disentanglement considers the feature importance of a GBT predicting each factor of variation from the latent codes. MIG and Modularity compute the pairwise mutual information matrix between factors and codes. The SAP Score computes the predictability of each factor of variation from each latent code using an SVM. Second, they aggregate this matrix into a score measuring some of its structural properties. This is typically implemented as a normalized gap between the largest and second largest entries in the factor-code matrix either row or column-wise. We argue that this second step is the one that most encodes the ``notion of disentanglement'' being measured by the score. However, the correlation between the scores may also be influenced by how the matrix is estimated. In the remainder of this section, we put under scrutiny these two steps, systematically analyzing their similarities, robustness, and biases.

\begin{figure}[ht]
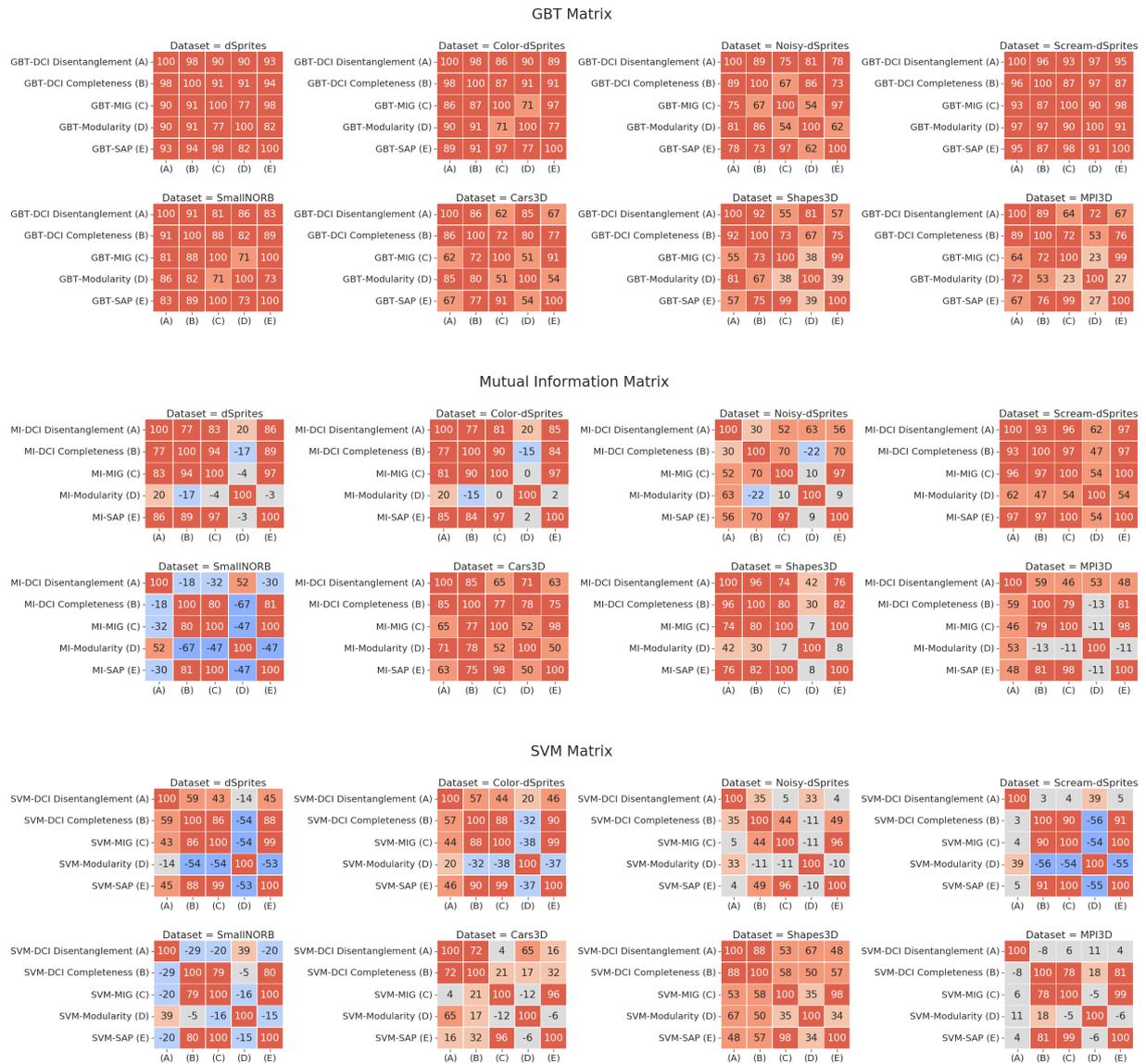

\centering\includegraphics[width=\textwidth]{sections/disent/unsupervised/autofigures/metrics_rank_gbt_matrix}\vspace{6mm}
\centering\includegraphics[width=\textwidth]{sections/disent/unsupervised/autofigures/metrics_rank_mi_matrix}\vspace{6mm}
\centering\includegraphics[width=\textwidth]{sections/disent/unsupervised/autofigures/metrics_rank_svm_matrix}\vspace{6mm}
\caption{\small Rank correlation of the different aggregations computed on the same matrix (GBT feature importance, mutual information, and predictability with a SVM). When the matrix is the same, MIG, SAP and DCI Completeness are much more correlated while the correlation with DCI Disentanglement decreases, highlighting the difference between completeness and disentanglement~\citep{eastwood2018framework}.
}\label{figure:metrics_rank_matrix}
\end{figure}

\subsection{What is the Difference Between the Aggregations? Is Compactness Equivalent to Disentanglement in Practice? }\label{sec:eval_same_concept}
In this section, we focus on the metrics that can be computed from observational data.
We question the ``notion of disentanglement'' implemented by the second step of DCI Disentanglement, MIG, Modularity, and SAP Score and look for differences between disentanglement and compactness in practice. These aggregations measure some structural properties of the statistical relation between factors and codes. To empirically understand the similarities and differences of these aggregations, we compare their result when evaluating the same input matrix in Figure~\ref{figure:metrics_rank_matrix}. We observe that the different aggregations seem to correlate well, but we note that this correlation is not always consistent across different matrices and data sets.
We note that MIG, SAP and DCI Completeness appear to be strongly correlated with each other when the matrix is the same. On the contrary, MIG/SAP and DCI Disentanglement seem to be consistently less correlated on the same matrix. The correlation between Modularity and the other scores varies dramatically depending on the matrix. This is not in contrast with Figure~\ref{figure:metrics_rank_correlation}, where we observed MIG being more correlated with DCI Disentanglement rather than SAP Score. Indeed, the dissimilarity between MIG and SAP depends on differences in the estimation of the matrix. 

These results may not be surprising, given the insights presented by~\citet{eastwood2018framework}. MIG and SAP compute the gap between the entries of the matrix per factor and therefore penalize compactness rather than disentanglement. In other words, they penalize whether a factor of variation is embedded in multiple codes but do not penalize the same code capturing multiple factors. DCI Disentanglement instead penalizes whether a code is related to multiple factors. Observing these differences in a large pool of trained models is challenging. First, the representations are not evenly distributed across the possible configurations (one-to-one, one-to-many, many-to-one, and many-to-many), and, for some of these relations (such as one-to-one and many-to-many), the scores behave similarly. Second, when comparing aggregations computed on different matrices, it is typically unclear where the difference is coming from. However, we believe it is important to understand these practical differences as enforcing different notions of disentanglement may not result in the same benefits downstream. 

\paragraph{Implications} We conclude that the similarity between the scores in Section~\ref{sec:metrics_agreement} is confounded by how the statistical relations are computed. Further, we note that one-to-one or many-to-many mappings seem to be preferred to one-to-many in the models we train, partially supporting the insights from~\citet{rolinek2018variational}.

\begin{figure}[ht]
\begin{subfigure}{0.3\textwidth}
\centering\includegraphics[width=\textwidth]{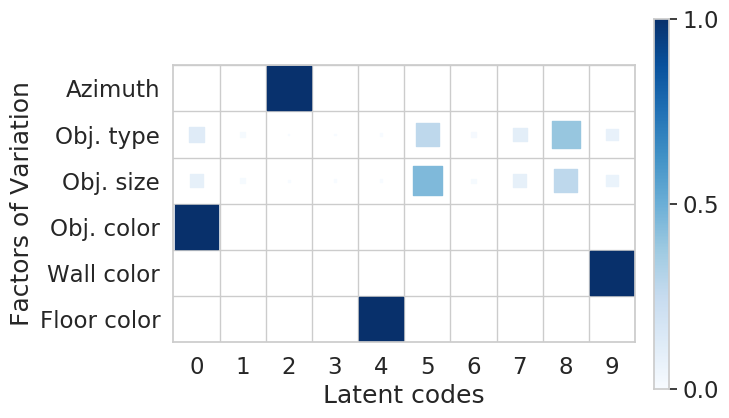}
\end{subfigure}%
\begin{subfigure}{0.3\textwidth}
\centering\includegraphics[width=\textwidth]{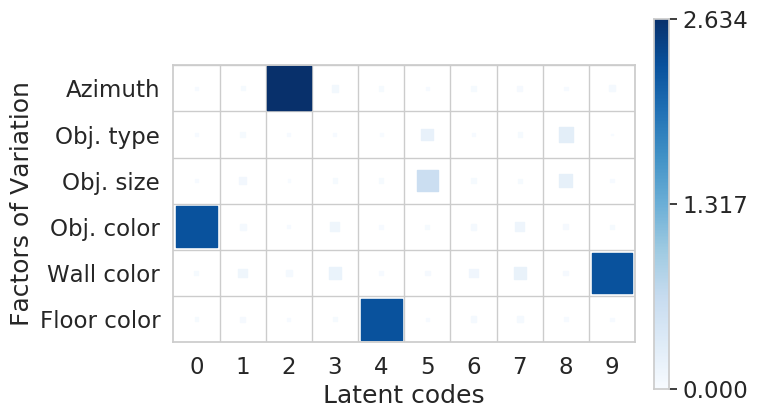}
\end{subfigure}
\begin{subfigure}{0.3\textwidth}%
\centering\includegraphics[width=\textwidth]{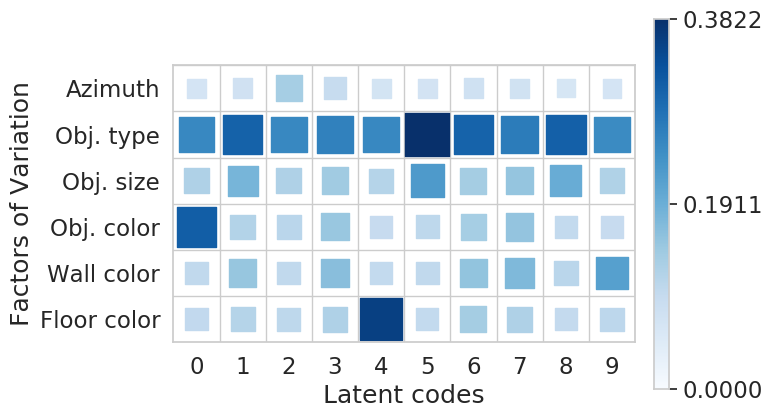}
\end{subfigure}

\begin{subfigure}{0.3\textwidth}
\centering\includegraphics[width=\textwidth]{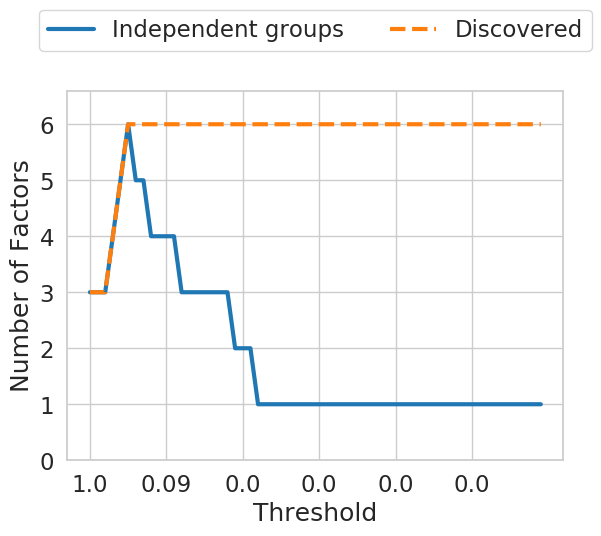}
\end{subfigure}%
\begin{subfigure}{0.3\textwidth}
\centering\includegraphics[width=\textwidth]{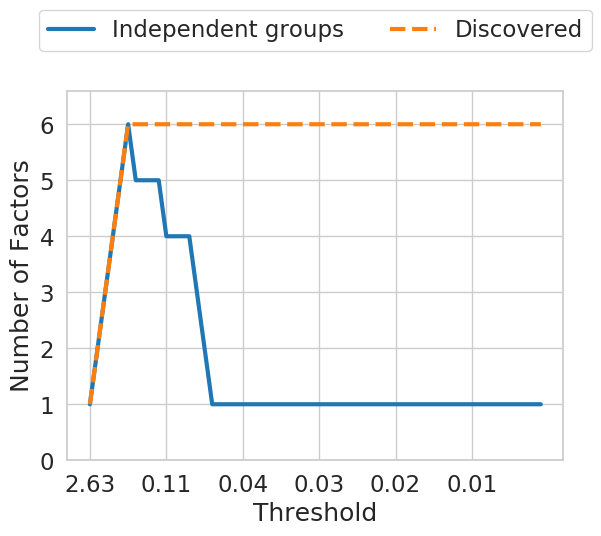}
\end{subfigure}
\begin{subfigure}{0.3\textwidth}%
\centering\includegraphics[width=\textwidth]{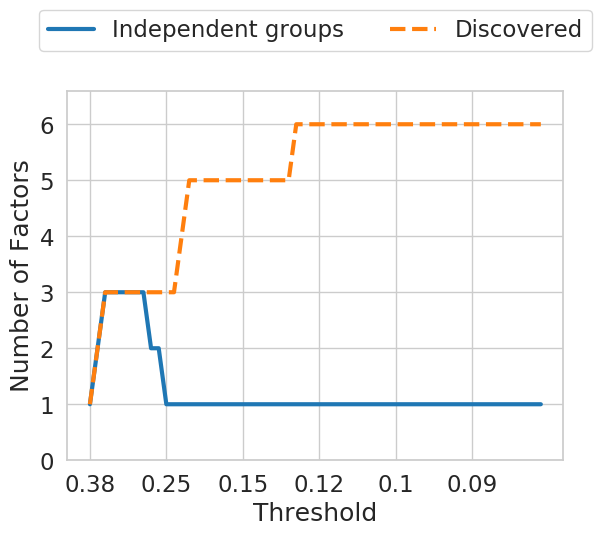}
\end{subfigure}

\begin{subfigure}{0.3\textwidth}
\centering\includegraphics[width=\textwidth]{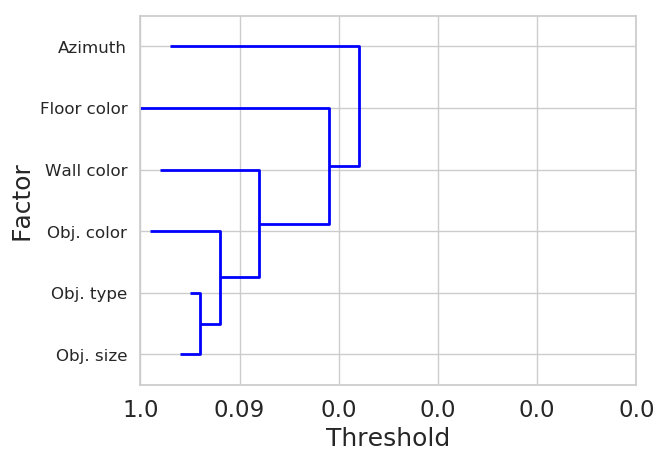}
\end{subfigure}%
\begin{subfigure}{0.3\textwidth}
\centering\includegraphics[width=\textwidth]{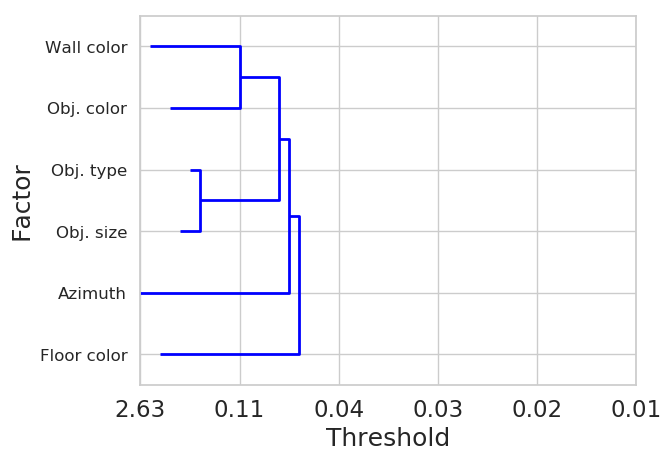}
\end{subfigure}
\begin{subfigure}{0.3\textwidth}%
\centering\includegraphics[width=\textwidth]{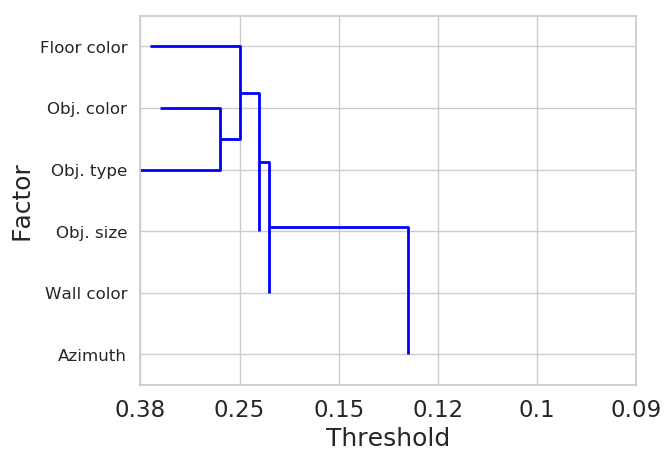}
\end{subfigure}
\caption{\small \looseness=-1Visualization of the relation between factors of variations and latent codes using for the model at the top of Figure~\ref{figure:counterexample_traversals}: (left) GBT feature importance as in the DCI Disentanglement score, (center) the mutual information as computed in the MIG and Modularity and, (right) SVM predictability as computed by the SAP Score. Top row: factor-code matrix. Middle row: independent-groups curve recording how many connected components of size larger than one there are in the factor-code bipartite graph defined by the matrix at a given threshold. Bottom row: dendrogram plot recording which factors are merged at which threshold. The long tail of the SVM importance matrix explains the weaker correlation between MIG and SAP Score in Figure~\ref{figure:metrics_rank_correlation} even though the scores are measuring a similar concept. The dendrogram plots computed from the independent-groups curve can be used to systematically analyze which factors are merged at which threshold by the different estimation techniques (\eg, SVM, GBT feature importance and mutual information).}\label{figure:counterexample}
\end{figure}

\subsection{Does the Estimation of Factor-Code Matrices Impact the Evaluation?}\label{sec:matrix}
\looseness=-1We continue to investigate the metrics computed from observational data and focus on the different matrices estimating the statistical relations between factors of variation and latent codes. First, we build new visualization tools to understand both what a model has learned and how it has been evaluated by the factor-code matrices.
\begin{figure}[t]
\begin{center}
\begin{subfigure}{0.3\textwidth}
\centering\includegraphics[width=\textwidth]{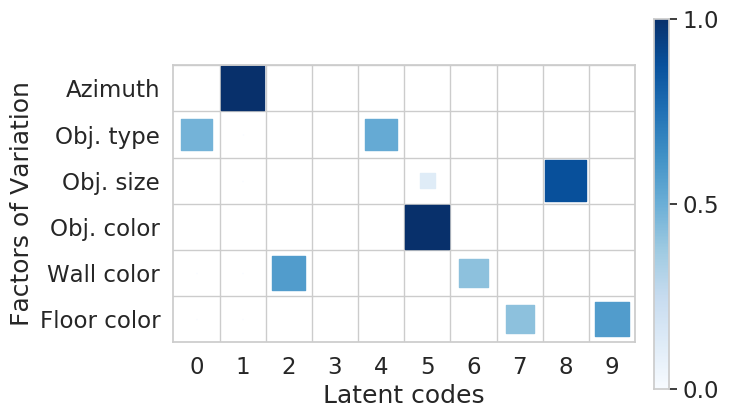}
\centering\includegraphics[width=\textwidth]{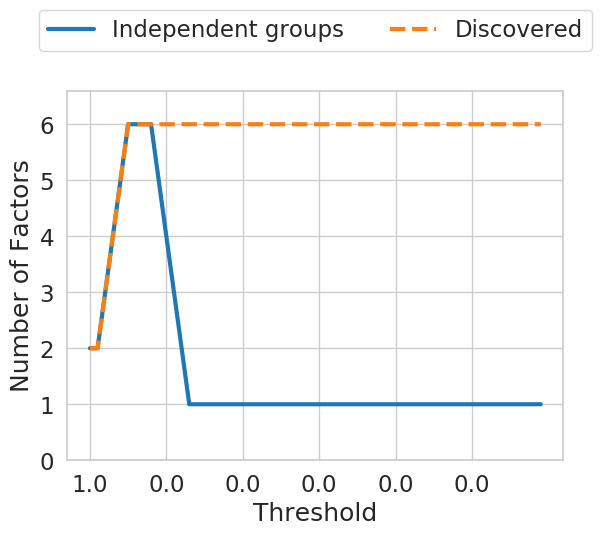}
\centering\includegraphics[width=\textwidth]{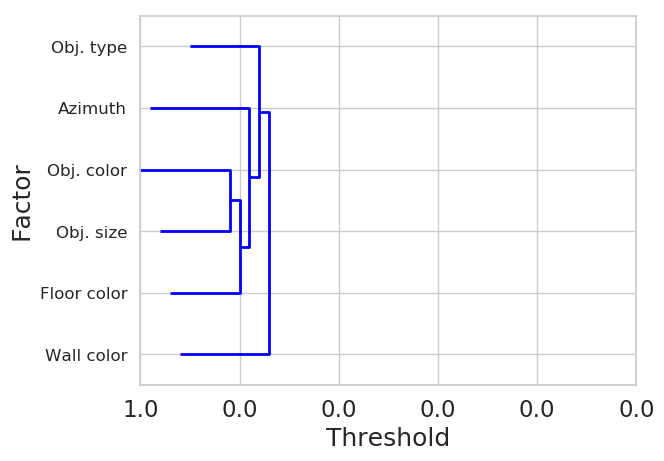}
\end{subfigure}%
\begin{subfigure}{0.3\textwidth}
\centering\includegraphics[width=\textwidth]{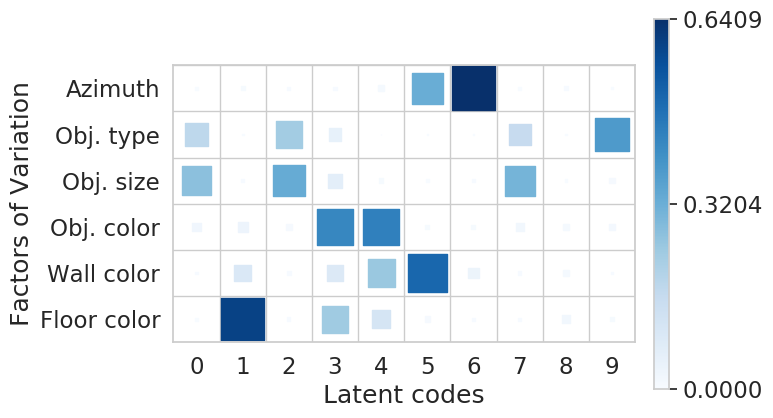}
\centering\includegraphics[width=\textwidth]{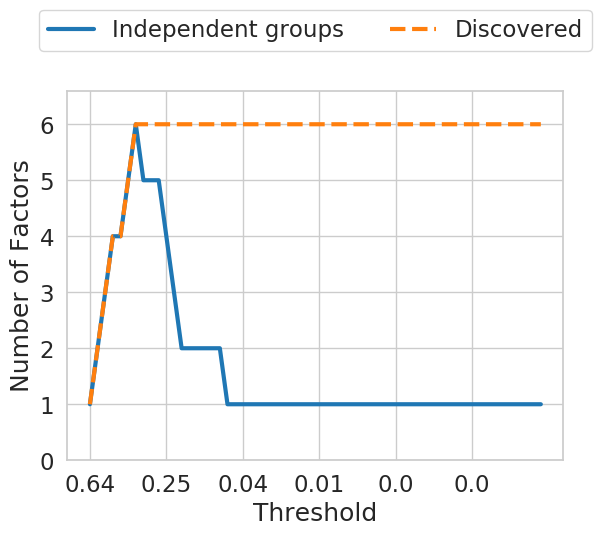}
\centering\includegraphics[width=\textwidth]{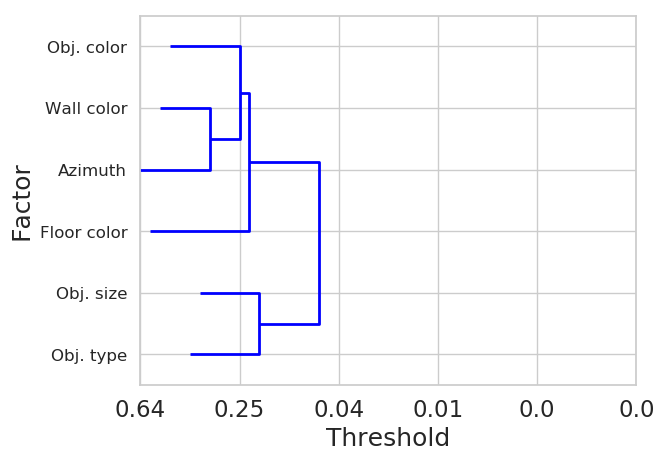}
\end{subfigure}%
\begin{subfigure}{0.3\textwidth}
\centering\includegraphics[width=\textwidth]{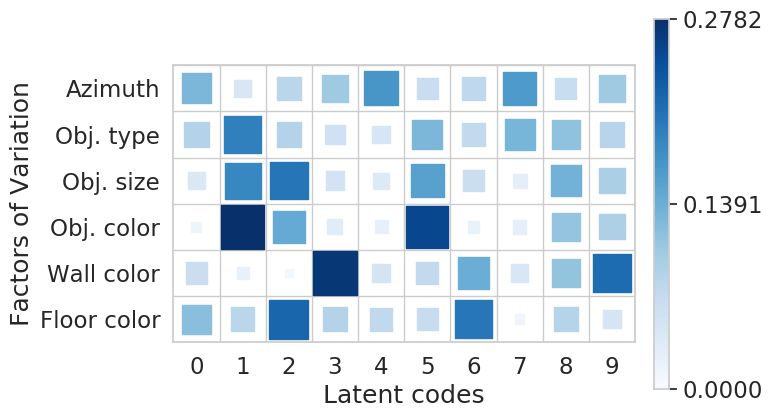}
\centering\includegraphics[width=\textwidth]{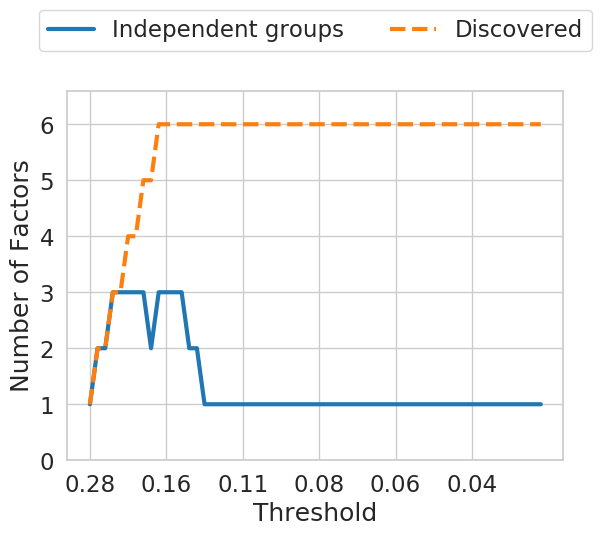}
\centering\includegraphics[width=\textwidth]{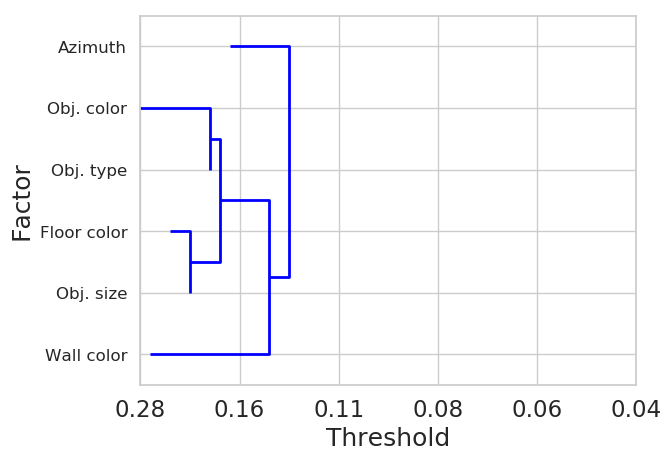}
\end{subfigure}
\end{center}
\caption{\small (top row) Visualization of the GBT importance matrix used in the DCI Disentanglement score for models with top (left), average (center), and worse (right) DCI Disentanglement on Shapes3D. (middle row) Independent-groups curves of the GBT importance matrix. (bottom row) Dendrogram plot recording when factors are merged. Comparing these plots with the ones in Figure~\ref{figure:precision_curves_mi}, we note that there are differences in the factor-code matrices. In particular, they disagree on which factors are most entangled. }~\label{figure:precision_curves_gbt}
\end{figure}

\begin{figure}[t]
\begin{center}
\begin{subfigure}{0.3\textwidth}
\centering\includegraphics[width=\textwidth]{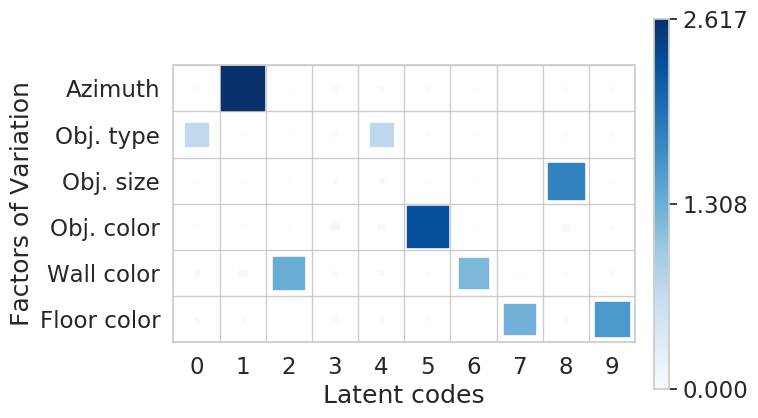}
\centering\includegraphics[width=\textwidth]{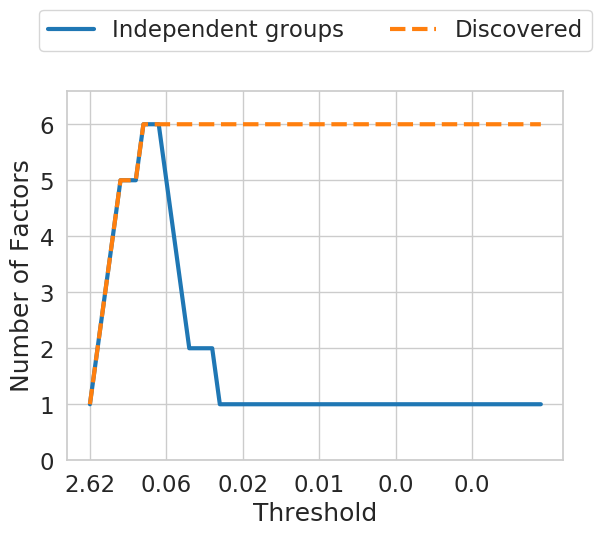}
\centering\includegraphics[width=\textwidth]{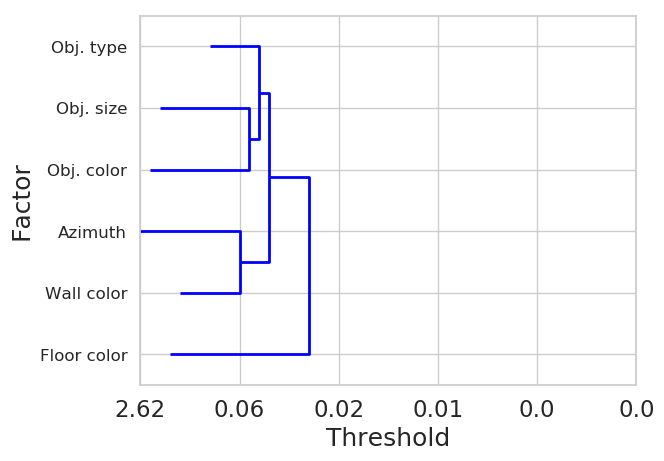}
\end{subfigure}%
\begin{subfigure}{0.3\textwidth}
\centering\includegraphics[width=\textwidth]{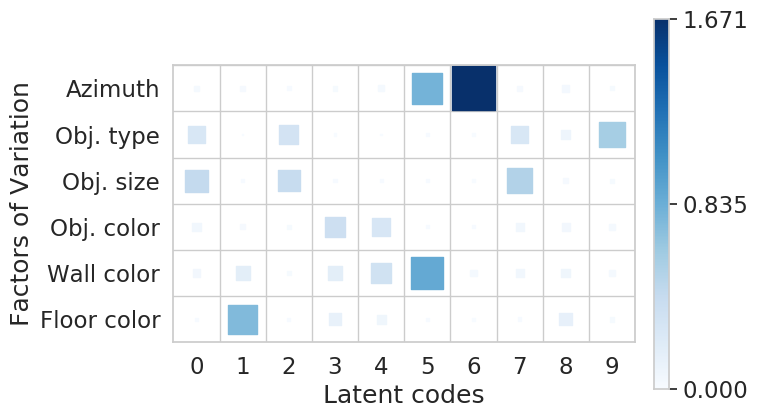}
\centering\includegraphics[width=\textwidth]{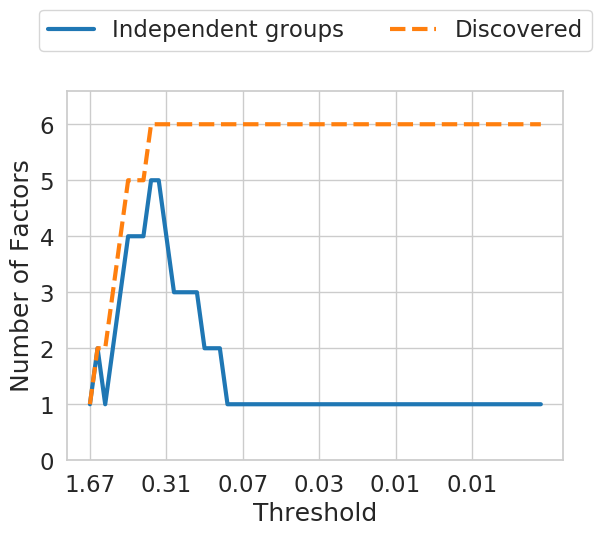}
\centering\includegraphics[width=\textwidth]{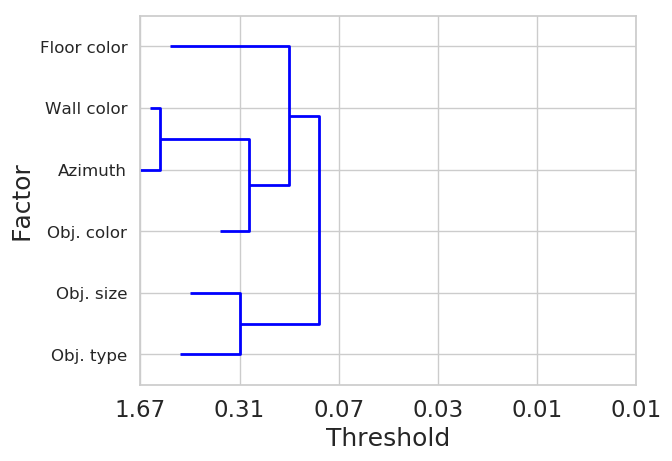}
\end{subfigure}%
\begin{subfigure}{0.3\textwidth}
\centering\includegraphics[width=\textwidth]{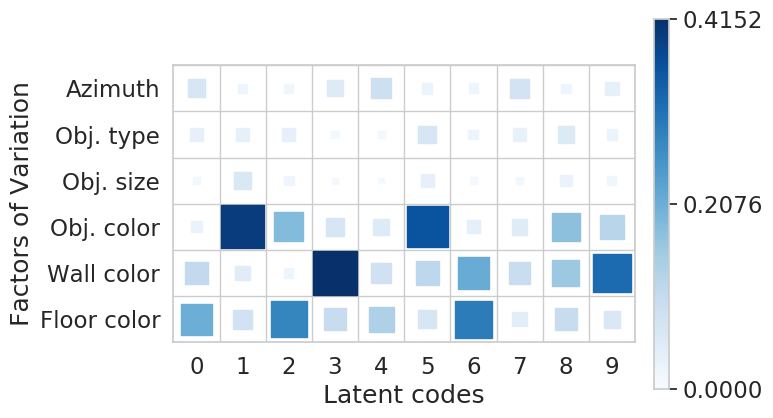}
\centering\includegraphics[width=\textwidth]{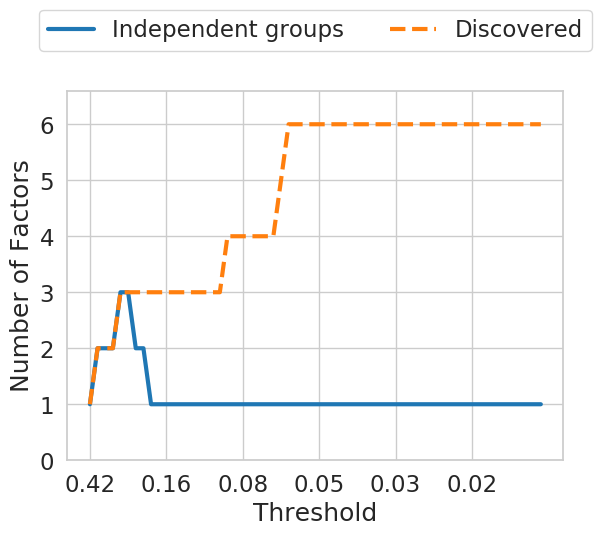}
\centering\includegraphics[width=\textwidth]{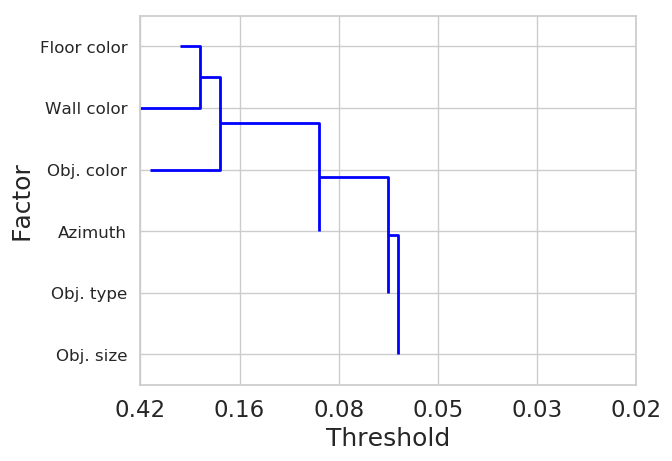}
\end{subfigure}%
\caption{\small 
(top row) Visualization of the mutual information matrix used in the MIG and Modularity scores for the same models of Figure~\ref{figure:precision_curves_gbt}. (middle row) Independent-groups curves of the mutual information matrix. (bottom row) Dendrogram plot recording when factors are merged. Comparing these plots with the ones in Figure~\ref{figure:precision_curves_gbt}, we note that there are differences in the factor-code matrices. In particular, they disagree on which factors are most entangled. }~\label{figure:precision_curves_mi}
\end{center}
\end{figure}

In Figure \ref{figure:counterexample} we visualize the model at the bottom of Figure~\ref{figure:counterexample_traversals}. In the first row, we plot the factor-codes matrices as learned by GBT feature importance, pairwise mutual information, and SVM predictability, respectively. We observe that for the GBT features and the mutual information matrix, the largest entries are the same. Still, the latter underestimates the effect of some dependencies, for example, object size and type in dimensions five and eight. The SVM feature importance, also agrees on some of the large values but exhibits a longer tail than the other matrices.

To further analyze the differences between the matrices, we view them as weights on the edges of a bipartite graph encoding the statistical relation between each factor of variation and code. We can now delete all edges with weight smaller than some threshold and count (i) how many factors of variation are connected with at least a latent code and (ii) the number of connected components with size larger than one. In Figure~\ref{figure:counterexample} (middle row), we plot these two curves computed on the respective matrices, and, in Figure~\ref{figure:counterexample} (bottom row), we record which factors are merged at which threshold.  Factors that are merged at a lower threshold are more entangled in the sense that are more statistically related to a shared latent dimension.

The long tail of the SVM importance matrix explains why we observed a weaker correlation between MIG and SAP Score in Figure~\ref{figure:metrics_rank_correlation} even though the scores are measuring a similar concept. Indeed, we can observe in the middle row of Figure~\ref{figure:counterexample} that the largest entries of the three matrices are distributed differently, in particular for the SVM predictability. Similarly, we can read in the dendrogram plot that the factors are merged in a different order for the SVM predictability compared to the other two matrices. We hypothesize that the long tail of the SVM predictability results from spurious correlations and optimization issues that arise from how the score computation (fitting a threshold separately on each code predicting each factor).

\looseness=-1In Figures~\ref{figure:precision_curves_gbt} and~\ref{figure:precision_curves_mi} we compare the factor-code matrices, independent-groups curves, and dendrograms for the best, average and worse model in terms of DCI Disentanglement. Figure~\ref{figure:precision_curves_gbt} shows the plots for the GBT (Gradient Boosted Trees) feature importance matrix used by the DCI Disentanglement score and Figure~\ref{figure:precision_curves_mi} the mutual information matrix of MIG and Modularity. By comparing these plots, we can clearly distinguish which model is the most disentangled but we again note differences in how the different matrices capture the factors of variation. In particular, we again observe that the two matrices may disagree on which factors are most entangled in the same model. For example, the GBT features computed on the model on the left suggest that object color and size are more entangled. In contrast, the mutual information matrix suggests azimuth and wall color.

Finally, we test whether the differences in the factor-code matrix impact the computation of the disentanglement scores. To do so, we compare the ranking produced by each aggregation computed on the different matrices. If the different matrices encode the same statistical relations, the ranking should also be similar.
We observe in Figure~\ref{figure:metrics_rank_sap_mig_aggregation} that the ranking seems to be generally different, and the level of correlation appears to depend on the data set. Overall, the aggregation of SAP Score and MIG seems to be more robust to changes in the estimation matrix compared to Modularity and DCI Disentanglement. 

\paragraph{Implications} Based on this result, we conclude that systematic differences in the estimation matrix may indeed impact the evaluation of disentanglement. It seems important for the evaluation that the statistical relations between factors and codes are robustly and consistently estimated. We observed that changing the estimation technique may produce different rankings of the models. It appears, therefore, important to not bias the evaluation by considering a single estimation technique unless reliability guarantees are also given.

\begin{figure}[pt]
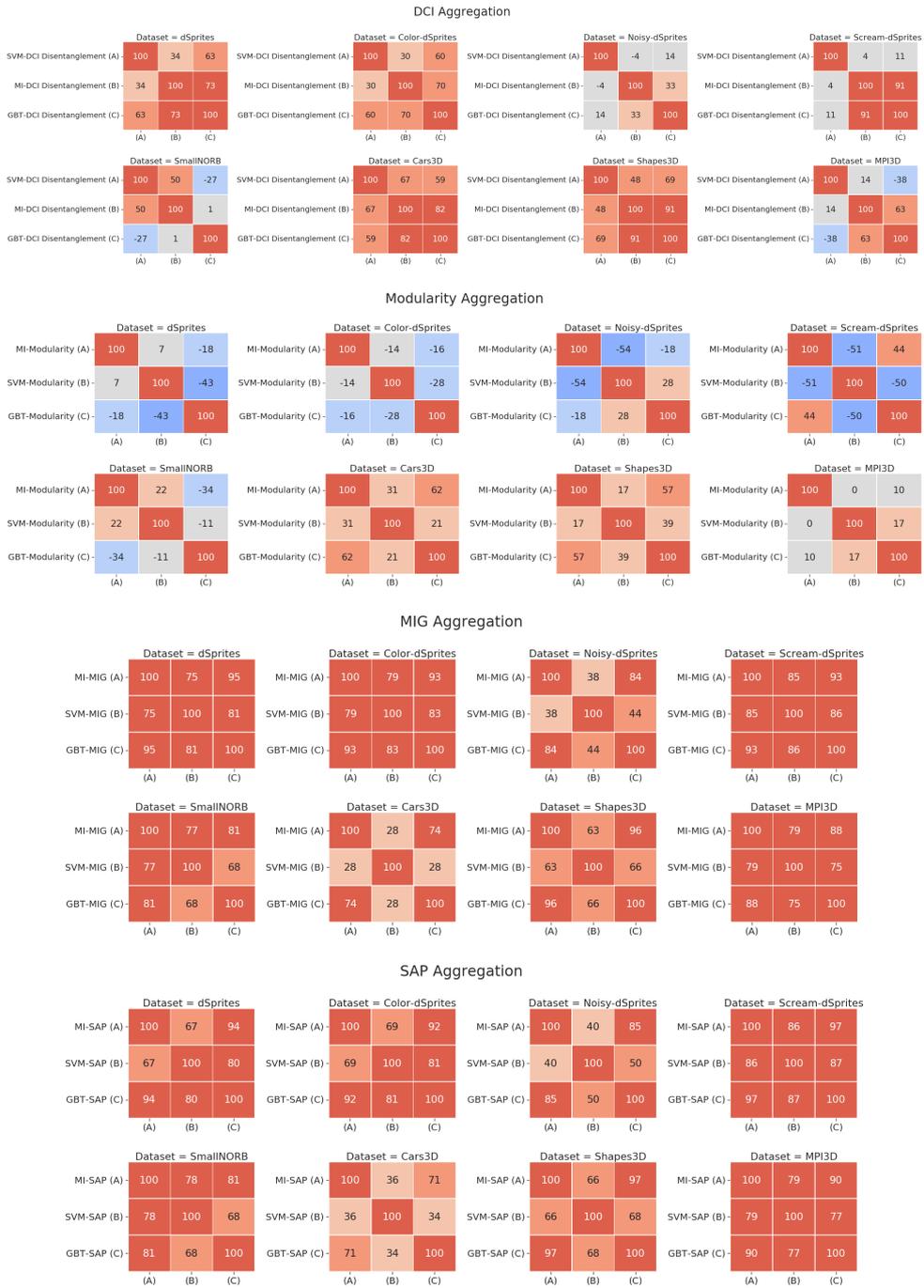

\centering\includegraphics[width=0.8\textwidth]{sections/disent/unsupervised/autofigures/metrics_rank_dci_aggregation}\vspace{3mm}
\centering\includegraphics[width=0.8\textwidth]{sections/disent/unsupervised/autofigures/metrics_rank_modularity_aggregation}\vspace{3mm}
\centering\includegraphics[width=0.7\textwidth]{sections/disent/unsupervised/autofigures/metrics_rank_mig_aggregation}\vspace{3mm}
\centering\includegraphics[width=0.7\textwidth]{sections/disent/unsupervised/autofigures/metrics_rank_sap_aggregation}
\caption{\small Rank correlation of DCI Disentanglement, Modularity, SAP Score and MIG aggregations on different matrices. The ranking seem to be generally different and data set dependant indicating that systematic differences in the estimation matrix may impact the evaluation of disentanglement. MIG and SAP aggregations appear to be more robust to changes in the estimation matrix.
}\label{figure:metrics_rank_sap_mig_aggregation}
\end{figure}

\subsection{Discussion} 

We conclude that the different disentanglement scores do not measure the same concept: they measure different notions of disentanglement (compactness versus disentanglement) that appear to be generally correlated in practice but not equivalent. 

In particular, MIG and SAP Score intend disentanglement differently than DCI Disentanglement. They rather measure completeness: they do not penalize multiple factors of variation being captured by a single latent dimension. Modularity seems to be more dependent on the estimation matrix as its correlation with the other scores changes considerably with different matrices. Furthermore, there are systematic differences between the different techniques to estimate the relation between factors of variation and latent codes that influence the correlation of the scores: the ranking of the models is different depending on the chosen estimation technique.

We argue that future works advancing the state-of-the-art in disentanglement, with or without any form of supervision, should reflect upon which notion of disentanglement they consider and how it is measured in the chosen evaluation protocol. 

\looseness=-1Not all the properties that are generally associated with the term ``disentanglement'' are necessarily related to all the scores considered in this chapter, and specific downstream tasks may require specific notions~\citep{locatello2019fairness,van2019disentangled,locatello2020weakly}. Further, separating the estimation of the statistical dependencies between factors of variation and codes from what the score is measuring may help clarify the properties that are being evaluated. As robustly capturing these statistical dependencies is a crucial step of the evaluation metrics that do not rely on interventions, we argue that future work on disentanglement scores should specifically highlight (i) how this estimation is performed precisely, (ii) its sample complexity/variance and (iii) biases (for example do they work well with coarse-grained as opposed to fine-grained factors of variation). Future research is necessary to understand both how estimation metrics overestimate or underestimate the amount of disentanglement and how to robustly aggregate this information into a score. Among the scores tested in this chapter, we recommend using the DCI aggregation, either with the GBT feature importance or the mutual information matrix, ideally both. 

\begin{figure}[pt]
\centering\includegraphics[width=\textwidth]{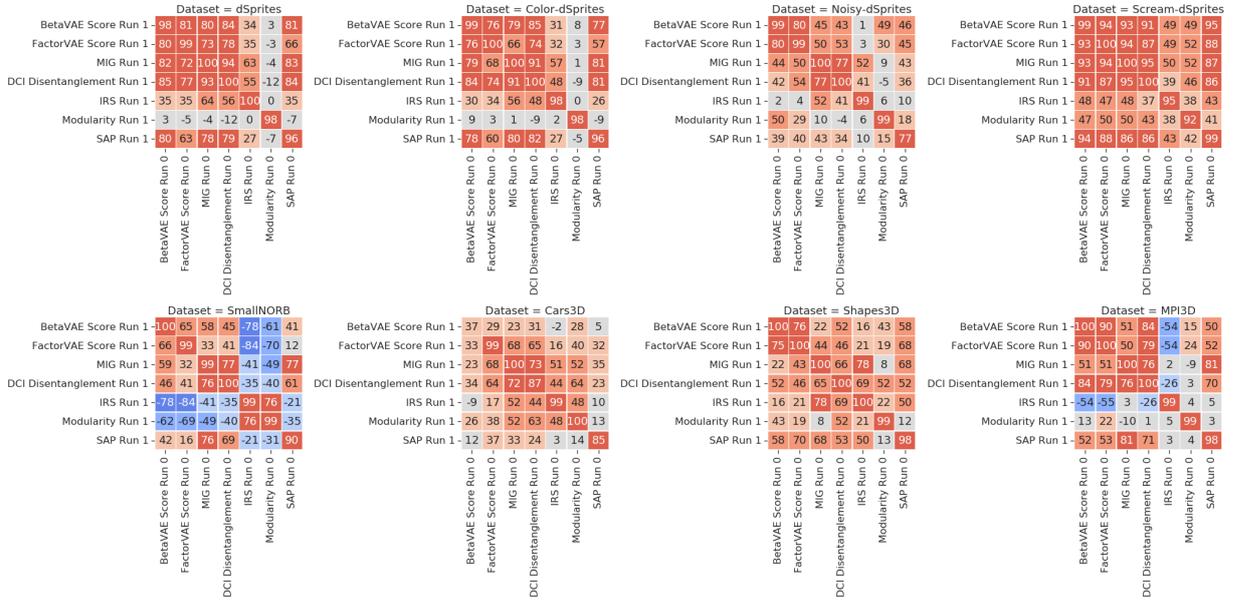}
\caption{\small Rank correlation of different metrics on different data sets across two runs. 
Overall, we observe that the disentanglement scores computed with \num{10000} examples are relatively stable.
}\label{figure:metrics_rank_correlation_10000}
\end{figure}
\begin{figure}[pt]
\centering\includegraphics[width=\textwidth]{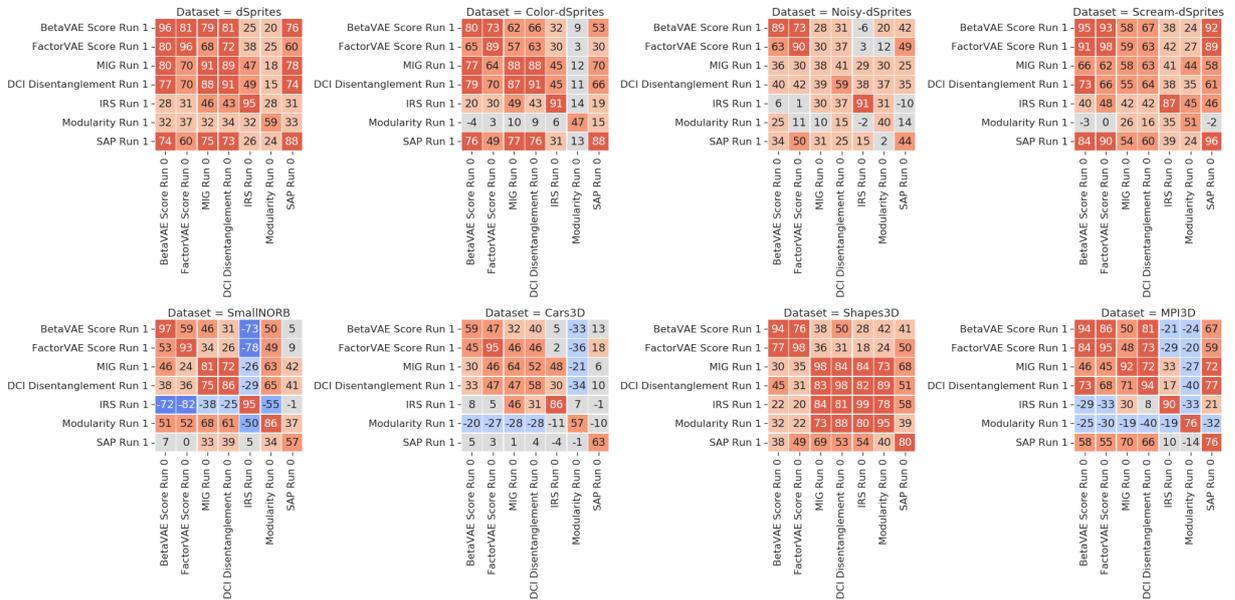}
\caption{\small Rank correlation of different metrics computed using \num{100} examples on different data sets across two runs. 
Overall, we observe that with fewer examples the disentanglement scores are considerably less stable.
}\label{figure:metrics_rank_correlation_100}
\end{figure}

\section{Is the Computation of the Disentanglement Scores Reliable?}
The computation of the disentanglement scores requires supervision, and having access to a large number of observations of $\rvz$ may be unreasonable. 
\looseness=-1On the other hand, for the purpose of this study, we are interested in a stable and reproducible experimental setup. In Figure~\ref{figure:metrics_rank_correlation_10000}, we observe that running the disentanglement scores twice yields comparable results with \num{10000} examples. Using just \num{100} examples may be feasible in practice as suggested by~\cite{locatello2019disentangling} but has less stable results as depicted in Figure~\ref{figure:metrics_rank_correlation_100}. We observe that not every score is equally sample efficient. The FactorVAE scores and the IRS seem to be the most efficient ones, followed by DCI Disentanglement and MIG.

\paragraph{Implications}
Computing the disentanglement scores on these data sets with \num{10000} examples yields stable results and is appropriate for this study. Finding sample efficient disentanglement scores is an important research direction for practical semi-supervised disentanglement~\citep{locatello2019disentangling}.



\chapter{Semi-Supervised Disentanglement}\label{cha:semi_sup}
In this chapter, we discuss the role of explicit supervision in the learning of disentangled representations. The presented work is based on \citep{locatello2019disentangling} and was developed in collaboration with  Michael Tschannen, Stefan Bauer, Gunnar R\"atsch, Bernhard Sch\"olkopf, and Olivier Bachem. This work was partially done when Francesco Locatello was at Google Research, Brain Team in Zurich.

\section{Motivation}

\begin{figure}
\begin{center}
    {\adjincludegraphics[scale=0.4, trim={0 {0.9\height} 0 0}, clip]{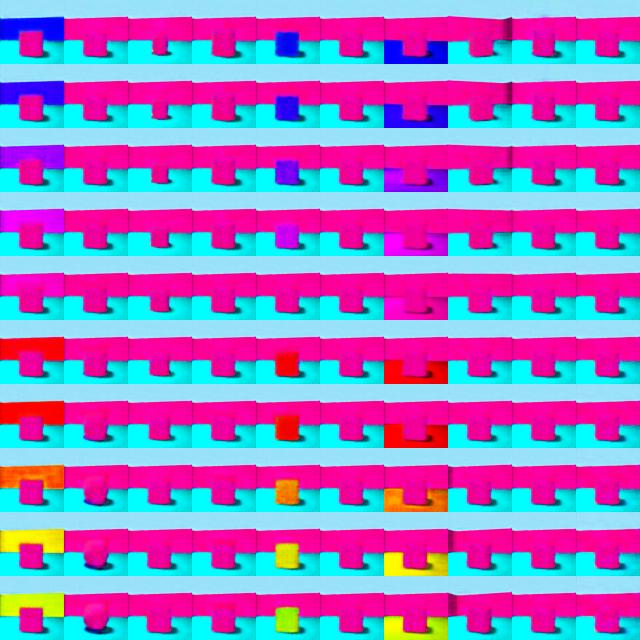}}\vspace{-0.3mm}
    {\adjincludegraphics[scale=0.4, trim={0 {.4\height} 0 {.5\height}}, clip]{sections/disent/semi_sup/autofigures/traversals0.jpg}}\vspace{-0.3mm}
    {\adjincludegraphics[scale=0.4, trim={0 0 0 {.9\height}}, clip]{sections/disent/semi_sup/autofigures/traversals0.jpg}}\vspace{2mm}
    {\adjincludegraphics[scale=0.4, trim={0 {0.9\height} 0 0}, clip]{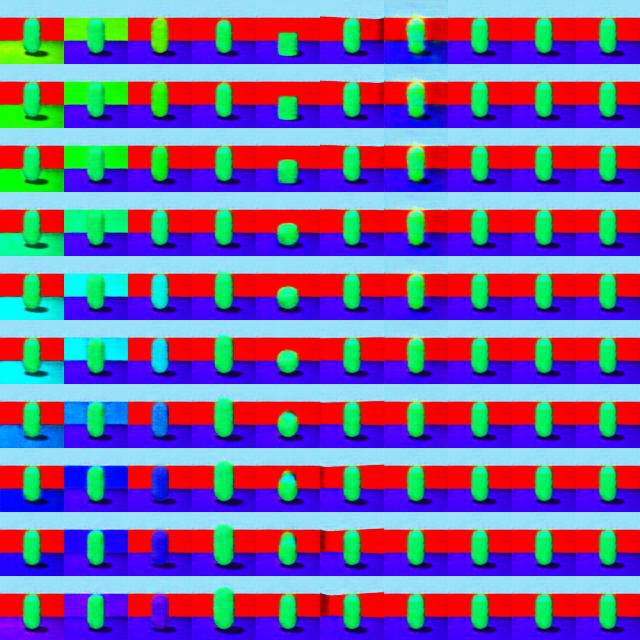}}\vspace{-0.3mm}
    {\adjincludegraphics[scale=0.4, trim={0 {.4\height} 0 {.5\height}}, clip]{sections/disent/semi_sup/autofigures/traversals4_s2.jpg}}\vspace{-0.3mm}
    {\adjincludegraphics[scale=0.4, trim={0 0 0 {.9\height}}, clip]{sections/disent/semi_sup/autofigures/traversals4_s2.jpg}}
\end{center}
    \caption{\looseness=-1Latent traversals (each column corresponds to a different latent variable being varied) on Shapes3D for the $\beta$-TCVAE model with best validation MIG (top) and for the semi-supervised $\beta$-TCVAE model with best validation loss (bottom), both using only \num{1000} labeled examples for validation and/or supervision. Both models appear to be visually well disentangled.}\label{figure:latent_traversal_s2}
    \vspace{-4mm}
\end{figure}

We discussed in Chapter~\ref{cha:unsup_dis} that the inductive biases of state-of-the-art methods may not be sufficient to reliably learn disentangled representations in practice and it is not clear how much additional supervision we would need. Further, there are many practical settings where one might have access to a limited amount of supervision, for example, through manual labeling of (some) factors of variation in a few training examples.
In this chapter, we investigate the impact of such supervision on state-of-the-art disentanglement methods and perform a large-scale study under well-defined and reproducible experimental conditions.
While human inspection can be used to select good model runs and hyperparameters (\eg \citet[Appendix 5.1]{higgins2017darla}), we argue that such supervision should be made explicit.
Hence, we consider the setting where one has access to annotations (which we call \textit{labels} in the following) of the latent variables $\rvz$ for a very limited number of observations $\rvx$, for example through human annotation.
Even though this setting is not universally applicable (\eg when the observations are not human interpretable) and a completely unsupervised approach would be elegant, collecting a small number of human annotations is simple and cheap via crowd-sourcing platforms such as Amazon Mechanical Turk, and is common practice in the development of real-world machine learning systems. As a consequence, the considered setup allows us to explicitly encode prior knowledge and biases into the learned representation via annotation, rather than relying solely on implicit biases such as the choice of network architecture with possibly hard-to-control effects. 
First, we investigate whether disentanglement scores are sample efficient and robust to imprecise labels.
Second, we explore whether it is more beneficial to incorporate the limited amount of labels available into training and thoroughly test the benefits and trade-offs of this approach compared to supervised validation.
For this purpose, we perform a reproducible large-scale experimental study\footnote{Reproducing these experiments requires approximately 8.57 GPU years (NVIDIA P100).}, training over \num{52000} models on four different data sets.
We found that unsupervised training with supervised validation enables reliable learning of disentangled representations. On the other hand, using some of the labeled data for training may be beneficial for disentanglement.
Overall, we show that a very small amount of supervision is enough to learn disentangled representations reliably as illustrated in Figure~\ref{figure:latent_traversal_s2}. 

\section{Unsupervised training with supervised model selection}\label{sec:validation}
\looseness=-1In this section, we investigate whether commonly used disentanglement metrics can be used to identify good models if a very small number of labeled observations is available.
While existing metrics are often evaluated using as much as \num{10000} labeled examples, it might be feasible in many practical settings to annotate \num{100} to \num{1000} data points and use them to obtain a disentangled representation.
At the same time, it is unclear whether such an approach would work as existing disentanglement metrics can be noisy (even with more samples), see Chapter~\ref{cha:eval_dis}.
Finally, we emphasize that the impossibility result of Chapter~\ref{cha:unsup_dis} does not apply in this setting as we do observe samples from $\rvz$.

\subsection{Experimental setup and approach}

\looseness=-1\textbf{Data sets.}
To reduce the number of models to train, we consider four data sets: \textit{dSprites}~\citep{higgins2016beta}, \textit{Cars3D}~\citep{reed2015deep}, \textit{SmallNORB}~\citep{lecun2004learning} and \textit{Shapes3D}~\citep{kim2018disentangling}.
For each data set, we assume to have either \num{100} or \num{1000} labeled examples available and a large amount of unlabeled observations.
We note that \num{100} labels correspond to labeling \num{0.01}\% of dSprites, \num{0.5}\% of Cars3D, \num{0.4}\% of SmallNORB and \num{0.02}\% of Shapes3D.

\textbf{Perfect vs. imprecise labels.}
In addition to using the perfect labels of the ground-truth generative model, we also consider the setting where the labels are imprecise. Specifically, we consider the cases were labels are \textit{binned} to take at most five different values, are \textit{noisy} (each observation of a factor of variation has 10\% chance of being random) or \textit{partial} (only two randomly drawn factors of variations are labeled). 
This is meant to simulate the trade-offs in the process of a practitioner quickly labeling a small number of images.

\textbf{Model selection metrics.}
We use MIG \citep{chen2018isolating}, DCI Disentanglement \citep{eastwood2018framework} and SAP score \citep{kumar2017variational} for model selection as they can be used on purely observational data. 

\textbf{Experimental protocol.}
We prepend the prefix $U/S$ for \emph{unsupervised} training with \emph{supervised} model selection to the method name. We consider 32 different experimental settings where an experimental setting corresponds to a data set (\textit{dSprites}/ \textit{Cars3D}/ \textit{SmallNORB}/ \textit{Shapes3D}), a specific number of labeled examples (\num{100}/\num{1000}), and a labeling setting (perfect/ binned/ noisy/ partial).
For each considered setting, we generate five different sets of labeled examples using five different random seeds.
For each of these labeled sets, we train cohorts of $\beta$-VAEs \citep{higgins2016beta}, $\beta$-TCVAEs \citep{chen2018isolating}, Factor-VAEs \citep{kim2018disentangling}, and DIP-VAE-Is \citep{kumar2017variational} where each model cohort consists of 36 different models with 6 different hyperparameters for each model and 6 random seeds.
For each of these \num{23040} models, we then compute all the model selection metrics on the set of labeled examples and use these scores to select the best models in each of the cohorts. Finally, we compute the BetaVAE score, the FactorVAE score, MIG, Modularity, DCI disentanglement, and SAP score for each model based on an additional test set of $\num{10000}$ samples.

\begin{figure}
\begin{center} 
\begin{subfigure}{0.8\textwidth}
\centering\includegraphics[width=\textwidth]{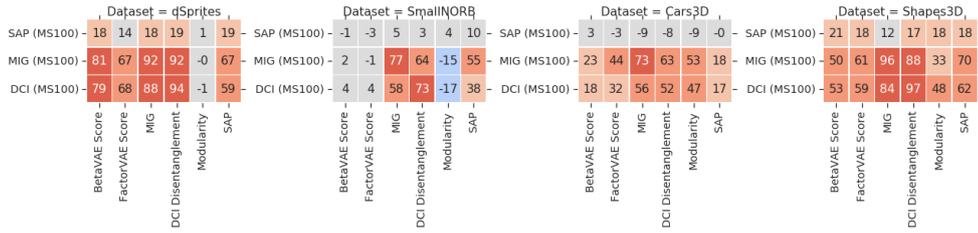}\caption{100 perfect labels.}
\end{subfigure}
\begin{subfigure}{0.8\textwidth}
\centering\includegraphics[width=\textwidth]{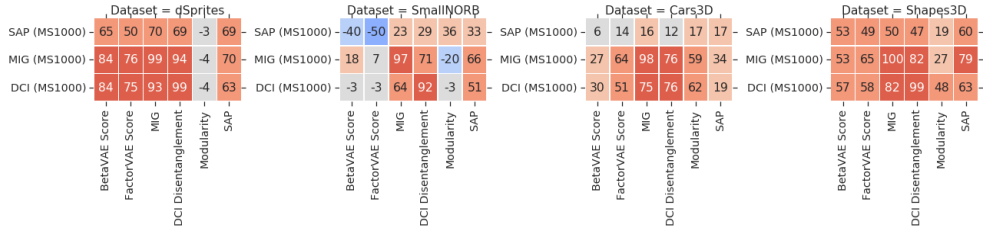}\caption{1000 perfect labels.}
\end{subfigure}
\begin{subfigure}{0.8\textwidth}
\centering\includegraphics[width=\textwidth]{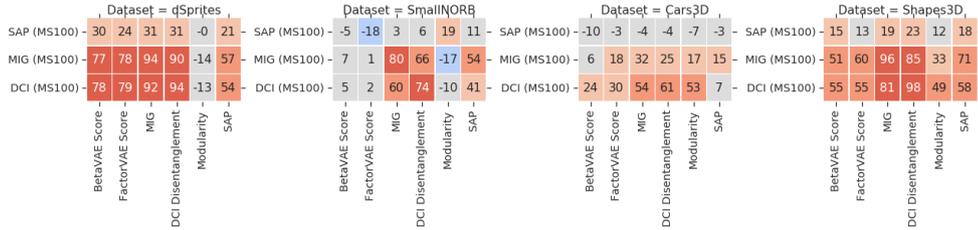}\caption{100 binned labels.}
\end{subfigure}
\begin{subfigure}{0.8\textwidth}
\centering\includegraphics[width=\textwidth]{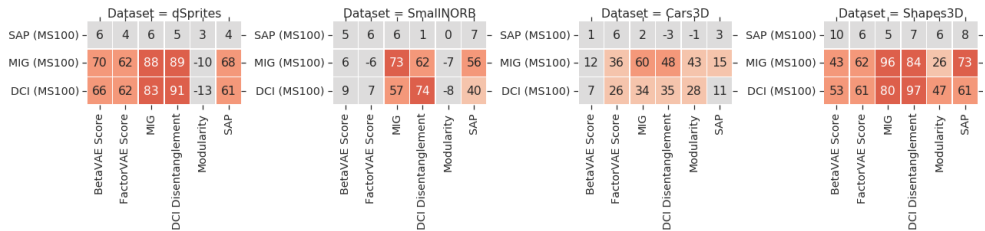}\caption{100 noisy labels.}
\end{subfigure}
\begin{subfigure}{0.8\textwidth}
\centering\includegraphics[width=\textwidth]{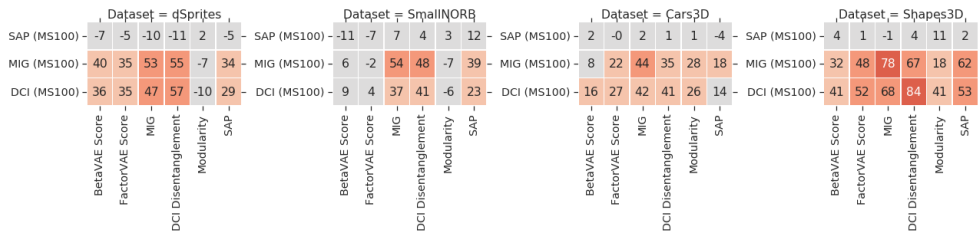}\caption{100 partial labels.}
\end{subfigure}
\end{center}
\vspace{-3mm}
\caption{Rank correlation of validation metrics and test metrics on dSprites. Validation metrics are computed on different types of labels.  Legend: (A) = BetaVAE Score,  (B) = FactorVAE Score, (C) = MIG, (D) = DCI Disentanglement, (E) = Modularity, (F) = SAP.
}\label{figure:metric_comparison_perfect_single}
\vspace{-5mm}
\end{figure}

\subsection{Key findings} \label{sec:us_keyfindings}
\looseness=-1In Figure~\ref{figure:metric_comparison_perfect_single}~(a), we show the rank correlation between the validation metrics computed on \num{100} samples and the test metrics on dSprites. 
We observe that MIG and DCI Disentanglement generally correlate well with the test metrics (with the only exception of Modularity) while the correlation for the SAP score is substantially lower.
This is not surprising given that the SAP score requires us to train a multiclass support vector machine for each dimension of $r(\rvx)$ predicting each dimension of $\rvz$. 
For example, on Cars3D the factor determining the object type can take 183 distinct values making it hard to train a classifier using only 100 training samples.
In Figure~\ref{figure:metric_comparison_perfect_single}~(b), we observe that the rank correlation improves considerably for the SAP score if we have \num{1000} labeled examples available and slightly for MIG and DCI Disentanglement. In Figure~\ref{figure:latent_traversal_s2} (top) we show latent traversals for the $U/S$ model achieving maximum validation MIG on \num{1000} examples on Shapes3D.
Figure~\ref{figure:metric_comparison_perfect_single}~(c) shows the rank correlation between the model selection metrics with binned values and the test metrics with exact labels. 
We observe that binned labeling does not seem detrimental to the effectiveness of model selection with few labels.
We interpret these results as follows: For disentanglement, fine-grained labeling is not critical as the different factors of variation can already be disentangled using coarse feedback.
Interestingly, the rank correlation of the SAP score and the test metrics improves considerably (particularly for 100 labels). 
This is to be expected, as now we only have five classes for each factor of variation so the classification problem becomes easier, and the estimate of the SAP score more reliable.
In Figure~\ref{figure:metric_comparison_perfect_single}~(d), we observe that noisy labels are only slightly impacting the performance. In Figure~\ref{figure:metric_comparison_perfect_single}~(e), we can see that observing only two factors of variation still leads to a high correlation with the test scores, although the correlation is lower than for other forms of label corruption.

\textbf{Conclusions.} 
From this experiment, we conclude that it is possible to identify good runs and hyperparameter settings on the considered data sets using the MIG and the DCI Disentanglement based on \num{100} labeled examples. 
The SAP score may also be used, depending on how difficult the underlying classification problem is. Surprisingly, these metrics are reliable, even if we do not collect the labels exactly. 
We conclude that labeling a small number of examples for supervised validation appears to be a reasonable solution to learn disentangled representations in practice. Not observing all factors of variation does not have a dramatic impact. Whenever possible, it seems better to label more factors of variation in a coarser way rather than fewer factors more accurately.

\section{Incorporating label information during training}\label{sec:semi-supervised}

Using labels for model selection---even only a small amount---raises the natural question of whether these labels should rather be used for training a good model directly.
In particular, such an approach also allows the structure of the ground-truth factors of variation to be used, for example, ordinal information.

\looseness=-1The key idea is that the limited labeling information should be used to ensure a latent space of the VAE with desirable structure w.r.t. the ground-truth factors of variation (as there is not enough labeled samples to learn a good representation solely from the labels).
We hence incorporate supervision by equipping the regularized ELBO with a constraint $R_s(q_\phi(\rvz|\rvx), \rvz)\leq \kappa$, where $R_s(q_\phi(\rvz|\rvx), \rvz)$ is a function computed on the (few) available observation-label pairs and $\kappa >0$ is a threshold.
We can now include this constraint into the loss as a regularizer under the Karush-Kuhn-Tucker conditions: 
\begin{align} \label{eq:rurs}
\max_{\phi, \theta} \quad \textsc{ELBO}(\phi, \theta) + \beta \E_{\rvx}R_u(q_\phi(\rvz|\rvx)) + \gamma_\text{sup}\E_{\rvx, \rvz} R_s(q_\phi(\rvz|\rvx), \rvz)
\end{align}
where $\gamma_\text{sup}>0$. 
We rely on the binary cross-entropy loss to match the factors to their targets, \ie, $R_s(q_\phi(\rvz | \rvx), \rvz) = -\sum_{i=1}^d z_i \log(\sigma(r(\rvx)_i)) + (1-z_i) \log(1-\sigma(r(\rvx)_i))$, where the targets $z_i$ are normalized to $[0,1]$, $\sigma(\cdot)$ is the logistic function and $r(\rvx)$ corresponds to the mean (vector) of $q_\phi(\rvz | \rvx)$.
When $\rvz$ has more dimensions than the number of factors of variation, only the first $d$ dimensions are regularized (where $d$ is the number of factors of variation).
While the $z_i$ do not model probabilities of a binary random variable but factors of variation with potentially more than two discrete states, we have found the binary cross-entropy loss to work empirically well out-of-the-box.
We also experimented with a simple $L_2$ loss $\|\sigma (r(\rvx)) - \rvz\|^2$ for $R_s$, but obtained considerably worse results than for the binary cross-entropy. 
Many other candidates for supervised regularizers could be explored in future work.

\textbf{Differences to prior work on semi-supervised disentanglement.}
Existing semi-supervised approaches tackle the different problem of disentangling some factors of variation that are (partially) observed from the others that remain entangled~\citep{reed2014learning,cheung2014discovering,mathieu2016disentangling,narayanaswamy2017learning,kingma2014semi}. 
In contrast, we assume to observe all ground-truth generative factors but only for a very limited number of observations.
Disentangling only some of the factors of variation from the others is an interesting extension of this study. However, it is not clear how to adapt existing disentanglement scores to this different setup as they are designed to measure the disentanglement of \textit{all} the factors of variation. 
We remark that the goal of the experiments in this section is to compare the two different approaches to incorporate supervision into state-of-the-art unsupervised disentanglement methods. 
\vspace{-2mm}
\subsection{Experimental setup}
\textbf{Experimental protocol.}
As in Section~\ref{sec:validation}, we compare the effectiveness of the ground-truth labels with binned, noisy, and partial labels on the performance of our semi-supervised approach. 
\looseness=-1To include supervision during training we split the labeled examples in a $90\%$/$10\%$ train/validation split. We consider 40 different experimental settings each corresponding to a data set (\textit{dSprites}/ \textit{Cars3D}/ \textit{SmallNORB}/ \textit{Shapes3D}), a specific number of labeled examples (\num{100}/\num{1000}), and a labeling setting (perfect/ binned/ noisy/ partial/randomly permuted). The randomly permuted labels are used to check the ordinal inductive bias of our loss and is not further discussed in the thesis, see~\citep{locatello2019disentangling}.
For each considered setting, we generate the same five different sets of labeled examples we used for the $U/S$ models.
For each of the labeled sets, we train cohorts of $\beta$-VAEs, $\beta$-TCVAEs, Factor-VAEs, and DIP-VAE-Is with the additional supervised regularizer $R_s(q_\phi(\rvz|\rvx), \rvz)$. 
Each model cohort consists of 36 different models with 6 different hyperparameters for each of the two regularizers and one random seed.
For each of these \num{28800} models, we compute the value of $R_s$ on the validation examples and use these scores to select the best method in each of the cohorts. 
For these models we use the prefix $S^2/S$ for \emph{semi-supervised} training with \emph{supervised} model selection and compute the same test disentanglement metrics as in Section~\ref{sec:validation}. 

\textbf{Fully supervised baseline.}
We further consider a fully supervised baseline where the encoder is trained solely based on the supervised loss with perfectly labeled training examples (again with a $90\%$/$10\%$ train/validation split). 
The supervised loss does not have any tunable hyperparameter, and for each labeled data set, we run cohorts of six models with different random seeds. 
For each of these \num{240} models, we compute the value of $R_s$ on the validation examples and use these scores to select the best method in the cohort.

\vspace{-2mm}
\subsection{Should labels be used for training?} \label{sec:labels_train}
\looseness=-1Each dot in the panels in Figure~\ref{figure:comparison_perfect}, corresponds to the median of the DCI Disentanglement score across the draws of the labeled subset (using 100 vs. 1000 examples for validation). For the $U/S$ models, we use MIG for validation (MIG has a higher rank correlation with most of the testing metrics than other validation metrics, see Figure~\ref{figure:metric_comparison_perfect_single}). 
From this plot, one can see that the fully supervised baseline performs worse than those that use unsupervised data. 
As expected, having more labels can improve the median disentanglement for the $S^2/S$ approaches (depending on the data set and the test metric) but does not improve much the $U/S$ approaches (recall that we observed in Figure~\ref{figure:metric_comparison_perfect_single} (a) that the validation metrics already perform well with 100 samples).

To test whether incorporating the label information during training is better than using it for validation only, we report in Figure~\ref{figure:bin_compare} (a) how often each approach outperforms all the others on a random disentanglement metric and data set. 
We observe that semi-supervised training often outperforms supervised validation.  
In particular, $S^2/S$-$\beta$-TC-VAE seems to improve the most, outperforming the $S^2/S$-Factor-VAE, which was the best method for \num{100} labeled examples. Using \num{100} labeled examples, the $S^2/S$ approach already wins in 70.5\% of the trials. 
We observe similar trends even when we use the testing metrics for validation (based on the full testing set) in the $U/S$ models, see~\citep{locatello2019disentangling}. The $S^2/S$ approach seems to improve training overall and to transfer well across the different disentanglement metrics. 
In Figure~\ref{figure:latent_traversal_s2} (bottom) we show the latent traversals for the best $S^2/S$ $\beta$-TCVAE using \num{1000} labeled examples. 
We observe that it achieves excellent disentanglement and that the unnecessary dimensions of the latent space are unused, as desired.

\begin{figure}
\centering\includegraphics[width=\textwidth]{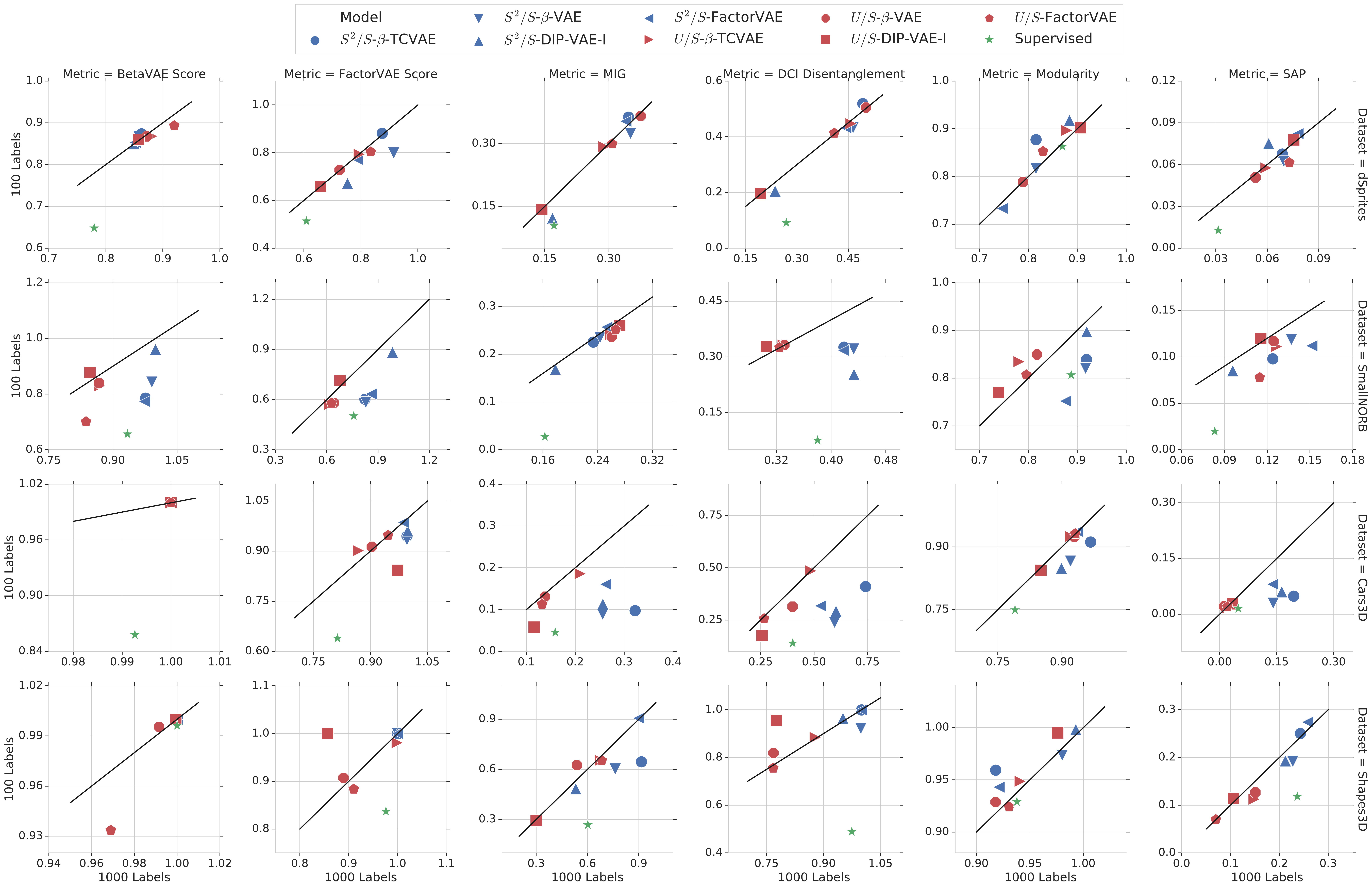}
\caption{Median across the draws of the labeled data set of the DCI Disentanglement test score on SmallNORB after validation with 100 and 1000 labeled examples. $U/S$ were validated with the MIG. 
}\label{figure:comparison_perfect}
\vspace{-3mm}
\end{figure}



\textbf{Conclusions:} Even though our semi-supervised training does not directly optimize the disentanglement scores, it seems beneficial compared to unsupervised training with supervised selection. The more labels are available, the larger the benefit. Finding extremely sample efficient disentanglement metrics is however an important research direction for practical applications of disentanglement.

\subsection{Robustness to imprecise labels} \label{sec:train_labels_imprecise}

\begin{figure}
\vspace{-2mm}
\begin{center}
\begin{subfigure}{0.22\textwidth}\centering\includegraphics[width=\textwidth]{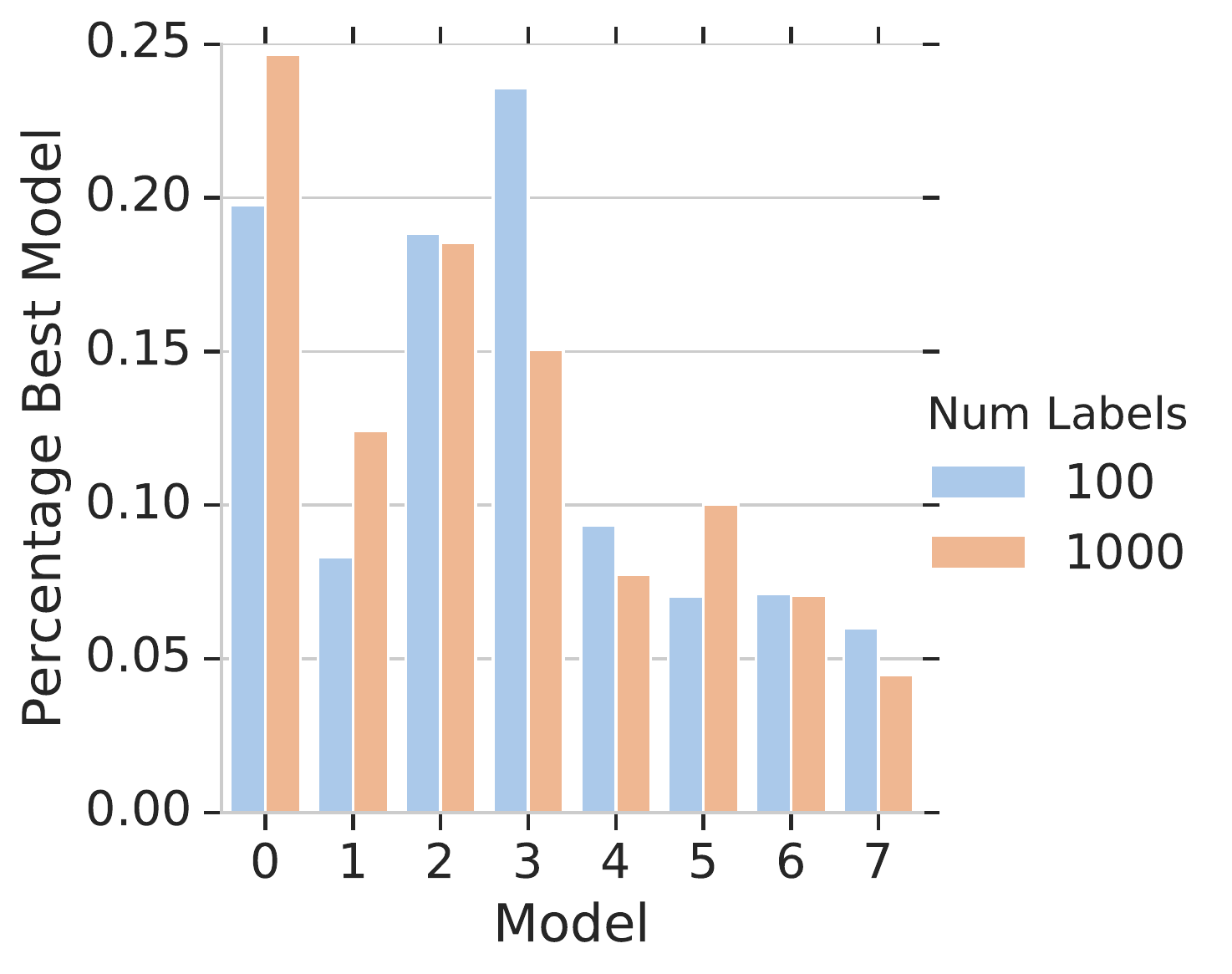}\caption{Perfect labels.}
\end{subfigure}%
\hspace{4mm}
\begin{subfigure}{0.22\textwidth}
\centering\includegraphics[width=\textwidth]{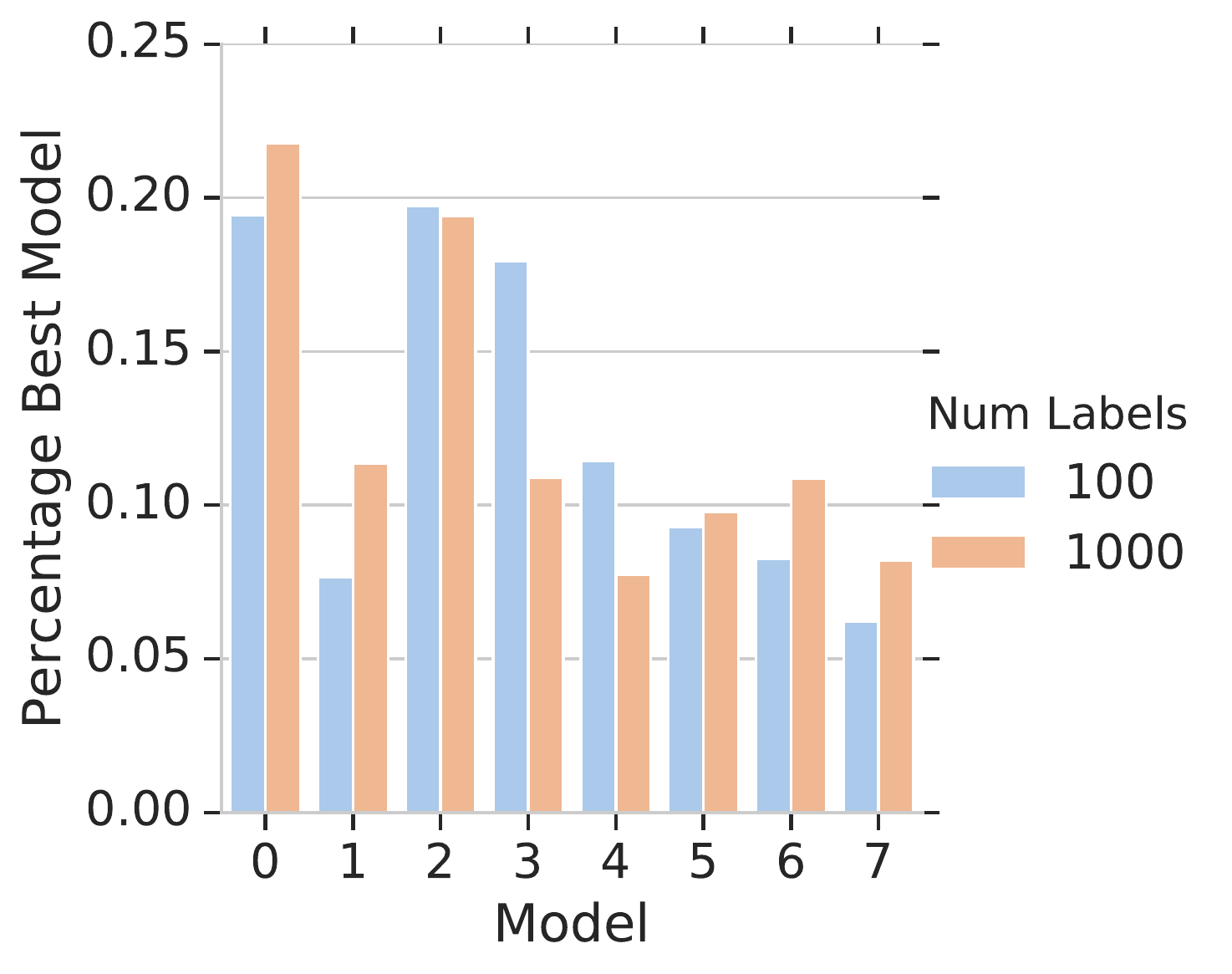}\caption{Binned labels.}
\end{subfigure}%
\hspace{4mm}
\begin{subfigure}{0.22\textwidth}
\centering\includegraphics[width=\textwidth]{sections/disent/semi_sup/autofigures/hist_best_method_noisy}\caption{Noisy labels.}
\end{subfigure}%
\begin{subfigure}{0.22\textwidth}
\centering\includegraphics[width=\textwidth]{sections/disent/semi_sup/autofigures/hist_best_method_partial}\caption{Partial labels.}
\end{subfigure}
\end{center}
\vspace{-3mm}
\caption{\looseness=-1Probability of each method being the best on a random test metric and a random data set after validation with different types of labels. Legend: 0=$S^2/S$-$\beta$-TCVAE, 1=$S^2/S$-$\beta$-VAE, 2=$S^2/S$-DIP-VAE-I, 3=$S^2/S$-FactorVAE, 4=$U/S$-$\beta$-TCVAE, 5=$U/S$-$\beta$-VAE, 6=$U/S$-DIP-VAE-I, 7=$U/S$-FactorVAE. Overall, it seem more beneficial to incorporate supervision during training rather than using it only for validation. Having more labels available increases the gap.}\label{figure:bin_compare}
\vspace{-4mm}
\end{figure}
\begin{figure}
\centering\includegraphics[width=\textwidth]{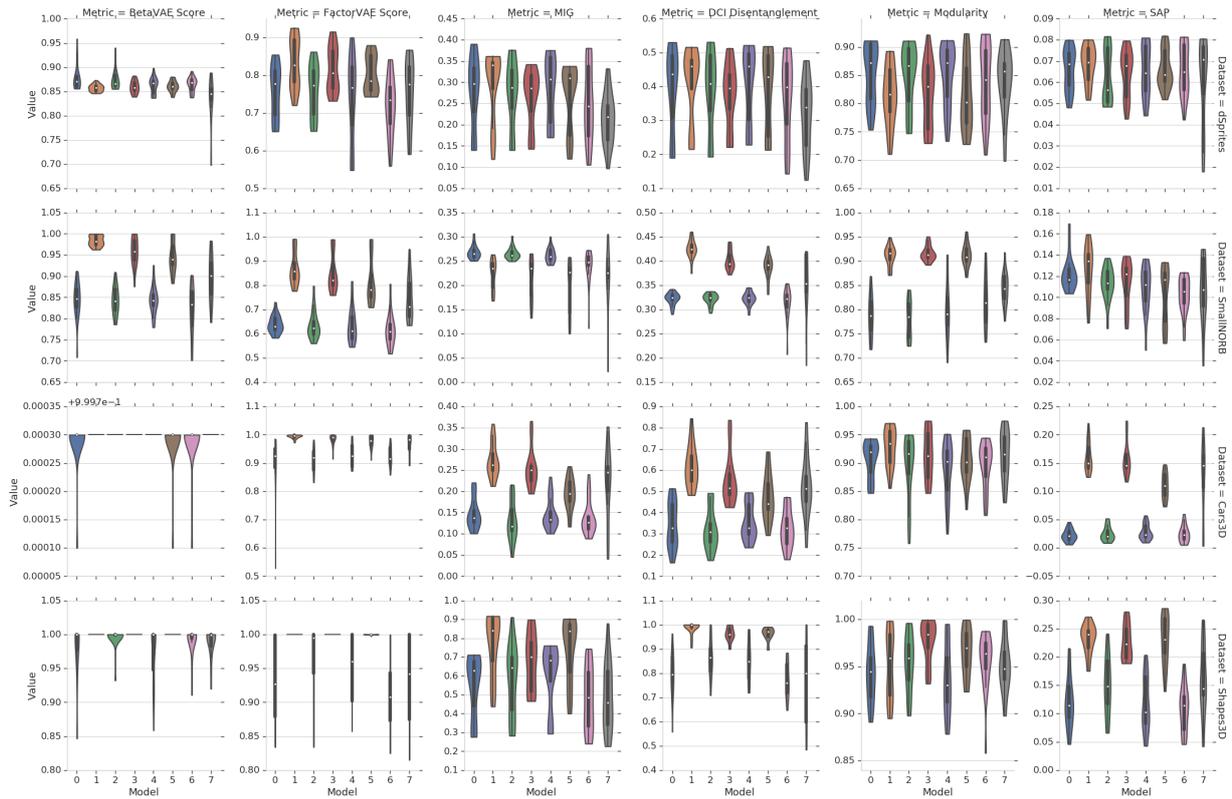}
\caption{\looseness=-1 Distribution of models trained with different types of labels with 1000 samples, $U/S$ validated with MIG. Legend: 0=$U/S$ perfect, 1=$S^2/S$ perfect, 2=$U/S$ binned, 3=$S^2/S$ binned, 4=$U/S$ noisy, 5=$S^2/S$ noisy, 6=$U/S$ partial, 7=$S^2/S$ partial.}\label{figure:bin_compare_violin}
\end{figure}

\looseness=-1
In Figure~\ref{figure:bin_compare_violin}, we observe that imprecise labels do not considerably worsen the performance of both the supervised validation and the semi-supervised training. 
Sometimes the regularization induced by simplifying the labels appears to improve generalization, arguably due to a reduction in overfitting.
We observe that the model selection metrics are slightly more robust than the semi-supervised loss. 
However, as shown in Figure~\ref{figure:bin_compare} (b-d), the semi-supervised approaches still outperform supervised model selection in 64.8\% and 67.5\% of the trials with 100 binned and noisy labels, respectively. The only exception appears to be with partial labels, where the two approaches are essentially equivalent (50.0\%) with 100 labeled examples, and the semi-supervised improves (62.6\%) only with 1000 labeled examples. 

\textbf{Conclusion:} The $S^2/S$ methods are also robust to imprecise labels. While the $U/S$ methods appear to be more robust, $S^2/S$ methods are still outperforming them.

\chapter{Fairness of Disentangled Representations}\label{cha:fairness}
\looseness=-1In this chapter, we discuss the fairness properties of disentangled representations. The presented work is based on \citep{locatello2019fairness} and was developed in collaboration with Gabriele Abbati, Tom Rainforth, Stefan Bauer, Bernhard Sch\"olkopf, and Olivier Bachem. This work was partially done when Francesco Locatello was at Google Research, Brain Team in Zurich.

\section{General Purpose Representations and Fairness}\label{sec:fairness}
In this chapter, we investigate the downstream usefulness of disentangled representations through the lens of fairness. For this, we consider the standard setup of disentangled representation learning, in which observations are the result of an (unknown) mixing mechanism of independent ground-truth factors of variation, as depicted in Figure~\ref{fig:graph}.
As one builds machine learning models for different tasks on top of such general purpose representations, it is not clear how the properties of the representations relate to the fairness of the predictions.
In particular, for different downstream prediction tasks, there may be different sensitive variables that we would like to be fair to.
This is modeled in our setting of Figure~\ref{fig:graph} by allowing one ground-truth factor of variation to be the target variable $\rvy$ and another one to be the sensitive variable $\rvs$.\footnote{Please see Section~\ref{sec:unfair} for how this is done in the experiments.}
There are two key differences to prior setups in the fairness literature:
First, we assume that one only observes the observations $\rvx$ when learning the representation $r(\rvx)$ and the target variable $\rvy$ only when solving the downstream classification task. 
The sensitive variable $\rvs$ and the remaining ground-truth factors of variation are not observed.
The second difference is that we assume that the target variable $\rvy$ and the sensitive variable $\rvs$ are independent.
While beyond the scope of this chapter, it would be interesting to study the setting where ground-truth factors of variations are dependent. 

To evaluate the learned representations $r(\rvx)$ of these observations, we assume that the set of ground-truth factors of variation include both a target factor $\rvy$, which we would like to predict from the learned representation, and an underlying sensitive factor $\rvs$, which we want to be fair to in the sense of demographic parity~\cite{calders2009building,zliobaite2015relation}, \ie such that $p(\hat \rvy = y | \rvs=s_1)=p(\hat \rvy = y | \rvs=s_2) \ \forall y, s_1, s_2$. 
The key difference to prior work is that in this setting, one never observes the sensitive variable $\rvs$ nor the other factors of variation except the target variable, which is only observed when learning the model for the downstream task. This setup is relevant when sensitive variables may not be recorded due to privacy reasons. 
Examples include learning general-purpose embeddings from a large number of images or building a world model based on video input of a robot.

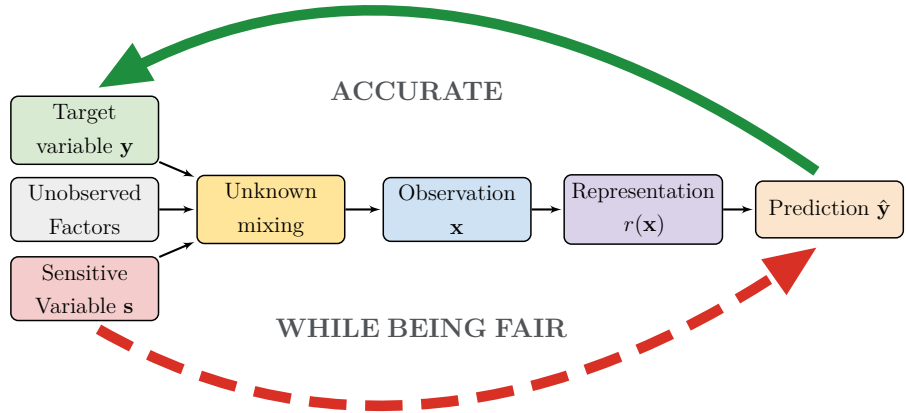
\begin{figure}
\makebox[\textwidth][c]{%
    \centering
    \definecolor{mygreen}{RGB}{30,142,62}
	\definecolor{myred}{RGB}{217,48,37}
	\scalebox{0.7}{\begin{tikzpicture}
		\tikzset{
			myarrow/.style={->, >=latex', shorten >=1pt, line width=0.4mm},
		}
		\tikzset{
			mybigredarrow/.style={dashed, dash pattern={on 20pt off 10pt}, >=latex', shorten >=3mm, shorten <=3mm, line width=2mm, draw=myred},
		}
		\tikzset{
			mybiggreenarrow/.style={->, >=latex', shorten >=3mm, shorten <=3mm, line width=2mm, draw=mygreen},
		}
		\node[text width=2.5cm,align=center,minimum size=2.6em,draw,thick,rounded corners, fill=gyellow] (unkn_mixing) at (-1.75,0) {Unknown mixing};
		\node[text width=2.5cm,align=center,minimum size=2.6em,draw,thick,rounded corners, fill=gblue] (obs) at (1.75,0) {Observation $\rvx$};
        \node[text width=2.7cm,align=center,minimum size=2.6em,draw,thick,rounded corners, fill=gpurple] (repr) at (5.25,0) {Representation $r(\rvx)$};
		\node[text width=2.5cm,align=center,minimum size=2.6em,draw,thick,rounded corners, fill=gyellow2] (pred) at (8.75,0) {Prediction $\hat{\rvy}$};
		\node[text width=2.5cm,align=center,minimum size=2.6em,draw,thick,rounded corners, fill=ggray] (unobs_fact) at (-5.25,0) {Unobserved Factors};
		\node[text width=2.5cm,align=center,minimum size=2.6em,draw,thick,rounded corners, fill=ggreen] (target_v) at (-5.25,1.5) {Target variable $\rvy$};
		\node[text width=2.5cm,align=center,minimum size=2.6em,draw,thick,rounded corners, fill=gred] (sens_v) at (-5.25,-1.5) {Sensitive Variable $\rvs$};
		\draw[myarrow] (repr) -- (pred);
		\draw[myarrow] (unkn_mixing) -- (obs);
		\draw[myarrow] (obs) -- (repr);
		\draw[myarrow] (unobs_fact) -- (unkn_mixing);
		\draw[myarrow] (sens_v) -- (unkn_mixing);
		\draw[myarrow] (target_v) -- (unkn_mixing);
		\draw[mybiggreenarrow, -triangle 45] (pred.north) to [out=145, in=30] (target_v.north);
		\draw[mybigredarrow, -triangle 45] (sens_v.south) to [out=-30, in=-145] (pred.south);
		\node[text width=6cm,align=center,minimum size=2.6em,draw=none,] (accurate) at (1.0,2.25) {\textcolor{ggray2}{\large \textbf{ACCURATE}}};
		\node[text width=6cm,align=center,minimum size=2.6em,draw=none,] (accurate) at (1.1,-2.25) {\textcolor{ggray2}{\large \textbf{WHILE BEING FAIR}}};
	\end{tikzpicture}}}
    \caption{\small Causal graph and problem setting. We assume the observations $\rvx$ are manifestations of independent factors of variation. We aim at predicting the value of some factors of variation $\rvy$ without being influenced by the \textit{unobserved} sensitive variable $\rvs$. Even though target and sensitive variable are in principle independent, they are entangled in the observations by an unknown mixing mechanism. Our goal for fair representation learning is to learn a good representation $r(\rvx)$ so that any downstream classifier will be both accurate and fair. Note that the representation is learned without supervision and when training the classifier we do not observe and do not know which variables are sensitive.}
    \label{fig:graph}
\end{figure}


\textbf{Why Can Representations be Unfair in this Setting?}
While the independence between the target variable $\rvy$ and the sensitive variable $\rvs$ may seem like an overly restrictive assumption, we argue that fairness is non-trivial to achieve even in this setting.
Since we only observe $\rvx$ or the learned representations $r(\rvx)$, the target variable $\rvy$ and the sensitive variable $\rvs$ can become conditionally dependent.
If we now train a prediction model based on $\rvx$ or $r(\rvx)$, there is no guarantee that predictions will be fair with respect to $\rvs$.

\looseness=-1There are additional considerations:
first, the following theorem shows that the fairness notion of demographic parity may not be satisfied even if we find the optimal prediction model (\ie, $p(\hat\rvy | \rvx) = p(\rvy | \rvx)$) on entangled representations (for example when the representations are the identity function, \ie $r(\rvx)=\rvx$).
\begin{theorem}\label{thm:perfect_fairness}
If $\rvx$ is entangled with $\rvs$ and $\rvy$, the use of a perfect classifier for $\hat\rvy$, \ie, $p(\hat\rvy | \rvx) = p(\rvy | \rvx)$, does not imply demographic parity, \ie, $p(\hat \rvy = y | \rvs=s_1)=p(\hat \rvy = y | \rvs=s_2), \forall y, s_1, s_2$.
\end{theorem}
While this result provides a worst-case example, it should be interpreted with care.
In particular, such instances may not allow for good and fair predictions regardless of the representations\footnote{In this case, even properties of representations such as disentanglement may not help.}, and real-world data may satisfy additional assumptions not satisfied by the provided counterexample.

Second, the unknown mixing mechanism that relates $\rvy$, $\rvs$ to $\rvx$ may be highly complex and in practice the downstream learned prediction model will likely not be equal to the theoretically optimal prediction model $p(\hat\rvy | r(\rvx))$.
As a result, the downstream prediction model may be unable to properly invert the unknown mixing mechanism and successfully separate $\rvy$ and $\rvs$, in particular as it may not be incentivized to do so.
Finally, implicit biases and specific structures of the downstream model may interact and lead to different overall predictions for different sensitive groups in $\rvs$.

\paragraph{Why Might Disentanglement Help?}
The key idea why disentanglement may help in this setting is that disentanglement promises to capture information about different generative factors in different latent dimensions.
This limits the mutual information between different code dimensions and encourages the predictions to depend only on the latent dimensions corresponding to the target variable and not to the one corresponding to the sensitive ground-truth factor of variation.
More formally, in the context of Theorem~\ref{thm:perfect_fairness}, consider a disentangled representation where the two factors of variations $\rvs$ and $\rvy$ are separated in independent components (say $r(\rvx)_y$ only depends on $\rvy$ and $r(\rvx)_s$ on $\rvs$).
Then, the optimal classifier can learn to ignore the part of its input which is independent of $\rvy$ since $p(\hat\rvy | r(\rvx)) = p(\rvy | r(\rvx)) = p(\rvy | r(\rvx)_y, r(\rvx)_s) =  p(\rvy | r(\rvx)_y)$ as $\rvy$ is independent from $r(\rvx)_s$. 
While such an optimal classifier on the representation $r(\rvx)$ might be fairer than the optimal classifier on the observation $\rvx$, it may also have a lower prediction accuracy.

\section{Do disentangled representations matter?}\label{sec:dis_matter}

\paragraph{Experimental conditions}
We adopt the setup of~\cite{locatello2019challenging} described in Chapter~\ref{cha:disent_background}, and use \num{12600} pre-trained models from Chapter~\ref{cha:unsup_dis}. We assume to observe a target variable $\rvy$ that we should predict from the representation while we do not observe the sensitive variable $\rvs$. 
For each trained model, we consider each possible pair of factors of variation as target and sensitive variables.
For the prediction, we consider the same gradient boosting classifier~\cite{friedman2001greedy} as in~Chapter~\ref{cha:unsup_dis}, which was trained on \num{10000} labeled examples (denoted by GBT10000) and which achieves higher accuracy than the cross-validated logistic regression.  
Then, we observe the values of all the factors of variations and have access to the whole generative model. 
With this, we compute the disentanglement metrics and use the following score to measure the unfairness of the predictions
\begin{align*}
    \texttt{unfairness}(\hat \rvy) = \frac{1}{| S |} \sum_s TV(p(\hat \rvy), p(\hat \rvy  \mid \rvs = s)) \ \forall \ y
\end{align*}
where $TV$ is the total variation. 
In other words, we compare the average total variation of the prediction after intervening on $\rvs$, thus directly measuring the violation of demographic parity.
The reported unfairness score for each trained representation is the average unfairness of all downstream classification tasks we considered for that representation. 

\begin{figure}
   \begin{center}
    \begin{subfigure}{0.4\textwidth}
    \includegraphics[width=\textwidth]{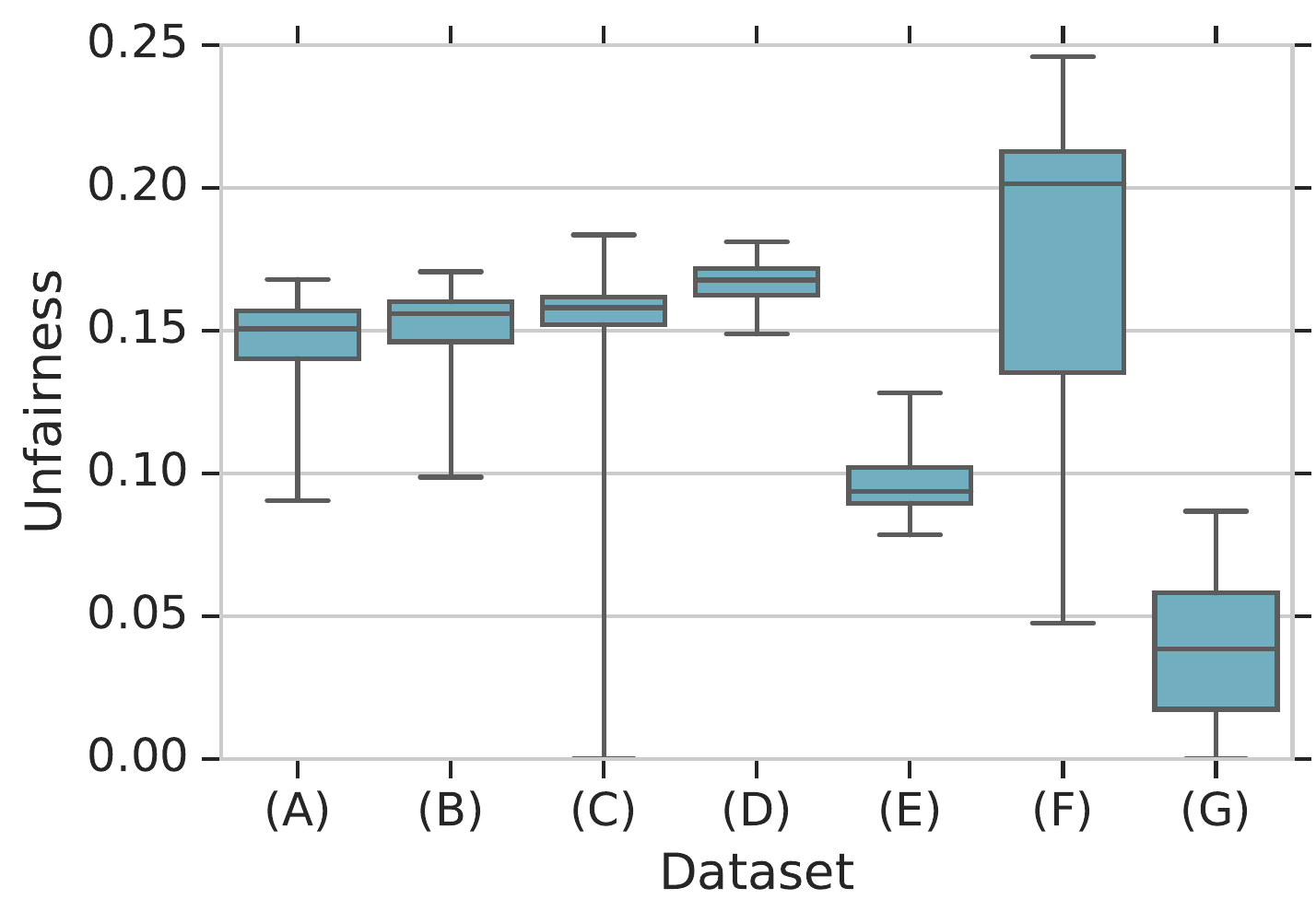}
    \end{subfigure}
    \begin{subfigure}{0.4\textwidth}
    \includegraphics[width=\textwidth]{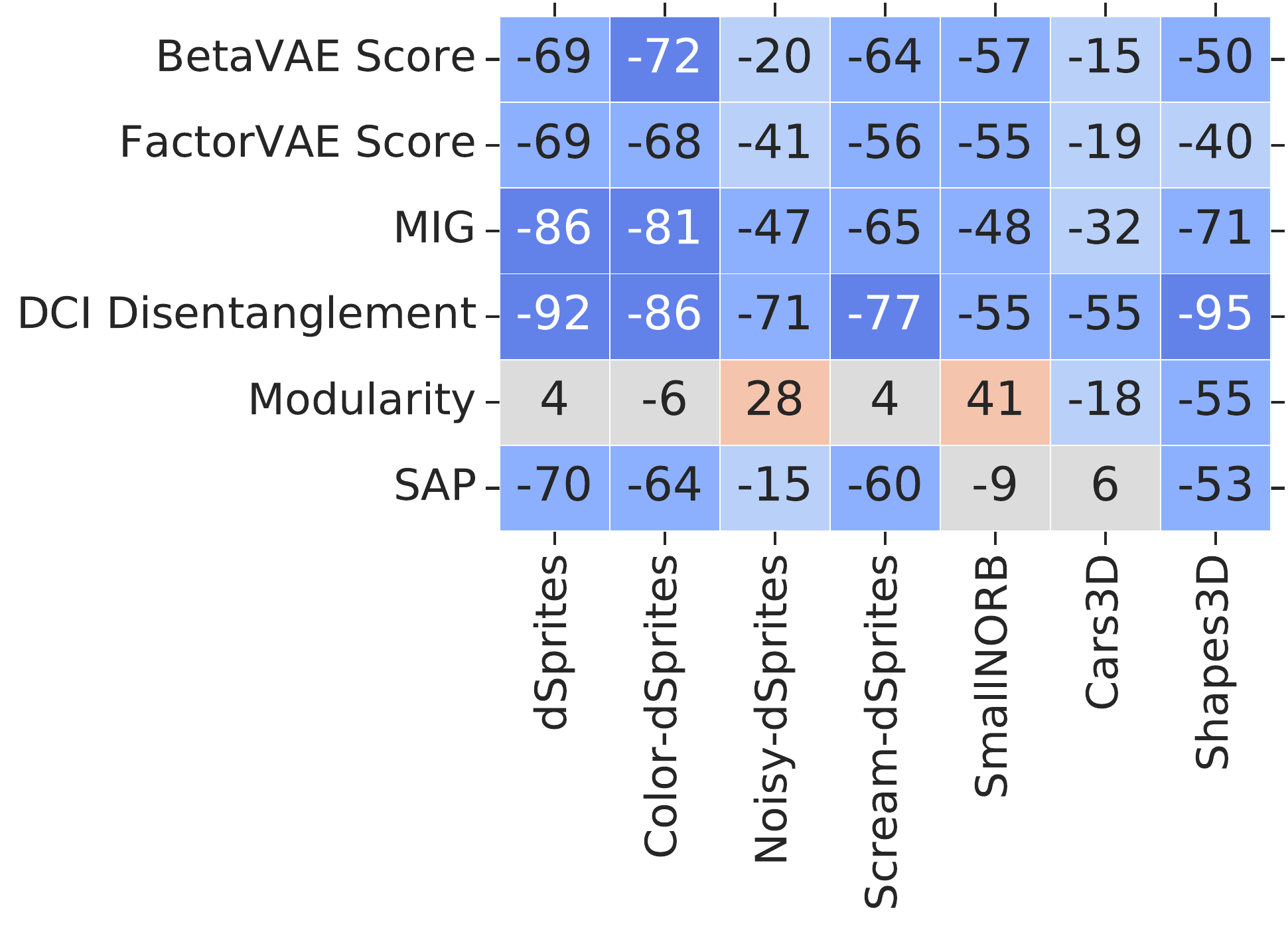}
    \end{subfigure}
    \end{center}
    \caption{\small (Left) Distribution of unfairness for learned representations. Legend: dSprites = (A), Color-dSprites = (B), Noisy-dSprites = (C), Scream-dSprites = (D), SmallNORB = (E), Cars3D = (F), Shapes3D = (G). (Right) Rank correlation of unfairness and disentanglement scores on the various data sets.}
     \label{fig:Unfairness}
     \centering
    \includegraphics[width=\textwidth]{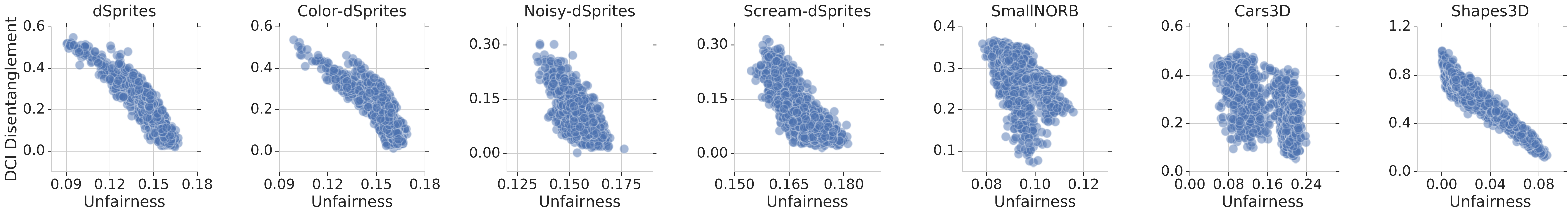}
    \caption{\small Unfairness of representations versus DCI Disentanglement on the different data sets.
    }
    \label{fig:Unfairness_vs_disentanglement_single}
\end{figure}

\subsection{The Unfairness of General Purpose Representations and the Relation to Disentanglement}\label{sec:unfair}

\looseness=-1In Figure~\ref{fig:Unfairness} (left), we show the distribution of unfairness scores for different representations on different data sets.
We clearly observe that learned representations can be unfair, even in the setting where the target variable and the sensitive variable are independent.
In particular, the total variation can reach as much as $15\%-25\%$ on five out of seven data sets.
This confirms the importance of trying to find general-purpose representations that are less unfair.
We also note that there is considerable spread in unfairness scores for different learned representations.
This indicates that the specific representation used matters and that predictions with lower unfairness can be achieved.
To investigate whether disentanglement is a useful property to guarantee less unfair representations, we show the rank correlation between a wide range of disentanglement scores and the unfairness score in Figure~\ref{fig:Unfairness} (right).
We observe that all disentanglement scores except Modularity appear to be consistently correlated with a lower unfairness score for all data sets.
While we have found the considered disentanglement metrics (except Modularity) to be correlated (see~Chapter~\ref{cha:eval_dis}), we observe differences in-between scores regarding the correlation with fairness:
Figure~\ref{fig:Unfairness} (right) indicates that DCI Disentanglement is correlated the most, followed by the Mutual Information Gap, the BetaVAE score, the FactorVAE score, the SAP score, and, finally Modularity.
The strong correlation of DCI Disentanglement is confirmed by Figure~\ref{fig:Unfairness_vs_disentanglement_single}, where we plot the Unfairness score against the DCI Disentanglement score for each model. Again, we observe that the large gap in unfairness seems to be related to differences in the representation.

\begin{figure}[h!]
    \begin{center}
    \centering
    \includegraphics[width=\textwidth]{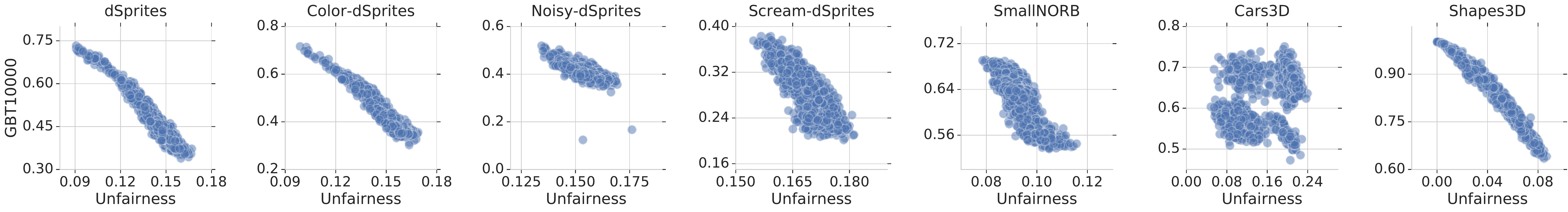}
    \caption{\small Unfairness of representations versus downstream accuracy on the different data sets.}
    \label{fig:GBT_vs_unfairness}
    \vspace{5mm}
    \begin{subfigure}{0.9\textwidth}
    \centering
    \includegraphics[width=\textwidth]{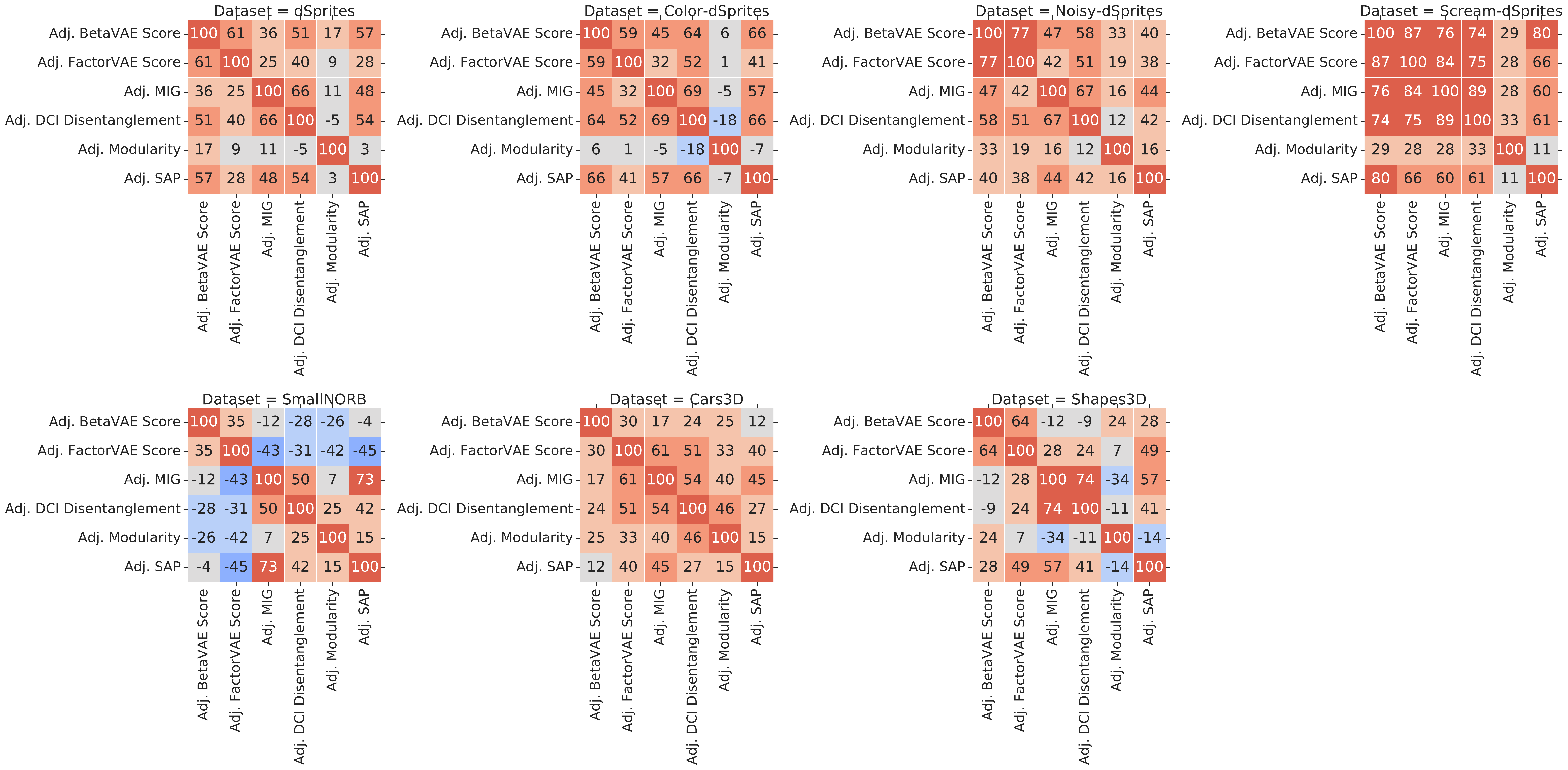}
    \end{subfigure}
    \hspace{2mm}
    \begin{subfigure}{0.9\textwidth}
    \centering
    \includegraphics[width=\textwidth]{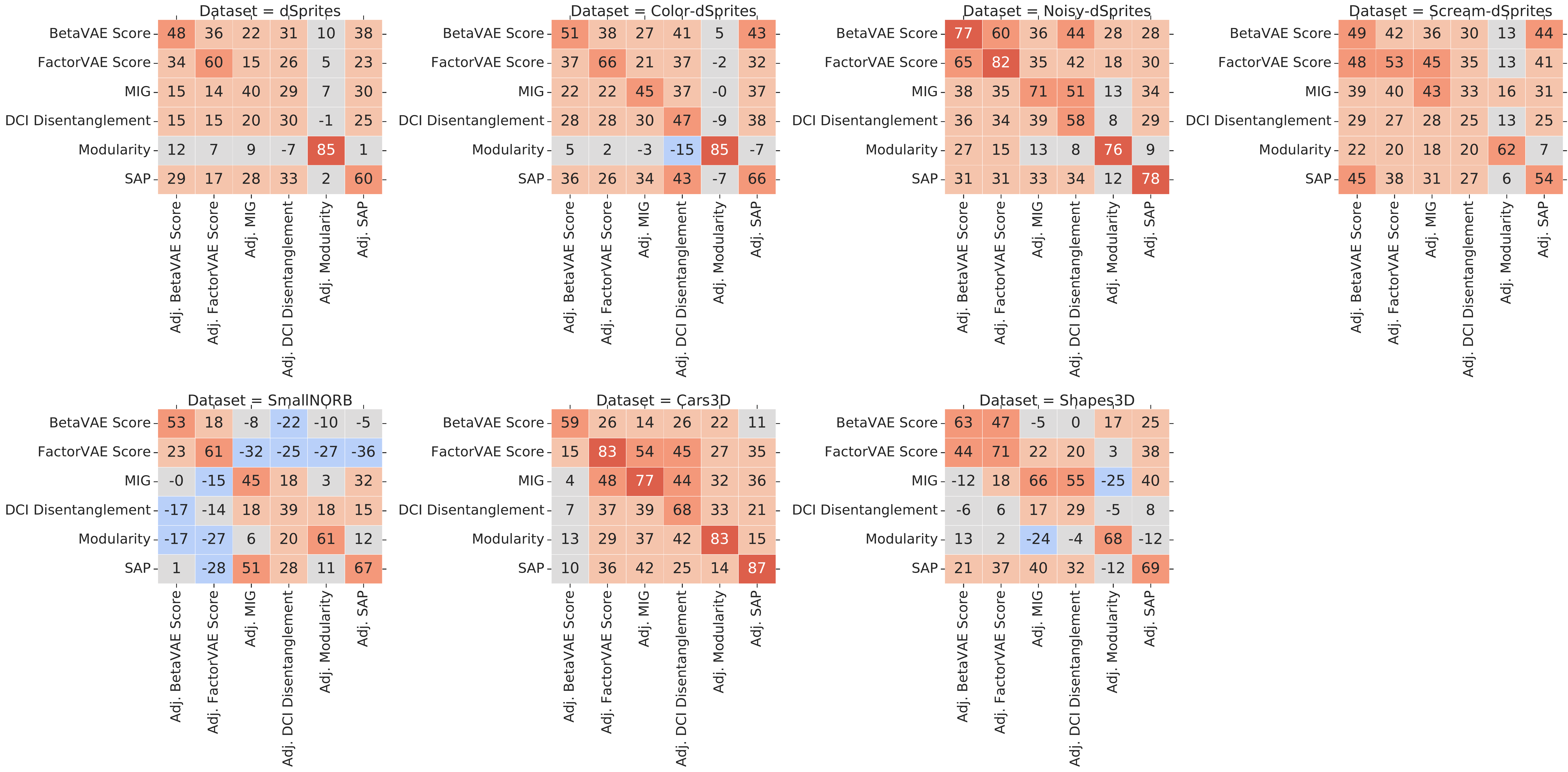}
     \end{subfigure}
     \end{center}
    \caption{\small Rank correlation between the adjusted disentanglement scores (top) and between original scores and the adjusted version (bottom).}
    \label{fig:adjusted_scores}
\end{figure}

These results provide an encouraging case for disentanglement being helpful in finding fairer representations.
However, they should be interpreted with care: 
Even though we have considered a diverse set of methods and disentangled representations, the computed correlation scores depend on the distribution of considered models.
If one were to consider an entirely different set of methods, hyperparameters, and corresponding representations, the observed relationship may differ.

\subsection{Adjusting for Downstream Performance}\label{sec:adj_unfair}

\looseness=-1In Chapter~\ref{cha:unsup_dis}, we have observed that disentanglement metrics are correlated with how well ground-truth factors of variations can be predicted from the representation using gradient boosted trees.
It is thus not surprising that the unfairness of a representation is also consistently correlated to the average accuracy of a gradient boosted trees classifier using \num{10000} samples (see Figure~\ref{fig:GBT_vs_unfairness}).
Now, we investigate whether disentanglement is also correlated with higher fairness if we compare representations with similar accuracy as measured by GBT10000 scores. \emph{Given two representations with the same downstream performance, is the more disentangled one also more fair?}
The key challenge is that for a given representation there may not be other ones with exactly the same downstream performance.

For this, we adjust all the disentanglement scores and the unfairness score for the effect of downstream performance. 
We use a k-nearest neighbors regression from Scikit-learn~\cite{scikitlearn} to predict, for any model, each disentanglement score and the unfairness from its five nearest neighbors in terms of GBT10000 (which we write as $N(\text{GBT10000})$). 
This can be seen as a one-dimensional non-parametric estimate of the disentanglement score (or fairness score) based on the GBT10000 score.
The adjusted metric is computed as the residual score after the average score of the neighbors is subtracted, namely
\begin{align*}
\texttt{Adj. Metric} = \texttt{Metric} - \frac{1}{5}\sum_{i \in N(\text{GBT10000})} \texttt{Metric}_{i}
\end{align*}
Intuitively, the adjusted metrics measure how much more disentangled/fairer a given representation is compared to an average representation with the same downstream performance.

\begin{figure}
\begin{center}
    {\adjincludegraphics[scale=0.4, trim={0 {0.9\height} 0 0}, clip]{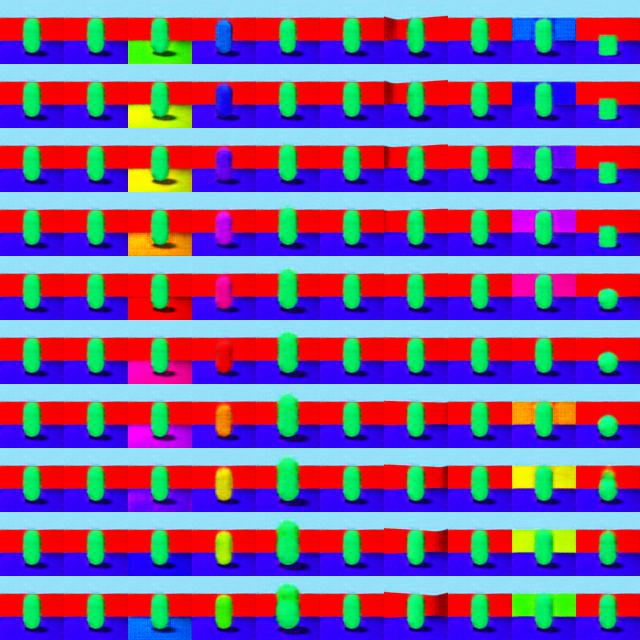}}\vspace{-0.3mm}
    {\adjincludegraphics[scale=0.4, trim={0 {.4\height} 0 {.5\height}}, clip]{sections/disent/fairness/figures/traversals4.jpg}}\vspace{-0.3mm}
    {\adjincludegraphics[scale=0.4, trim={0 0 0 {.9\height}}, clip]{sections/disent/fairness/figures/traversals4.jpg}}\vspace{2mm}
\end{center}
    \caption{\small Latent traversals (each column corresponds to a different latent variable being varied) on Shapes3D for the model with best adjusted MIG. }\label{figure:latent_traversal}
    \vspace{-4mm}
\end{figure}

\begin{wrapfigure}{l}{0.42\textwidth}
    \includegraphics[width=0.42\textwidth]{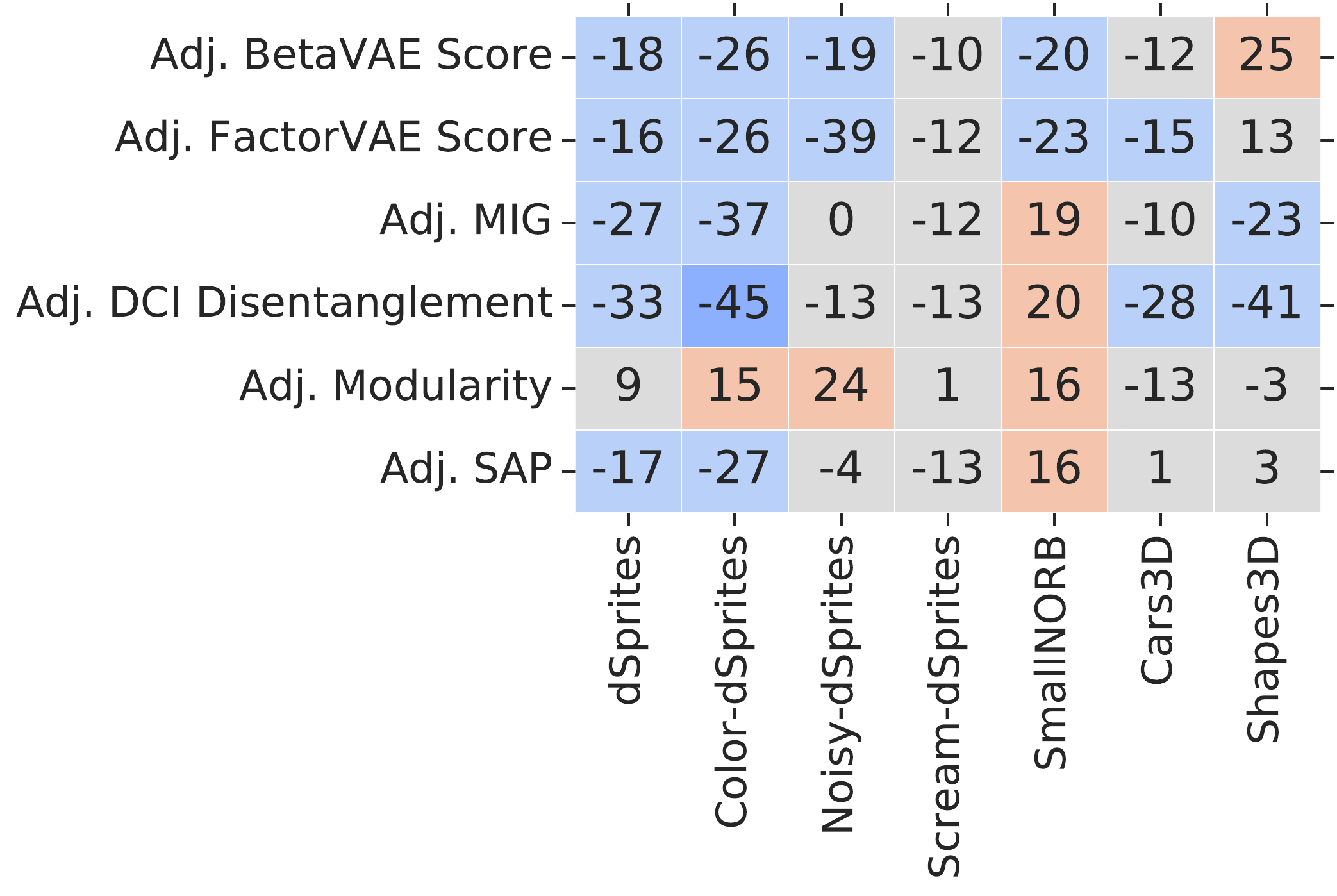}
    \caption{\small \looseness=-1Rank correlation between adjusted disentanglement and unfairness on the various data sets.}
    \label{fig:rank_fair_metrics}
    \vspace{-5mm}
\end{wrapfigure}In Figure~\ref{fig:adjusted_scores} (top), we observe that the rank correlation between the adjusted disentanglement scores (except Modularity) is consistenly positive.
This indicates that the adjusted scores do measure a similar property of the representation even when adjusted for performance. The only exception appears to be SmallNORB, where the adjusted DCI Disentanglement, MIG, and SAP score correlate with each other but do not correlate well with the BetaVAE and FactorVAE score (which only correlate with each other). On Shapes3D we observe a similar result, but the correlation between the two groups of scores is stronger than on SmallNORB. 
Similarly, Figure~\ref{fig:adjusted_scores} (bottom) shows the rank correlation between the disentanglement metrics and their adjusted versions.
As expected, we observe that there still is a positive correlation. 
This indicates the adjusted scores still capture a part of the unadjusted score. This result appears to be consistent across the different data sets, again with the exception of SmallNORB.
As a sanity check, we finally confirm by visual inspection that the adjusted metrics still measure disentanglement. In Figure~\ref{figure:latent_traversal}, we plot latent traversals for the model with the highest adjusted MIG score on Shapes3D and observe that the model appears well disentangled.

\looseness=-1Finally, Figure~\ref{fig:rank_fair_metrics} shows the rank correlation between the adjusted disentanglement scores and the adjusted fairness score for each of the data sets.
Overall, we observe that higher disentanglement still seems to be correlated with increased fairness, even when accounting for downstream performance.
Exceptions appear to be the adjusted Modularity score, the  adjusted BetaVAE, and the FactorVAE score on Shapes3D, and the adjusted MIG, DCI Disentanglement, Modularity and SAP on SmallNORB.
As expected, the correlations appear to be weaker than for the unadjusted scores (see Figure~\ref{fig:Unfairness} (right)), but we still observe some residual correlation.

\paragraph{How do we Identify Fair Models?}
In this chapter, we observed that disentangled representations might help training fairer classifiers. This leaves us with the question: \emph{how can we find fair representations?} In Chapter~\ref{cha:unsup_dis}, we showed that without access to supervision or inductive biases, disentangled representations cannot be identified. Due to the high correlation between disentanglement, downstream performance with GBT, and fairness, one could use downstream performance as a proxy for fairness. This leads to classifiers that are fairer 84.2\% of the time on the models we considered. This assumes that disentanglement is the only variable explaining prediction accuracy and fairness. Since disentanglement is likely not the only confounder, model selection based on downstream performance is not guaranteed to be fairer than random model selection as we have seen in Theorem~\ref{thm:perfect_fairness}. We will return to this issue in Chapter~\ref{cha:weak}, where we show that relying on weak-supervision alone, one can select good models that are disentangled and useful on several tasks, including this fairness setting.

\section{Proof of Theorem~\ref{thm:perfect_fairness}}
\label{sec:proof}
\begin{proof}
Our proof is by counter example. We present a simple case for which $\rvy$ is predicted from $\rvx$ in such a way that $p(\hat\rvy | \rvx) = p(\rvy | \rvx)$, but which does not satisfy demographic parity. 

We assume all variables to be Bernoulli-distributed $p(\rvs =1)= q, \, 0<q<1$ and $p(\rvy =1)= b, \, 0<b<1$ and our mixing mechanism to be $\rvx = \min(\rvy, \rvs)$. The assumption of demographic parity yields:
\begin{align*}
DP &\implies \sum_{\rvx \in \{0,1\}}p(\hat \rvy = 1, \rvx | \rvs = 1) = \sum_{\rvx \in \{0,1\}} p(\hat \rvy = 1 , \rvx| \rvs = 0) .
\end{align*}
Using the causal Markov condition~\cite{PetJanSch17}, we can rewrite $p(\hat \rvy = 1, \rvx | \rvs) = p(\hat \rvy = 1| \rvx) p(\rvx | \rvs)$ and thus
\begin{align*}
DP &\implies \sum_{\rvx \in \{0,1\}}p(\hat \rvy = 1| \rvx) p(\rvx | \rvs = 1) = \sum_{\rvx \in \{0,1\}} p(\hat \rvy = 1| \rvx) p(\rvx | \rvs = 0)
\end{align*}
The rest of the proof follows as a proof by contradiction.  Assuming that the classifier satisfies $p(\hat\rvy | \rvx) = p(\rvy | \rvx)$, we have
\begin{align*}
DP &\implies \sum_{\rvx \in \{0,1\}}p( \rvy = 1| \rvx) \left[ p(\rvx | \rvs = 1) - p(\rvx | \rvs = 0)\right] = 0. \\
\intertext{At this point, using the fact that $\rvx = \min(\rvy, \rvs)$, we have $p(\rvx = 0 | \rvs = 1)=p(\rvy=0)$,~ $p(\rvx = 0 | \rvs = 0)=1$,~ $p( \rvy = 1| \rvx = 1)=1$,~ $p(\rvx = 1 | \rvs = 1) = p(\rvy=1)$, ~and~ $p(\rvx = 1| \rvs = 0)=0$, therefore}
DP &\implies p( \rvy = 1| \rvx = 0) \left[ p(\rvy = 0) - 1\right]+1\cdot\left[ p(\rvy = 1) - 0\right] = 0\\
&\implies -b \cdot p( \rvy = 1| \rvx = 0)+b = 0\\
&\implies p( \rvy = 1| \rvx = 0) = 1 \\
&\implies p( \rvx = 0 | \rvy = 1) p(\rvy = 1) = p( \rvx = 0) \\
&\implies p( \rvs = 0) p(\rvy = 1) = p( \rvx = 0) \\
&\implies (1-q) b = (1-q)+q(1-b) \\
&\implies 0 = (1-q)(1-b)+q(1-b) \\
&\implies b=1
\end{align*}
Hence we have our desired contradiction as, by assumption, $b<1$.
\end{proof}



\chapter{Weakly-Supervised Disentanglement}\label{cha:weak}
\looseness=-1In this chapter, we propose a new framework for disentanglement relying on weak supervision. The presented work is based on \citep{locatello2020weakly} and was developed in collaboration with Ben Poole, Gunnar R\"atsch, Bernhard Sch\"olkopf, Olivier Bachem, and Michael Tschannen. This work was partially done when Francesco Locatello was at Google Research, Brain Team in Zurich.

\section{Problem Setting}
\looseness=-1Many data modalities are \textit{not} observed as \iid samples from a distribution~\cite{dayan1993improving,storck1995reinforcement,hochreiter1999feature,bengio2013representation,PetJanSch17,thomas2018disentangling,scholkopf2019causality}. Changes in natural environments, which typically correspond to changes of only a few underlying factors of variation, provide a weak supervision signal for representation learning algorithms~\cite{foldiak1991learning,schmidt2007learning,bengio2017consciousness,bengio2019meta}. 
State-of-the-art weakly-supervised disentanglement methods~\cite{bouchacourt2017multi,gvae2019,shu2020weakly} assume that observations belong to annotated groups where two things are known at training time: (i) the relation between images in the same group, and (ii) the group each image belongs to. \citet{bouchacourt2017multi,gvae2019} consider groups of observations differing in precisely one of the underlying factors. An example of such a group are images of a given object with a fixed orientation, in a fixed scene, but of varying color. \citet{shu2020weakly} generalized this notion to other relations (\eg, single shared factor, ranking information).
In general, precise knowledge of the groups and their structure may require either explicit human labeling or at least strongly controlled acquisition of the observations. As a motivating example, consider the video feedback of a robotic arm. In two temporally close frames, both the manipulated objects and the arm may have changed their position, the objects themselves may be different, or the lighting conditions may have changed due to failures.

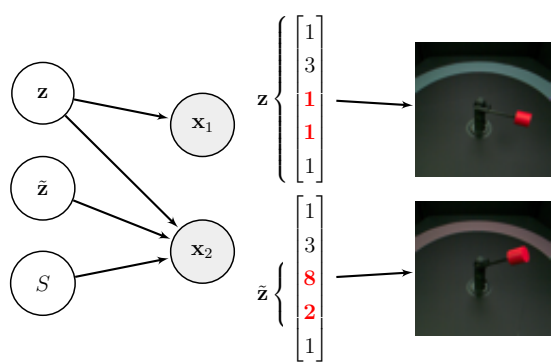
\begin{wrapfigure}{l}{0.5\textwidth}
\hspace{0.3cm}\makebox[0.42\textwidth][c]{%
    \centering
    \definecolor{mygreen}{RGB}{30,142,62}
	\definecolor{myred}{RGB}{217,48,37}
	\definecolor{myw}{RGB}{255,255,255}
	\scalebox{0.7}{
	\centering
	\begin{minipage}{.28\textwidth}
	\begin{tikzpicture}[scale=1.2]
		\tikzset{
			myarrow/.style={->, >=latex', shorten >=1pt, line width=0.4mm},
		}
		\tikzset{
			mybigredarrow/.style={dashed, dash pattern={on 20pt off 10pt}, >=latex', shorten >=3mm, shorten <=3mm, line width=2mm, draw=myred},
		}
		\tikzset{
			mybiggreenarrow/.style={->, >=latex', shorten >=3mm, shorten <=3mm, line width=2mm, draw=mygreen},
		}
		\node[text width=.8cm,align=center,minimum size=2.6em,draw,thick,circle, fill=myw] (z) at (-5.25,1.5) {$\rvz$};
		\node[text width=.8cm,align=center,minimum size=2.6em,draw,thick,circle, fill=myw] (zt) at (-5.25,0) {$\tilde\rvz$};
		\node[text width=.8cm,align=center,minimum size=2.6em,draw,thick,circle, fill=myw] (s) at (-5.25,-1.5) {$S$};
		
		\node[text width=.8cm,align=center,minimum size=2.6em,draw,thick,circle, fill=ggray] (x1) at (-2.75,1) {$\rvx_1$};
		
		\node[text width=.8cm,align=center,minimum size=2.6em,draw,thick,circle, fill=ggray] (x2) at (-2.75,-1) {$\rvx_2$};
		\draw[myarrow] (z) -- (x1);
		\draw[myarrow] (z) -- (x2);
		\draw[myarrow] (zt) -- (x2);
		\draw[myarrow] (s) -- (x2);

	\end{tikzpicture}
	\end{minipage}%
	\begin{minipage}{.35\textwidth}
	\begin{tikzpicture}
		\tikzset{
			myarrow/.style={->, >=latex', shorten >=1pt, line width=0.4mm},
		}
		\tikzset{
			mybigredarrow/.style={dashed, dash pattern={on 20pt off 10pt}, >=latex', shorten >=3mm, shorten <=3mm, line width=2mm, draw=myred},
		}
		\tikzset{
			mybiggreenarrow/.style={->, >=latex', shorten >=3mm, shorten <=3mm, line width=2mm, draw=mygreen},
		}
		\node[fill=myw] (z1) at (-5.25,1.7) {
		$\begin{bmatrix}
           1 \\
           3 \\
           \color{red}\textbf{1} \\
           \color{red}\textbf{1} \\
           1 \\
         \end{bmatrix}
		$};

		\node[] (z) at (-5.75,1.7) {
		$\rvz\left\{
            \begin{array}{ c }
            \\
            \\
            \\
            \\
            \\
            \end{array}
        \right.
		$};
		
		\node[] (zt) at (-5.75,-2) {
		$\tilde\rvz\left\{
            \begin{array}{ c }
            \\
            \\
            \end{array}
        \right.
		$};

		\node[fill=myw] (z2) at (-5.25,-1.7) {$\begin{bmatrix}
           1 \\
           3 \\
           \color{red}\textbf{8} \\
           \color{red}\textbf{2} \\
           1 \\
         \end{bmatrix}
		$};

		\node[inner sep=0pt] (x1) at (-2.,1.5) {\includegraphics[width=.45 \textwidth]{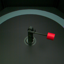}};
		
		\node[inner sep=0pt] (x2) at (-2.,-1.5) {\includegraphics[width=.45\textwidth]{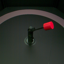}};
		
		\draw[myarrow] (z1) -- (x1);
		\draw[myarrow] (z2) -- (x2);

	\end{tikzpicture}
	\end{minipage}
	}}
	
    \caption{\small (\textbf{left}) The generative model: We observe pairs of observations $(\rvx_1, \rvx_2)$ sharing a random subset $S$ of latent factors: $\rvx_1$ is generated by $\rvz$; $\rvx_2$ is generated by combining the subset $S$ of $\rvz$ and resampling the remaining entries (modeled by $\tilde \rvz$). (\textbf{right}) Real-world example of the model: A pair of images from MPI3D~\cite{gondal2019transfer} where all factors are shared except the first degree of freedom and the background color (red values). This corresponds to a setting where few factors in a causal generative model change, which, by the {\em independent causal mechanisms} principle, leaves the others invariant \cite{scholkopf2012causal}.}
    \label{fig:disent_cause}
    \vspace{-0.5cm}
\end{wrapfigure}

\textbf{Weakly-Supervised Generative Model} We study learning of disentangled image representations from paired observations, for which some (but not all) factors of variation have the same value. This can be modeled as sampling two images from the causal generative model with an intervention~\cite{PetJanSch17} on a random subset of the factors of variation. Our goal is to use the additional information given by the pair (as opposed to a single image) to learn disentangled image representations. 
We generally do not assume knowledge of which or how many factors are shared, \ie, we do not require a controlled acquisition of the observations. This observation model applies to many practical scenarios. For example, we may want to learn a disentangled representation of a robot arm observed through a camera: In two temporally close frames, some joint angles will likely have changed, but others will have remained constant. Other factors of variation may also change independently of the actions of the robot. An example can be seen in Figure~\ref{fig:disent_cause} (right), where the first degree of freedom of the arm and the color of the background changed. More generally, this observation model applies to many natural scenes with moving objects~\citep{foldiak1991learning}. 
\newcommand{\dx}{\mathrm{d}}
\looseness=-1
For simplicity of exposition, we assume that the number of factors $k$ in which the two observations differ is constant (we present a strategy to deal with varying $k$ in Section~\ref{sec:algorithms}). 
The generative model is given by 
\begin{align}
    p(\rvz) &= \prod_{i=1}^d p(z_i), \quad p(\tilde \rvz) = \prod_{i=1}^{k} p(\tilde z_i), \quad S \sim p(S) \label{eq:generative-model-1}\\
    \rvx_1 &= g^\star(\rvz), \qquad 
    \rvx_2 = g^\star(f(\rvz, \tilde \rvz, S)), \label{eq:generative-model-4}
\end{align}
where $S$ is the subset of shared indices of size $d-k$ sampled from a distribution $p(S)$ over the set 
$\gS = \{ S \subset [d] \colon |S|=d-k \}$, and the $p(z_i)$ and $p(\tilde z_j)$ are all identical. 
The generative mechanism is modeled using a function
$g^\star \colon \gZ \to \gX$, with $\gZ = \supp(\rvz) \subseteq \sR^d$ and $\gX \subset \sR^m$, which maps the latent variable to observations of dimension $m$, typically $m \gg d$. To make the relation between $\rvx_1$ and $\rvx_2$ explicit, we use a function $f$ obeying
\begin{equation*}
    f(\rvz, \tilde \rvz, S)_S = \rvz_S \qquad \text{and} \qquad f(\rvz, \tilde \rvz, S)_{\bar S} = \tilde \rvz 
\end{equation*}
\looseness=-1with $\bar S = [d] \backslash S$. Intuitively, to generate $\rvx_2$, $f$ selects entries from $\rvz$ with index in $S$ and substitutes the remaining factors with $\tilde \rvz$, thus ensuring that the factors indexed by $S$ are shared in the two observations. 
The generative model \eqref{eq:generative-model-1}--\eqref{eq:generative-model-4} does not model additive noise; we assume that noise is explicitly modeled as a latent variable and its effect is manifested through $g^\star$ as done by~\cite{bengio2013representation,higgins2018towards,higgins2016beta,suter2018interventional,reed2015deep,lecun2004learning,kim2018disentangling,gondal2019transfer}.
For simplicity, we consider the case where groups consisting of two observations (pairs), but extensions to more than two observations are possible~\cite{gresele2019incomplete}. 

\section{Identifiability and Algorithms}
First, we show that, as opposed to the unsupervised case in Chapter~\ref{cha:unsup_dis}, the generative model \eqref{eq:generative-model-1}--\eqref{eq:generative-model-4} is identifiable under weak additional assumptions. Note that the joint distribution of all random variables factorizes as
\begin{equation}
    p(\rvx_1, \rvx_2, \rvz, \tilde \rvz, S) = p(\rvx_1 | \rvz)p(\rvx_2 | f(\rvz, \tilde \rvz, S)) p(\rvz) p(\tilde \rvz) p(S) \label{eq:full-joint}
\end{equation}
where the likelihood terms have the same distribution, \ie, $p(\rvx_1 | \bar \rvz) = p(\rvx_2 | \bar \rvz), \forall \bar \rvz \in \supp(p(\rvz))$. We show that to learn a disentangled generative model of the data $p(\rvx_1, \rvx_2)$ it is therefore sufficient to recover a factorized latent distribution with factors $p(\hat z_i)= p(\hat{\tilde{z}}_j)$, a corresponding likelihood $q(\rvx_1|\cdot)=q(\rvx_2|\cdot)$, as well as a distribution $p(\hat S)$ over $\gS$, which together satisfy the constraints of the true generative model \eqref{eq:generative-model-1}--\eqref{eq:generative-model-4} and match the true $p(\rvx_1, \rvx_2)$ after marginalization over $\hat \rvz,\hat{\tilde{\rvz}}, \hat S$ when substituted into \eqref{eq:full-joint}.

\begin{theorem} \label{thm:identifiability}
Consider the generative model \eqref{eq:generative-model-1}--\eqref{eq:generative-model-4}. Further assume that $p(z_i) = p(\tilde z_i)$ are continuous distributions, $p(S)$ is a distribution over $\gS$ s.t. for $S, S' \sim p(S)$ we have $P(S \cap S' = \{i\}) > 0, \forall i \in [d]$. Let $g^\star \colon \gZ \to \gX$ in \eqref{eq:generative-model-4} be smooth and invertible on $\gX$ with smooth inverse (\ie, a diffeomorphism). Given unlimited data from $p(\rvx_1, \rvx_2)$ and the true (fixed) $k$, consider all tuples $(p(\hat z_i), q(\rvx_1|\hat \rvz), p(\hat S))$ obeying these assumptions and matching $p(\rvx_1, \rvx_2)$ after marginalization over $\hat \rvz,\hat{\tilde{\rvz}}, \hat S$ when substituted in \eqref{eq:full-joint}. Then, the posteriors $q(\hat \rvz | \rvx_1) = q(\rvx_1|\hat \rvz)p(\hat\rvz)/p(\rvx_1)$ are disentangled in the sense that the aggregate posteriors $q(\hat \rvz) = \int q(\hat \rvz|\rvx_1) p(\rvx_1) \dx \rvx_1 = \iint q(\hat \rvz|\rvx_1) p(\rvx_1|\rvz) p(\rvz)\dx \rvz \dx \rvx_1$ are coordinate-wise reparameterizations of the ground-truth prior $p(\rvz)$ up to a permutation of the indices of $\rvz$.
\end{theorem}

\looseness=-1\textbf{Discussion} 
Under the assumptions of this theorem, we established that all generative models that match the true marginal over the observations $p(\rvx_1, \rvx_2)$ must be disentangled. Therefore, constrained distribution matching is sufficient to learn disentangled representations. Formally, the aggregate posterior $q(\hat \rvz)$ is a coordinate-wise reparameterization of the true distribution of the factors of variation (up to index permutations). In other words, there exists a one-to-one mapping between every entry of $\rvz$ and a unique matching entry of $\hat \rvz$, and thus a change in a single coordinate of $\rvz$ implies a change in a single matching coordinate of $\hat \rvz$~\cite{bengio2013representation}. Changing the observation model from single \iid observations to non-\iid pairs of observations generated according to the generative model \eqref{eq:generative-model-1}--\eqref{eq:generative-model-4} allows us to bypass the non-identifiability result of Chapter~\ref{cha:unsup_dis}. 
Our result requires strictly weaker assumptions than the result of~\citet{shu2020weakly} as we do not require group annotations, but only knowledge of $k$.  As we shall see in Section~\ref{sec:algorithms}, $k$ can be cheaply and reliably estimated from data at run-time. 
Although the weak assumptions of Theorem~\ref{thm:identifiability} may not be satisfied in practice, we will show that the proof can inform practical algorithm design.

\subsection{Practical Adaptive Algorithms} \label{sec:algorithms}
We conceive two $\beta$-VAE \cite{higgins2016beta} variants tailored to the weakly-supervised generative model \eqref{eq:generative-model-1}--\eqref{eq:generative-model-4} and a selection heuristic to deal with unknown and random $k$. We will see that these simple models can very reliably learn disentangled representations. 

\looseness=-1The key differences between theory and practice are that: (i) we use the ELBO and amortized variational inference for distribution matching (the true and learned distributions will not exactly match after training), (ii) we have access to a finite number of data only, and (iii) the theory assumes known, fixed $k$, but $k$ might be unknown and random. 

\looseness=-1\textbf{Enforcing the Structural Constraints }
Here we present a simple structure for the variational family that allows us to tractably perform approximate inference on the weakly-supervised generative model.
First note that the alignment constraints imposed by the generative model (see \eqref{eq:gconstraint1v2} and \eqref{eq:gconstraint2v2} evaluated for $g = g^\star$ in Section~\ref{sec:proof_ident}) imply for the true posterior
\begin{align}
p(z_i | \rvx_1) &= p(z_i | \rvx_2) \quad \forall i \in S, \label{eq:pconstraint1}\\
p(z_i | \rvx_1) &\neq p(z_i | \rvx_2) \quad \forall i \in \bar S, \label{eq:pconstraint2}
\end{align}
(with probability $1$)
and we want to enforce these constraints on the approximate posterior $q_\phi(\hat\rvz | \rvx)$ of our learned model. However, the set $S$ is unknown. To obtain an estimate $\hat S$ of $S$ we therefore choose for every pair $(\rvx_1, \rvx_2)$ the $d-k$ coordinates with the smallest $\KL(q_\phi(\hat z_i | \rvx_1)||q_\phi(\hat z_i | \rvx_2))$. To impose the constraint \eqref{eq:pconstraint1} we then replace each shared coordinate with some average $a$ of the two posteriors
\begin{align*}
    \tilde q_\phi(\hat z_i | \rvx_1) &= a(q_\phi(\hat z_i | \rvx_1), q_\phi(\hat z_i | \rvx_2)) \quad &\forall i \in \hat S, \\
    \tilde q_\phi(\hat z_i | \rvx_1) &= q_\phi(\hat z_i | \rvx_1) &\text{else,}
\end{align*}
and obtain $\tilde q_\phi(\rvz_i | \rvx_2)$ in analogous manner. As we later simply use the averaging strategies of the Group-VAE (GVAE)~\citep{gvae2019} and the Multi Level-VAE (ML-VAE)~\citep{bouchacourt2017multi}, we term variants of our approach which infers the groups and their properties adaptively \textit{Adaptive-Group-VAE} (Ada-GVAE) and \textit{Adaptive-ML-VAE} (Ada-ML-VAE), depending on the choice of the averaging function $a$. We then optimize the following variant of the $\beta$-VAE objective 
\begin{align}
    \max_{\phi, \theta} \mathbb{E}_{(\rvx_1, \rvx_2)}&\mathbb{E}_{\tilde q_\phi(\hat\rvz | \rvx_1)} \log(p_\theta(\rvx_1|\hat\rvz)) \nonumber \\
    &+ \mathbb{E}_{\tilde q_\phi(\hat\rvz | \rvx_2)} \log(p_\theta(\rvx_2|\hat\rvz)) \nonumber \\
    &- \beta D_{KL} \left(\tilde q_\phi(\hat\rvz || \rvx_1)|p(\hat\rvz)\right) \nonumber \\
    &- \beta D_{KL} \left(\tilde q_\phi(\hat\rvz || \rvx_2)|p(\hat\rvz)\right), \label{eq:mod-elbo}
\end{align}
where $\beta \geq 1$ \cite{higgins2016beta}. The advantage of this averaging-based implementation of \eqref{eq:pconstraint1}, over implementing it, for instance, via a $\KL$-term that encourages the distributions of the shared coordinates $\hat S$ to be similar, is that averaging imposes a hard constraint in the sense that $q_\phi(\hat\rvz|\rvx_1)$ and $q_\phi(\hat\rvz|\rvx_2)$ can jointly encode only one value per shared coordinate. This in turn implicitly enforces the constraint \eqref{eq:pconstraint2} as the non-shared dimensions need to be efficiently used to encode the non-shared factors of $\rvx_1$ and $\rvx_2$. 

\looseness=-1We emphasize that the objective \eqref{eq:mod-elbo} is a simple modification of the $\beta$-VAE objective and is very easy to implement. Finally, we remark that invoking Theorem~4 of~\cite{khemakhem2019variational}, we achieve consistency under maximum likelihood estimation up to the equivalence class in our Theorem~\ref{thm:identifiability}, for $\beta=1$ and in the limit of infinite data and capacity. 

\textbf{Inferring $k$} In the (practical) scenario where $k$ is unknown, we use the threshold
$$
    \tau = \frac{1}{2}(\max_i \delta_i + \min_i \delta_i),
$$
where $\delta_i = \KL(q_\phi(\hat z_i | \rvx_1)||q_\phi(\hat z_i | \rvx_2))$, and average the coordinates with $\delta_i < \tau$. This heuristic is inspired by the ``elbow method''~\cite{ketchen1996application} for model selection in $k$-means clustering and $k$-singular value decomposition and we found it to work surprisingly well in practice (see the experiments in Section~\ref{sec:results}). This estimate relies on the assumption that not all factors have changed. All our adaptive methods use this heuristic.
Although a formal recovery argument cannot be made for arbitrary data sets, inductive biases may limit the impact of an approximate $k$ in practice. We further remark that this heuristic always yields the correct $k$ if the encoder is disentangled.

\looseness=-1\textbf{Relation to Prior Work} Closely related to the proposed objective \eqref{eq:mod-elbo} the GVAE of~\citet{gvae2019} and the ML-VAE of~\citet{bouchacourt2017multi} assume $S$ is known and implement $a$ using different averaging choices. 
Both assume Gaussian approximate posteriors where $\mu_j, \Sigma_j$ are the mean and variance of $q(\hat \rvz_{ S}|\rvx_j)$ and $\mu, \Sigma$ are the mean and variance, of $\tilde q(\hat \rvz_{ S}|\rvx_j)$.
For the coordinates in $S$, the GVAE uses a simple arithmetic mean ($\mu = \frac12(\mu_1 + \mu_2)$ and $\Sigma = \frac12(\Sigma_1 + \Sigma_2)$) and the ML-VAE takes the product of the encoder distributions, with $\mu, \Sigma$ taking the form:
\begin{align*}
    \Sigma^{-1} = \Sigma^{-1}_1 + \Sigma^{-1}_2, \quad \mu^T = (\mu_1^T\Sigma_1^{-1} + \mu_2^T\Sigma_2^{-1})\Sigma.
\end{align*}
Our approach critically differs in the sense that $S$ is not known and needs to be estimated for every pair of images.

\looseness=-1Recent work combines non-linear ICA with disentanglement~\cite{khemakhem2019variational,sorrenson2020disentanglement}. 
Critically, these approaches are based on the setup of~\citet{hyvarinen2018nonlinear} which requires access to label information $\rvu$ such that $p(\rvz|\rvu)$ factorizes as $\prod_i p(z_i|\rvu)$. In contrast, we base our work on the setup of~\citet{gresele2019incomplete}, which only assumes access to two \textit{sufficiently distinct views} of the latent variable. \citet{shu2020weakly} train the same type of generative models over paired data but use a GAN objective where inference is not required. However, they require known and fixed $k$ as well as annotations of which factors change in each pair.

\section{Experimental Results}\label{sec:results}

\looseness=-1\textbf{Experimental Setup} We use the five data sets where the observations are generated as deterministic functions of the factors of variation: \textit{dSprites}~\cite{higgins2016beta}, \textit{Cars3D}~\cite{reed2015deep}, \textit{SmallNORB}~\cite{lecun2004learning}, \textit{Shapes3D}~\cite{kim2018disentangling}, and \textit{MPI3D}, the real-world robotics data set we introduced in~\cite{gondal2019transfer}.
\looseness=-1Our unsupervised baselines correspond to a cohort of \num{9000} unsupervised models ($\beta$-VAE~\citep{higgins2016beta}, AnnealedVAE~\citep{burgess2018understanding}, Factor-VAE~\citep{kim2018disentangling}, $\beta$-TCVAE~\citep{chen2018isolating}, DIP-VAE-I and II~\citep{kumar2017variational}), from Chapter~\ref{cha:unsup_dis}. To evaluate the representations, we consider the disentanglement metrics in Chapter~\ref{cha:eval_dis}.

\looseness=-1To create data sets with weak supervision from the existing disentanglement data sets, we first sample from the discrete $\rvz$ according to the ground-truth generative model~\eqref{eq:generative-model-1}--\eqref{eq:generative-model-4}. Then, we sample either one factor (corresponding to sparse changes) or $k$ factors of variation (to allow potentially denser changes) that may not be shared by the two images and re-sample those coordinates to obtain $\tilde \rvz$.  This ensures that each image pair differs in at most $k$ factors of variation (although changes are typically sparse and some pairs may be identical). For $k$ we consider the range from $1$ to $d-1$. This last setting corresponds to the case where all but one factor of variation are re-sampled. We study both the case where $k$ is constant across all pairs in the data set and where $k$ is sampled uniformly in the range $[d-1]$ for every training pair ($k=\texttt{Rnd}$ in the following).

For each data set, we train four weakly-supervised methods: Our adaptive and vanilla (group-supervision) variants of GVAE~\cite{gvae2019} and ML-VAE~\cite{bouchacourt2017multi}. For each approach we consider six values for the regularization strength 
and 10 random seeds, training a total of \num{6000} weakly-supervised models. We perform model selection using the weakly-supervised reconstruction loss (\ie, the sum of the first two terms in \eqref{eq:mod-elbo})\footnote{Training loss and the ELBO correlate similarly with disentanglement.}. We stress that we \textit{do not require labels for model selection}.


\subsection{Is Weak Supervision Enough for Disentanglement?}\label{sec:results:learning_weak_sup}
\looseness=-1In Figure~\ref{fig:selection}, we compare the performance of the weakly-supervised methods with $k=\texttt{Rnd}$ against the unsupervised methods. 
Unlike in unsupervised disentanglement with $\beta$-VAEs where $\beta\gg 1$ is common, we find  $\beta=1$ (the ELBO) performs best in most cases. We clearly observe that weakly-supervised models outperform the unsupervised ones.
The Ada-GVAE performs similarly to the Ada-ML-VAE. For this reason, we focus the following analysis on the Ada-GVAE, and refer to~\citep{locatello2020weakly} for the Ada-ML-VAE results.

\begin{figure*}[t]
    \centering
    \includegraphics[width=0.9\linewidth]{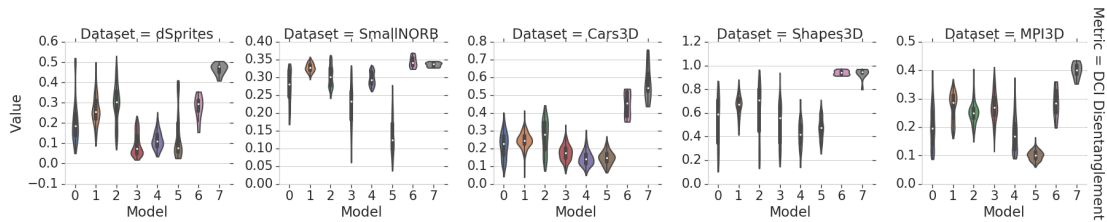}
    \vspace{-4mm}
    \caption{\small Our adaptive variants of the group-based disentanglement methods (models 6 and 7) considerably and consistently outperform unsupervised methods. In particular, the Ada-GVAE consistently yields the same or better performance than the Ada-ML-VAE. In this experiment, we consider the case where the number of shared factors of variation is randomly sampled for each pair ($k=\texttt{Rnd}$).
    Legend: 0=$\beta$-VAE, 1=FactorVAE, 2=$\beta$-TCVAE, 3=DIP-VAE-I, 4=DIP-VAE-II, 5=AnnealedVAE, 6=Ada-ML-VAE, 7=Ada-GVAE}
    \label{fig:selection}
    \vspace{-3mm}
\end{figure*}

\looseness=-1\textbf{Summary} With weak supervision, we reliably learn disentangled representations that outperform unsupervised ones. Our representations are competitive even if we perform fully supervised model selection on the unsupervised models.

\subsection{Are our methods adaptive to different values of $k$?}\label{sec:results_sup} 
In Figure~\ref{fig:cd_sweep_labels} (left), we report the performance of Ada-GVAE without model selection for different values of $k$ on MPI3D. We observe that Ada-GVAE is indeed adaptive to different values of $k$ and it achieves better performance when the change between the factors of variation is sparser. Note that our method is agnostic to the sharing pattern between the image pairs. In  applications where the number of shared factors is known to be constant, the performance may thus be further improved by injecting this knowledge into the inference procedure.

\looseness=-1\textbf{Summary} Our approach makes no assumption of which and how many factors are shared and successfully adapts to different values of $k$. The sparser the difference on the factors of variation, the more effective our method is in using weak supervision and learning disentangled representations.

\begin{figure}[t]
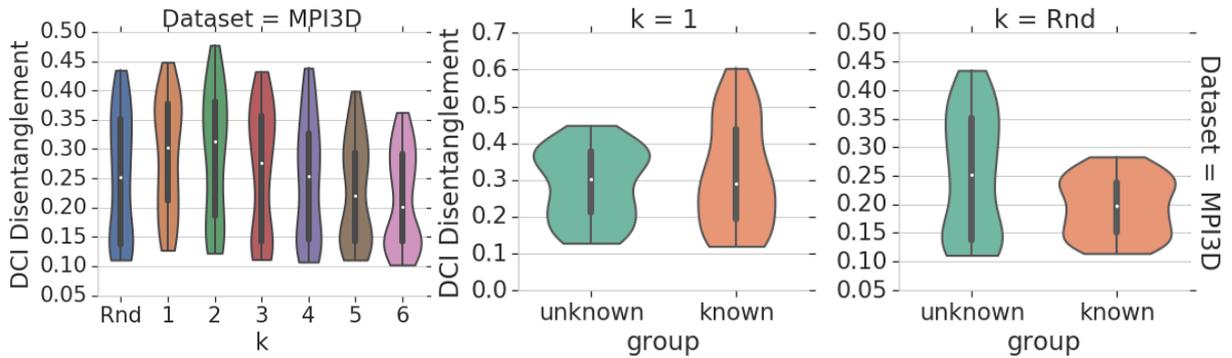

    \centering
    \includegraphics[width=0.35\linewidth]{sections/disent/weak/autofigures/cd_sweep_single.pdf}%
    \includegraphics[width=0.65\linewidth]{sections/disent/weak/autofigures/cd_labels_short.pdf}
    \vspace{-8mm}
    \caption{\small (\textbf{left}) Performance of the Ada-GVAE with different $k$ on MPI3D. The algorithm adapts well to the unknown $k$ and benefits from sparser changes. (\textbf{center} and \textbf{right}) Comparison of Ada-ML-VAE with the vanilla ML-VAE which assumes group knowledge. We note that group knowledge may improve performance (\textbf{center}) but can also hurt when it is incomplete (\textbf{right}).}
    \label{fig:cd_sweep_labels}
    \vspace{-4mm}
\end{figure}

\subsection{Supervision-performance trade-offs}\label{sec:results_group}
\looseness=-1The case $k=1$ where we actually know which factor of variation is not shared was previously considered in~\cite{bouchacourt2017multi,gvae2019,shu2020weakly}. Clearly, this additional knowledge should lead to improvements over our method. On the other hand, this information may be correct but incomplete in practice: For every pair of images, we know about one factor of variation that has changed but it may not be the only one. We therefore also consider the setup where $k=\texttt{Rnd}$ but the algorithm is only informed about one factor. Note that the original GVAE assumes group knowledge, so we directly compare its performance with our Ada-GVAE. We defer the comparison with ML-VAE~\cite{bouchacourt2017multi} and with the GAN-based approaches of~\cite{shu2020weakly} to the paper~\citep{locatello2020weakly}.

In Figure~\ref{fig:cd_sweep_labels} (center and right), we observe that when $k=1$, the knowledge of which factor was changed generally improves the performance of weakly-supervised methods on MPI3D. On the other hand, the GVAE is not robust to incomplete knowledge as its performance degrades when the factor that is labeled as non-shared is not the only one. This may not come as a surprise as group-based disentanglement methods all assume that the group knowledge is precise.

\looseness=-1\textbf{Summary} Whenever the groups are fully and precisely known, this information can be used to improve disentanglement. Even though our adaptive method does not use group annotations, its performance is often comparable to the methods of~\cite{bouchacourt2017multi,gvae2019,shu2020weakly}. On the other hand, in practical applications there may not be precise control of which factors have changed. In this scenario, relying on incomplete group knowledge considerably harms the performance of GVAE and ML-VAE as they assume exact group knowledge. A blend between our adaptive variant and the vanilla GVAE may further improve performance when only partial group knowledge is available.

\subsection{Are weakly-supervised representations useful?}\label{sec:results_downstream}
In this section, we investigate whether the representations learned by our Ada-GVAE are useful on a variety of tasks. 
We show that representations with small weakly-supervised reconstruction loss (the sum of the first two terms in \eqref{eq:mod-elbo}) achieve improved downstream performance as in Chapter~\ref{cha:unsup_dis}, improved downstream generalization~\cite{PetJanSch17} under covariate shifts~\cite{shimodaira2000improving,quionero2009dataset,ben2010impossibility}, fairer downstream predictions in the setting of Chapter~\ref{cha:fairness}, and improved sample complexity on an abstract reasoning task~\cite{van2019disentangled}.  
To the best of our knowledge, strong generalization under covariate shift has not been tested on disentangled representations before. 

\looseness=-1\textbf{Key insight} We remark that the usefulness insights of our previous work~\citet{locatello2019challenging,locatello2019disentangling,locatello2019fairness,van2019disentangled} (partially discussed in Chapters~\ref{cha:unsup_dis}--\ref{cha:fairness}) are based on the assumption that disentangled representations can be learned in the first place, ideally without observing the factors of variation. They consider models trained without supervision and argue that \textit{some} of the \textit{supervised disentanglement scores} (which require explicit labeling of the factors of variation) correlate well with desirable properties. \emph{In stark contrast, we here show that all these properties can be achieved simultaneously using only weakly-supervised data.}

\subsubsection{Downstream performance}
In this section, we consider the prediction task of Chapter~\ref{cha:unsup_dis}
\looseness=-1In Figure~\ref{fig:downstream} (left), we observe that the weakly-supervised reconstruction loss of Ada-GVAE is generally anti-correlated with downstream performance. The best weakly-supervised disentanglement methods thus learn representations that are useful for training accurate classifiers downstream.

\textbf{Summary} The weakly-supervised reconstruction loss of our Ada-GVAE is a useful proxy for downstream accuracy. 

\begin{figure*}[t]
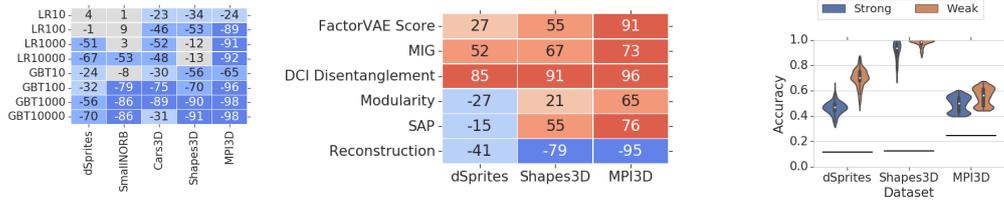

    \begin{center}
    \begin{minipage}{.25\textwidth}
    \centering\includegraphics[width=0.8\linewidth]{sections/disent/weak/autofigures/disentanglement_vs_downstream_groupVAE.pdf}
    \end{minipage}%
    \begin{minipage}{.4\textwidth}
    \includegraphics[width=0.8\linewidth]{sections/disent/weak/autofigures/disentanglement_vs_downstream_strong.pdf}
    \end{minipage}%
    \begin{minipage}{.3\textwidth}
    \includegraphics[width=0.65\linewidth]{sections/disent/weak/autofigures/generalization_gap.pdf}
    \end{minipage}
    \end{center}
    \vspace{-3mm}
    \caption{\small (\textbf{left}) Rank correlation between our weakly-supervised reconstruction loss and performance of downstream prediction tasks with logistic regression (LR) and gradient boosted decision-trees (GBT) at different sample sizes for Ada-GVAE. We observe a general negative correlation that indicates that models with a low weakly-supervised reconstruction loss may also be more accurate. (\textbf{center}) Rank correlation between the strong generalization accuracy under covariate shifts and disentanglement scores as well as weakly-supervised reconstruction loss, for Ada-GVAE. (\textbf{right}) Distribution of vanilla (weak) generalization and under covariate shifts (strong generalization) for Ada-GVAE. The horizontal line corresponds to the accuracy of a naive classifier based on the prior only.}
    \label{fig:downstream}
    \vspace{-3mm}
\end{figure*}

\subsubsection{Generalization Under Covariate Shift}
Assume we have access to a large pool of unlabeled paired data and our goal is to solve a prediction task for which we have a smaller labeled training set. Both the labeled training set and test set are biased,  but with different biases. For example, we want to predict object shape but our training set contains \textit{only} red objects, whereas the test set does \textit{not} contain \textit{any} red objects.
We create a biased training set by performing an intervention on a random factor of variation (other than the target variable), so that its value is constant in the whole training set. We perform another intervention on the test set, so that the same factor can take all other values. We train a GBT classifier on 10000 examples from the representations learned by Ada-GVAE. For each target factor of variation, we repeat the training of the classifier 10 times for different random interventions. For this experiment, we consider only dSprites, Shapes3D and MPI3D since Cars3D and SmallNORB are too small (after an intervention on their most fine grained factor of variation, they only contain 96 and 270 images respectively).

\looseness=-1In Figure~\ref{fig:downstream} (center) we plot the rank correlation between disentanglement scores and weakly-supervised reconstruction, and the results for generalization under covariate shifts for Ada-GVAE. We note that both the disentanglement scores and our weakly-supervised reconstruction loss are correlated with strong generalization. In Figure~\ref{fig:downstream} (right), we highlight the gap between the performance of a classifier trained on a normal train/test split (which we refer to as \textit{weak} generalization) as opposed to this covariate shift setting. We do not perform model selection, so we can show the performance of the whole range of representations.
We observe that there is a gap between weak and strong generalization but the distributions of accuracies overlap and are considerably better than a naive classifier based on the prior distribution of the classes.

\textbf{Summary} Our results provide compelling evidence that disentanglement is useful for strong generalization under covariate shifts. The best Ada-GVAE models in terms of weakly-supervised reconstruction loss seem to be useful for training classifiers that generalize under covariate shifts.

\begin{figure*}[t]
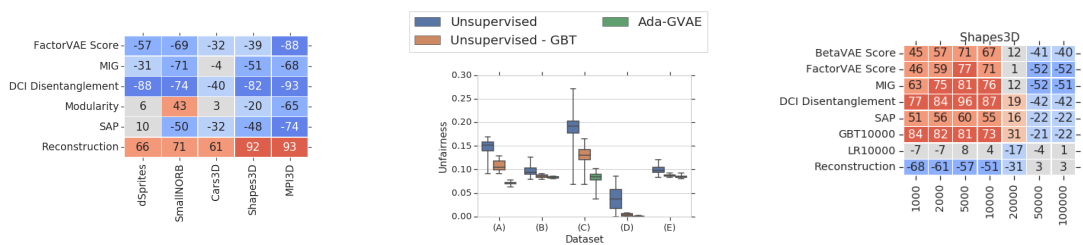

    \centering
    \begin{minipage}{.25\textwidth}
    \includegraphics[width=\linewidth]{sections/disent/weak/autofigures/unfairness_disentanglement_rank.pdf}
    \end{minipage}%
    \qquad\quad
    \begin{minipage}{.25\textwidth}
    \includegraphics[width=\linewidth]{sections/disent/weak/autofigures/Unfairness.pdf}
    \end{minipage}%
    \qquad
    \begin{minipage}{.25\textwidth}
    \includegraphics[width=\linewidth]{sections/disent/weak/autofigures/rank_scores_vs_steps_number.pdf}
    \end{minipage}
    \vspace{-0.2cm}
    \caption{\small \looseness=-1(\textbf{left}) Rank correlation between both disentanglement scores and our weakly-supervised reconstruction loss with the unfairness of GBT10000 on all the data sets for Ada-GVAE. (\textbf{center}) Unfairness of the unsupervised methods with the semi-supervised model selection heuristic of~\cite{locatello2019fairness} and our weakly-supervised Ada-GVAE with $k=1$.
    (\textbf{right}) Rank correlation with down-stream accuracy of the abstract visual
    reasoning models of~\cite{van2019disentangled} throughout training (\ie, for different sample sizes).
    }\vspace{-0.3cm}
    \label{fig:fairness_abstract}
\end{figure*}

\subsubsection{Fairness}
\looseness=-1We revisit our findings in Chpater~\ref{cha:fairness}, where we showed that disentangled representations may be useful to train robust classifiers that are fairer to unobserved sensitive variables independent of the target variable. We observed a strong correlation between demographic parity~\cite{calders2009building,zliobaite2015relation} and disentanglement, but the applicability of the presented approach was limited by the fact that disentangled representations are difficult to identify without access to explicit observations of the factors of variation (see Chapters~\ref{cha:unsup_dis} and~\ref{cha:eval_dis}).

In Figure~\ref{fig:fairness_abstract} (left), we show that the weakly-supervised reconstruction loss of our Ada-GVAE correlates with unfairness as strongly as the disentanglement scores, even though the former can be computed without observing the factors of variation. In particular, we can perform model selection without observing the sensitive variable. In Figure~\ref{fig:fairness_abstract} (center), we show that our Ada-GVAE with $k=1$ and model selection allows us to train and identify fairer models compared to the unsupervised models of Chapter~\ref{cha:fairness}. Furthermore, our previous model selection heuristic was based on downstream performance which requires knowledge of the target variable. From both plots we conclude that our weakly-supervised reconstruction loss is a good proxy for unfairness and allows us to train fairer classifiers in the setup of Chapter~\ref{cha:fairness} even if the sensitive variable is not observed.

\textbf{Summary} We showed that using weak supervision, we can train and identify fairer classifiers in the sense of demographic parity~\cite{calders2009building,zliobaite2015relation}. As opposed to Chapter~\ref{cha:fairness}, we do not need to observe the target variable and yet, our principled weakly-supervised approach outperforms their semi-supervised heuristic. 

\subsubsection{Abstract visual reasoning} 
\looseness=-1Finally, we consider the abstract visual reasoning task of~\citet{van2019disentangled}. This task is based on Raven's progressive matrices~\cite{raven1941standardization} and requires completing the bottom right missing panel of a sequence of context panels arranged in a $3\times 3$ grid as shown in Figure~\ref{fig:abstract_reasoning_app}. The algorithm is presented with six potential answers and needs to choose the correct one. To solve this task, the model has to infer the abstract relationships between the panels.
We replicate the experiment of~\citet{van2019disentangled} on Shapes3D under the same exact experimental conditions.

\begin{figure}[t]
    \centering
    \includegraphics[width=\linewidth]{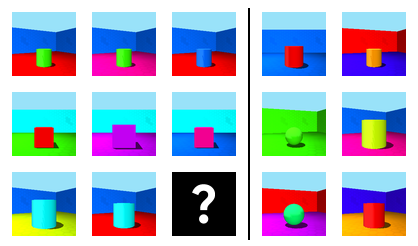}
    \caption{\small Example of the abstract visual reasoning task of~\cite{van2019disentangled}. The solution is  the panel in the central row on the right. }
    \label{fig:abstract_reasoning_app}
\end{figure}

\looseness=-1In Figure~\ref{fig:fairness_abstract} (right), one can see that at low sample sizes, the weakly-supervised reconstruction loss is strongly anti-correlated with performance on the abstract visual reasoning task. As previously observed by~\citet{van2019disentangled}, this benefit only occurs at low sample sizes.

\textbf{Summary} We demonstrated that training a relational network on the representations learned by our Ada-GVAE improves its sample efficiency. This result is in line with the findings of~\citet{van2019disentangled} where disentanglement was found to correlate positively with improved sample complexity.

\section{Proof of Theorem~\ref{thm:identifiability}}\label{sec:proof_ident}
Recall that the true marginal likelihoods $p(\rvx_1|\cdot) = p(\rvx_2|\cdot)$, are completely specified through the smooth, invertible function $g^\star$. The corresponding posteriors $p(\cdot | \rvx_1) = p(\cdot | \rvx_2)$ are completely determined by $g^{\star-1}$. The model family for candidate marginal likelihoods $q(\rvx_1| \cdot) = q(\rvx_2| \cdot)$ and corresponding posteriors $q(\cdot | \rvx_1) = q(\cdot | \rvx_2)$ are hence conditional distributions specified by the set of smooth invertible functions $g \colon \gZ \to \gX$ and their inverses $g^{-1}$, respectively.

In order to prove identifiability, we show that every candidate posterior distribution $q(\hat \rvz| \rvx_1)$ (more precisely, the corresponding $g$) on the generative model~\eqref{eq:generative-model-1}--\eqref{eq:generative-model-4} satisfying the assumptions stated in Theorem~\ref{thm:identifiability} inverts $g^\star$ in the sense that the aggregate posterior $q(\hat \rvz) = \int q(\hat \rvz|\rvx_1) p(\rvx_1) \dx \rvx_1$ is a coordinate-wise reparameterization of $p(\rvz)$ up to permutation of the indices.
Crucially, while neither the latent variables nor the shared indices are directly observed, observing pairs of images allows us to verify whether a candidate distribution has the right factorization \eqref{eq:full-joint} and sharing structure imposed by $S$ or not.

The proof is composed of the following steps:
\begin{enumerate}
    \item We characterize the constraints that need to hold for the posterior $q(\hat \rvz|\rvx_1)$ (the associated $g^{-1}$) inverting $g^\star$ for fixed~$S$.
    \item We parameterize all candidate posteriors $q(\hat \rvz|\rvx_1)$ (the associated $g^{-1}$) as a function $g^\star$ for a fixed $S$.
    \item We show that, for fixed $S$, $q(\hat \rvz|\rvx_1)$ (the associated $g^{-1}$) has two disentangled coordinate subspaces, one corresponding to $S$ and one corresponding to $\bar S$, in the sense that varying $\rvz_S$ and keeping $\rvz_{\bar S}$ fixed results in changes of the coordinate subspace of $\hat \rvz$ corresponding to $S$ only, and vice versa.
    \item We show that randomly sampling $S$ implies that every candidate posterior has an aggregated posterior which is a coordinate-wise reparameterization of the distribution of the true factors of variation.
\end{enumerate}

\paragraph{Step 1} We start by noting that since any continuous distribution can be obtained from the standard uniform distribution (via the inverse cumulative distribution function), it is sufficient to simply set $p(\hat \rvz)$ to the $d$-dimensional standard uniform distribution and try to recover an axis-aligned, smooth, invertible function $g \colon \gZ \to \gX$ (which completely characterizes $q(\rvx_1 | \hat \rvz)$ and $q(\hat \rvz|\rvx_1)$ via its inverse) as well as the distribution $p(S)$.

Next, assume that $S$ is fixed but unknown, \ie, the following reasoning is conditionally on $S$. By the generative process \eqref{eq:generative-model-1}--\eqref{eq:generative-model-4} we know that all smooth, invertible candidate functions $g$ need to obey with probability $1$ (and irrespective of whether $p(\hat \rvz)$ or $p(\rvz)$ is used)
\begin{align}
    g^{-1}_i(\rvx_1) &= g^{-1}_i(\rvx_2) \quad \forall i \in T, \label{eq:gconstraint1v2}\\
    g^{-1}_j(\rvx_1) &\neq g^{-1}_j(\rvx_2) \quad \forall i,j \in \bar T, \label{eq:gconstraint2v2}
\end{align}
for all $(\rvx_1,\rvx_2) \in \supp(p(\rvx_1,\rvx_2|S))$, where $T \in \gS$ is arbitrary but fixed. $T$ indexes the the coordinate subspace in the image of $g^{-1}$ corresponding to the unknown coordinate subspace $S$ of shared factors of $\rvz$. Note that choosing $T \in \gS$ requires knowledge of $k$ ($d$ can be inferred from $p(\rvx_1, \rvx_2)$). Also note that $g^\star$ satisfies \eqref{eq:gconstraint1v2}--\eqref{eq:gconstraint2v2} for $T=S$. 

\paragraph{Step 2} All smooth, invertible candidate functions can be written as $g = g^\star \circ h$, where $h\colon [0,1]^d \to \gZ$ is a smooth invertible function with smooth inverse (using that the composition of smooth invertible functions is smooth and invertible) that maps the $d$-dimensional uniform distribution to $p(\rvz)$. 

We have $g^{-1} = h^{-1} \circ g^{\star-1}$ i.e., $g^{-1}(\vx_1) = h^{-1}(g^{\star-1}(\vx_1)) = h^{-1}(\rvz)$ and similarly $g^{-1}(\vx_2) = h^{-1}(f(\rvz, \tilde \rvz, S))$. Expressing now \eqref{eq:gconstraint1v2}--\eqref{eq:gconstraint2v2} through $h$ we have with probability $1$
\begin{align}
    h^{-1}_i(\rvz) &= h^{-1}_i(f(\rvz, \tilde \rvz, S)) \quad \forall i \in T, \label{eq:rconstraint1v2}\\
    h^{-1}_j(\rvz) &\neq h^{-1}_j(f(\rvz, \tilde \rvz, S)) \quad \forall i,j \in \bar T. \label{eq:rconstraint2v2}
\end{align}
Thanks to invertibility and smoothness of $h$ we know that $h^{-1} $ maps the coordinate subspace $S$ of $\gZ$ to a $(d-k)$-dimensional submanifold $\gM_S$ of $[0,1]^d$ and the coordinate subspace $\bar S$ to a $k$-dimensional sub-manifold $\gM_{\bar S}$ of $[0,1]^d$ that is disjoint from $\gM_S$. 

\paragraph{Step 3}
Next, we shall see that for a fixed $S$ the only admissible functions $h \colon [0,1]^d \rightarrow\sZ^d$ are identifying two groups of factors (corresponding to two orthogonal coordinate subspaces): Those in $S$ and those in $\bar S$. 

To see this, we prove that $h$ can only satisfy \eqref{eq:rconstraint1v2}--\eqref{eq:rconstraint2v2} if it aligns the coordinate subspace $S$ of $\gZ$ with the coordinate subspace $T$ of $[0,1]^d$ and $\bar S$ with $\bar T$. In other words, $\gM_S$ and $\gM_{\bar S}$ lie in the coordinate subspaces $T$ and $\bar T$, respectively, and the Jacobian of $h^{-1}$ is block diagonal with blocks of coordinates indexed by $T$ and $\bar T$.

By contradiction, if $\gM_{\bar S}$ does not lie in the coordinate subspace $\bar T$ then \eqref{eq:rconstraint1v2} is violated as $h$ is smooth and invertible but its arguments obey $\rvz_i \neq f(\rvz, \tilde \rvz, S)_i=\tilde \rvz_i$ for every $i \in \bar S$ with probability $1$.

Likewise, if $\gM_S$ does not lie in the coordinate subspace $T$ then \eqref{eq:rconstraint2v2} is violated as $h$ is smooth and invertible but its arguments satisfy $\rvz_{S} = f(\rvz, \tilde \rvz, S)_{S}$ with probability $1$. 

As a result, \eqref{eq:rconstraint1v2} and \eqref{eq:rconstraint2v2} can only be satisfied if $h^{-1}$ maps each coordinate in $S$ to a unique matching coordinate in $T$. In other words there exists a permutation $\pi$ on $[d]$ such that $h^{-1}$ can be simplified as $h^{-1} = \tilde h$, where
\begin{align}
    h^{-1}_T(\rvz) &= \tilde{h}_T(\rvz_{\pi(S)}) \label{eq:reparam-fixed-s-1v2} \\
    h^{-1}_{\bar T}(\rvz) &= \tilde{h}_{\bar T}(\rvz_{\pi(\bar S)}). \label{eq:reparam-fixed-s-2v2}
\end{align}
Note that the permutation is required because the choice of $T$ is arbitrary. This implies that the Jacobian of $\tilde{h}$ is block diagonal with blocks corresponding to coordinates indexed by $T$ and $\bar T$ (or equivalently $S$ and $\bar S$).

For fixed $S$, i.e., considering $p(\rvx_1,\rvx_2|S)$, we can recover the groups of factors in $g^\star_S$ and $g^\star_{\bar S}$ up to permutation of the factor indices. Note that this does not yet imply that we can recover all axis-aligned $g$ as the factors in $g_T$ and $g_{\bar T}$ may still be entangled with each other, i.e., $\tilde{h}$ is not axis aligned within $T$ and $\bar T$.

\paragraph{Step 4}
If now $S$ is drawn at random, we observe a mixture of distributions $p(\rvx_1,\rvx_2| S)$ (but not $S$ itself) and $g$ needs to associate every $(\rvx_1,\rvx_2) \in \supp(p(\rvx_1,\rvx_2|S))$ with one and only one $T$ to satisfy \eqref{eq:gconstraint1v2}--\eqref{eq:gconstraint2v2}, for every $S \in \supp(p(S))$.

Indeed, suppose that $(\rvx_1, \rvx_2)$ are distributed according to a mixture of $p(\rvx_1,\rvx_2| S=S_1)$ and $p(\rvx_1,\rvx_2| S=S_2)$ with $S_1, S_2 \in \supp(p(S)), S_1 \neq S_2$. Then \eqref{eq:gconstraint1v2} can only be satisfied with probability $1$ for a subset of coordinates of size $|S_1 \cap S_2| < d-k$ due to invertibility and smoothness of $g$, but $|T| = d-k$. The same reasoning applies for mixtures of more than two subsets of $p(\rvx_1,\rvx_2| S)$. Therefore, \eqref{eq:gconstraint1v2} cannot be satisfied for $(\rvx_1, \rvx_2)$ drawn from a mixture of distribution $p(\rvx_1,\rvx_2| S)$ but associated with a single $T$.

Conversely, for a given $S$, all $(\rvx_1,\rvx_2) \in \supp(p(\rvx_1,\rvx_2|S))$ need to be associated with the same $T$ due to invertibility and smoothness of $g$. in more detail, all $(\rvx_1,\rvx_2) \in \supp(p(\rvx_1,\rvx_2|S))$ will share the same $d-k$-dimensional coordinate subspace due to \eqref{eq:rconstraint1v2}--\eqref{eq:rconstraint2v2} and therefore cannot be associated with two different $T$ as $|T|=d-k$.

Further, note that due to the smoothness and invertibility of $g$, for every pair of associated $S_1, T_1$ and $S_2, T_2$ we have $|S_1 \cap S_2| = |T_1 \cap T_2|$ and $|S_1 \cup S_2| = |T_1 \cup T_2|$. The assumption
\begin{equation} \label{eq:ps-conditionv2}
    P(S \cap S' = \{i\}) > 0 \quad \forall i \in [d] \quad \text{and} \quad S, S' \sim p(S)
\end{equation}
hence implies that we ``observe'' every factor through $(\rvx_1, \rvx_2) \sim p(\rvx_1, \rvx_2)$ as the intersection of two sets $S_1,S_2$, and this intersection will be reflected as the intersection of the corresponding two coordinate subspaces $T_1, T_2$. This, together with \eqref{eq:reparam-fixed-s-1v2}--\eqref{eq:reparam-fixed-s-2v2} finally implies
\begin{align}
    h^{-1}_i(\rvz) &= \tilde{h}_i(z_{\pi(i)}) \quad \forall i \in [d]
\end{align}
for some permutation $\pi$ on $[d]$. This in turns imply that the Jacobian of $\tilde{h}$ is diagonal. 

Therefore, by change of variables formula we have
\begin{equation} \label{eq:reparam}
q(\hat\rvz) = p(\tilde{h}( \rvz_{\pi([d])})) \left|\mathrm{det}\frac{\partial}{\partial\rvz_{\pi([d])}}\tilde{h}\right| = \prod_{i=1}^d p(\tilde{h}_i( z_{\pi(i)})) \left|\frac{\partial}{\partial z_{\pi(i)}}\tilde{h}_i\right|
\end{equation}
where the second equality is a consequence of the Jacobian being diagonal, and $|\partial \tilde{h}_i / \partial z_{\pi(i)}| \neq 0, \forall i,$ thanks to $\tilde{h}\colon \gZ \to [0,1]^d$ being invertible on $\gZ$. From \eqref{eq:reparam}, we can see that $q(\hat\rvz)$ is a coordinate-wise reparameterization of $p(\rvz)$ up to permutation of the indices. As a consequence, a change in a coordinate of $\rvz$ implies a change in the unique corresponding coordinate of $\hat\rvz$, so $q(\hat \rvz|\rvx_1)$ (or, equivalently, $g$) disentangles the factors of variation.

\paragraph{Final remarks} The considered generative model is identifiable up to coordinate-wise reparametrization of the factors. $p(S)$ can then be recovered $p(\rvx_1, \rvx_2)$ via $g$. Note that \eqref{eq:ps-conditionv2} effectively ensures that to a weak supervision signal is available for each factor of variation.

\chapter{Learning About Multiple Objects}\label{cha:slot_attn}
In this chapter, we discuss how to learn a modular representation of multiple entities. The presented work is based on \citep{locatello2020object} and was developed in collaboration with Dirk Weissenborn, Thomas Unterthiner, Aravindh Mahendran, Georg Heigold, Jakob Uszkoreit, Alexey Dosovitskiy, and Thomas Kipf. This work was done when Francesco Locatello was interning at Google Research, Brain Team in Amsterdam. Francesco Locatello and Thomas Kipf contributed equally to this publication. Francesco Locatello contributed to the development of the module (which started before he joined and was the result of a team effort) and its implementation. He was responsible for the analysis of its properties and the experiments on set prediction. Thomas Kipf was responsible for the experiments on Object Discovery which are not part of this dissertation. Figures~\ref{fig:slota} and~\ref{fig:slota_b} were done by Francesco Locatello and Thomas Kipf. Code available at \url{https://github.com/google-research/google-research/tree/master/slot_attention}.

\section{Motivation}
Perhaps, the most compelling arguments in favor of the usefulness of learning structured representations come from reinforcement learning and physical modeling, where several state-of-the-art approaches are trained from the internal representation of a simulator~\citep{battaglia2016interaction,sanchez2020learning} or of a game engine~\citep{berner2019dota,vinyals2019grandmaster}. As such, learning object-centric representations of complex scenes is a promising step towards enabling efficient abstract reasoning from low-level perceptual features. Yet, most deep learning approaches learn distributed representations that do not capture the compositional properties of natural scenes. Even the disentangled representations we presented in Chapters~\ref{cha:unsup_dis}--\ref{cha:weak} all learn vector representations that are the output of a CNN. Not only they introduce an unnatural ordering among the factors of variation, but also reduce the scene to a representation of its \textit{features} rather than the objects that compose it. Consider the example of a robotic arm manipulating an object and assume we have learned a disentangled representation for this physical system. If we added a second object to the arena, the representation would either have to ignore this new object or stop being disentangled.

\begin{wrapfigure}{l}{0.5\textwidth}
    \centering
        \includegraphics[width=0.5\textwidth]{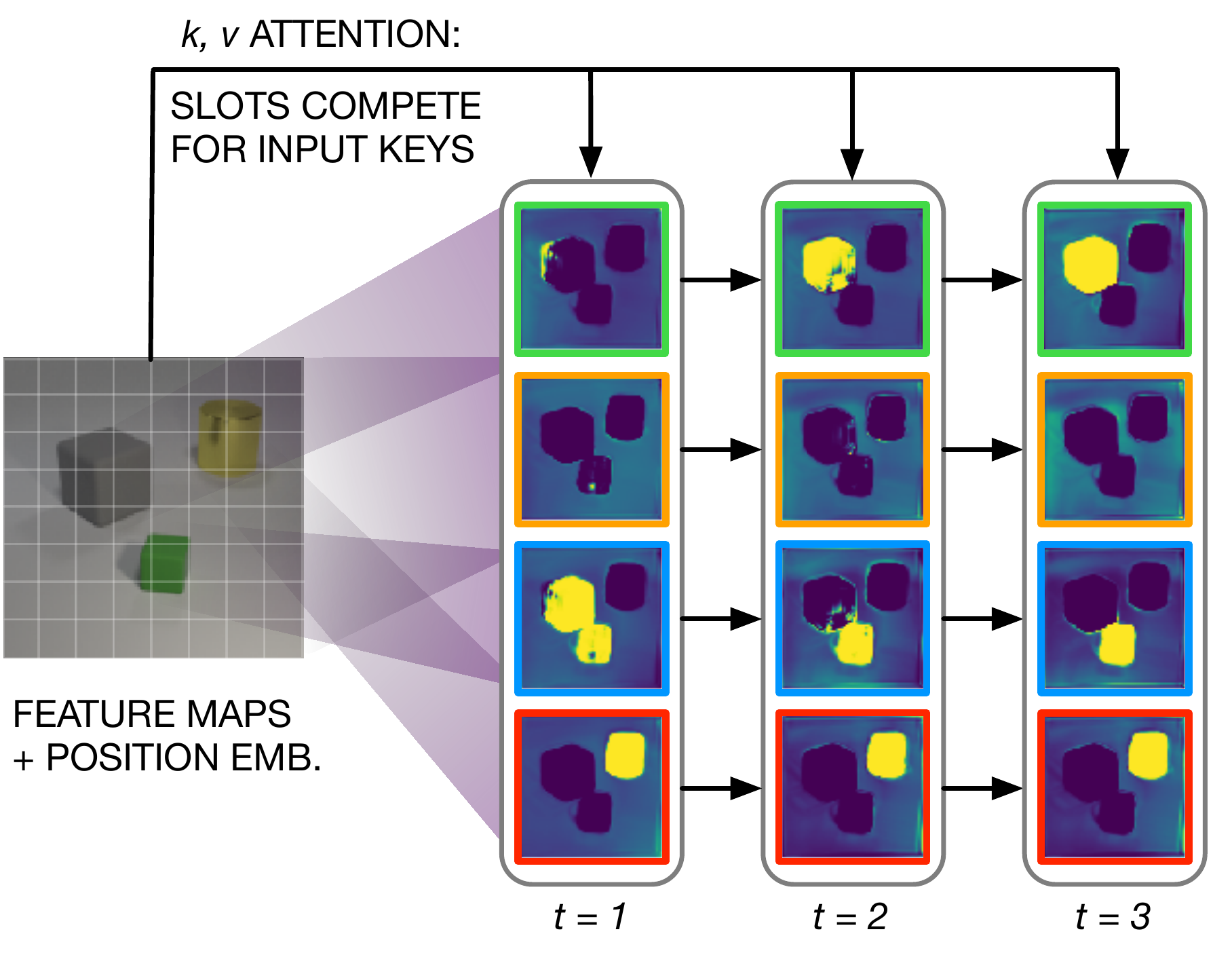}
    \caption{\small Slot Attention module. Using iterative attention, \sam learns a mapping from distributed representations to a set of slots. Slots compete to explain parts of the input.}\label{fig:slota}
\end{wrapfigure}
As a step in the direction of learning abstract variables from high-dimensional observations, we introduce the Slot Attention module, a differentiable \textit{interface} between perceptual representations (e.g., the output of a CNN) and a \textit{set} of variables called \textit{slots}. Using an iterative attention mechanism, \sam produces a set of output vectors with permutation symmetry. Unlike \textit{capsules} used in Capsule Networks~\citep{sabour2017dynamic,hinton2018matrix}, slots produced by \sam do not specialize to one particular type or class of object, which could harm generalization. Instead, they act akin to \textit{object files}~\citep{kahneman1992reviewing}, i.e., slots use a common representational format: each slot can store (and bind to) any object in the input. This allows \sam to generalize in a systematic way to unseen compositions, more objects, and more slots.

\section{Slot Attention}\label{sec:slot_attention_module}
\looseness=-1 The Slot Attention module (Figure~\ref{fig:slota}) maps from a set of $N$ input feature vectors to a set of $K$ output vectors that we refer to as \textit{slots}. Each vector in this output set can, for example, describe an object or an entity in the input. The overall module is described in Algorithm~\ref{algo:slot_attention} in pseudo-code.

Slot Attention uses an iterative attention mechanism to map from its inputs to the slots. Slots are initialized at random and thereafter refined at each iteration $t=1\ldots T$ to bind to a particular part (or grouping) of the input features. Randomly sampling initial slot representations from a common distribution allows Slot Attention to generalize to a different number of slots at test time.

At each iteration, slots \textit{compete} for explaining parts of the input via a softmax-based attention mechanism \citep{bahdanau2014neural,luong2015effective,vaswani2017attention} and update their representation using a recurrent update function. The final representation in each slot can be used in downstream tasks such as supervised set prediction (Figure~\ref{fig:slota_b}) or   unsupervised object discovery~\citep{locatello2020object}.

\begin{algorithm}[t!]
\caption{Slot Attention module. The input is a set of $N$ vectors of dimension $D_{\inp}$ which is mapped to a set of $K$ slots of dimension $D_{\slots}$. We initialize the slots by sampling their initial values as independent samples from a Gaussian distribution with shared, learnable parameters $\mu\in\mathbb{R}^{D_{\slots}}$ and $\sigma\in\mathbb{R}^{D_{\slots}}$. In our experiments we set the number of iterations to $T=3$.}\label{algo:slot_attention}
\begin{algorithmic}[1]
  \STATE \textbf{Input}: $\inp\in\mathbb{R}^{N \times D_{\inp}}$, $\slots \sim \mathcal{N}(\mu, \ \mathrm{diag(\sigma)})\in\mathbb{R}^{K \times D_{\slots}}$\\[0.2em]
  \STATE \textbf{Layer params}: $k, \ q, \ v$: linear projections for attention; $\texttt{GRU}$; $\texttt{MLP}$;  \layernorm \,(x3)\\[0.2em]
  \STATE \quad $\inp = \layernorm\,(\inp)$\label{algo_step:norm_input}\\[0.2em]
  \STATE \quad \textbf{for} $t = 0\ldots T$ \\[0.2em]
  \STATE \qquad $\slots\texttt{\_prev} = \slots$\\[0.2em]
  \STATE \qquad $\slots = \layernorm\,(\slots)$\\[0.2em]
  \STATE \qquad $\attn = \softmax{\frac{1}{\sqrt{D}} k(\inp) \cdot q(\slots)^T,\,\texttt{axis=`slots'}} $  \hfill\COMMENT{{\color{gray}\# norm.~over slots}}\label{algo_step:softmax}
  \STATE \qquad $\updates = \texttt{WeightedMean}\,(\texttt{weights=}\texttt{attn}+\epsilon,\,\texttt{values=}v(\inp))$  \hfill\COMMENT{{\color{gray} \# aggregate}}\label{algo_step:updates}\\[0.2em]
  \STATE \qquad $\slots = \gru{\texttt{state=}\slots\texttt{\_prev}}{\texttt{inputs=}\updates} $  \hfill\COMMENT{{\color{gray}\# GRU update (per slot)}}\\[0.2em]
  \STATE \qquad $\slots \mathrel{+}= \mlp\,(\layernorm\,(\slots))$  \hfill\COMMENT{{\color{gray}\# optional residual MLP (per slot)}}\\[0.2em]
  \STATE \quad \textbf{return} $\slots$\label{algo_step:return}\\[0.2em]
\end{algorithmic}
\end{algorithm}

We now describe a single iteration of Slot Attention on a set of input features, $\inp\in\mathbb{R}^{N\times D_{\inp}}$, with $K$ output slots of dimension $D_{\slots}$ (we omit the batch dimension for clarity). We use learnable linear transformations $k$, $q$, and $v$ to map inputs and slots to a common dimension $D$.

Slot Attention uses dot-product attention~\citep{luong2015effective} with attention coefficients that are normalized over the slots, i.e., the queries of the attention mechanism. This choice of normalization introduces competition between the slots for explaining parts of the input. We further follow the common practice of setting the softmax temperature to a fixed value of $\sqrt{D}$~\citep{vaswani2017attention}:
\begin{equation}\label{eqn:slot_attention_key_query}
      \attn_{i, j} \coloneqq \frac{e^{M_{i, j}}}{\sum_l e^{M_{i, l}}} \qquad \text{where} \qquad M \coloneqq \frac{1}{\sqrt{D}}k(\inp) \cdot q(\slots)^T \in\mathbb{R}^{N\times K}.
\end{equation}
\looseness=-1 In other words, the normalization ensures that attention coefficients sum to one for each individual input feature vector, which prevents the attention mechanism from ignoring parts of the input. To aggregate the input values to their assigned slots, we use a weighted mean as follows:
\begin{equation}\label{eqn:slot_attention_value}
    \updates \coloneqq W^T \cdot v(\inp) \in \mathbb{R}^{K \times D}  \qquad \text{where} \qquad W_{i, j} \coloneqq \frac{\attn_{i, j}}{\sum_{l=1}^N\attn_{l, j}} \ .
\end{equation}
The weighted mean helps improve stability of the attention mechanism (compared to using a weighted sum) as in our case the attention coefficients are normalized over the slots. In practice we further add a small offset $\epsilon$ to the attention coefficients to avoid numerical instability.

The aggregated \texttt{updates} are finally used to update the slots via a learned recurrent function, for which we use a Gated Recurrent Unit (GRU)~\citep{cho2014learning} with $D_\slots$ hidden units. We found that transforming the GRU output with an (optional) multi-layer perceptron (MLP) with ReLU activation and a residual connection~\citep{he2016deep} can help improve performance. Both the GRU and the residual MLP are applied independently on each slot with shared parameters.
We apply layer normalization (LayerNorm)~\citep{ba2016layer} both to the inputs of the module and to the slot features at the beginning of each iteration and before applying the residual MLP. While this is not strictly necessary, we found that it helps speed up training convergence. We refer to the paper~\citep{locatello2020object} for a detailed ablation study. The overall time-complexity of the module is $\mathcal{O}\left(T\cdot D \cdot N\cdot K \right)$.

We identify two key properties of \sam: (1) permutation invariance with respect to the input (i.e., the output is independent of permutations applied to the input and hence suitable for sets) and (2) permutation equivariance with respect to the order of the slots (i.e., permuting the order of the slots after their initialization is equivalent to permuting the output of the module). More formally:
\begin{proposition}\label{thm:layer_invarianvce_equivariance}
\looseness=-1Let $\textnormal{SlotAttention}(\textnormal{\inp}, \textnormal{\slots})\in\mathbb{R}^{K\times D_{\textnormal{\slots}}}$ be the output of the Slot Attention module (Algorithm~\ref{algo:slot_attention}), where $\textnormal{\inp}\in\mathbb{R}^{N\times D_{\textnormal{\inp}}}$ and $\textnormal{\slots}\in\mathbb{R}^{K\times D_{\textnormal{\slots}}}$\,. Let $\pi_{i}\in\mathbb{R}^{N\times N}$ and $\pi_{s}\in\mathbb{R}^{K\times K}$ be arbitrary permutation matrices. Then, the following holds: 
\begin{align*}
    \textnormal{SlotAttention}(\pi_i \cdot \textnormal{\inp}, \pi_s\cdot \textnormal{\slots}) = \pi_s \cdot \textnormal{SlotAttention}(\textnormal{\inp}, \textnormal{\slots})\,.
\end{align*}
\end{proposition}

\textbf{Relation with Capsules} \
The capsules used in Capsules Networks~\citep{sabour2017dynamic, hinton2018matrix,tsai2020capsules} communicate with an iterative \textit{routing} mechanism similar to ours. The closest is the inverted dot-product attention routing~\citep{tsai2020capsules}, which does not respect permutation symmetry as input-output pairs do not share parameters.

\textbf{Relation with interacting memory models} \
Interactive memory models~\citep{watters2017visual,van2018relational,kipf2020contrastive,santoro2018relational,zambaldi2018relational,watters2019cobra,stanic2019r,goyal2019recurrent,veerapaneni2020entity} are structured architectures for reasoning. They typically consist in a set of slots that are updated with a recurrent function and periodically communicate. Slots are symmetric (with the exception of~\citep{goyal2019recurrent}). Both~\citep{santoro2018relational} and \citep{goyal2019recurrent} use vanilla attention to map from the input to the slots without the competition driven by the normalization wrt the slots. Further, they consider temporal data with the recurrent function updating the state of the slots across time-steps rather than refining the abstraction for a single input.

\textbf{Relation with mixtures of experts} \ Expert models~\citep{jacobs1991adaptive,parascandolo2018learning,locatello2018clustering} are opposite to our approach as they do not share parameters between individual experts. Their goal is the specialization of the experts to different tasks or examples    .

\textbf{Relation with soft clustering} \
\looseness=-1Our routing procedure is related to soft k-means clustering~\cite{bauckhage2015lecture} (where slots corresponds to centroids), except we use a dot product similarity with learned linear projections and a parameterized, learnable update function.

\textbf{Relation with recurrent attention} \
Recurrent attention models~\citep{mnih2014recurrent,gregor2015draw,eslami2016attend,ren2017end,kosiorek2018sequential} and recurrent models for set predictions~\citep{stewart2016end,romera2016recurrent} have been applied to image modeling and scene decomposition but they infer the slots one at the time in an auto-regressive manner.



\section{Example Application: Set Prediction}
\label{sec:set_prediction}
Set representations are already used in tasks across many data modalities ranging from point cloud prediction~\citep{achlioptas2018learning,fan2017point}, classifying multiple objects in an image~\citep{zhang2019deep}, or generation of molecules with desired properties~\citep{de2018molgan,simonovsky2018graphvae}.
In the example considered in this section, we are given an input image and a set of prediction targets, each describing an object in the scene. The key challenge in predicting sets is that there are $K!$ possible equivalent representations for a set of $K$ elements, as the order of the targets is arbitrary. This inductive bias needs to be explicitly modeled in the architecture to avoid discontinuities in the learning process, e.g.~when two semantically specialized slots swap their content throughout training~\citep{zhang2019deep,zhang2019fspool}. \sam directly maps from set to set (of different cardinalities) using only a few attention iterations and a single task-specific loss function.
The output order of \sam is random and independent of the input order. Therefore, \sam can be used to turn a distributed representation of an input scene into a set representation where each object can be separately classified with a standard classifier as shown in Figure~\ref{fig:slota_b}.

\begin{figure}[t]
        {
            \includegraphics[width=\textwidth]{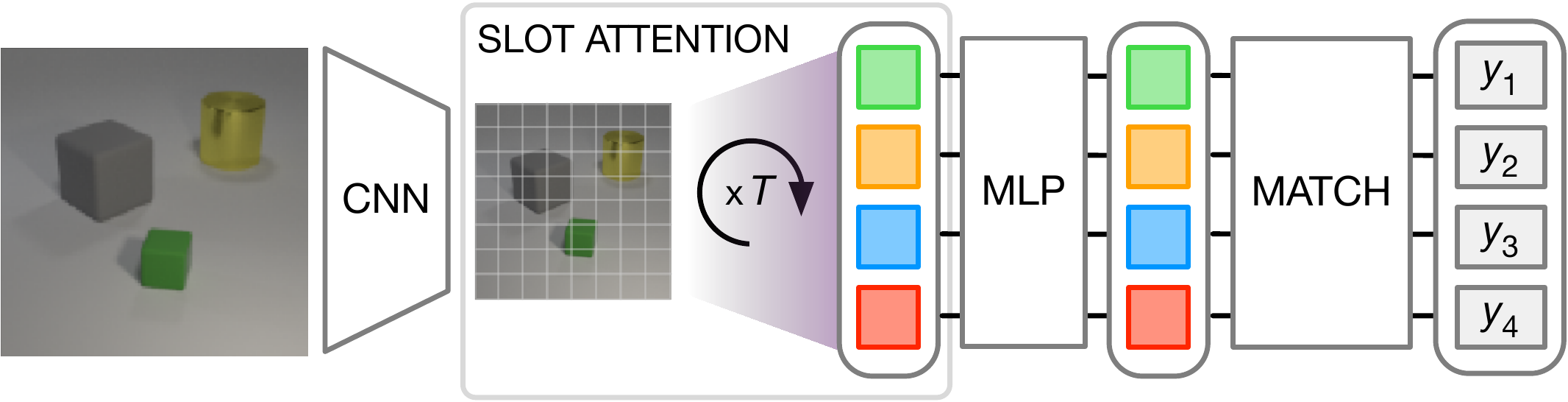}
            \caption{\small Set prediction architecture. A CNN backbone learns a distributed representation of each input image. \sam extract a representation of each object from which we classifiy its properties with a MLP with shared parameters.}
            \label{fig:slota_b}
        }
    \end{figure}

\textbf{Other neural networks for sets} \
\looseness=-1A range of recent methods explore learning representations for set structured data~\citep{lin2017structured,zaheer2017deep,zhang2019fspool} and set generation~\citep{zhang2019deep,rezatofighi2020learn,kosiorek2020conditional}. For set-to-set mappings~\citep{vaswani2017attention,lee2018set} that preserve the cardinality of the input (\ie, the output set has the same number of elements as the input) the standard are graph neural networks~\citep{scarselli2008graph,li2015gated,kipf2016semi,battaglia2018relational} and the self-attention mechanism~\citep{vaswani2017attention}. \citep{ying2018hierarchical,lee2018set,carion2020detr} learn an \textit{ordered} representation of the output set with learned per-element initialization, which prevents these approaches from generalizing to a different set cardinality at test time. The Deep Set Prediction Network (DSPN)~\citep{zhang2019deep,huang2020set} is the only approach that respects permutation symmetry by running an inner gradient descent loop for each example, which requires many steps for convergence and careful tuning of several loss hyperparmeters. 
Most related approaches, including DiffPool~\citep{ying2018hierarchical}, Set Transformer~\citep{lee2018set}, DSPN~\citep{zhang2019deep}, and DETR~\citep{carion2020detr} use a learned per-element initialization (\ie, separate parameters for each set element), which prevents these approaches from generalizing to more set elements at test time. 

\subsection{Experiments}

\textbf{Architecture} \ Our encoder consists of two components: (i) a CNN backbone augmented with positional embeddings, followed by (ii) a Slot Attention module. The output of \sam is a set of slots, that represent a grouping of the scene (e.g.~in terms of objects).
For each slot, we apply a MLP with parameters shared between slots. As the order of both predictions and labels is arbitrary, we match them using the Hungarian algorithm~\citep{kuhn1955hungarian}. For details on the implementation we refer to original paper~\citep{locatello2020object} and the code\footnote{\url{https://github.com/google-research/google-research/tree/master/slot_attention}}.

\looseness=-1\textbf{Metrics} \
Following~\citet{zhang2019deep}, we use $K=10$ object slots and compute the Average Precision (AP) as commonly used in object detection~\citep{everingham2015pascal}. A prediction (object properties and position) is considered correct if there is a matching object with exactly the same properties (shape, material, color, and size) within a certain distance threshold ($\infty$ means we do not enforce any threshold).
The predicted position coordinates are scaled to $[-3, 3]$. We zero-pad the targets and predict an additional indicator score in $[0, 1]$ corresponding to the presence probability of an object (1 means there is an object) which we then use as prediction confidence to compute the AP.

\textbf{Results} \
\looseness=-1In Figure~\ref{fig:prop_pred} (left) we report results in terms of Average Precision for supervised object property prediction on CLEVR10 (using $T=3$ for Slot Attention at both train and test time). We compare to both the DSPN results of~\citep{zhang2019deep} and the Slot MLP baseline. Overall, we observe that our approach matches or outperforms the DSPN baseline. The performance of our method degrades gracefully at more challenging distance thresholds (for the object position feature) maintaining a reasonably small variance. Note that the DSPN baseline~\citep{zhang2019deep} uses a significantly deeper ResNet 34~\citep{he2016deep} image encoder and use a different weight for the distance penalty in the loss. In Figure~\ref{fig:prop_pred} (center) we observe that increasing the number of attention iterations at test time generally improves performance. Slot Attention  can naturally handle more objects at test time by changing the number of slots. In Figure~\ref{fig:prop_pred} (right) we observe that the AP degrades gracefully if we train a model on CLEVR6 (with $K=6$ slots) and test it with more objects.

\looseness=-1Intuitively, to solve this set prediction task each slot should attend to a different object. In Figure~\ref{fig:masks_prop_pred}, we visualize the attention maps of each slot for two CLEVR images. In general, we observe that the attention maps naturally segment the objects. We remark that the method is only trained to predict the property of the objects, without any segmentation mask. Quantitatively, we can evaluate the Adjusted Rand Index (ARI) scores of the attention masks. On CLEVR10 (with masks), the attention masks produced by Slot Attention achieve an ARI~\citep{rand1971objective,hubert1985comparing} of $78.0\% \pm 2.9$ (to compute the ARI we downscale the input image to $32 \times 32$). 

Finally, we investigate a variant of our model where we make a prediction and compute the loss at every iteration. A similar experiment was reported in~\citep{zhang2019deep} for the DSPN model. Also, we remark that DSPN uses a different scale for the position coordinates of objects by default. Therefore, we further compare against a version of our model where we similarly tune the importance of position in the loss.
A scale of 1 corresponds to our default coordinate normalization of $[0, 1]$, whereas larger scales correspond to a $[0, \texttt{scale}]$ normalization of the coordinates (or shifted by an arbitrary constant). In Figure~\ref{fig:prop_pred_hist_loss}, we observe that computing the loss at each step in Slot Attention improves the AP score at all distance thresholds as opposed to DSPN, where it is only beneficial at small distance thresholds. We conjecture that this is an optimization issue in DSPN. As expected, increasing the importance of accurately modeling position in the loss impacts the AP positively at smaller distance thresholds, but can otherwise have a negative effect on predicting other object attributes correctly.

\textbf{Summary} \ Slot Attention learns a representation of objects for set-structured property prediction tasks and achieves results competitive with a prior state-of-the-art approach while being significantly easier to implement and tune. Further, the attention masks naturally segment the scene, which can be valuable for debugging and interpreting the predictions of the model.

\begin{figure}[t]
    \centering
    \begin{subfigure}[b]{0.29\textwidth}
    \includegraphics[width=\textwidth]{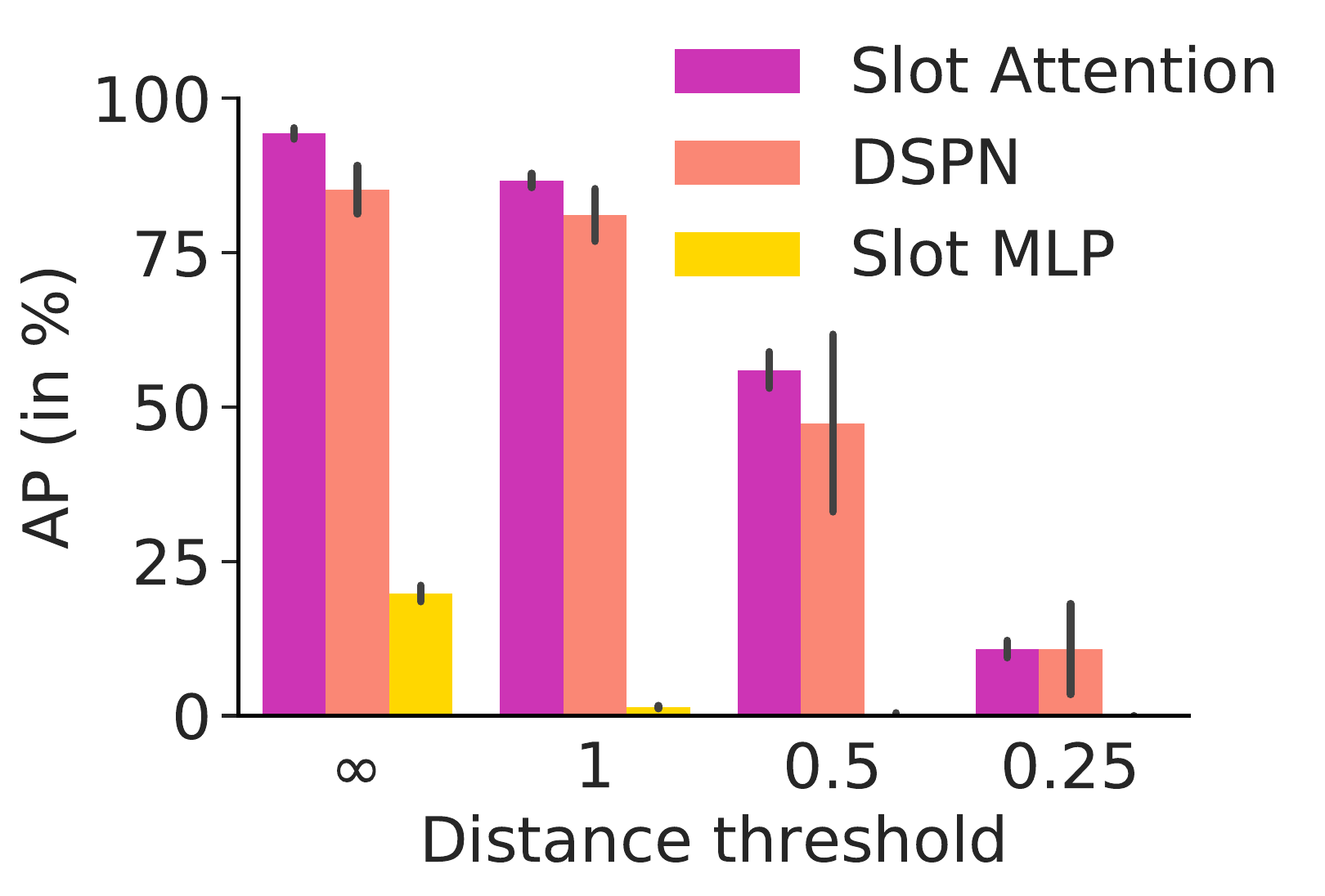}
    \end{subfigure}
    ~\quad
    \begin{subfigure}[b]{0.31\textwidth}
    \includegraphics[width=\textwidth]{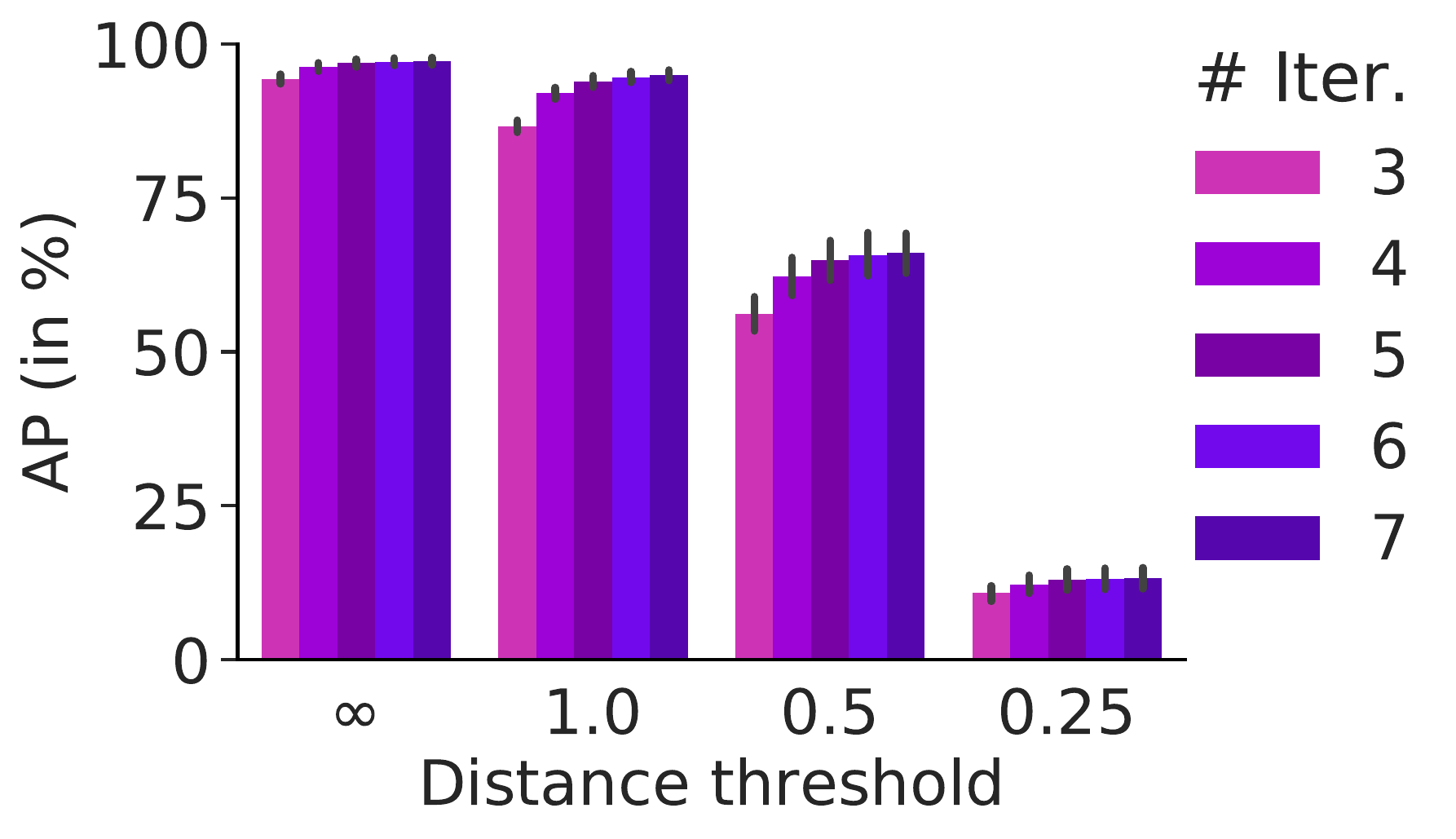}
    \end{subfigure}
    ~\quad
    \begin{subfigure}[b]{0.31\textwidth}
    \includegraphics[width=\textwidth]{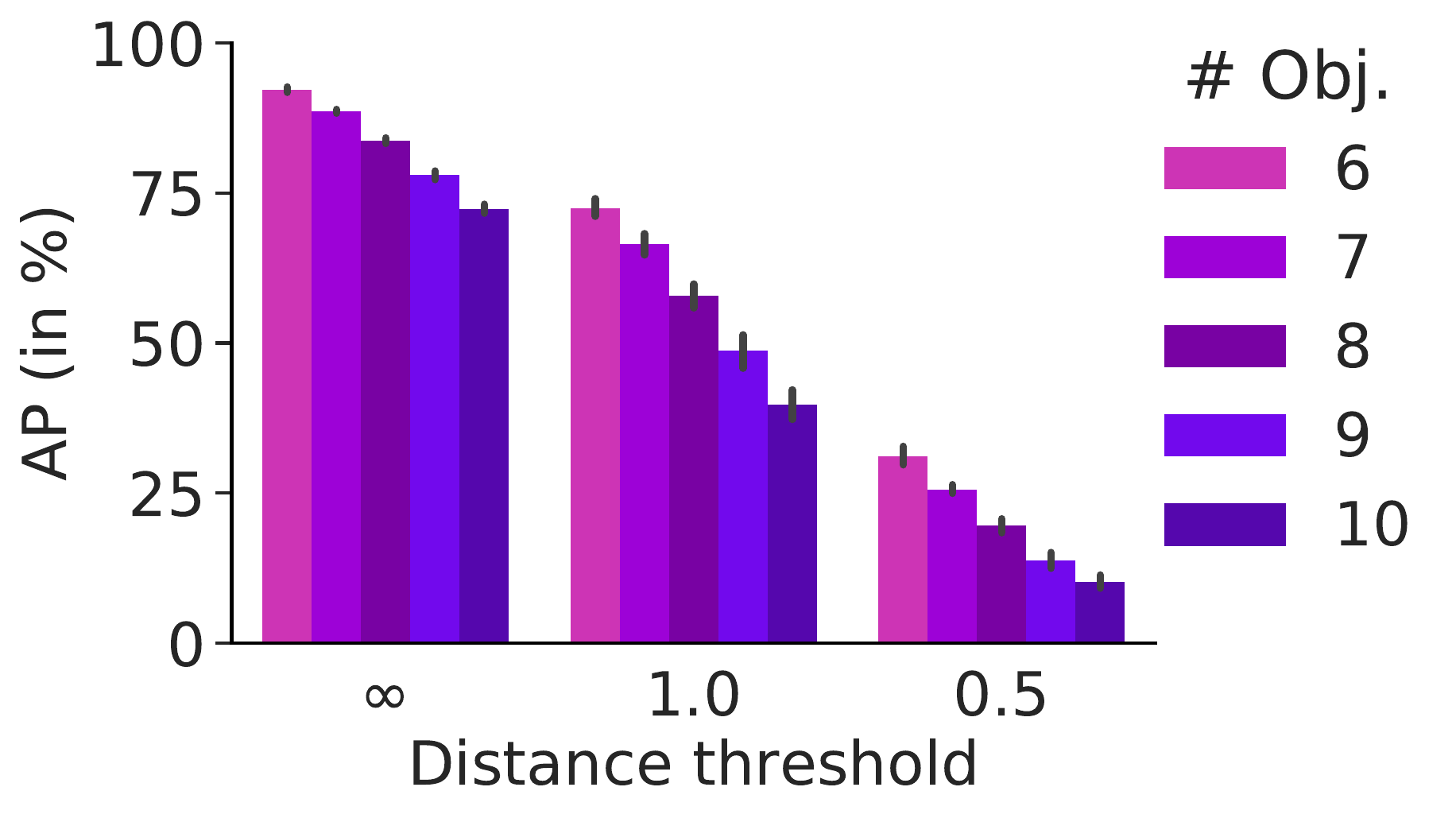}
    \end{subfigure}
    \vspace{-1mm}
    \caption{\small (\textbf{Left}) AP at different distance thresholds on CLEVR10 (with $K=10$).
    (\textbf{Center}) AP with different number of iterations. The models are trained with 3 iterations and tested with iterations ranging from 3 to 7. (\textbf{Right}) AP for Slot Attention trained on CLEVR6 ($K=6$) and tested on scenes containing exactly $N$ objects (with $N=K$ from $6$ to $10$).}
    \label{fig:prop_pred}
    \vspace{-1mm}
\end{figure}

\begin{figure}[t]
    \centering
    \includegraphics[width=0.75\textwidth,trim={0 0.3cm 0 0.25cm},clip]{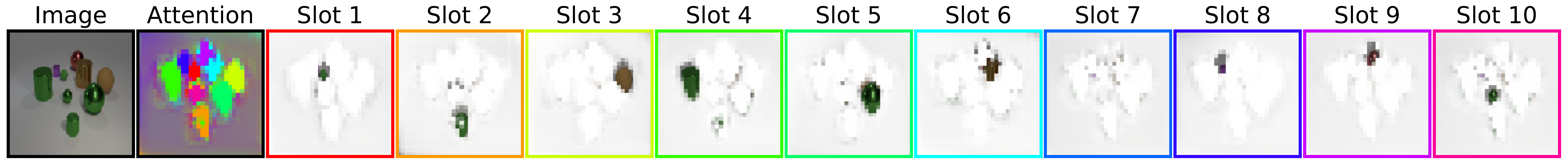}
    \includegraphics[width=0.75\textwidth,trim={0 0.3cm 0 1.3cm},clip]{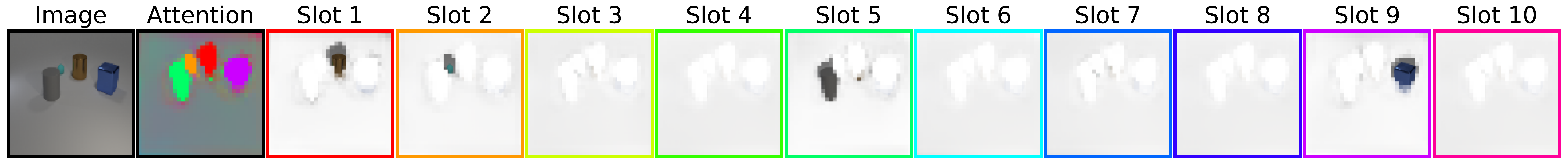}
    \caption{\small Visualization of the attention masks on CLEVR10 for two examples with 9 and 4 objects, respectively, for a model trained on the property prediction task. The masks are upsampled to $128\times 128$ for this visualization to match the resolution of input image.}
    \label{fig:masks_prop_pred}
    \vspace{-2.5mm}
\end{figure}

\begin{figure}[ht]
    \centering
    \includegraphics[width=0.85\textwidth,trim={0 0.3cm 0 0.25cm},clip]{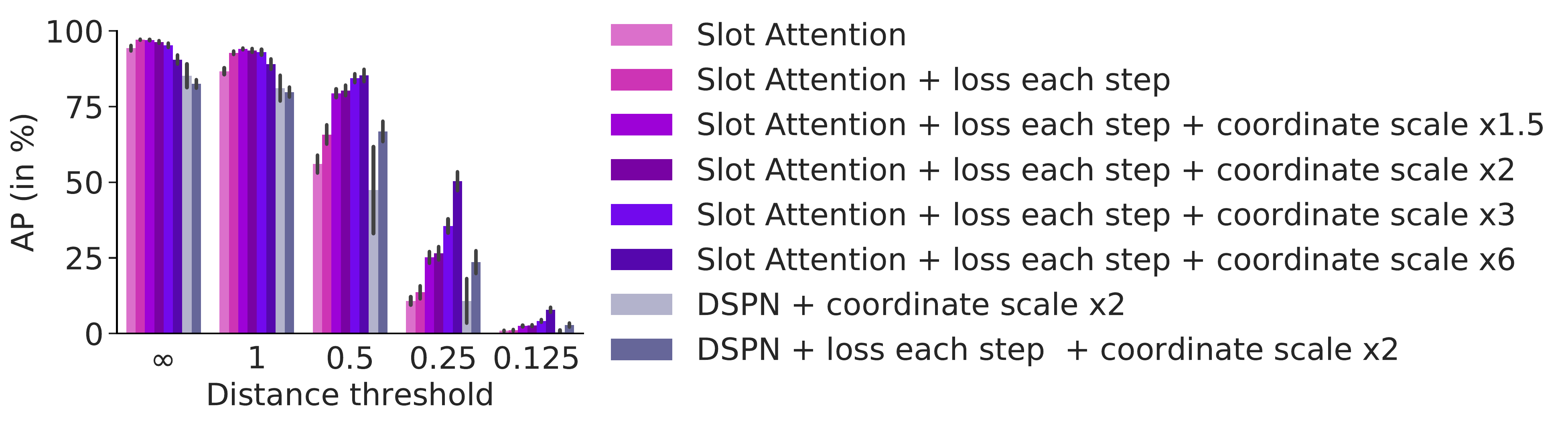}
    \caption{\small Computing the loss at each iteration generally improves results for both Slot Attention and the DSPN (while increasing the computational cost as well). As expected, tuning the importance of modeling the position of objects in the loss positively impacts the AP at small distance thresholds where the position of objects needs to be predicted more accurately.}
    \label{fig:prop_pred_hist_loss}
\end{figure}

\section{Proof of Proposition~\ref{thm:layer_invarianvce_equivariance}}\label{app:proof}
Let us first formally define permutation invariance and equivariance.
\begin{definition}[Permutation Invariance]
A function $f: \mathbb{R}^{M\times D_1} \rightarrow \mathbb{R}^{M\times D_2}$ is \emph{permutation invariant} if for any arbitrary permutation matrix $\pi\in\mathbb{R}^{M\times M}$ it holds that:
\begin{equation*}
    f(\pi x) = f(x) \, .
\end{equation*}
\end{definition}
\begin{definition}[Permutation Equivariance]
A function $f: \mathbb{R}^{M\times D_1} \rightarrow \mathbb{R}^{M\times D_2}$ is \emph{permutation equivariant} if for any arbitrary permutation matrix $\pi\in\mathbb{R}^{M\times M}$ it holds that:
\begin{equation*}
    f(\pi x) = \pi f(x) \, .
\end{equation*}
\end{definition}

The proof is straightforward and is reported for completeness. We rely on the fact that the sum operation is permutation invariant.

\paragraph{Linear projections} As the linear projections are applied independently per slot/input element with shared parameters, they are permutation equivariant.

\paragraph{Equation~\ref{eqn:slot_attention_key_query}}
The dot product of the attention mechanism (i.e. computing the matrix $M\in\mathbb{R}^{N\times K}$) involves a sum over the feature axis (of dimension $D$) and is therefore permutation equivariant w.r.t.~both input and slots. The output of the softmax is also equivariant, as:
\begin{align*}
    \softmax{\pi_s\cdot \pi_i \cdot M}_{(k,l)} = \frac{e^{(\pi_s\cdot \pi_i \cdot M)_{k,l}}}{\sum_s e^{(\pi_s\cdot \pi_i \cdot M)_{k,s}}} =  \frac{e^{M_{\pi_i(k),\pi_s(l)}}}{\sum_{\pi_s(l)} e^{M_{\pi_i(k),\pi_s(l)}}} = \softmax{M}_{(\pi_i(k),\pi_s(l))}\,,
\end{align*}
where we indicate with e.g. $\pi_i(k)$ the transformation of the coordinate $k$ with the permutation matrix $\pi_i$. The second equality follows from the fact that the sum is permutation invariant.

\paragraph{Equation~\ref{eqn:slot_attention_value}}
The matrix product in the computation of the $\updates$ involves a sum over the input elements which makes the operation invariant w.r.t.~permutations of the input order (i.e. $\pi_i$) and equivariant w.r.t.~the slot order (i.e. $\pi_s$).

\paragraph{Slot update:} The slot update applies the same network to each slot with shared parameters. Therefore, it is a permutation equivariant operation w.r.t.~the slot order.
\paragraph{Combining all steps:}
As all steps in the algorithms are permutation equivariant wrt $\pi_s$, the overall module is permutation equivariant. On the other hand, Equation~2 is permutation invariant w.r.t.~to $\pi_i$. Therefore, after the first iteration the algorithm becomes permutation invariant w.r.t.~the input order.

\part{Concluding Remarks}


\chapter{Discussion}
\label{cha:conclusion}
In this chapter, we summarize our results and discuss their implications. This chapter is based on~\citet{locatello2020stochastic}, Locatello* and Raj* et al.~\citep{locatello2018matching},~\citet{locatello2017greedy},~\citet{LocKhaGhoRat18}, Locatello* and Dresdner* et al.~\citep{locatello2018boosting},~\citet{locatello2019challenging},~\citet{locatello2020sober},~\citet{locatello2020commentary},~\citet{locatello2019disentangling},~\citet{locatello2019fairness},~\citet{locatello2020weakly}, Locatello* and Kipf* et al.~\citep{locatello2020object}, Sch\"olkopf* and Locatello* et al.~\citep{scholkopf2020towards} and was developed with the authors thereof.

\looseness=-1In this dissertation, we investigated how to incorporate specific structural constraints into the solution of a learning algorithm. From the optimization perspective in Part~\ref{part:opt}, we focused on constraints that can be written as combinations of atoms. Although restrictive, the analysis of this setting resulted in several algorithmic and theoretical contributions.
We unified the convergence analysis of several greedy optimization methods, developing new solvers for SDPs and Boosting Variational Inference along the way. 
In Chapter~\ref{ch:SFW}, we introduced a scalable stochastic FW-type method for solving convex optimization problems with affine constraints and demonstrated empirical superiority of our approach in various numerical experiments. In particular, we consider the case of stochastic optimization of SDPs for which we give the first projection-free algorithm. Our algorithm's convergence rate is asymptotically identical to that of~\citep{mokhtari2020stochastic}, indicating that the additional affine constraint does not impact the asymptotic convergence. The convergence in feasibility gap is also close to deterministic variants~\citep{yurtsever2018conditional}.
In Chapter~\ref{cha:mpcd}, we presented a unified analysis of matching pursuit and coordinate descent. Using the natural connection between the two algorithms, we obtain tight rates for steepest coordinate descent and the first accelerated rate for matching pursuit and steepest coordinate descent. Our affine invariant rates for MP are more elegant than the ones from~\citep{locatello2017unified}. Furthermore, we discussed the relation between the steepest and the random directions by viewing the latter as an approximate version of the former. This allows us to discuss general rates for random pursuits as well. In Chapter~\ref{cha:cone}, we extend this new proof to the case of conic constraints, leading to the first principled and convergent non-negative matching pursuit algorithm.
In Chapter~\ref{cha:boostingVI}, we discussed the theoretical convergence of the boosting variational inference paradigm, delineating explicitly the assumptions that are required for the previously conjectured rates. Further, we have presented algorithmic enhancements allowing us to incorporate black box VI solvers into a general gradient boosting framework based on the Frank-Wolfe algorithm.  This is an important step forward to add boosting VI to the standard toolbox of Bayesian inference.

\looseness=-1From the representation learning perspective described in Part~\ref{part:rep_learn}, we focused on learning factors of variation from data. This problem has been deemed to be a fundamental step in the direction of the emerging field of \textit{Causal Representation Leaerning}~\citep{scholkopf2019causality,scholkopf2020towards}. The goal of the new research program is to learn representations that expose causal structure and support specific causal statements (\eg certain interventional and counterfactual questions). As described in Chapter~\ref{cha:disent_background}, disentanglement can be seen as recovering the noise variables of a structural causal model from high dimensional observations. Our work clearly delineates when certain causal representations can and cannot be learned. In practice, we observed that existing architectural inductive biases are not sufficient to overcome the theoretical limits. We proposed new identifiable settings and presented evidence on the usefulness of these representations.
In Chapter~\ref{cha:unsup_dis}, we gave an impossibility result for the unsupervised learning of disentangled representations. Then, we investigated the performance of six state-of-the-art
disentanglement methods and, in Chapter~\ref{cha:eval_dis}, disentanglement metrics. Overall, we found that (i) A factorizing aggregated posterior (which is sampled) does not seem to necessarily imply that the dimensions in the representation (which is taken to be the mean) are uncorrelated. (ii) Random seeds and hyperparameters seem to matter more than the model but tuning seems to require supervision. (iii) We did not observe that increased disentanglement necessarily implies a decreased sample complexity of learning downstream tasks. 
Our results highlight an overall need for supervision. In theory, inductive biases are crucial to distinguish among equally plausible generative models. In practice, we did not find a reliable strategy to choose hyperparameters without supervision. Recent work~\cite{duan2019heuristic} proposed a stability based heuristic for unsupervised model selection. Further exploring these techniques may help us understand the practical role of inductive biases and implicit supervision. Otherwise, we advocate to consider different settings, for example, when limited explicit~\cite{locatello2019disentangling} or weak supervision~\cite{bouchacourt2017multi,gresele2019incomplete,locatello2020weakly,shu2020weakly} is available.
Our study also highlights the need for a sound, robust, and reproducible experimental setup on a diverse set of data sets.
In our experiments, we observed that the results may be easily misinterpreted if one only looks at a subset of the data sets or the scores. As current research is typically focused on the synthetic data sets of~\cite{higgins2016beta,reed2015deep,lecun2004learning,kim2018disentangling,locatello2019challenging} --- with only a few recent exceptions~\cite{gondal2019transfer} --- we advocate for insights that generalize across data sets rather than individual absolute performance.
In Chapter~\ref{cha:semi_sup}, we investigated the role and benefit of explicit supervision for disentangled representations. We found that both supervised model selection and semi-supervised training and validation with imprecise and partial labels (inherent with human annotation) are viable solutions to  apply these approaches in real-world machine learning systems. Our findings provide practical guidelines for practitioners to develop such systems and, as we hope, will help to advance disentanglement research towards more practical data sets and tasks. 
In Chapter~\ref{cha:fairness}, we observe the first empirical evidence that disentanglement might prove beneficial to learn fair representations, supporting the conjectures of~\cite{kumar2017variational}.  We show that general purpose representations can lead to unfair predictions, even if the sensitive variable and target variable are independent and one only has access to observations that depend on both of them. We extensively discuss the relation between fairness, downstream accuracy, and disentanglement. These findings may serve as motivation to further explore the usefulness of disentanglement for robust and fair classification. They are especially relevant in the context of the General Data Protection Regulation laws in Europe.
In Chapter~\ref{cha:weak}, we proposed a new setting for learning disentangled representations considering pairs of non-i.i.d. observations sharing a sparse but unknown, random subset of factors of variation. We demonstrated that, under certain technical assumptions, the associated disentangled  generative model is identifiable. Importantly, we show that our training metrics correlate with disentanglement and are useful to select models with strong performance on a diverse suite of downstream tasks \emph{without using supervised disentanglement metrics}, relying exclusively on weak supervision. This result is of great importance as the community is becoming increasingly interested in the practical benefits of disentangled representations~\cite{van2019disentangled,locatello2019fairness,creager2019flexibly,chao2019hybrid,iten2020discovering,chartsias2019disentangled,higgins2017darla}. 
In Chapter~\ref{cha:slot_attn}, we discuss a conceptual limitation of disentanglement that we believe is important in the pursuit of learning causal representations. We propose a new architectural component learning a set of abstract high-level variables with a common representational format that respects permutation symmetry. While we did not apply Slot Attention to disentanglement, we successfully showed the generalization to a variable number of objects in the supervised setting. Our approach is significantly simpler than competitive set prediction architectures~\citep{zhang2019deep,zhang2019fspool} and achieve better performance.

\section{Limits of this Dissertation} 
We start from the premise that structure is a useful property and concern ourselves with how to learn and enforce such structure. While we observed some evidence that structure is indeed useful (\eg matrix completion experiments in Chapter~\ref{ch:SFW}, downstream tasks in Chapters~\ref{cha:unsup_dis},~\ref{cha:fairness}, and~\ref{cha:weak}), we did not address this claim systematically on applications in the wild. The work presented in Part~\ref{part:opt} is often very general, and we did not further explore applications (for example, in Chapters~\ref{cha:mpcd} and~\ref{cha:cone}), and our Boosting Black-Box VI approach is still cumbersome to implement properly. 
Our work on disentanglement significantly advanced the state of the field. However, it is unclear what is the practical benefit of chasing performance on the data sets we considered. To counter this, we often compared trends rather than absolute performance. 
While our research agenda is not at odds with the trend of scaling model capacity and data set size, we overall did not personally investigate this direction on real-life benchmarks with clear and immediate industrial applications (although our work on semi-supervised disentanglement, fairness and Slot Attention generated more interest from industry). In subsequent work~\citep{trauble2020independence,dittadi2020transfer}, we successfully scaled the methods described in Chapter~\ref{cha:weak} to a real robotic platform, however the usefulness of disentanglement for real world control tasks is still largely unproven.
With Slot Attention we outperformed other set prediction architectures, but we did not scale it to the level of~\cite{carion2020detr}, which is in turn not state-of-the-art in object detection, where dedicated architectures are still superior.
From the conceptual standpoint, Slot Attention does not know about objects per-se as the segmentation is solely driven by the downstream task. Slot Attention does not distinguish between clustering objects, colors, or simply spatial regions and completely relies on the downstream task to drive the specialization to objects. Further, the positional encoding used in our experiments is absolute and hence our module is not equivariant to translations.

\section{Final Comments, Lessons Learned, and Open Problems}
\paragraph{Inductive Biases and Supervision}
\looseness=-1In our research agenda, we explored various forms of supervision. Motivated by our theoretical impossibility result in Chapter~\ref{cha:unsup_dis}, we explored the role of inductive biases for disentanglement. As supervision seemed to be crucial for model selection, we explored semi- and weakly-supervised approaches. Clearly, unsupervised methods would be more elegant. Still, we would argue that real-world applications do not only come with precise requirements but also with some prior knowledge about the data collection process. In the disentanglement challenge we organized for NeurIPS 2019, the participants did not even have access to the training set, and had to develop their models in a purely unsupervised way. In hindsight, one may argue that this setting is somewhat rare. We believe that it is prudent to incorporate knowledge of the data collection process into the model. For example, several works on weakly-supervised disentanglement were published in 2020 with our~\citep{locatello2020weakly} described in Chapter~\ref{cha:weak}, e.g.~\citep{shu2020weakly,khemakhem2019variational,sorrenson2020disentanglement} to name a few. They all consider slightly different sources of weak supervision, which is, in our view, an asset. The tale of a ``single model to rule them all'' is at odds with the fact that specific inductive biases from the data collection process are likely useful for both performance and data efficiency. Designing models that can flexibly incorporate these biases is an important research direction that brings machine learning closer to causality research, where assumptions (in the form of a graph, or faithfulness and Markov conditions) play a more explicit role in the modeling.

\paragraph{Practical Benefits}
\looseness=-1In our research, we extensively investigated whether higher disentanglement scores lead to usefulness downstream. Surprisingly, our first finding in Chapter~\ref{cha:unsup_dis} was that disentanglement did not seem to be useful for increased sample efficiency for downstream tasks. Subsequently, we explored different settings that more clearly benefit from the imposed structure~\citep{locatello2019fairness,van2019disentangled,trauble2020independence,dittadi2020transfer}. The lesson we learned here is that it is unlikely that a single structural property would be useful for any possible downstream task. Instead, different tasks come with different requirements, which speaks in favor of modular architectures that can be efficiently re-purposed and recombined. One would almost wish for a dictionary of inductive biases, much like the ones described in Part~\ref{part:opt} of this dissertation, that are dynamically combined depending on the task. An interesting candidate approach is using independent causal mechanisms \citep{pearl2009causality, peters2017elements}, with the sparse causal shifts assumption~\citep{scholkopf2020towards}.

\paragraph{Experimental Setup and Diversity of Data Sets}
\looseness=-1Our experimental studies highlight the need for a sound, robust, and reproducible experimental setup on a diverse set of data sets in order to draw valid conclusions. 
We have observed that it is easy to draw spurious conclusions from experimental results if one only considers a subset of methods, metrics, and data sets. Hence,
we argue that it is crucial for future work to perform experiments on a wide variety of data sets to see whether conclusions and insights are generally applicable.
This is particularly important in the setting of disentanglement learning as experiments are largely performed on toy-like data sets. Likewise, concerning the evaluation, it is important for future work to be explicit about which the properties of the learned representation are deemed useful and how they are being evaluated.

\paragraph{Open Problems in Optimization}
There are few open problems from Part~\ref{part:opt}. First, our rates for SHCGM of Chapter~\ref{ch:SFW} are suboptimal. In particular, it should be possible to match the rate with deterministic gradients~\citep{yurtsever2018conditional}. A feasible approach may be using the proof technique we developed in~\citep{negiar2020stochastic} for separable objectives (and constraints). Regardless of their theoretical advantages, widespread adoption of these new SDP solvers will require more work, in particular, impressive applications to unprecedented scales.
Second, one of the main reasons FW fell out of favor in the 80s was its lack of an accelerated rate. While the conditional gradient sliding~\citep{Lan2016} and Catalyst~\citep{NIPS2015_5928} achieve a faster rate, the resulting algorithm is not faster in practice. Our hope was that by connecting Frank-Wolfe with Coordinate Descent through Matching Pursuit and the Non-Negative variant (which is an intermediate case between the FW and MP) we could prove an accelerated rate for FW. This did not work out and remains an open problem. Finally, we hope iterative inference algorithms inspired by those in this dissertation can be used in combination with meta-learning to learn a dictionary of modules and inductive biases that are dynamically recombined across multiple tasks.

\paragraph{Causal Representations}
From Part~\ref{part:rep_learn}, there are several open problems in the direction of causal representation learning. First, it is unclear how to model disentanglement in the multi-object case as a structural causal model. Second, while we can recover exogenous variables with techniques like the ones described in this dissertation, even with correlations~\citep{trauble2020independence,dittadi2020transfer}, we still do not directly have access to the causal variables and do not perform inference on the graph structure. This is problematic because the decoder cannot distinguish between causal graphs in the same Markov equivalence class. Therefore, the traversals of a disentangled model trained on a data set with causally dependent factors of variation~\citep{trauble2020independence,dittadi2020transfer} may not correspond to correct interventional distributions under the causality interpretation described in Chapter~\ref{cha:disent_background}.
Third, we learn disentangled representations in a very controlled setting: all images come from the same causal graph. The applicability of this setting to real-world data sets is unclear. 
Understanding under which conditions causal variables and non-linear causal relations can be learned, which training frameworks allow to best exploit the scalability of machine learning approaches to the problem of learning structure, and providing compelling evidence on the advantages over (non-causal) statistical representations in terms of generalization, re-purposing, and transfer of causal modules on real-world tasks all remain open questions.

\cleardoublepageempty{}


\pdfbookmark[-1]{Lists of Tables}{lot}
\listoftables{}
\cleardoublepageempty{}
\pdfbookmark[-1]{Lists of Figures}{lof}
\listoffigures{}
\cleardoublepageempty{}
\pdfbookmark[-1]{Lists of Algorithms}{loa}
\listofalgorithms{}
\cleardoublepageempty{}

\pdfbookmark[-1]{Bibliography}{book:bibliography}
\printbibliography{}
\cleardoublepageempty{}
\clearpage
\cleardoublepageempty{}
\end{document}